\documentclass[a4paper]{article}[12pts]

\usepackage[english]{babel}
\usepackage[utf8x]{inputenc}
\usepackage[T1]{fontenc}

\normalsize

\usepackage[a4paper,top=1in,bottom=1in,left=1in,right=1in,marginparwidth=1in]{geometry}

\usepackage{natbib}
\usepackage{framed}
\usepackage{setspace}
\usepackage{amsmath}
\usepackage{amssymb} 
\usepackage{amsthm}
\usepackage{amsfonts, mathtools}
\usepackage{algorithm,algpseudocode}
\usepackage{bm}
\usepackage{bbm}
\usepackage{comment}
\usepackage{graphicx}
\usepackage{multirow, booktabs}
\usepackage{caption}
\usepackage{subcaption}
\usepackage[colorinlistoftodos]{todonotes}
\usepackage{hyperref}
\hypersetup{colorlinks=true, allcolors=blue}

\usepackage{multirow}
\usepackage{textgreek}
\usepackage{adjustbox}
\usepackage{mdframed}

\definecolor{azure}{rgb}{0.9, 0.95, 1.0}

\theoremstyle{plain}
\newtheorem{theorem}{Theorem}[section]
\newtheorem{proposition}[theorem]{Proposition}
\newtheorem{lemma}[theorem]{Lemma}
\newtheorem{corollary}[theorem]{Corollary}
\theoremstyle{definition}
\newtheorem{definition}[theorem]{Definition}
\newtheorem{assumption}[theorem]{Assumption}

\theoremstyle{remark}

\renewcommand{\thefootnote}{\fnsymbol{footnote}}
\DeclareMathOperator*{\argmax}{arg\,max}

\title{How Humans Help LLMs: Assessing and Incentivizing Human Preference Annotators}
\author{Shang Liu$^{*,1}$, Hanzhao Wang$^{*,2}$, Zhongyao Ma$^{3}$, Xiaocheng Li$^{1}$}
\date{\small
Imperial College Business School, Imperial College London$^{1}$ \\
University of Sydney Business School, University of Sydney$^{2}$ \ \ \ \ Meta$^{3}$}

\begin{document}
\maketitle
\onehalfspacing

\def\thefootnote{*}\relax\footnotetext{Equal contribution. Corresponding to s.liu21@imperial.ac.uk and hanzhao.wang@sydney.edu.au.}

\begin{abstract}
Human-annotated preference data play an important role in aligning large language models (LLMs). In this paper, we study two connected questions: how to monitor the quality of human preference annotators and how to incentivize them to provide high-quality annotations. In current practice, expert-based monitoring is a natural workhorse for quality control, but it performs poorly in preference annotation because annotators are heterogeneous and downstream model performance is an indirect and noisy proxy for annotation quality. We therefore propose a self-consistency monitoring scheme tailored to preference annotation, and analyze the statistical sample complexity of both methods. This practitioner-facing analysis identifies how many inspected samples are needed to reliably assess an annotator and shows when self-consistency monitoring can outperform expert-based monitoring. We then use the resulting monitoring signal as the performance measure in a principal-agent model, which lets us study a second sample-complexity question: how many monitored samples are needed before simple contracts perform close to the ideal benchmark in which annotation quality is perfectly observable. The model rationalizes practical bonus schemes, highlights the complementarity between monitoring and contracts, and extends the principal-agent literature by considering a continuous action space for the agent. Under this continuous action space, we show that this shortfall scales as $\Theta(1/\sqrt{\mathcal{I} n \log n})$ for binary contracts and $\Theta(1/(\mathcal{I}n))$ for linear contracts, where $\mathcal{I}$ is the Fisher information of the monitoring method and $n$ is the number of samples used for performance assessment; we further show that the linear contracts are rate-optimal among general contracts. This contrasts with the known result that binary contracts are optimal and of $\exp(-\Theta(n))$ when the action space is discrete \citep{frick2023monitoring}. Throughout the paper, we use real preference annotation data to accompany our discussions.
\end{abstract}

\section{Introduction}

Human-annotated preference data have been playing a critical role in aligning large language models (LLMs) and other multi-modal foundation models. Millions of preference samples annotated by human annotators are aggregated from public and private data sources, and then used in the post-training/alignment of the state-of-the-art LLMs, more specifically, in the stage of RLHF \citep{ouyang2022training} and DPO \citep{rafailov2024direct}. The literature on aligning LLMs has been focused on developing alignment algorithms and evaluating the performance of an aligned model. In alignment procedures, human-annotated data are often used as a gold standard. Sometimes people are aware of the presence of potential mistakes in human annotations, but do not make any special treatment for these mistakes; partly because these mistakes are hard to identify given the volume of the data. 

Once human-annotated data are received from an upstream source, things that one can do in the downstream training of LLMs are probably limited. In this paper, we take an upstream perspective -- we raise and investigate the question of how to evaluate the performance of human annotators and accordingly incentivize them to produce high-quality annotations. In the first place, we do not assume that human annotators have a bad intention of providing low-quality data. However, low-quality annotations can commonly happen for reasons such as exhaustion of the annotators and lack of expertise. Without a monitoring system to assess the performance of the annotators, human nature will cause them to provide random effort-saving annotations. In building such monitoring systems, one challenge is that the data to be annotated usually does not have a true answer, and this is in particular the case of annotating data for persona LLMs \citep{tseng2024two}. In our paper, we first address the question of how to build the monitoring system and then analyze effective mechanisms that can incentivize human annotators. We hope our work makes some first-step progress in answering the question of how humans help AI in the future of human-AI interaction.

From a modeling viewpoint, we develop a probability model (in Section \ref{sec:setup}) to characterize the behavior of annotators in preference data annotation and such a model complements the existing preference probability model which does not account for the human factors of the annotators. We also build utility models (in Section \ref{sec:pa_models}) for both the annotator (as an agent) and the company (as a principal) under a principal-agent model. These models not only lay the foundations for our discussions but can also be used for the future development of better annotation systems and for understanding the business models of data provider companies such as Scale AI, Appen, Outlier, etc.

In assessing the performance of human annotators, we note two new challenges arising in the context of language preference annotation compared to traditional data annotation: (1) unlike tasks such as image classification, there is typically no ``true label,'' because different annotators may legitimately hold different preferences; and (2) the annotations matter mainly because they are later used in downstream LLM post-training/alignment \citep{ouyang2022training,rafailov2024direct}, yet downstream model performance is an aggregate, delayed, and noisy proxy for any one annotator's effort. Specifically, current LLM evaluations are conducted in dozens of metrics and the training is based on the mix of annotators. These two features make quality control for LLM preference annotation fundamentally different from conventional annotation pipelines. The paper therefore proceeds in two layers. Section \ref{sec:assess} addresses the \emph{monitoring} layer: it asks how a company should assess an annotator when there is no sample-level ground truth, and it compares expert-based monitoring with self-consistency monitoring. Section \ref{sec:incentivizing} addresses the \emph{incentive} layer: it takes the resulting monitoring signal as the contractible outcome in a principal-agent model and studies how simple contracts can induce higher annotation quality. In this formulation, heterogeneity is handled by the design of the monitoring technology, while the downstream role of the data enters through the principal's reduced-form value for higher annotation quality.

Section \ref{sec:assess} is therefore practitioner-facing and data-driven. We take expert-based monitoring as the current workhorse for industry quality control, analyze why it can become sample-inefficient in preference annotation, and propose self-consistency monitoring as an improved agreement-based alternative tailored to this setting. At this operational level, Section \ref{sec:assess} answers: how should a company monitor annotators, what statistical properties do the monitoring rules have, and how many inspected samples are necessary for a reliable assessment? In this sense, Section \ref{sec:assess} is both a methodological contribution and a call for greater industry awareness that monitoring design is itself a first-order determinant of annotation quality.

Section \ref{sec:incentivizing} builds on this monitoring analysis and takes an economic view of how the monitoring signal translates into incentives, principal utility, and system-level surplus. We construct a principal-agent model in which the principal (the LLM-developing company) offers a contract for data annotation and the agent (the annotator) then chooses an annotation quality level $\eta \in [0,1]$. The model captures a moral-hazard problem: the company cannot directly observe quality, while the annotator bears the effort cost of producing it. This creates a gap between an ideal benchmark with perfect observability of annotation quality and the best outcome attainable under realistic imperfect monitoring. We prove that, under mild assumptions, this gap scales as $\Theta(1/\sqrt{\mathcal{I} n \log n})$ for binary contracts, in which payment has two levels (Theorem \ref{thm:binary}), and $\Theta(1/(\mathcal{I}n))$ for linear contracts, in which payment is a linear function of average monitored quality (Theorem \ref{thm:linear_contract}), when $\mathcal{I}$ is the Fisher information of a single assessment sample and $n$ samples are used for monitoring. Furthermore, we show that the lower bound for any type of contract must be of $\Omega(1/(\mathcal{I}n))$, which implies the rate-optimality of linear contracts. These results provide another example to justify simple mechanisms for incentivizing human annotators. Our technical result differs from the discrete-action case studied by \citet{frick2023monitoring}, where the corresponding gap for binary contracts is $\exp(-\Theta(n))$ and is optimal among general contracts. Technically, because classical large-deviation tools fail in the continuous-action case, our analysis relies on a fine-grained study of binomial tail probabilities under incentives. This answers the open question posed by \citet{frick2023monitoring} and may be of independent interest.

Put differently, the two sections study the sample complexity of achieving high-quality preference annotation from two complementary angles: Section \ref{sec:assess} studies statistical sample complexity (how many monitored samples are needed to learn or test annotation quality), while Section \ref{sec:incentivizing} studies economic sample complexity (how many monitored samples are needed before simple contracts perform close to the ideal benchmark with perfect observability).

\subsection{Related literature}

\paragraph{Annotation monitoring and management}\  

The ultimate goal of designing monitoring systems and mechanisms is to incentivize human annotators to provide high-quality annotations for reinforcement learning from human feedback (RLHF) and direct preference optimization (DPO). RLHF is a prominent framework for aligning large language models (LLMs) with human preferences. Under this framework, one first trains a reward model using preference data, then applies reinforcement learning to fine-tune the LLM to produce high-reward responses \citep{askell2021general,ouyang2022training,ziegler2019fine}.  For a detailed survey on RLHF, we refer readers to \cite{kaufmann2023survey}. 

A crucial factor underlying RLHF is the quality of the preference data. Recent studies \citep{gao2024impact,chowdhury2024provably,wang2024secrets} demonstrate that alignment performance is highly sensitive to data quality and can deteriorate significantly in the presence of noisy samples or flipped labels. In practice, preference noise has been observed at rates exceeding 20\% in many datasets \citep{zhao2023slic,munos2023nash,cui2023ultrafeedback,touvron2023llama}. To mitigate the effects of such noise during RLHF or other optimization stages, various approaches have been introduced, including filtering out noisy data \citep{gao2024impact,liang2024robust}, applying label smoothing \citep{wang2024secrets}, and designing robust loss functions \citep{gao2024impact, wang2024secrets,liang2024robust}. Different from all these works, our paper takes the perspective of better designing the human annotator system to improve the data quality. 

Beyond these post-annotation techniques, many studies address data quality management during data annotation, which involves both quality estimation and improvement. Quality estimation methods verify whether the annotated data meets the required standards. They may involve expert- or AI-based monitoring \citep{pustejovsky2012natural,qian2021annotation,northcutt2021confident,klie2024efficient,ghosal2022cicero}, injecting annotation tasks with known answers to check consistency \citep{callison2010creating}, measuring inter-annotator agreement (IAA) through various metrics \citep{krippendorff1989content,krippendorff2004reliability,artstein2008inter,monarch2021human}, and designing contracts that compare annotations from different annotators in crowdsourcing (e.g., \cite{miller2005eliciting,bacon2012predicting,cai2015optimum,dasgupta2013crowdsourced}). There is also a line of works on analyzing the value of the quality estimation/monitoring \citep[e.g.,][]{holmstrom1979moral,Jewitt,singh1985monitoring,kim1995efficiency}. If the annotated data does not meet the required standards, improvement strategies should be adopted, such as re-annotation with updated guidelines or retrained annotators \citep{bareket2021neural,klie2024analyzing,ghosal2022cicero} and data filtering \citep{bastan2020author}. For a comprehensive discussion of annotation quality management, we refer to the survey \cite{klie2024analyzing}. 

However, to the best of our knowledge, these methods suffer from two aforementioned challenges when applied to preference data annotation. First, the intrinsic heterogeneity among annotators renders traditional approaches inapplicable, as there is no ground-truth label available for each data sample. Second, the unclear relationship between annotation quality and the performance of downstream tasks complicates the development of a comprehensive evaluation metric for the entire annotated dataset. Together, these challenges hinder the adoption of existing methods in the context of preference data annotation.

\paragraph{Contract Design}\ 

Contract design investigates how to formulate contracts that incentivize agents to pursue the principal's objectives, particularly under conditions of information asymmetry. As a powerful tool, contract design has been widely applied to various problems. For instance, in operations management, \cite{de2019designing,corbett1999designing,corbett2005optimal} propose different contracts for supply chain management with diverse objectives, \cite{adida2019contracts} study contract design for referral services in healthcare, and \cite{jain2013optimal} explore the design of optimal contracts for outsourcing repair and restoration services. In the realm of machine learning, \cite{goldwasser2021interactive} investigate interactive proof systems for PAC verification, while \cite{ivanov2024principal} focus on designing contracts to align the preferences of principals and agents in reinforcement learning. Moreover, \cite{ananthakrishnanknowledge} examine the achievability of optimal outcomes that a fully informed player could secure despite inherent uncertainties in strategic interactions. Another related line of research considers data markets and the pricing of data \citep[e.g.,][]{agarwal2019marketplace,chen2022selling,ho2014adaptive,acemoglu2022too,moscarini2002law}; for a survey, see \cite{bergemann2019markets}. These works primarily address the selling of annotated data to buyers, whereas our focus is on assessing and incentivizing annotators.

Recent papers \cite{ananthakrishnan2024delegating,saig2024delegated} propose simple-form contracts to address machine learning delegation problems under various assumptions about the principal’s utilities, proving the (near-)optimality of these contracts. In their frameworks, agent effort is represented by the quantity of collected data, each data point is assumed to be correctly annotated, and the principal’s utility depends on the accuracy of the resulting machine learning model; hence, the contract is based on accuracy. In contrast, we consider settings where the agent’s effort is reflected in annotation quality and the desired label is influenced by the annotator's own preference. Here, the principal’s utility depends on the quality of annotations, and the contract is based on (tested) annotation quality.

In the context of large language models (LLMs), several studies have explored contract- and mechanism-design approaches to address incentive-related challenges. For example, \cite{saig2024incentivizing} employ algorithmic contract design to enhance the quality of generated content from LLMs. In another direction, \cite{duetting2024mechanism} design auctions that aggregate outputs from multiple LLMs for advertising in an incentive-compatible manner, \cite{harris2023algorithmic} propose a Bayesian persuasion framework with generative AI simulating receiver behavior, and \cite{sun2024mechanism} develop mechanisms for fine-tuning LLMs that aggregate reward models from multiple agents. A recent work \cite{hao2024online} introduces an online learning mechanism that addresses strategic human annotators in RLHF by formulating a novel dynamic Bayesian game. In their setting, agents (annotators) adversarially misreport their preference probabilities to sway the principal's aggregation toward their own interests, while the principal seeks to learn the (unknown) most accurate agent through repeated interactions and minimize regret. In contrast, our agents randomly misreport labels to maximize the difference between the expected utility and their effort, and our principal's goal is to incentivize agents to produce high-quality annotations.

A stream of works has also focused on the theoretical analysis of contract theory \citep[e.g.,][]{dutting2021complexity,dutting2019simple,alon2022bayesian,10.1145/3670865.3673607}. For comprehensive overviews, see \cite{dutting2024algorithmic,NBERw13480}. Among these studies, our work is most closely related to work on linear and binary contracts. For example, \cite{holmstrom1987aggregation,herweg2010binary,georgiadis2020optimal,lopomo2011knightian} study settings with binary contracts, while \cite{holmstrom1987aggregation,walton2022general,carroll2015robustness,barron2020optimal} analyze linear contracts. In contrast to these works, we analyze how quickly the principal's payoff under binary and linear contracts approaches the ideal benchmark that would be attainable if annotation quality were perfectly observable, rather than asking whether those contracts are optimal.

The most relevant paper to our study is \cite{frick2023monitoring}, which studies how quickly the principal's payoff approaches that same ideal benchmark as the amount of data increases (thereby revealing the agent's effort). Specifically, they analyze the convergence rates of binary and linear contracts, showing that binary contracts can achieve the optimal convergence rate while linear contracts yield a slower rate. They further provide a ranking method for monitoring technologies (i.e., for a given agent effort, how the monitoring data are generated) based on principal utility. There are two main differences between their work and ours: (i) they assume a discrete action (effort) space for agents, whereas we extend the analysis to a continuous space with a new proof scheme; and (ii) they rank given monitoring technologies, whereas we design a monitoring technology tailored to language preference data and compare it with classical alternatives from both the assessment and the principal-utility perspectives, using both theoretical and numerical results.

\section{Problem Setup}

\label{sec:setup}

In this section, we first introduce an idealized setting for human annotators, which is commonly assumed as true and widely used as the backbone in developing reward/preference models for aligning LLMs. And then we propose a more practical setting to model the annotation behavior. Consider the task of preference (reward) modeling based on pairwise preference data. Each data sample consists of a tuple 
$$(x, y_1, y_2, Z)$$
where $x\in \mathcal{X}$ denotes a prompt/instruction, $y_1, y_2\in \mathcal{Y}$ are two candidate responses to $x$, and $Z\in\mathcal{Z}$ is a random variable that denotes the feedback indicating the preference between $y_1$ and $y_2$. 
In the canonical setup (\cite{bai2022training, ouyang2022training} among others), the label $Z$ takes binary values, i.e., $\mathcal{Z} = \{0, 1\}$. Furthermore one assumes $Z$ is a Bernoulli random variable such that
\begin{equation}
\mathbb{P}(Z=1) = 1-\mathbb{P}(Z=0) = \mathbb{P}\left(y_{1} \succ y_{2}\ | \ x\right).
\label{eqn:Z_binary}
\end{equation}

\begin{assumption}
Assume the label $Z$ is produced by human annotators and it follows \eqref{eqn:Z_binary}.
\label{assm_basic}
\end{assumption}

This assumption has been (unconsciously) widely used as the backbone of training mainstream preference and reward models; some literature (e.g., \cite{gao2024impact,liang2024robust,wang2024secrets}) discusses the case of noisy or poisoned labels of $Z$ without \eqref{eqn:Z_binary}. Differently, our paper focuses on assessing whether human annotators meet such an assumption (and to what extent), and motivating human annotators to do so.


We denote the true preference probability as a function of $(x, y_1, y_2)$
\begin{equation}
p(x,y_1,y_2) \coloneqq \mathbb{P}(y_1 \succ y_2\ |\ x).  
\label{eqn:true_preference}
\end{equation}
Let $\mathcal{P}$ denote the probability distribution from which $(x,y_1,y_2)$ is sampled. Then, the probability value $p(x,y_1,y_2)$ can be viewed as a random variable accordingly. 

In the light of human annotation under Assumption \ref{assm_basic}, we can think of the probability $p(x, y_1, y_2)$ as the preference between $y_{i,1}$ and $y_{i,2}$ among the whole population, and $Z_i$, the label annotated by one human annotator, is a random draw from the population. We refer to \cite{sun2024rethinking, liu2024reward} for more discussion on the probability model behind the annotation procedure. Throughout the paper, we focus on this binary annotation setup and leave the more complicated annotation setups such as ranking and ordinal feedback for future studies. 

The downstream task in this paper is \emph{LLM post-training/alignment} based on pairwise preference data. Given an annotated dataset
$$\mathcal{D}_{\mathcal{Z}} \coloneqq \{(x_i, y_{i,1}, y_{i,2}, Z_i)\}_{i=1}^N,$$
there are two standard downstream uses. In a reward-modeling/RLHF pipeline \citep{ouyang2022training}, one first learns a reward model $r_{\theta}(x,y)$ from $\mathcal{D}_{\mathcal{Z}}$ and then uses that model to fine-tune the LLM. In a DPO-style pipeline \citep{rafailov2024direct}, one uses the same preference pairs directly to update the policy without a separate reward-model training stage. Thus, throughout the paper, ``downstream performance'' means the performance of the post-trained/aligned model that is ultimately produced from these annotations.

For example, for the reward-modeling component in an RLHF pipeline, the canonical object to learn is a probability model $\hat{p}_\theta$ induced by a reward function $r_{\theta}(x,y): \mathcal{X}\times \mathcal{Y}\rightarrow \mathbb{R}$ with parameter $\theta\in \Theta$, where the preference probability $\hat{p}_\theta$ between $(x, y_1)$ and $(x, y_2)$ is modeled by the relative relationship of $r_\theta(x, y_1)$ and $r_\theta(x, y_2)$. One prevalent specification is the Bradley--Terry model \citep{bradley1952rank}
\[
\hat{p}_\theta(y_1 \succ y_2 | x) \coloneqq \frac{\exp\big(r_{\theta}(x, y_1)\big)}{\exp\big(r_{\theta}(x, y_1)\big) + \exp\big(r_{\theta}(x, y_2)\big)}.
\]
The learned reward model $r_{\theta}(x,y)$ is then used as one ingredient in the later post-training/alignment of the LLM.

\subsection{Annotator probability model}

While the above discusses an idealized setting for preference annotation, we introduce a more realistic probability model to characterize the behavior of annotators. Let a binary random variable $V\in\{0,1\}$ represent whether the annotator treats one sample carefully. Specifically, we consider the following model 
\begin{equation}
\begin{split}
\mathbb{P}(Z=1|V=0, x, y_1, y_2) = \mathbb{P}(Z=0|V=0, x, y_1, y_2) &= \frac{1}{2},\\
\mathbb{P}(Z=1|V=1, x, y_1, y_2) = 1 - \mathbb{P}(Z=0|V=1, x, y_1, y_2) &= p(x, y_1, y_2).
\end{split}
\label{eqn:Z_annotation_quality}
\end{equation}
The probability model \eqref{eqn:Z_annotation_quality} reduces to \eqref{eqn:Z_binary} when $\mathbb{P}(V=0) = 0$. In the other extreme of $\mathbb{P}(V=0) = 1$, the annotations are produced by coin flips. Essentially, $V=1$ indicates that the annotator is fully committed to the sample, whereas $V=0$ indicates that the annotator assigns the label randomly.

Throughout the paper, we consider a simplified case where
\begin{equation}
\mathbb{P}(V=1|x,y_1,y_2) \equiv \eta \text{ \ for all \ } (x,y_1,y_2).
\label{eqn:annotation_eta}
\end{equation}
In other words, the random variable $V$ is a Bernoulli variable independent of $(x,y_1,y_2),$ and for all samples, the probability that the annotator fully commits ($V=1$) is always $\eta$. In the light of \eqref{eqn:annotation_eta}, $\eta$ is a key parameter that reflects the annotator's commitment level and consequently the quality of the annotation. Figure \ref{fig:gao_noise_alignment} \citep[Figure 2]{gao2024impact} illustrates how annotation quality affects downstream performance in LLM post-training. From the company’s perspective, a central goal is to assess annotators by estimating $\eta$. Conversely, annotators may choose $\eta$ strategically based on the design of the annotation system. We examine these two perspectives in the following sections.

\begin{figure}[ht!]
  \centering
  \begin{minipage}[b]{0.48\textwidth}
    \centering
    \includegraphics[width=\textwidth]{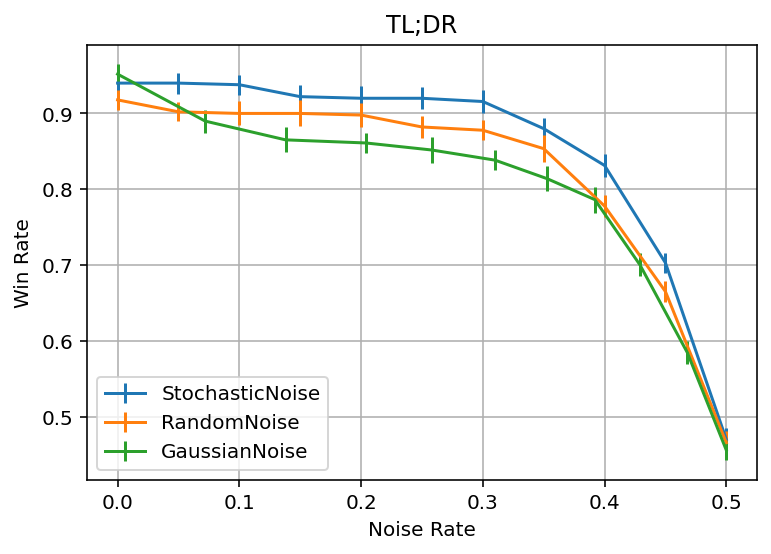}
    \vspace{2mm}
    
    {\small (a) Summarization (TL;DR)}
  \end{minipage}
  \hfill
  \begin{minipage}[b]{0.48\textwidth}
    \centering
    \includegraphics[width=\textwidth]{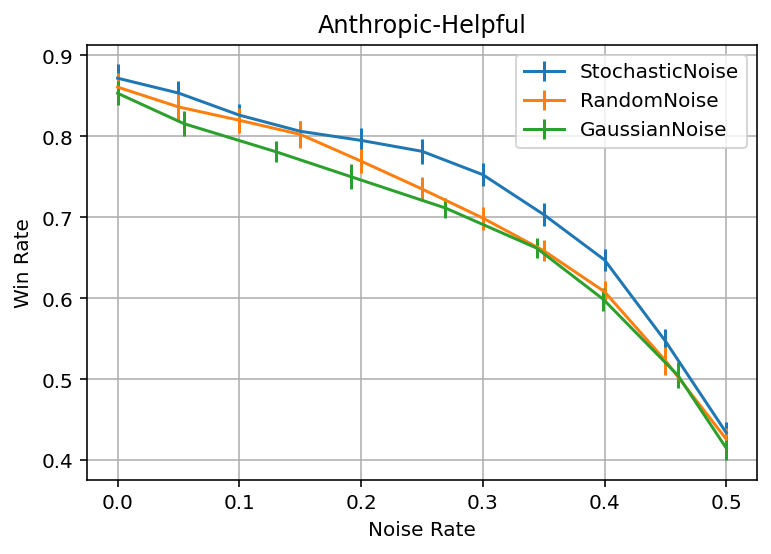}
    \vspace{2mm}
    
    {\small (b) Dialogue (Anthropic-Helpful)}
  \end{minipage}
  
\caption{Aggregate relationship between preference-noise rate ($x$-axis) and downstream alignment performance ($y$-axis; higher is better) in LLM post-training \citet[Figure~2]{gao2024impact}. The two panels correspond to summarization and dialogue tasks, respectively. The figure highlights the importance of annotation quality for system-level performance. Experimental details can be found in \citet{gao2024impact}.}
  \label{fig:gao_noise_alignment}
\end{figure}



\section{Assessing Human Annotators: Monitoring Technologies for Practice}
\label{sec:assess}

In this section, we study the problem of assessing human annotators. Before proceeding, we highlight two challenges inherent to this task. First, annotators are intrinsically heterogeneous: for a sample $(x, y_1, y_2)$ to be labeled, some annotators may judge $y_1$ to be better, while others may prefer $y_2$ (see Appendix~\ref{appx:example} for several examples). As a result, when an annotator assigns a particular label to a sample, we cannot conclude with certainty that the label is mistaken. Second, although some existing works such as \citet{gao2024impact} have studied the impact of noisy annotations on the downstream tasks, the exact relationship between annotation quality and downstream alignment performance is unclear. What makes the evaluation of an annotator more difficult is that, in practice, each annotator's labeled data are pooled with data from other annotators and are often further combined with data generated from other sources, such as AI-generated data, during the alignment phase. This creates an additional attribution problem. Taken together, these factors make it nearly impossible to reliably assess an annotator's performance solely through downstream alignment outcomes.

With these challenges in mind, we analyze two assessment methods: the canonical expert-based monitoring approach and our proposed self-consistency monitoring approach. We take a data-driven practitioner's perspective: before deciding what contract to write, a company must first determine how annotation quality will be monitored. We begin with expert-based monitoring, since it is the natural operational default, and then propose self-consistency monitoring as an improved agreement-based alternative that is better suited to preference annotation. The main output of this section is the informativeness of the monitored signal per inspected sample, which is precisely the object that Section~\ref{sec:incentivizing} later converts into incentives through contract design.

\subsection{Expert-based monitoring}

We first consider the expert-based monitoring system, commonly used in traditional production quality control. An expert can be a manager, an inspector, or a senior-level annotator who inspects (a proportion of) the annotated data carefully to evaluate the quality of the labeled data produced by one annotator. We formulate the problem as a hypothesis testing problem:
$$H_0: \eta\le \eta_0, \ \ \ \ H_1: \eta\ge \eta_1$$
where $0 \leq \eta_0 < \eta_1 \leq 1$. Here $\eta_1$ can be some target level that the annotator is expected to achieve, say, a full commitment on $\eta_1=95\%$ of the samples. Meanwhile, $\eta_0$ is some penalty threshold; if the annotator does not fully commit on $\eta_0=80\%$ of the samples, they will face some penalty.

\begin{algorithm}[ht!]
\caption{Expert-based monitoring}
\centering
\label{alg_expert}
\begin{algorithmic}
\State \textbf{Input:} A set of $N$ samples $\mathcal{D}_{\mathcal{Z}}$ produced by one annotator 
\State \textbf{Output:} Assert $\eta\le \eta_0$ or $\eta\ge \eta_1$ 
\State Randomly select a subset of $n\le N$ samples $\mathcal{D}_{\text{test}}\subseteq \mathcal{D}_{\mathcal{Z}}$ without replacement and let
\[
\mathcal{D}_{\text{test}} \coloneqq \{(x_i, y_{i, 1}, y_{i, 2}, Z_i)\}_{i=1}^n.
\]
\State The expert's monitoring decision is based on some testing function $\Psi: \mathcal{D}_{\text{test}}\rightarrow \{0,1\}$
\State \textbf{Return} $\Psi(\mathcal{D}_{\text{test}})$
\end{algorithmic}
\end{algorithm}

Algorithm \ref{alg_expert} describes the generic procedure of expert-based monitoring. The first step is to sample a subset $\mathcal{D}_{\text{test}}$ of $n$ annotations from the whole data $\mathcal{D}_{\mathcal{Z}}.$ Without loss of generality, we let the test data $\mathcal{D}_{\text{test}}$ be the first $n$ annotations in the overall $N$ annotations. The expert's testing decision can be viewed as a function of $\mathcal{D}_{\text{test}}$ indicating the acceptance ($\Psi(\mathcal{D}_{\text{test}})=0$)  or rejection ($\Psi(\mathcal{D}_{\text{test}})=1$)  of $H_0.$

\begin{proposition}
\label{prop:info_lower_bound}
The following inequality holds for any $0 \leq \eta_0 < \eta_1 \leq 1,$
\[
\inf_{\Psi}\left\{ \mathbb{P}\left(\Psi(\mathcal{D}_{\text{test}}) = 1 \middle | \eta \leq \eta_0\right) + \mathbb{P}\left(\Psi(\mathcal{D}_{\text{test}}) = 0 \middle | \eta \geq \eta_1\right) \right\} \geq \frac12 \cdot\exp(- n D_{\mathrm{kl}}(\mathcal{P}_{\eta_0} \| \mathcal{P}_{\eta_1})),
\]
where the infimum over $\Psi$ is taken with respect to any measurable function and the probability $\mathbb{P}(\cdot)$ on the left hand side is with respect to the law of \eqref{eqn:Z_annotation_quality} and \eqref{eqn:annotation_eta}. Here $\mathcal{P}_{\eta_0}$ and $\mathcal{P}_{\eta_1}$ on the right hand side refer to the joint distribution of $(x,y_1,y_2,Z)$ under the law of \eqref{eqn:Z_annotation_quality} and \eqref{eqn:annotation_eta}, with $\eta=\eta_0$ and $\eta=\eta_1$, respectively. 
\end{proposition}

The proposition gives a lower bound for the sum of the two types of errors, and it follows from a standard application of Le Cam's method. We note that there is basically no restriction imposed on $\Psi$, and $\Psi$ can even utilize the knowledge of $p(x,y_1,y_2)$ -- true probability \eqref{eqn:true_preference}. This means the lower bound is (probably far) more optimistic than what people can achieve in a real-world scenario. 

To provide more intuitions on the lower bound, Figure \ref{fig:main_hist_LB} plots the lower bound on four different preference datasets. Specifically, for each dataset, we either use its original preference probability $p(x,y_1,y_2)$ or calibrate its corresponding open-source preference model and use it as an estimate of the true preference probability. In some sense, we can think that the preference models have a paramount performance in comparison to human experts on the annotation task, and hence it is a legitimate proxy of $p(x,y_1,y_2)$  (more details on the datasets and calibrations are deferred to Appendix \ref{appx:fig_hist_LB}). For the three preference datasets of PKU, Helpsteer, and Ultra, it takes from 100 to 500 samples to reduce the sum of two types of errors to a reasonable level. The explanation is that for these datasets, the preference between $y_{i,1}$ and $y_{i,2}$ is vague for most of the samples (see examples in Appendix \ref{appx:example}). For the preference dataset of Skywork, the sample number is much smaller as the preference between $y_{i,1}$ and $y_{i,2}$ is mostly clear. To interpret the result, let's say each annotator produces 1000 annotations per week, and if we perform a quality assessment every week, the result implies that we need to examine a few hundred samples for each annotator, which makes this expert-based monitoring infeasible. If the company has the budget to hire experts to examine these many samples, the company can directly hire these experts to annotate the samples, and then this causes another layer of the problem -- how to assess the expert.

\begin{figure}[ht!]
  \centering
  \begin{minipage}[b]{0.23\textwidth}
    \centering
\includegraphics[width=\textwidth]
{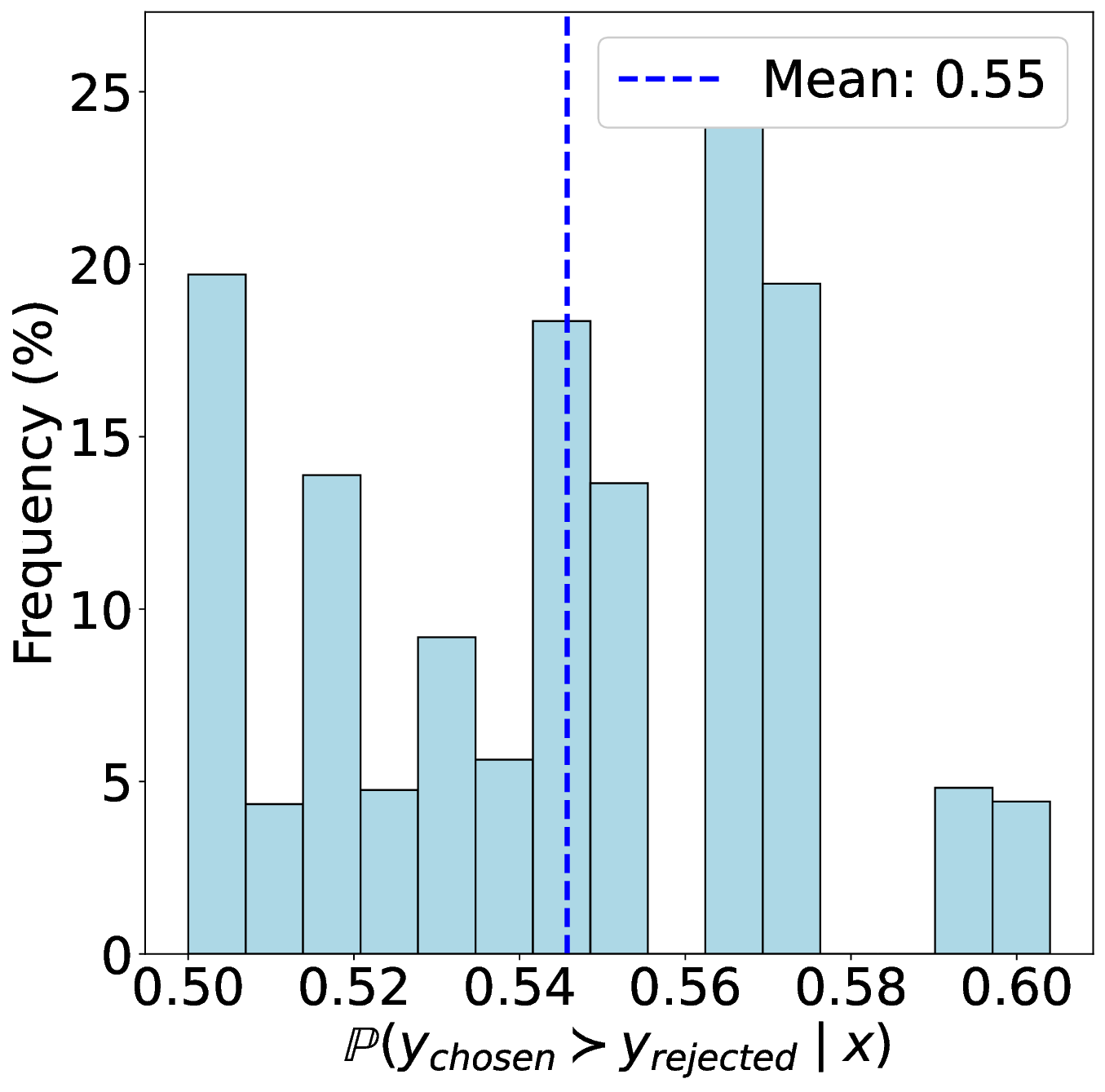}
 
  \end{minipage}
    \hfill
  \begin{minipage}[b]{0.23\textwidth}
    \centering
\includegraphics[width=\textwidth]
{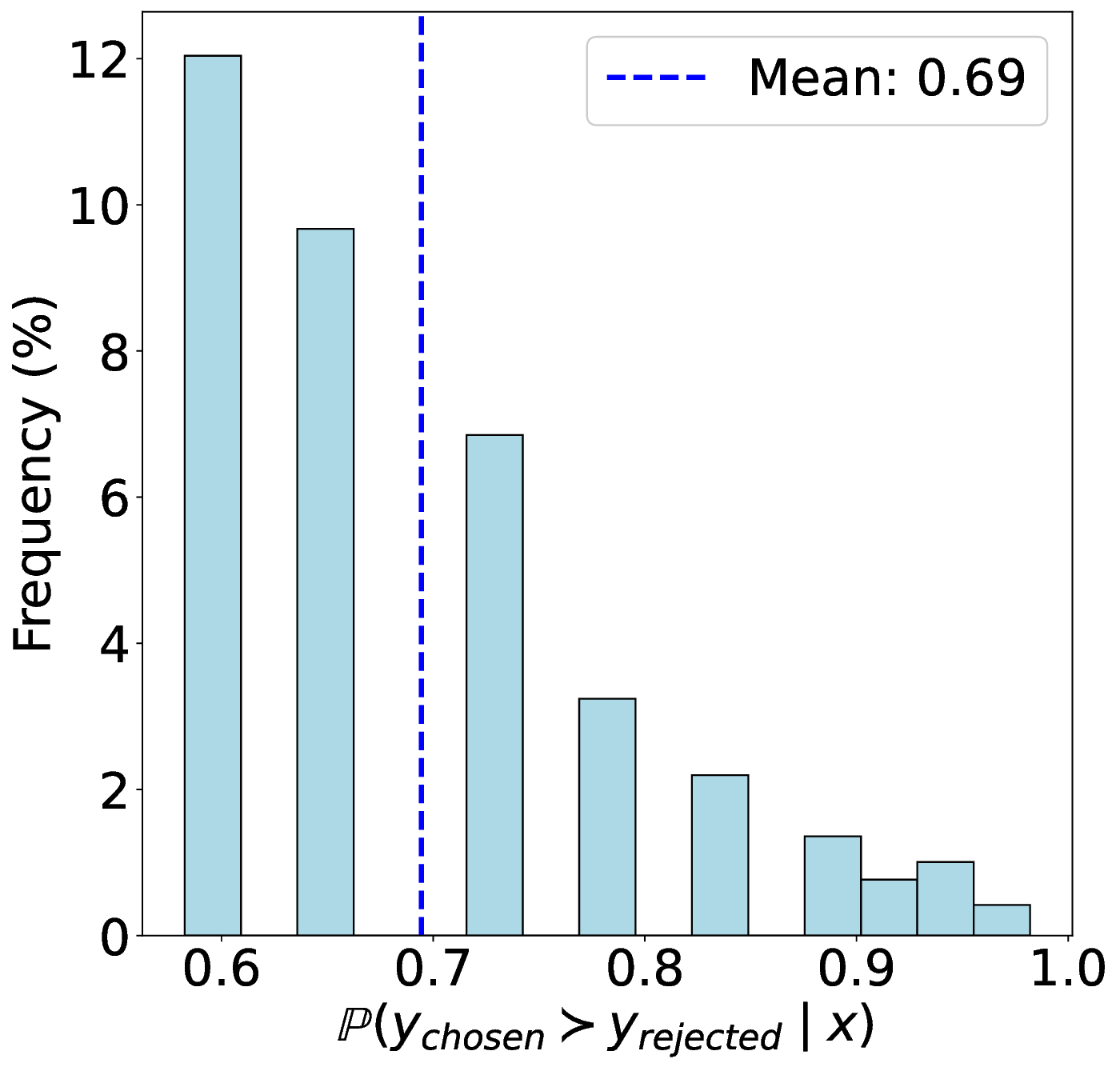}
  \end{minipage}
    \hfill
    \begin{minipage}[b]{0.23\textwidth}
    \centering
\includegraphics[width=\textwidth]
{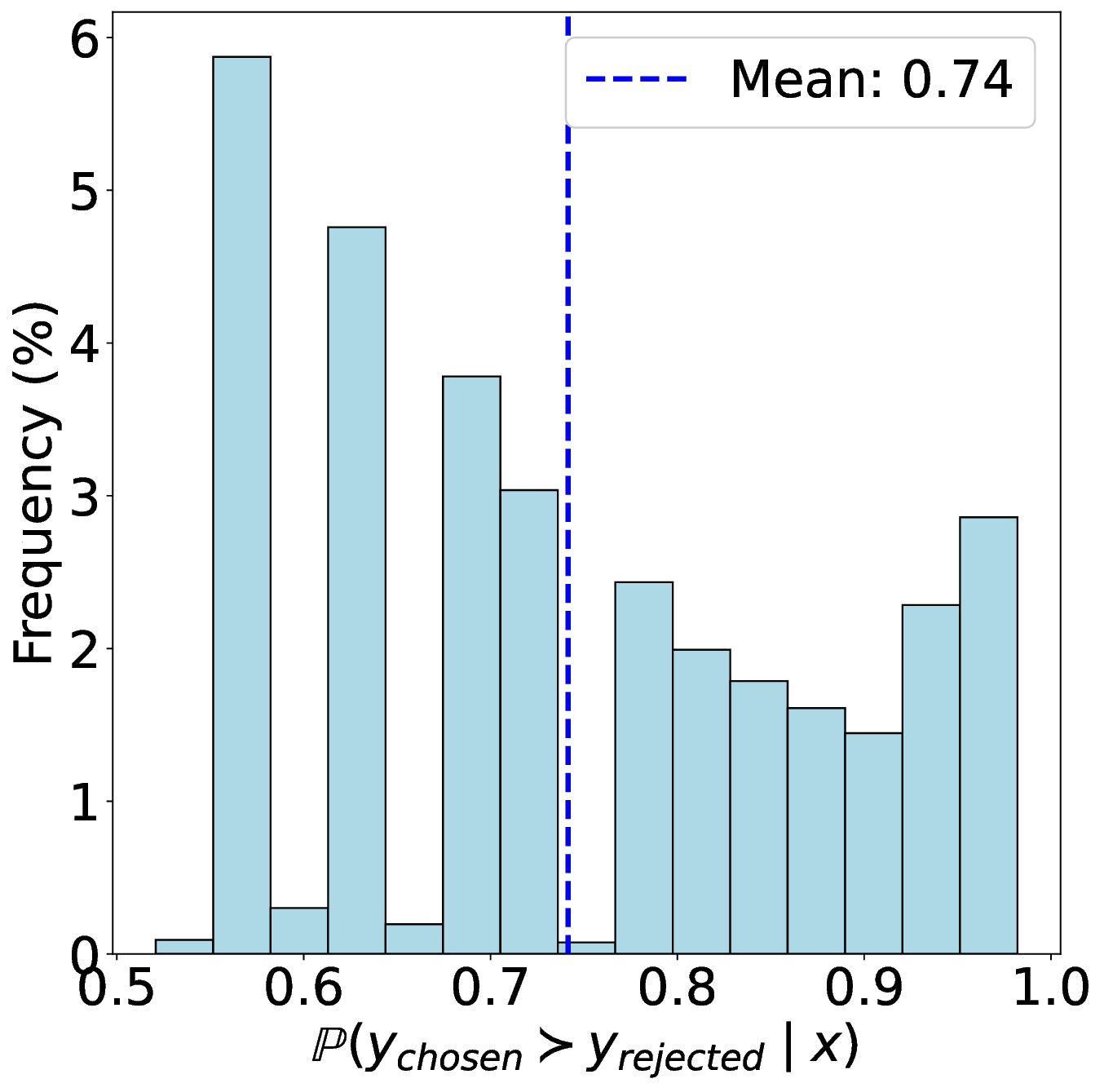}
  \end{minipage}
  \hfill  
  \begin{minipage}[b]{0.23\textwidth}
    \centering
\includegraphics[width=\textwidth]
{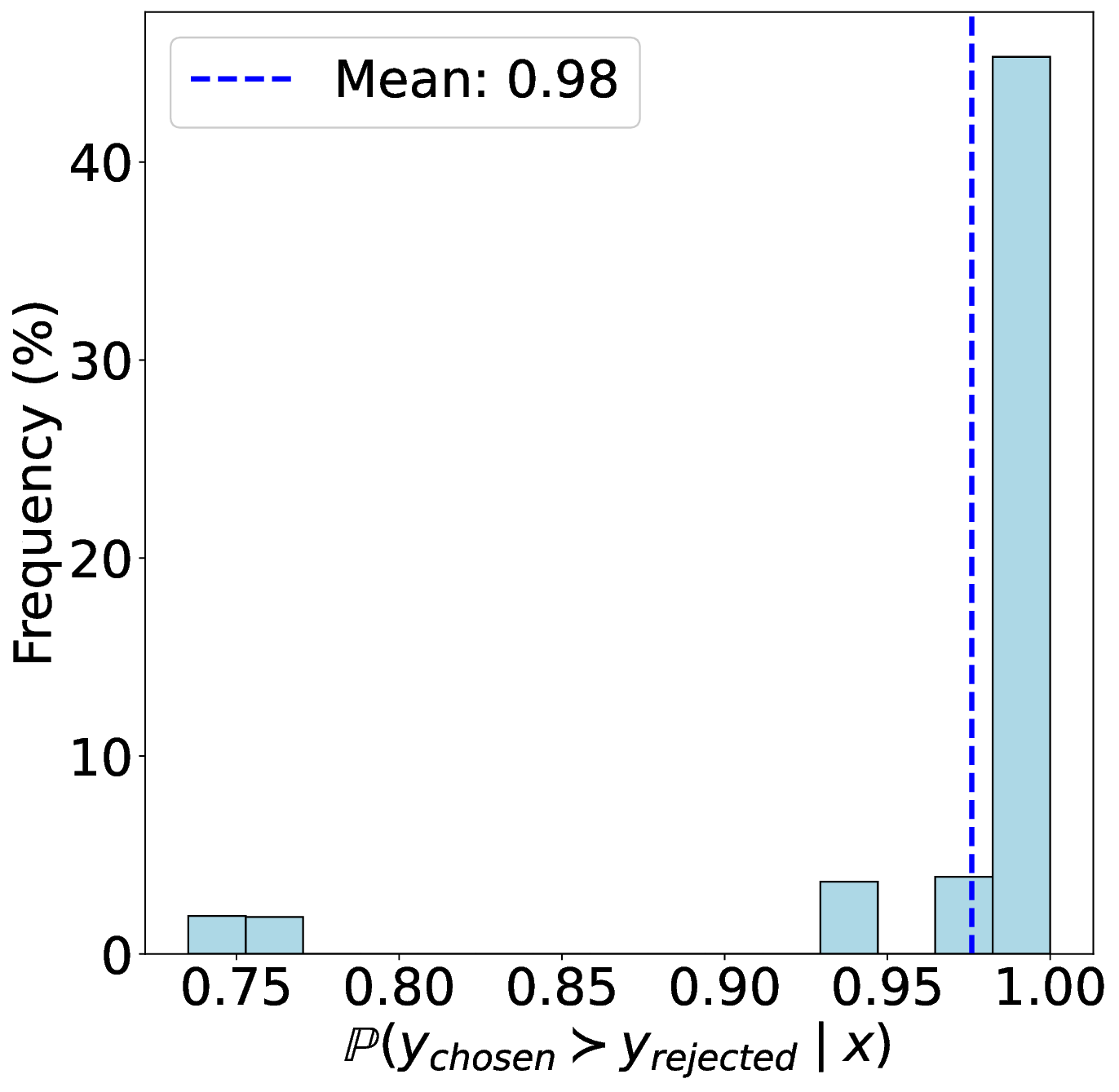}
  \end{minipage}
  \begin{minipage}[b]{0.23\textwidth}
    \centering
\includegraphics[width=\textwidth]
{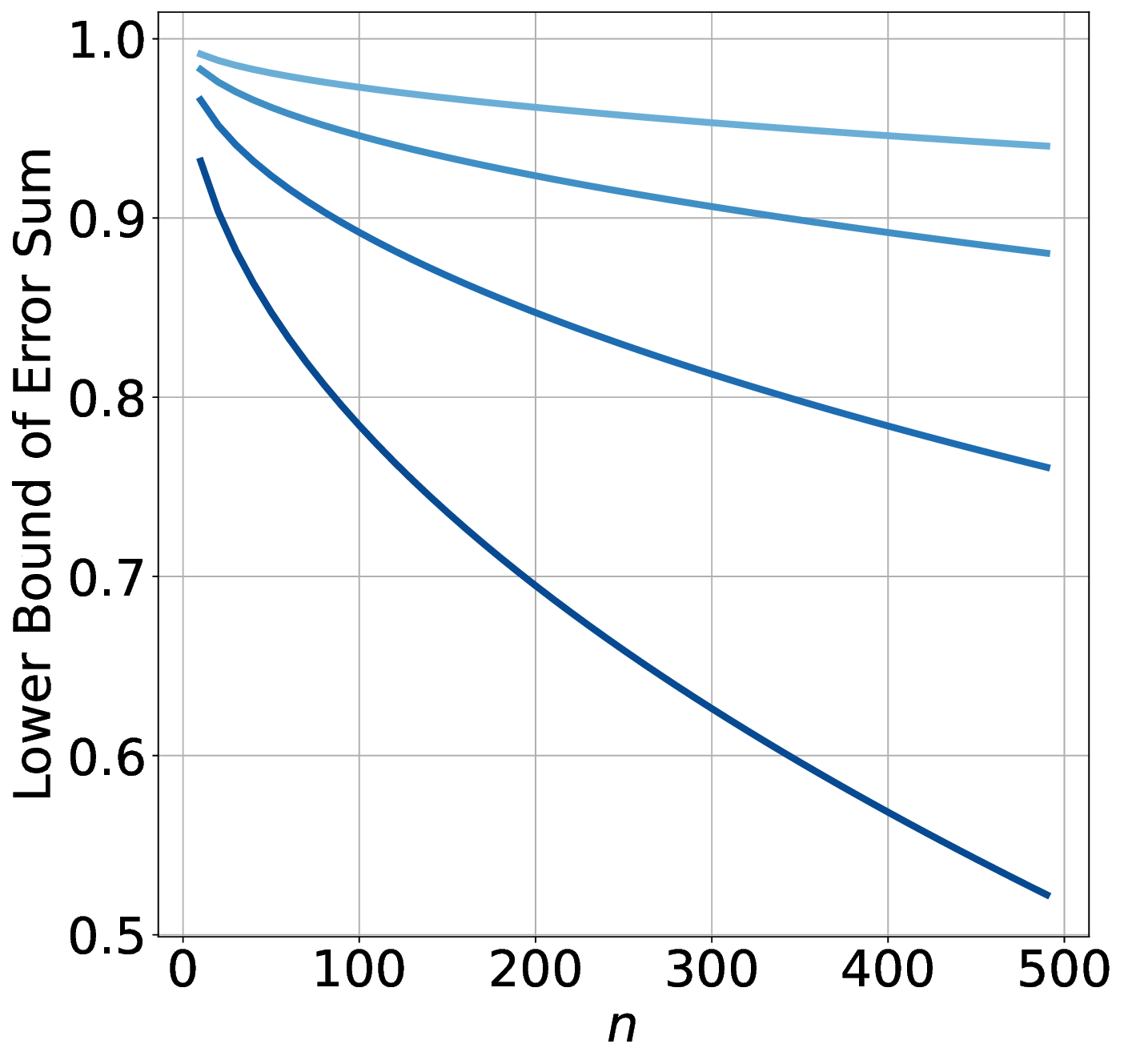}
\vspace{2mm}
    {\small (a) PKU}
  \end{minipage}
    \hfill
  \begin{minipage}[b]{0.23\textwidth}
    \centering
\includegraphics[width=\textwidth]
{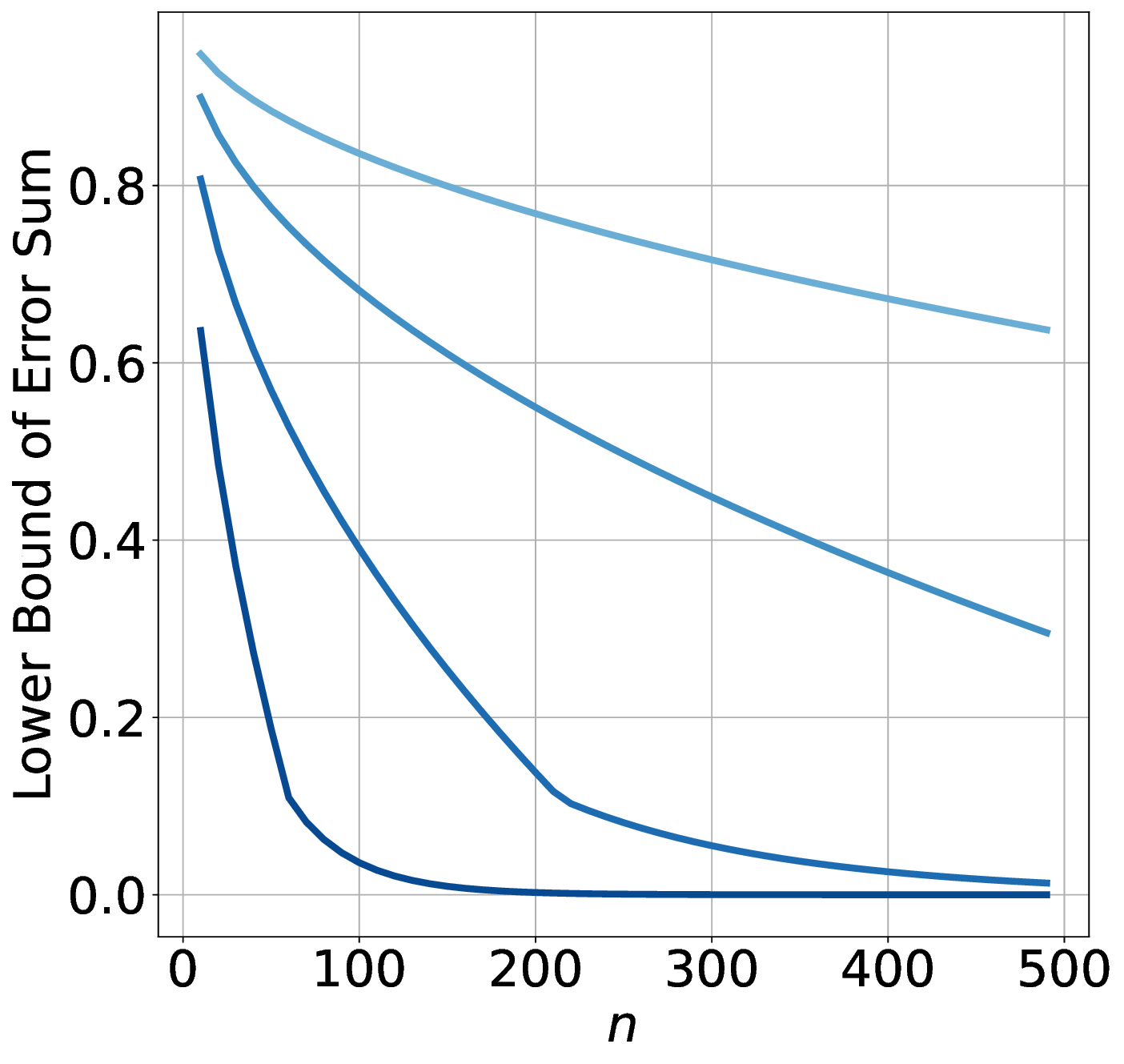}
\vspace{2mm}
    {\small (b) Helpsteer}
  \end{minipage}
    \hfill
    \begin{minipage}[b]{0.23\textwidth}
    \centering
\includegraphics[width=\textwidth]
{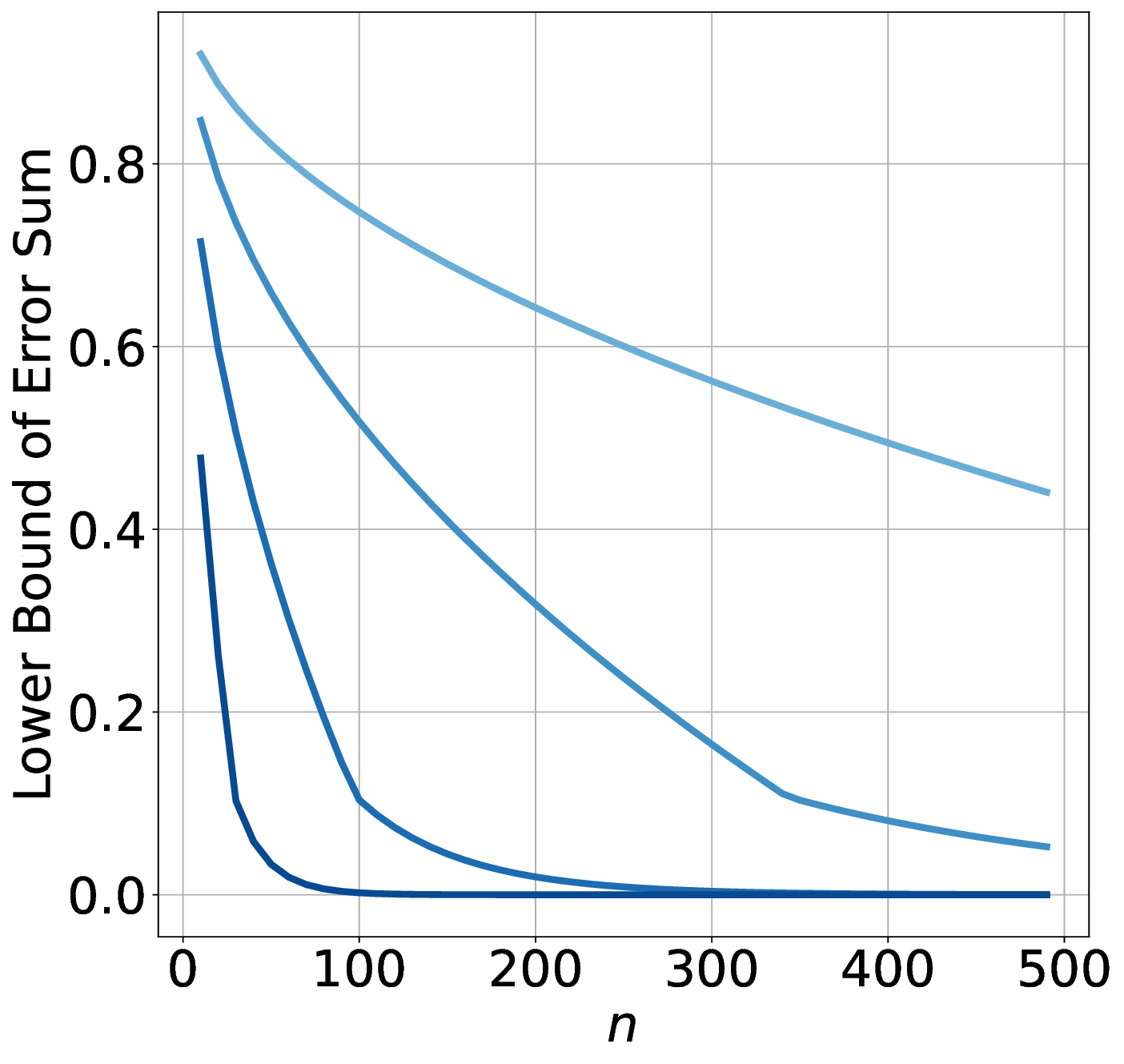}
\vspace{2mm}
    {\small (c) Ultra}
  \end{minipage}
  \hfill  
  \begin{minipage}[b]{0.23\textwidth}
    \centering
\includegraphics[width=\textwidth]
{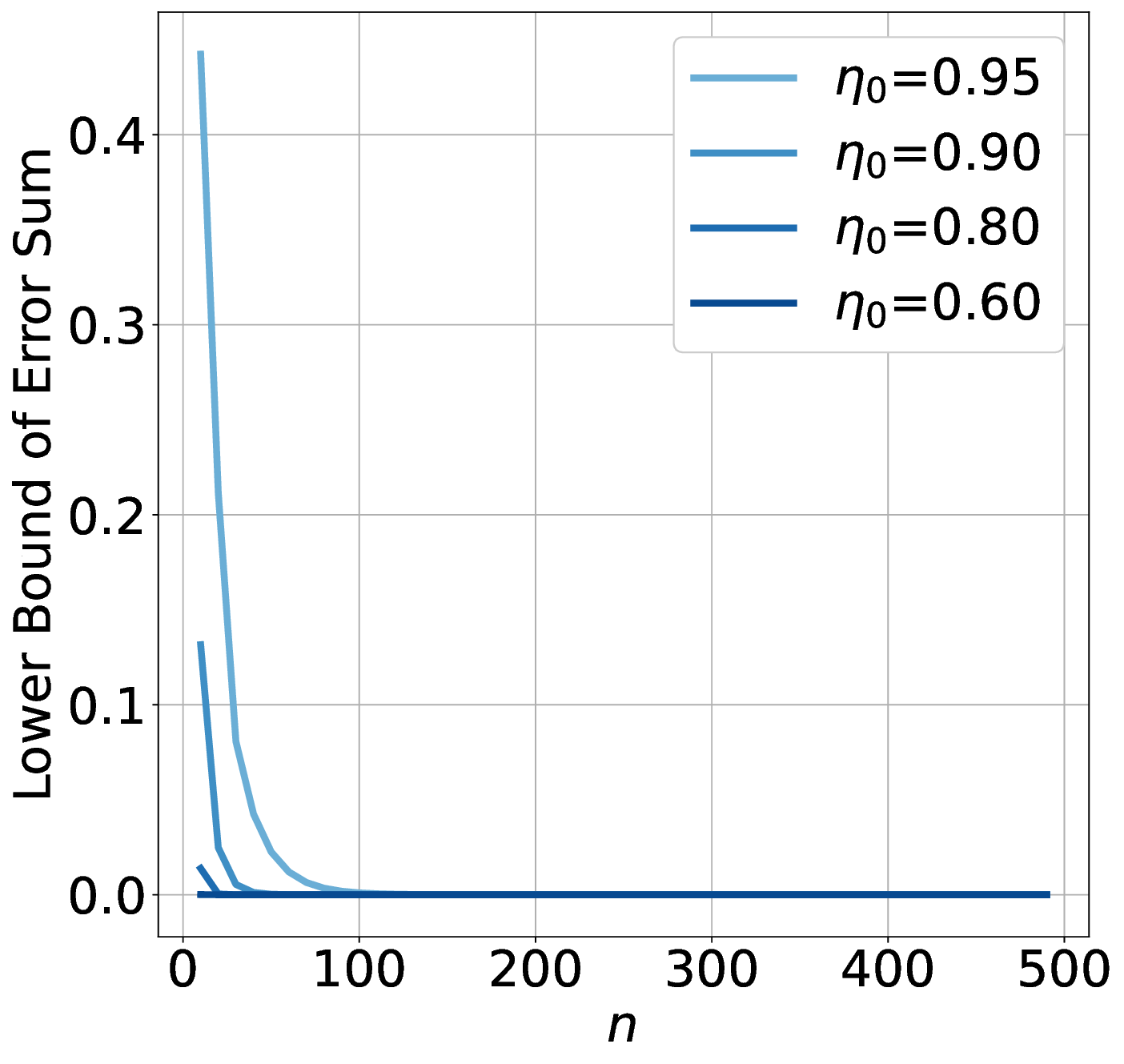}
\vspace{2mm}
    {\small (d) Skywork}
  \end{minipage}
  \caption{ How expert-based monitoring fails on real preference data. Upper four plots: histograms of $\mathbb{P}(y_{\text{chosen}} \succ y_{\text{rejected}} \mid x)$ ($y_{\text{chosen}}$ and $y_{\text{rejected}}$ represent the chosen/preferred and rejected responses, respectively). Lower four plots: the lower bound of the sum of two types of errors against the number of tested annotations $n$ at different $\eta_0$ with $\eta_1=1$ (see Proposition \ref{prop:info_lower_bound}). The observations align with Proposition \ref{prop:info_lower_bound}: the lower bound (i) decreases monotonically with $n$ and increases with $\eta_0$, and (ii) depends on the underlying distribution of preference probabilities. Note that the PKU dataset, where preference probabilities are mostly around 1/2, faces higher errors in assessing annotation quality than datasets (e.g., Skywork) where preference probabilities deviate further from 1/2. See Appendix \ref{appx:fig_hist_LB} for the setup and additional results with $\eta_1<1$.}
\label{fig:main_hist_LB}
\end{figure}

\subsubsection{Agreement-based test -- a practical monitoring algorithm}

Proposition \ref{prop:info_lower_bound} presents a lower bound and some negative numerical results on the practical feasibility (even in the best and most optimistic case) of expert-based monitoring. Now we turn to a more practical scenario by specifying the test function $\Psi$, which provides additional insights and a finer characterization of the test error. Consider an agreement-based test as follows, which is commonly used in the literature of data quality control \citep{krippendorff1989content,krippendorff2004reliability,artstein2008inter,monarch2021human}. For each sample $(x_i,y_{i,1},y_{i,2})$, let $Z_i$ denote the annotator's label and $Z_{i}^{(e)}$ denote the label of the expert. Accordingly, we define the agreement variable 
\[
A_i \coloneqq \begin{cases}
1, \quad & \text{if } Z_i = Z_i^{(e)}, \\
0, \quad & \text{otherwise}.
\end{cases}
\]


It is very natural to suppose that $Z_i$ and $Z_i^{(e)}$ are conditionally independent based on the observation of each $(x_i, y_{i,1}, y_{i,2})$. Let the preference probability of the expert
$$p_{e}(x,y_1,y_2) \coloneqq \mathbb{P}\left(Z_{i}^{(e)}=1\Big|(x,y_1,y_2)\right).$$
Then we have the probability of agreement
$$\mathbb{P}(A_i=1|(x_i,y_{i,1},y_{i,2})) = 2\eta\left(p(x,y_{i,1},y_{i,2})-1/2\right)\left(p_e(x,y_{i,1},y_{i,2})-1/2\right)+1/2.$$

Define
\[
c(x, y_1, y_2) \coloneqq 4\left(p(x,y_{i,1},y_{i,2})-1/2\right)\left(p_e(x,y_{i,1},y_{i,2})-1/2\right)
\]
to quantify the degree of (underlying) agreement between the agent's and the principal's preferences. Intuitively, the sign of $c(x, y_1, y_2)$ indicates whether the preferences between the annotator and the expert are aligned (positive sign) or misaligned (negative sign), while its absolute value reflects the strength of the alignment or misalignment.

Moreover, denote
\[
c_i \coloneqq c(x_i, y_{i,1}, y_{i,2}) \quad \text{and} \quad \bar{c} \coloneqq \mathbb{E}[c_i]
\]
where the expectation is taken with respect to the sample $(x_i, y_{i,1}, y_{i,2})$.

Consider an expert-based monitoring system as Algorithm \ref{alg_expert} that uses the collection $\bm{A} \coloneqq \{A_i\}_{i=1}^n$ for the test $\Psi$. Then the following proposition shows that a test based on the average of the agreement variables $$\bar{A}\coloneqq \frac{1}{n}\sum_{i=1}^n A_i$$ is uniformly most powerful. It also relates the distribution of $\bar{A}$ with the quantity of $\bar{c}.$
\begin{proposition} We have
$n \cdot \bar{A}$ follows a binomial distribution $\mathrm{Binomial}\left(n, \frac{1+\bar{c} \eta}{2}\right)$.
If $\bar{c}$ is known, then $\bar{A}$ is a sufficient statistic of $\eta$. In addition, one can build a Neyman-Pearson uniformly most powerful test for $\eta$ based on $\bar{A}$.
\label{prop:UMP}
\end{proposition}

For the later contracting analysis, Proposition \ref{prop:UMP} provides the key reduced-form takeaway from expert-based monitoring: after aggregation, the monitored statistic depends on the annotator's (average) quality only through the pass probability $\frac{1+\bar{c}\eta}{2}$. Thus, expert-based monitoring furnishes a noisy but contractible signal whose informativeness is summarized by $\bar{c}$. Section \ref{sec:incentivizing} will use exactly this type of reduced-form signal when studying contracts. 

Proposition~\ref{prop:info_lower_bound} gives a lower bound for the error of a hypothesis testing problem with respect to $\eta$. The following proposition considers the estimation problem, and one can yield a similar lower bound by applying Le Cam's two-point method. 


\begin{proposition}
\label{prop:lecam_lower_bound}
Denote $\bm{A} = \{A_1, \dots, A_n\}$'s joint distribution under $\eta$ by $\mathcal{Q}_{\eta}^{\otimes n}$. Denote $\hat{\eta}$ as any estimator of $\eta$ based on $\bm{A}$, we have
\begin{equation}
\begin{split}
\inf_{\hat{\eta}} \sup_{\eta \in [0, 1]} \mathbb{E}\left[|\hat{\eta} - \eta|\right]
&\geq \sup_{\eta_0, \eta_1 \in [0, 1]} \frac12 \cdot |\eta_0 - \eta_1| \cdot \left(1-\mathrm{TV}(\mathcal{Q}_{\eta_0}^{\otimes n}, \mathcal{Q}_{\eta_1}^{\otimes n})\right)\\
&\geq \sup_{\eta_0, \eta_1 \in [0, 1]} \frac14 \cdot |\eta_0 - \eta_1| \cdot \exp\left(-nD_{\mathrm{kl}}(\mathcal{Q}_{\eta_0} \| \mathcal{Q}_{\eta_1})\right)
= \Omega\left(\frac{1}{|\bar{c}|\sqrt{n}}\right).
\end{split}
\label{eq:lecam_continuous_lower_bound}
\end{equation}
\end{proposition}

From the proposition, a larger $|\bar{c}|$ can reduce the minimax lower bound of the error for any possible estimator of $\eta$. And $|\bar{c}|$ is large if the sample $(x, y_1, y_2)$ has an apparent preference meaning in the eyes of both the annotator and the expert; in such a case, both $|p(x,y_1,y_2) - 1/2|$ and $|p_e(x,y_1,y_2) - 1/2|$ are large on expectation and share the same sign for most of the samples (or disagree with each other on most samples, which is unlikely in the real cases). From Proposition~\ref{prop:info_lower_bound} and Figure \ref{fig:main_hist_LB}, we know that the lower bound in the previous proposition will be reduced when $|p(x,y_1,y_2) - 1/2|$ is large on expectation. Now in this specialized setting, we additionally expect $|p_e(x,y_1,y_2) - 1/2|$ to be large. The quantity $|p_e(x,y_1,y_2) - 1/2|$ will be maximized when the expert has the exact knowledge of $p(x,y_1,y_2)$ which is hardly practical. 


\textbf{Discussion of Proposition~\ref{prop:info_lower_bound} and Proposition~\ref{prop:UMP}. }
Proposition~\ref{prop:info_lower_bound} and Proposition~\ref{prop:UMP} address the same question at two different levels: can we reliably tell, from a limited audit, whether an annotator is producing high-quality labels or not?

Proposition~\ref{prop:info_lower_bound} is a  best-case benchmark: it lower bounds the smallest possible total testing error over all decision rules that could use the inspected items, and it is stated under an intentionally optimistic scenario (the test can even leverage the true underlying preference probabilities $p(x,y_1,y_2)$).
For example, from Proposition~\ref{prop:info_lower_bound}, if most comparisons are inherently close calls (i.e., many comparisons have preference probability near $1/2$ under a population level), then the data generated by a ``high-effort'' annotator and a ``low-effort'' annotator look statistically similar, so any  monitoring method needs many audited samples to separate the two with small error.

Proposition~\ref{prop:UMP} does not imply that we have a test that attains (or matches) the lower bound in Proposition~\ref{prop:info_lower_bound}.
Instead, Proposition~\ref{prop:UMP} analyzes a specific, implementable monitoring approach that is common in practice: the auditor/expert provides their own label $Z_i^{(e)}$, and the principal uses only the agreement indicators $A_i=\mathbb{I}\{Z_i=Z_i^{(e)}\}$.
The average agreement $\bar A$ summarizes the relevant evidence about the annotator's quality $\eta$ (when $\bar c$ is known), then by Proposition~\ref{prop:UMP}, the optimal rule within this agreement-based experiment is a simple threshold test in $\bar A$ (Neyman--Pearson uniformly most powerful).

Therefore, we present Proposition~\ref{prop:info_lower_bound} as an optimistic lower-bound benchmark that highlights the intrinsic difficulty from the heterogeneous preference of annotators, and Proposition~\ref{prop:UMP} as a complementary characterization of the best achievable test in a practically relevant (agreement-based) monitoring implementation, rather than as a claim of matching the information-theoretic benchmark.

\subsection{Self-consistency monitoring}

Now we present the second assessment method which we call \textit{self-consistency monitoring}. It aims to address the heterogeneity among different annotators, and it also saves the additional time and money cost of expert-based monitoring. The idea is to duplicate some samples in the dataset for the annotator to label them twice and then to check if the two labels are consistent. Such an idea is very natural and often adopted in survey design. The full procedure is described in Algorithm \ref{alg_self}.

\begin{algorithm}[ht!]
\centering
\caption{Self-consistency monitoring}
\label{alg_self}
\begin{algorithmic}
\State \textbf{Input:} A set of $N+n$ samples $\mathcal{D}_{\mathcal{Z}}'$ produced by one annotator. Among the $N+n$ samples, there are $n$ duplicated samples, i.e., $n$ random samples from the original $\mathcal{D}_{\mathcal{Z}}$ where each appears twice in $\mathcal{D}_{\mathcal{Z}}'$. Let 
$Z_i$ and $Z_{i}'$ denote the annotated labels for these $n$ samples for their first and second appearance, and define the self-consistency variable as
$$A_{i} \coloneqq \begin{cases}
1, \quad & \text{if } Z_i = Z_i', \\
0, \quad & \text{otherwise}.
\end{cases} $$
\State \textbf{Output:} Assert $\eta\le \eta_0$ or $\eta\ge \eta_1$ 
\State Based on the self-consistency variables, define
$$\bm{A} = \{A_{i}\}_{i=1}^n$$
\State The self-consistency monitoring decision is based on some testing function $\Psi: \bm{A} \rightarrow \{0,1\}$
\State \textbf{Return} $\Psi(\bm{A})$
\end{algorithmic}
\end{algorithm}

In comparison with expert-based monitoring, self-consistency monitoring employs the annotators themselves as the experts to assess their annotations. We can augment the previous probability model \eqref{eqn:Z_annotation_quality} and \eqref{eqn:annotation_eta} to capture the annotator's behavior in the second annotation of the same sample. Specifically, for a sample $(x,y_1,y_2)$, let $Z$ and $Z'$ be the two annotations produced by the same annotators. Consider the following probability model 
\begin{equation}
\begin{split}
\mathbb{P}\left(Z=Z'|V=0, x, y_{1}, y_{2}\right) &=\frac{1}{2},\\
\mathbb{P}(Z=Z'|V=1, x, y_1, y_2)& = 1 -\frac{\delta}{2}.
\end{split}
\label{eqn:Z_annotation_quality_new}
\end{equation}
where $\delta\in[0,1]$ is the disagreement probability strength under full commitment. To interpret the model, when the annotator assigns the label randomly for the sample ($V=0$), then the two annotations agree with 1/2 probability. When the annotator fully commits on the sample ($V=1$), it may still happen that the two annotations are different, and we model this disagreement probability with $\delta/2$, i.e., with probability $\delta$ the annotator may randomly guess a label conditional on the full commitment.  In the ideal case, the annotator memorizes precisely their preference over all the samples and $\delta\equiv 0$; in practice, even under a full commitment, $\delta$ can still be positive but it should be reasonably small. Then under the probability model of \eqref{eqn:Z_annotation_quality}, \eqref{eqn:annotation_eta} and \eqref{eqn:Z_annotation_quality_new}, the agreement probability on the $i$-th duplicate sample
\begin{equation}
\mathbb{P}\left(A_i = 1|x_{i},y_{i,1},y_{i,2}\right) = \eta(1-\delta)/2 + 1/2.
\label{eqn:agreement_self_consist} 
\end{equation}

This reduced-form agreement probability is precisely the object needed in Section \ref{sec:incentivizing}. Once self-consistency monitoring is adopted, the principal can contract on the empirical agreement rate, and the coefficient multiplying $\eta$ summarizes how informative each monitored sample is about the annotator's quality.

\begin{figure}[ht!]
  \centering
  \begin{minipage}[b]{0.23\textwidth}
    \centering
\includegraphics[width=\textwidth]
{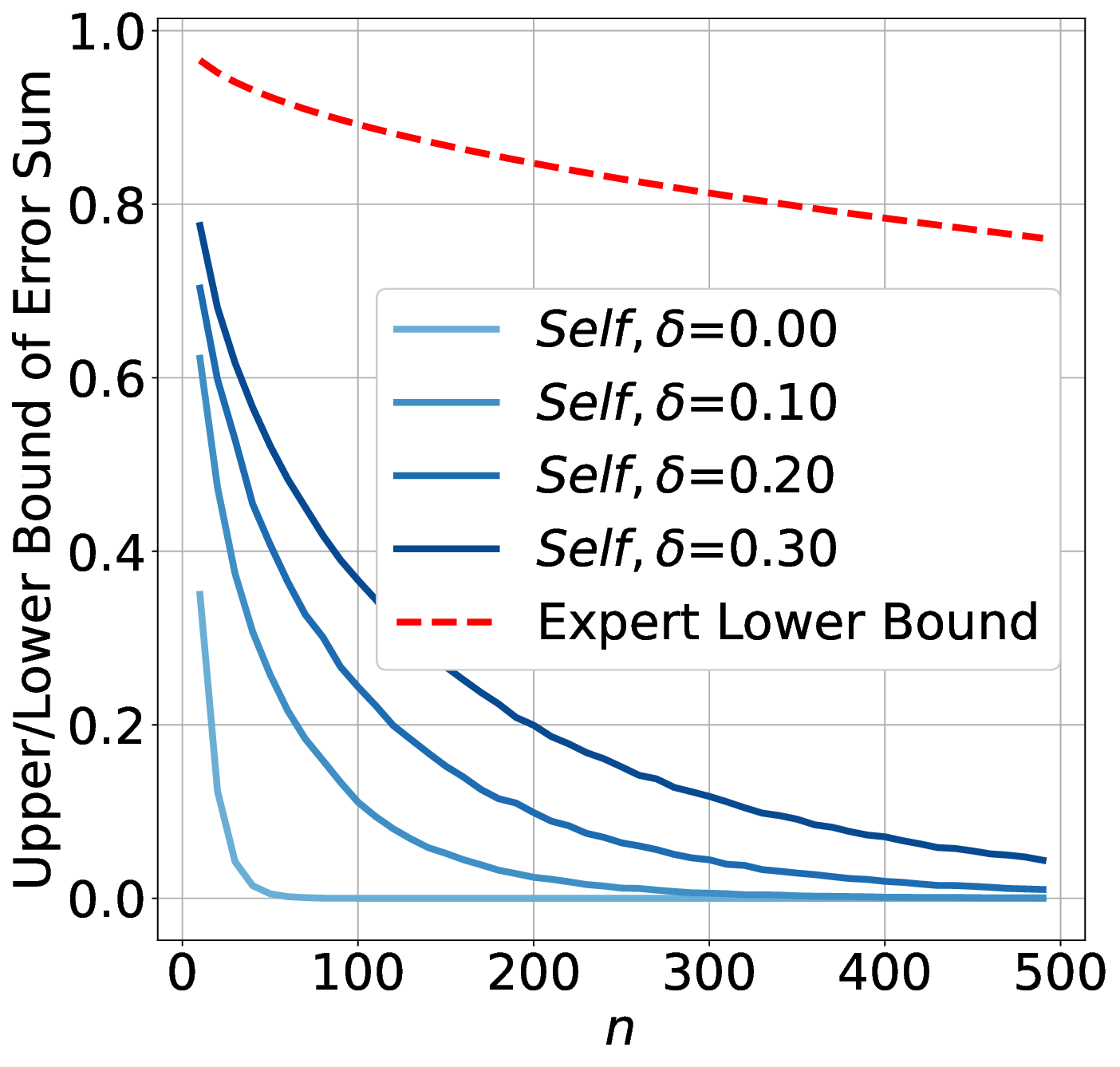}
    \vspace{2mm} {\small (a) PKU, $\eta_0=0.8$}
  \end{minipage}
    \hfill
  \begin{minipage}[b]{0.23\textwidth}
    \centering
\includegraphics[width=\textwidth]
{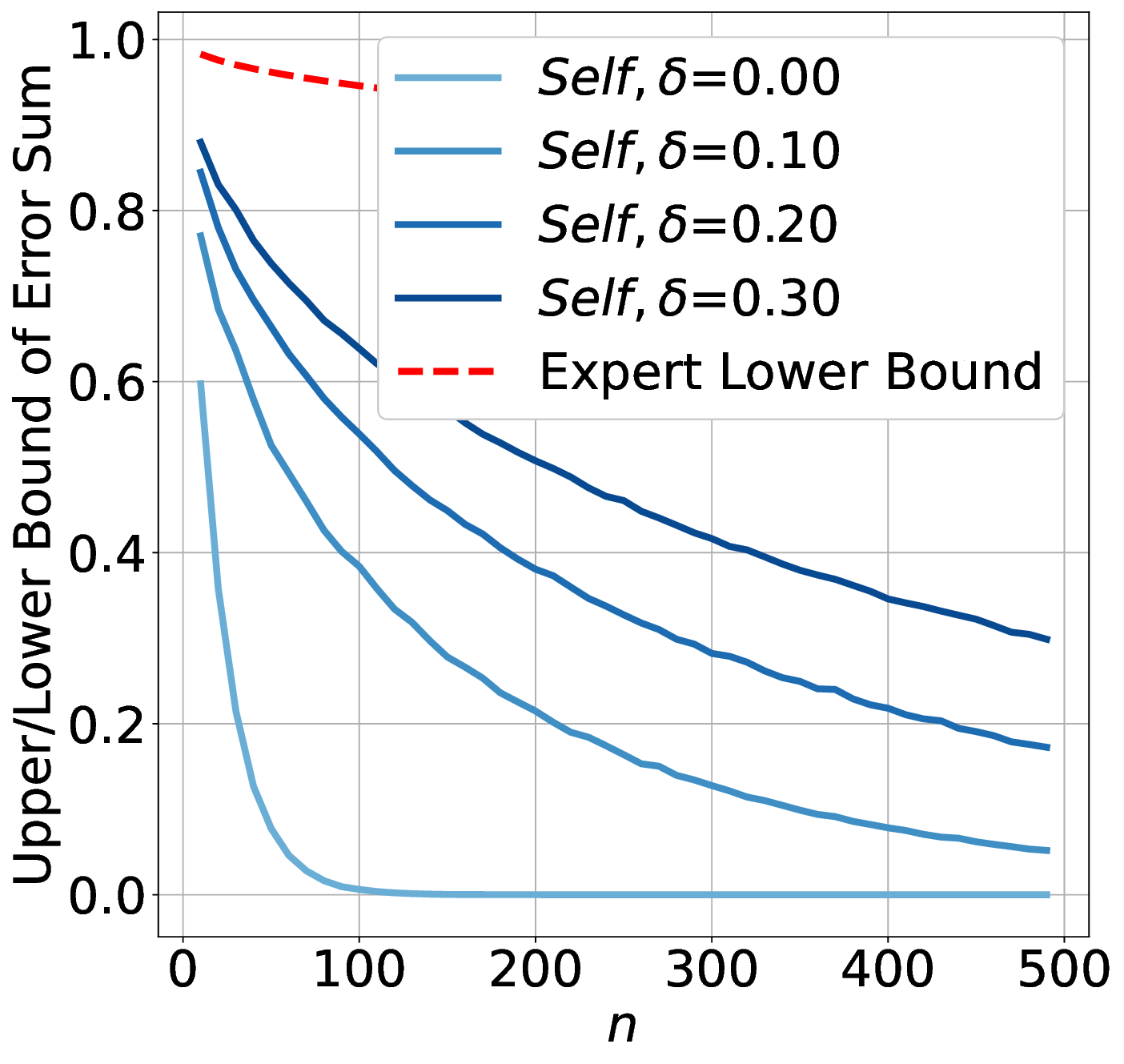}
    \vspace{2mm} {\small (b) PKU, $\eta_0=0.9$}
  \end{minipage}
    \hfill
    \begin{minipage}[b]{0.23\textwidth}
    \centering
\includegraphics[width=\textwidth]
{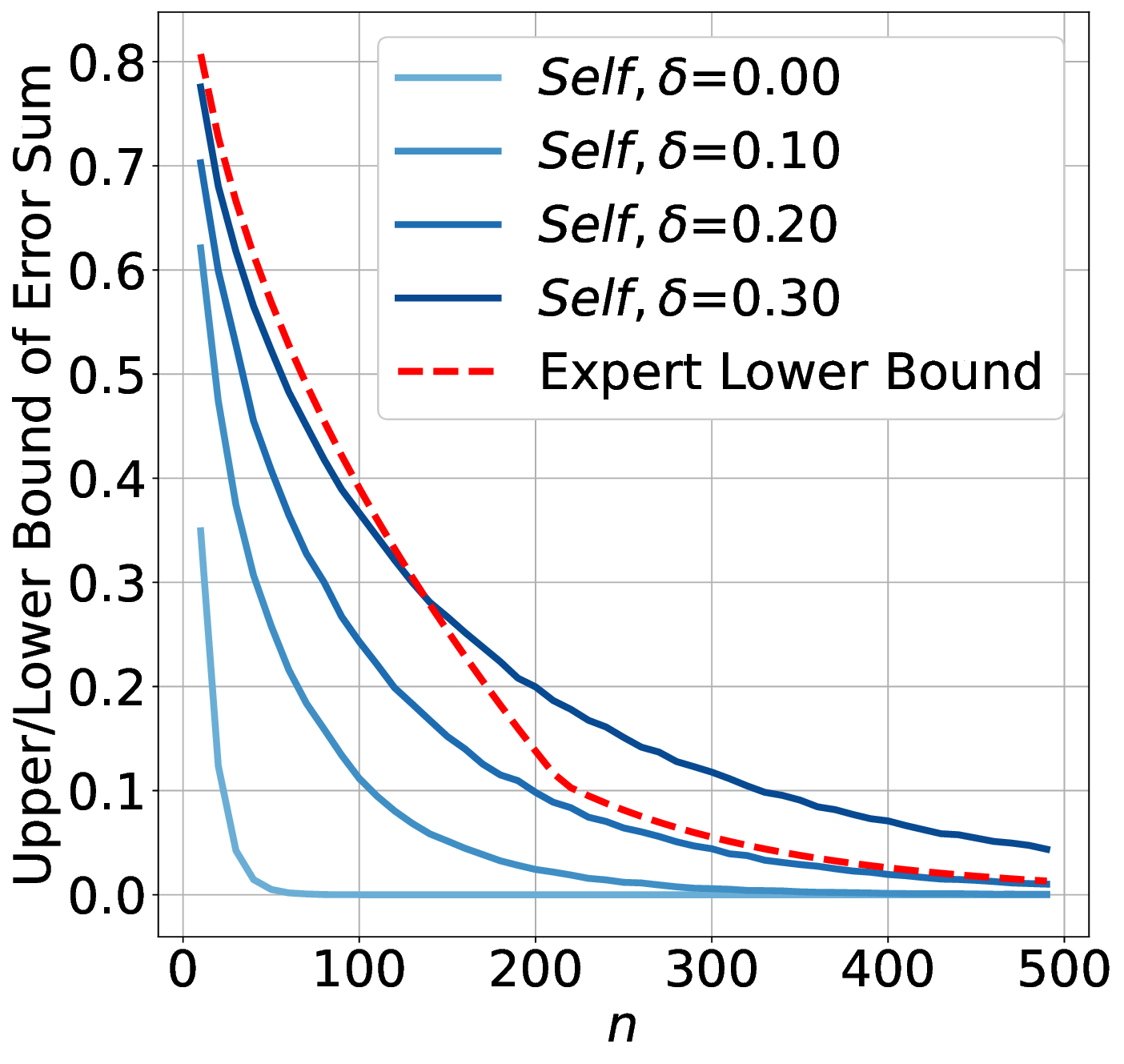}
    \vspace{2mm} {\small (c) Helpsteer, $\eta_0=0.8$}
  \end{minipage}
  \hfill  
  \begin{minipage}[b]{0.23\textwidth}
    \centering
\includegraphics[width=\textwidth]
{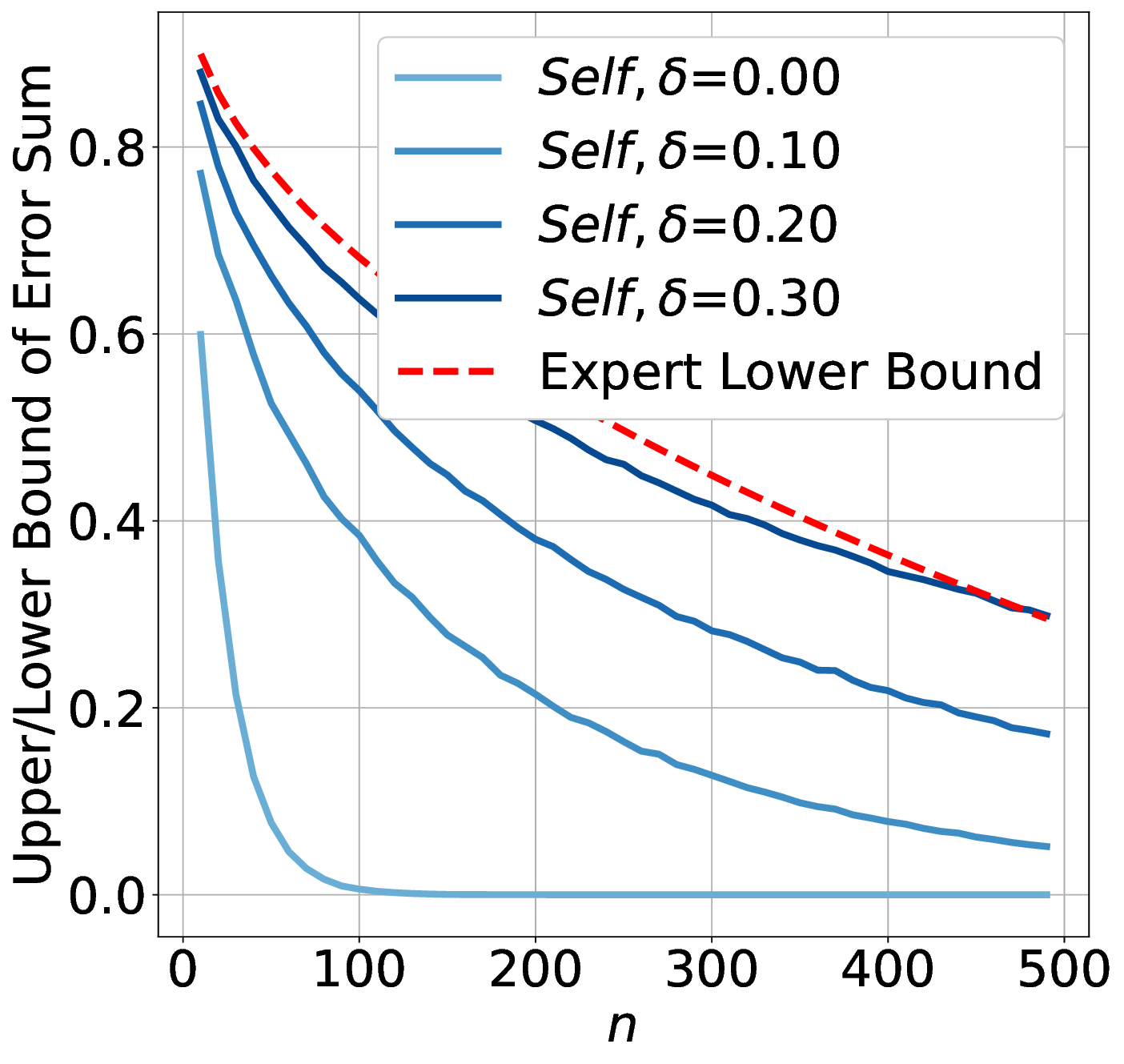}
   \vspace{2mm} {\small (d) Helpsteer, $\eta_0=0.9$}
  \end{minipage}
  \caption{ Comparison between self-consistency monitoring (upper bound) and expert-based monitoring (lower bound). For the sum of two types of errors, we plot the \textbf{upper bound} for self-consistency monitoring with various values of $\delta$ (blue, thick line) and the \textbf{lower bound} for expert-based monitoring (red, dashed line), evaluated at $\eta_0 \in \{0.8, 0.9\}$ and $\eta_1 = 1$ for two datasets. Even with a nontrivial disagreement probability $\delta$, self-consistency monitoring outperforms expert-based monitoring over a wide range of $n$, especially when the average preference probability is near $1/2$ (e.g., PKU). See Appendix~\ref{appx:fig_self_UB} for details on the experimental setup and additional results with $\eta_1 < 1$.}
\label{fig:main_self_UB}
\end{figure}

Figure \ref{fig:main_self_UB} plots the realized error of self-consistency monitoring against the lower bound of Proposition \ref{prop:info_lower_bound}. The fact that the realized bound is better than the theoretical lower bound is not a contradiction, but it is because of the introduction of duplicated samples in Algorithm \ref{alg_self}, and the additional structure of \eqref{eqn:agreement_self_consist}. Another benefit of self-consistency monitoring is that its performance is contingent on the parameter $\delta$ but not on the underlying distribution $p(x,y_1,y_2).$ Thus it resolves the challenge of annotator heterogeneity in performance assessment, and also saves the additional costs in hiring experts. To some extent, we can think of self-consistency monitoring as having the annotators themselves as the experts to assess their own annotation performance. In this light, the annotator uses their second label to the same sample to judge their full commitment to the first label, and as the judgment is from the same annotator, there is no heterogeneity among annotators involved. 

 \begin{proposition}
\label{prop:KL_compare}
Let $\hat{\eta}$ be any estimator of $\eta$ based on $\bm{A}$. The minimax lower bound of the estimation error for the self-consistency monitoring method is
\begin{equation}
\inf_{\hat{\eta}} \sup_{\eta \in [0, 1]} \mathbb{E}\left[|\hat{\eta} - \eta|\right] = \Omega \left(\frac{1}{(1-\delta) \sqrt{n}}\right)
\end{equation}
where the expectation is taken with respect to the randomness of $\bm{A}$ in Algorithm \ref{alg_self}.
\end{proposition}




While Figure \ref{fig:main_self_UB} provides numerical evidence on the advantage of self-consistency monitoring, Proposition \ref{prop:KL_compare} shows that this advantage also appears at the level of minimax lower bounds. When the within-annotator agreement signal is sufficiently informative: for example, when $1-\delta \geq \bar{c}= \mathbb{E}[4|p(x,y_1,y_2) - 1/2|\cdot |p_e(x,y_1,y_2)- 1/2|]$, self-consistency monitoring achieves a smaller lower bound than expert-based monitoring. This condition is natural in preference annotation: a fully committed annotator should exhibit high within-annotator consistency over time, so $\delta$ should be small. This statistical advantage is not only of independent interest for practitioner-facing quality control. It is also the bridge to Section \ref{sec:incentivizing}: a monitoring technology that extracts more information about $\eta$ per inspected sample reduces the number of monitored samples needed later for the contract to generate strong incentives and a small utility gap.

\section{Incentivizing Human Annotators under Imperfect Monitoring}
\label{sec:incentivizing}

In the previous section, we have introduced two monitoring algorithms for the human preference data, which are the expert-based monitoring (Algorithm~\ref{alg_expert}) and the self-consistency monitoring (Algorithm~\ref{alg_self}). We have also examined the information-theoretic limit for those monitoring methods to test \emph{how many samples} are needed to get a concrete estimation of the annotation quality within some \emph{error bounds}. The company can get some estimation of the annotator's commitment level $\eta$ using the algorithms and analysis in Section~\ref{sec:assess}; however, in real interaction between the company and the annotator, we may wonder:
\begin{center}
    \emph{How can the company regulate the behavior of the annotator via those monitoring methods? What is the economic outcome of the monitoring methods?}
\end{center}

In this section, we dive into the investigation of the above questions. To start with, we assume that the company is offering some \emph{contract} to the annotator, and their behaviors are modeled via the classic principal-agent model \citep{harris1979optimal, holmstrom1979moral, grossman1992analysis, laffont2009theory, frick2023monitoring}. There are two roles in this model: the company (the principal), and the annotator (the agent). In this framework, the agent only acts for their own interests, where \emph{moral hazard} takes place. The principal needs to design a contract to shape the agent's behavior, with the aim of inducing actions beyond serving only the interests of the agent. A well-designed contract mitigates the problem of moral hazard. In the context of data annotation, moral hazard arises when annotators provide low-quality annotations without bearing the full consequences of doing so. In our setting, the LLM company acts as the principal by designing a contract, and the annotator acts as the agent by choosing an annotation quality level once the contract is offered.

The underlying source of moral hazard is the information asymmetry between the two parties. In the principal-agent model, this asymmetry arises because the agent's action is not directly observed by the principal. In our setting, the hidden action is the annotation quality $\eta$ (see the probability model in \eqref{eqn:Z_annotation_quality} and \eqref{eqn:annotation_eta}). The agent knows their own annotation quality $\eta$, whereas the principal cannot directly observe it without an assessment system such as those introduced in the previous section. The statistical analysis in Section~\ref{sec:assess} determines the information structure of the contracting problem. Once the principal can observe only a noisy pass/fail signal generated by either expert-based monitoring or self-consistency monitoring, the economic problem becomes how to map that signal into payment.

At this point, the paper shifts from an operational question to an economic one. The issue is no longer simply whether a company can assess an annotator, but rather how the chosen monitoring technology shapes the incentives, the principal's utility, and the economic value created by the annotation pipeline. Put differently, Section~\ref{sec:assess} asks how many samples are needed to measure quality, whereas Section~\ref{sec:incentivizing} asks how many samples are needed to monetize that measurement through contract design.

\subsection{Principal-agent model for data annotation}

\label{sec:pa_models}

We first present the agent's utility function. As noted earlier, the action of the agent is the annotation quality $\eta \in [0, 1]$. It doesn't hurt to consider the total annotation number $N$ to be fixed. The utility function of the agent is modeled by 
\[
U_a(w_a, \eta) \coloneqq G_a(w_a) - E(\eta).
\]
Here $w_a$ is the wealth/money payment received by the agent for annotating $N$ samples, and $w_a$ is jointly determined by the contract (will be defined shortly) and the annotated data. The function $G_a: \mathbb{R}\rightarrow \mathbb{R}$ is the monetary utility of the agent and is strictly increasing as the payment $w_a$ increases. The effort function $E:[0,1]\rightarrow \mathbb{R}$ is a function of $\eta$ (given the total annotation number $N$ is fixed). 

\begin{assumption}
\label{assm:basic_assm_1}
Assume that the monetary utility $G_a$ is twice continuously differentiable, monotonically increasing, and strictly concave and that the effort $E$ is purely a function of $\eta$ (for fixed $N$ annotations) that is twice continuously differentiable, monotonically increasing, and convex, i.e.,
\[
\frac{\mathrm{d} G_a}{\mathrm{d} w_a} > 0, \quad \frac{\mathrm{d}^2 G_a}{\mathrm{d} w_a ^{\ 2}} < 0;
\]
\[\frac{\mathrm{d} E}{\mathrm{d} \eta} > 0, \quad \frac{\mathrm{d}^2 E}{\mathrm{d} \eta^2} \geq 0.
\]
\end{assumption}

Assumption \ref{assm:basic_assm_1} imposes some basic assumptions on the functions $G_a$ and $E$. The increasingness of $G_a$ is natural, and the concavity of $G_a$ is also known as the risk aversion of the agent. In the risk-neutral case of a linear $G_a$, the discussion of contract design under the principal-agent model becomes trivial (as well known in the literature \citep{harris1979optimal}). The assumption of $E$ simply reflects human nature: to maintain a high annotation quality (large $\eta$) requires a large effort, and the paid effort becomes even larger when the agent wants to achieve an even higher quality (say, $\eta = 0.96$) from an already high quality (say, $\eta = 0.95$). The major implication of Assumption \ref{assm:basic_assm_1} for our discussions is to guarantee that (one of) the optimal mixed strategy of the agent must be a pure strategy, so we can focus on the pure strategy only when considering the agent's action.

Next, we model the principal's utility function by
\[
U_p(w_a, \eta) \coloneqq -w_a + \mu(\eta).
\]
Here $w_a$ is the payment made by the principal (company) to the agent (annotator) and $\eta$ is the annotation quality -- a decision made by the annotator. The function $\mu: [0,1] \rightarrow \mathbb{R}$ models the utility function of the company with respect to the data quality $\eta$, for which we make the following assumption.

\begin{assumption}
Assume that $\mu$ is a twice continuously differentiable, increasing, and concave function of $\eta$, i.e.,
\begin{equation*}
 \frac{\mathrm{d} \mu}{\mathrm{d} \eta} \geq 0, \quad \frac{\mathrm{d}^2 \mu}{\mathrm{d} \eta^2} \leq 0.
\end{equation*}    
\label{assm:basic_assm_2}
\end{assumption}
The increasingness of $\mu$ is intuitive in that better data quality leads to large utility. The concavity of $\mu$ is widely acknowledged and verified by the downstream performance of an ML/LLM model under noisy data \citep{gao2024impact,chowdhury2024provably,liang2024robust}.

We note that an important quality appearing in both the agent and the principal utility functions is the payment $w_a.$ Under the principal-agent model, $w_a$ is determined through a (sequence of) contract(s) $F_n \in \mathcal{F}_n$. Each possible $F_n$ is a function that maps the assessment dataset $\mathcal{D}_n$ to the payment to the annotator,
$$w_a \coloneqq F_n(\mathcal{D}_n).$$
The assessment dataset $\mathcal{D}_n$ is a dataset of $n$ samples used to assess the performance of the annotator. For example, $\mathcal{D}_n$ can be the dataset $\mathcal{D}_{\text{test}}$ in expert-based Algorithm \ref{alg_expert} or the dataset $\bm{A}$ in self-consistency Algorithm \ref{alg_self}. The function $F_n$ prescribes how much money $w_a$ is paid to the annotator based on the $\mathcal{D}_n$. It is a \textit{sequence} of contracts in that it is indexed by $n$; for a different sample size $n$, the function can be different. The contract class $\mathcal{F}_n$ is determined by the principal, and then the agent's decision of $\eta$ is made in the knowledge and under the constraint of $\mathcal{F}_n$. In this light, $\mathcal{F}_n$ provides a handle for the principal to regulate the agent's behavior. In the following two sections, we further specify two classes of possible contracts, binary contract and linear contract. 



With the definition of utility functions and contracts, the dynamics between the principal and the agent can be formulated as a constrained optimization problem. The idealized scenario is called the \textit{first-best} solution which corresponds to the following program
\begin{equation}
\begin{split}
\mathcal{C} \coloneqq \max_{F_n \in \mathcal{F}_n, \eta \in [0, 1]} \, & \mathbb{E}\left[U_p(F_n(\mathcal{D}_n), \eta)\right], \\
\text{s.t. } & \mathbb{E}\left[U_a(F_n(\mathcal{D}_n), \eta)\right] \geq U_0
\end{split}
\label{eq:first_best}
\end{equation}
where the expectation is taken with respect to the randomness of the annotation data $\mathcal{D}_n$. Here the right-hand side of the constraint $U_0$ is called the leisure utility which denotes the utility the agent gains by rejecting the contract (say, not to work for the company). The constraint requires the agent's expected utility under the contract to be above this threshold $U_0$. Such a constraint is often referred to as the \textit{individual rationality} (IR). The objective function is to maximize the expected utility of the principal. We say it is an idealized scenario in that the problem treats both the contract $F_n$ and the annotation quality $\eta$ as the decision variables. Equivalently, under the first-best formulation, we can think of the principal has a perfect monitoring system and thus an exact knowledge of $\eta$; and consequently, the principal can design the contract in a way to control $\eta$. However, in reality, the principal can only optimize the contract $F_n$ but cannot even observe $\eta$, let alone optimize it. In this sense, the program \eqref{eq:first_best} characterizes the best-achievable utility for the principal for the case that the principal can even control $\eta$.

\begin{assumption}
\label{assm:basic_assm_3}
The wage is bounded, i.e., there exist constants $\underline{w}, \overline{w}\in \mathbb{R}$ such that
\[
F_n(\mathcal{D}_n) \in [\underline{w}, \overline{w}]
\]
almost surely for all $F_n \in \mathcal{F}_n$. 
Besides, the wage range is rich enough to cover all the efforts (plus the leisure utility), i.e.
\[
E([0, 1]) + U_0 \subset \mathrm{interior}\left(G_a([\underline{w}, \overline{w}])\right),
\]
\end{assumption}
The first part of the assumption simply requires the wage $w_a$ paid to the agent to be bounded. This is without loss of generality in that if the wage $w_a$ is too low (say, lower than $U_0 + E(0)$), then the agent would quit the contract; if $w_a$ is too high (say, higher than $\mu(1) - \mu(0)$), then the principal would choose not to provide the contract. The second part of the assumption is to consider a rich enough wage range to ensure that (under perfect monitoring/the first-best setting) the principal could offer a contract to make the agent achieve any effort option $\eta \in [0, 1]$.

\begin{proposition}
\label{prop:first_best_sol}
Under Assumptions \ref{assm:basic_assm_1}, \ref{assm:basic_assm_2}, and \ref{assm:basic_assm_3}, the first-best problem \eqref{eq:first_best} has a unique optimal solution. In addition, the optimal annotation quality
\[
\eta^* \coloneqq \argmax_{\eta \in [0, 1]} \, -G^{-1}_a(E(\eta) +U_0 ) + \mu(\eta),
\]
and the principal pays the agent $w_a^* \coloneqq G^{-1}_a(E(\eta^*) + U_0 )$ units of wealth. Here $G^{-1}_a$ is the inverse function of $G_a$, the existence of which is guaranteed by Assumption \ref{assm:basic_assm_1}.
\end{proposition}

The proposition characterizes the optimal quality $\eta^*$ and the optimal payment amount $w_a^*$ under the first-best program.


A more practical setting is that the principal decides the contract $F_n$ and the agent decides the annotation quality $\eta$. This corresponds to the so-called second-best solution which corresponds to the following bi-level optimization problem
\begin{equation}
\begin{split}
\mathcal{C}_n \coloneqq \max_{F_n \in \mathcal{F}_n} \, & \mathbb{E}\left[U_p(F_n(\mathcal{D}_n), \eta_a(F_n))\right], \\
\text{s.t. } & \mathbb{E}\left[U_a(F_n(\mathcal{D}_n), \eta_a(F_n))\right] \geq U_0, \\
& \eta_a(F_n) \in \argmax_{\eta \in [0, 1]} \mathbb{E}\left[U_a(F_n(\mathcal{D}_n), \eta)\right].
\end{split}
\label{eq:second_best}
\end{equation}
Here, as before, the expectation is taken with respect to the randomness of the annotation data $\mathcal{D}_n$. Compared to \eqref{eq:first_best}, there is an additional last constraint in \eqref{eq:second_best}. The constraint says that the annotation quality is determined by the annotator maximizing their expected utility, known as \textit{incentive compatibility} (IC). As a result, the outer program \eqref{eq:second_best} contains only one decision variable $F_n$ for the principal to optimize over. 

It is important to note that the annotator's chosen quality $\eta_a(F_n)$ depends on the contract $F_n$. This dependence follows directly from the sequential structure of the game: the principal first offers a contract, and the agent then chooses an action after observing it. In this formulation, although the principal can solve for the equilibrium quality $\eta_a(F_n)$ within the model (i.e., the principal knows enough about the agent’s utility and the monitoring technology to solve for the best-response mapping), that quality is determined through the contract and is not itself a directly contractible object. Thus, contract design still depends only on the monitored dataset $\mathcal{D}_n$. Hence, $\eta$ is neither directly observed nor directly contracted upon; the power of the contract remains limited by the informativeness of the monitoring technology characterized in Section~\ref{sec:assess}. We also note that the individual rationality constraint is kept in the formulation in \eqref{eq:second_best} as a threshold for the minimum utility offered to the agent such that the agent accepts the contract. Formulating the principal's contract design problem as maximizing the expected utility subject to the IR and the IC constraints is the standard practice in the literature of the principal-agent model \citep{holmstrom1979moral, grossman1992analysis, holmstrom1987aggregation}.

Another more restricted definition of second-best solutions examined in some literature \citep{frick2023monitoring} requires that $\eta_a(F_n) \equiv \eta^*$ (See Proposition \ref{prop:first_best_sol} for the definition of $\eta^*$).
\begin{equation}
\begin{split}
\tilde{\mathcal{C}}_n \coloneqq \max_{F_n \in \mathcal{F}_n} \, & \mathbb{E}\left[U_p(F_n(\mathcal{D}_n), \eta^*)\right], \\
\text{s.t. } & \mathbb{E}\left[U_a(F_n(\mathcal{D}_n), \eta^*)\right] \geq U_0, \\
& \eta^* \in \argmax_{\eta \in [0, 1]} \mathbb{E}\left[U_a(F_n(\mathcal{D}_n), \eta)\right].
\end{split}
\label{eq:restricted_second_best}
\end{equation}

For the above three programs \eqref{eq:first_best}, \eqref{eq:second_best}, and \eqref{eq:restricted_second_best}, we know 
\begin{equation*}
\mathcal{C} \geq \mathcal{C}_n \geq \tilde{\mathcal{C}}_n.
\end{equation*}

While ${\mathcal{C}}_n$ and $\tilde{\mathcal{C}}_n$ represent a more practical setting, the question is how large the gap is between these two and the ideal case of $\mathcal{C}.$ The analysis of the gap sheds light on how we should design the contract class $\mathcal{F}_n$ -- the domain of $F_n$. A proper $\mathcal{F}_n$ leads to a smaller gap which means an effective regulation of the agent's behavior in the choice of $\eta$, and the regulation is achieved implicitly by the design of $\mathcal{F}_n$. 

A recurring puzzle in contract theory is that optimal (second-best) contracts can be highly complex and difficult to implement, while the contract forms that are most attractive in practice, like binary bonus schemes or linear pay-for-performance rules, are typically not fully optimal. This tension raises a natural quantitative question: 
\begin{center}
    \emph{Are those simple contracts optimal? If not, how costly is simplicity?}
\end{center}

Rather than evaluating simple contracts only qualitatively, we study their performance through convergence rates of the principal's payoff toward the first-best benchmark as monitoring becomes more informative (e.g., as the assessment sample size $n$ increases). Convergence rates provide a transparent measure for comparing simple contract classes and for understanding how quickly improved monitoring mitigates the inefficiency gap created by restricting attention to implementable payment rules. More specifically, we study two important types of contracts: the \emph{binary contracts} and the \emph{linear contracts}. The reason for considering the binary contracts is that the existing work of \citet{frick2023monitoring} has shown that if the action space is discrete, then the binary contracts are optimal among all types of contracts. However, we shall see later that the claim is no longer true if the action space is continuous as in our model; the binary contracts are sub-optimal. Instead, the linear contracts achieve optimal rate when comparing with the lower bound based on Fisher information.


\subsection{Binary contracts}
Binary contracts refer to those contracts $F_n$'s that map the assessment dataset $\mathcal{D}_n$ onto only two values. It is widely adopted due to its simplicity. We describe the structure of binary contracts in Algorithm \ref{alg_binary_contr}. A binary contract $F_n$ is specified by three components, a test $\Psi$, base salary $w$, and bonus $w_b$. The test $\Psi$ takes the data $\mathcal{D}_n$ as the input and gives a binary output indicating whether the annotator passes the quality assessment. If so, the company pays the annotator the sum of the base salary and the bonus $w+w_b$; if not, the company only pays the base salary. 

\begin{algorithm}[ht!]
\centering
\caption{Binary contract}
\label{alg_binary_contr}
\begin{algorithmic}
\State \textbf{Input:} A dataset $\mathcal{D}_n$ used to assess the annotator performance and a binary contract $F_n = (\Psi, w, w_b)$
\State \textcolor{blue}{\%\% $\Psi(\mathcal{D}_n)\in \{0,1\}$ represents some test as the ones in Algorithm \ref{alg_expert} and Algorithm \ref{alg_self}}
\State \textcolor{blue}{\%\% $w$ is the base salary and $w_b$ is the bonus for passing the performance assessment}
\State Perform the test $\Psi(\mathcal{D}_n)$
\If{$\Psi(\mathcal{D}_n)=0$, i.e., the annotator fails the performance test}
\State The company pays the annotator $w_a = w$
\Else
\State  The company pays the annotator $w_a = w+w_b$
\EndIf
\end{algorithmic}
\end{algorithm}

In the finite-action case where the agent can only choose from a finite set of possible actions, \citet{frick2023monitoring} shows that the gap between $\mathcal{C}_n$ and $\mathcal{C}$ is of an optimal order $\exp(-\Theta(n))$. However, their analysis's large deviation theory tools no longer apply to the continuous case, and the authors leave the continuous case as an open question therein. For the problem of human annotation, the annotation quality $\eta$ is naturally a continuous variable that the annotator chooses as a decision, and there are infinite possibilities for the choice of $\eta$.

For our analysis, we consider a simple setting in quality monitoring/assessment that is also general enough to cover many application contexts including ours. Each inspected sample $D_i\in \mathcal{D}_n$ ($i=1,...,n$) passes the assessment with a certain probability, and the probability is a function of the effort: the more effort the agent spends on the sample, the higher the chance that the sample passes the assessment. For the preference data annotation, if the LLM company adopts our proposed self-consistency monitoring, then each sample $A_i$ is the agreement variable defined in Algorithm \ref{alg_self}. Then when the two labels agree, $A_i=1$ and the sample passes the assessment; the pass probability $P(A_i=1)=\frac{1+c\eta}{2}$ for $c=1-\delta$ (see \eqref{eqn:agreement_self_consist}). We show that in contrast to the discrete case, binary contracts of the (restricted) second-best problem no longer bear an exponential convergence rate to the first-best. Instead, the convergence rate is far worse in the continuous case: if we inspect $n$ samples, then the gap between the first-best and the second-best is $\Theta(1/\sqrt{n \log n})$.

\begin{assumption}
\label{assm:basic_assm_4}
The unique first-best solution $\eta^*$ is bounded away from 0 and 1:
\[
\eta^* \in (0, 1).
\]
\end{assumption}
The assumption is to avoid trivial solutions such as $\eta^* = 0$ or those solutions that have no randomness in the outcome (say, $\eta^* = 1$ and $c = 1$ so the agreement/consistency happens with probability $\frac{1+c\eta^*}{2} = 1$).

\begin{assumption}
\label{assm:basic_assm_5}
    Assume that the assessment dataset $\mathcal{D}_n$ consists of $n$ i.i.d. samples $D_i$'s. Further, denote the joint probability of the outcome $\mathcal{D}_n$ under the agent's action $\eta$ as $p_\eta^n(\mathcal{D}_n)$, and assume that $\frac{\partial }{\partial \eta} p_\eta^n(\mathcal{D}_n)$ exists and is continuous with respect to $\eta$.
\end{assumption}

The assumption is non-restrictive and applies to general distributions such as exponential family distributions, including algorithms in Section~\ref{sec:assess} which induce Bernoulli distributions as a special case.

For now, we focus on the Bernoulli distributions introduced by Algorithms~\ref{alg_expert} and~\ref{alg_self}. Let $m$ denote a monitoring technology (for example, either Algorithm~\ref{alg_expert} or~\ref{alg_self}) and suppose its assessment variable $D_i = A_i^{(m)}\in\{0,1\}$ satisfies
\begin{equation}
\mathbb{P}\!\left(A_i^{(m)}=1\, \middle|\, \eta\right)=\frac{1+c_m \eta}{2},
\qquad c_m\in(0,1].
\label{eq:generic_monitoring_strength}
\end{equation}
For the expert-based agreement test in Section \ref{sec:assess}, $c_{\mathrm{exp}}=\bar c$. For the self-consistency monitoring in Section \ref{sec:assess}, let $c_{\mathrm{self}}=(1- \delta)$ be the coefficient multiplying $\eta$ in \eqref{eqn:agreement_self_consist}.

Also write
\[
p_{m,\eta} \coloneqq \frac{1+c_m\eta}{2}.
\]
We denote by
$\mathcal{C}_{n,m}^{\text{bin}}$ and $\tilde{\mathcal{C}}_{n,m}^{\text{bin}}$
the corresponding second-best values under monitoring technology $m$.

\begin{theorem}
\label{thm:binary}
By restricting the contract class $\mathcal{F}_n$ to binary contracts $\mathcal{F}_n^{\text{bin}}$ as described in Algorithm~\ref{alg_binary_contr} and inspecting the monitoring method $m$, we denote the induced second-best values by $\mathcal{C}_{n,m}^{\text{bin}}$ and $\tilde{\mathcal{C}}_{n,m}^{\text{bin}}$, accordingly. Under Assumptions \ref{assm:basic_assm_1}, \ref{assm:basic_assm_2}, \ref{assm:basic_assm_3}, \ref{assm:basic_assm_4}, and \ref{assm:basic_assm_5}, we have
\[
\mathcal{C} - \tilde{\mathcal{C}}_{n,m}^{\text{bin}} = \Theta\left(1\middle/\sqrt{\mathcal{I}(m, \eta^*)\,  n \log n}\right);
\quad
\mathcal{C} - \mathcal{C}_{n,m}^{\text{bin}} = \Theta\left(1\middle/\sqrt{\mathcal{I}(m, \eta^*) \, n \log n}\right)
\]
where $n$ is the number of samples in the assessment dataset $\mathcal{D}_n$. Here, $\mathcal{I}(m, \eta)$ is the Fisher information of each $A_i$ at $\eta_a = \eta$ using monitoring method $m$:
\[
\mathcal{I}(m,\eta) \coloneqq \mathbb{E}\Big[\Big(\frac{\partial}{\partial \eta} \log p_{m,\eta}(A_i)\Big)^2\Big].
\]
For Bernoulli distributions in Algorithms~\ref{alg_expert} and~\ref{alg_self} with $\eta^*$ being the first-best effort in Proposition \ref{prop:first_best_sol},
\[
\mathcal{I}(m,\eta^*) = \frac{c_m^2}{1-c_m^2 (\eta^*)^2}.
\]
\end{theorem}

Theorem \ref{thm:binary} shows the gap between the first-best and the second-best solutions for binary contracts. To our knowledge, it is the first such result for the principal-agent model where the agent's decision is continuous. The key to our analysis is (i) to use the first-order condition of the agent's utility optimization problem and (ii) a careful analysis of the tail probability of binomial distribution. It studies how effective it is for the simple and easy-to-implement binary contract, which is also known as the bonus scheme in practice. We will see later that the rate of the binary contract is sub-optimal when comparing to the lower bound (Theorem~\ref{thm:lower_bound}).

Another side note is that the dependence on the Fisher information $\mathcal{I}(m ,\eta)$ here is unavoidable; in fact, that amount plays a central role in this section. A core intuition brought about by this theorem is that the more informative the monitoring method is, the larger the constant $c_m$ is, and henceforth the smaller the gap. This result partly justifies the choice of a more informative monitoring method in Section~\ref{sec:assess} from another perspective of economics. We will go back to the discussion of such an intuition later when we see $\mathcal{I}(m ,\eta)$ in other results.

In the following, we outline the proof sketch and we note that the upper bound proof is constructive, which characterizes that the optimal contract is a \emph{threshold} contract.

\subsubsection*{Proof Sketch of Theorem \ref{thm:binary}}
The proof of the theorem does not fully follow the standard method in the literature, so we provide a proof sketch here and defer the full details to Appendix \ref{subsec:proof_of_binary}.

The proof consists of three parts: the upper bound, the lower bound, and the constant of Fisher information. The dependence on $m$ is abbreviated in the first two parts for notation simplicity.

\textbf{Upper bound.} We prove the upper bound by considering a subclass of binary contracts: \textit{threshold contracts}. Suppose the principal pays the agent a bonus $w_b$ if $\sum_{i=1}^n A_i \geq k$ for some $k = \Theta(n)$ as the decision threshold. To interpret the contract under self-consistency monitoring, it means if for $k$ out of the $n$ duplicated samples, the annotator gives consistent labels, then the annotator can win a bonus $w_b.$ The analysis of the threshold contract consists of four steps as follows.

\textbf{Step 1. Reduce the problem to tail probability estimation.} Show the gap between the first-best and the second-best is approximately the payoff's variance due to the agent's risk-averse nature. This step is standard in the literature \citep{laffont2009theory, frick2023monitoring}.

\textbf{Step 2. Examine the curve of the marginal incentive.} Show that the marginal incentive $w_b \cdot \frac{\partial \mathbb{P}(\sum_{i=1}^n A_i \geq k)}{\partial \eta}$ for the agent to increase the quality $\eta$ is approximately a bell-shaped curve centered near the decision threshold $k/n$. The term ``approximately a bell-shaped curve'' means that the marginal incentive behaves similarly to a normal random variable with mean around $k/n$ and variance $\Theta(1/n)$.

\textbf{Step 3. Estimate the agent's decision via the first-order condition.} From the first-order condition of the agent's utility maximization, we know the agent's choice of $\eta_a$ happens at the intersection between the bell-shaped marginal incentive curve and the marginal effort curve $\frac{\mathrm{d} E(\eta)}{\mathrm{d} \eta}$. The two curves intersect twice at some $\eta_1$ and $\eta_2$ such that $p(\eta_1) < \frac{k-1}{n-1} < p(\eta_2)$, and the agent's choice to maximize their utility must be the one on the right. Then we prove that the intersection point $\eta_2$ happens at a distance of $\Theta(\sqrt{\log n/n})$ from the contract's threshold $k/n$ due to the bell-curve approximation.

\textbf{Step 4. Estimate the tail probability that the agent does not get the bonus.} By a careful examination of the properties of binomial distributions and the first-order condition, the probability that the agent exactly reaches the decision threshold $\mathbb{P}(\sum_{i=1}^n A_i = k | \eta_2)$ is of $\Theta(1/n)$. Then we estimate the ratio $\frac{\mathbb{P}(\sum_{i=1}^n A_i = j | \eta_2)}{\mathbb{P}(\sum_{i=1}^n A_i = j-1 | \eta_2)}$. Finally, we use the sum of geometric sequences to derive bounds for the tail probability (and henceforth the variance), which is of $\Theta(1/\sqrt{n \log n})$.

\textbf{Lower bound.} As noted in Step 1 of the upper bound proof, the gap is approximately the payoff's variance. The lower bound proof reduces to showing that the contract that minimizes the payoff's variance must be a threshold scheme.

\textbf{Step 1. Extend the combinatorial rejection region of the monitoring test to continuous regions.} The original test that decides whether the agent receives the bonus can be expressed by the rule that if $n\cdot \bar{A} \notin \mathcal{K}$, then the agent receives the bonus otherwise not. The set $\mathcal{K} \subset \{0,1,...,n\}$ is the rejection region. We extend this monitoring test by a rule that if $n\cdot \bar{A} = k$, then with probability $w_k \in [0,1]$ the agent gets the bonus for $k=1,...,n$. Then the rejection rule is parameterized by $w_k$'s and we can optimize $w_k$'s as a handle to minimize the variance.

\textbf{Step 2. Formulate the optimization problem of $w_k$'s to minimize the variance.} The objective is to minimize $\text{Bonus}^2 \cdot \mathrm{Var}(\mathbb{P}(\text{Bonus}))$. The first-order condition requires that the bonus times the marginal incentive at $\eta = \eta^*$ must be a constant, implying that the $\text{Bonus}$ term can be substituted by the inversion of $\sum_{k} w_k \cdot \frac{\partial \mathbb{P}(n \cdot \bar{A} = k)}{\partial \eta}\vert_{\eta = \eta^*}$. The term $\mathbb{P}(\text{Bonus})$ is $\sum_{k} w_k \cdot \mathbb{P}(n \cdot \bar{A} = k|\eta^*)$. The problem is now formulated as an optimization problem of different $w_k$'s.

\textbf{Step 3. Inspecting the local minima to show that any optimal binary contract must be thresholding.} By calculating the derivatives, show that all local minima must be $w_k = 0$ for $k < k_0$ and $w_k = 1$ for $k > k_0$. Then the minimizer must be some threshold scheme (up to one single term $k = k_0$). For large enough $n$, any single term $\mathbb{P}(n \cdot \bar{A} = k_0)$ does not change the estimation. From the discussions on the threshold scheme for the upper bound analysis, we know that the variance is of $\Theta(1/\sqrt{n \log n})$.

\textbf{Constant of Fisher information.} This part is straightforward once we replace the constants of $c^2/(1-c^2 \eta^2)$ in the proof by Fisher information.

\subsubsection*{Relation to Discrete Action Space}
The closest result to our Theorem~\ref{thm:binary} is the finite-action result of \citet{frick2023monitoring}, where their result and our continuous-action result concern different benchmarks and uses different math tools. If the action set $[0,1]$ is replaced by a finite grid $\mathcal{S}_r\subset[0,1]$ with mesh size $r$, let $\mathcal{C}^{\mathrm{disc}}(r)$ denote the corresponding first-best value. Since $\mathcal{S}_r\subset[0,1]$, we have $\mathcal{C}^{\mathrm{disc}}(r)\leq \mathcal{C}$.

To be more specific, let's take the discretized action space over $[0, 1]$ as an example, where there are $K$ possible actions in $\mathcal{S}_r\subset[0,1]$: $0\leq \eta_1 < \dots < \eta_K \leq 1$. Assume the minimal $\eta_{k+1} - \eta_k$ is lower bounded by some small $r$ (mesh size). Then, there is a gap between the first-best problem~\eqref{eq:first_best} and the discretized first-best:
\begin{equation}
\begin{split}
\mathcal{C}^{\mathrm{disc}}(r) \coloneqq \max_{F_n \in \mathcal{F}_n, \eta = \eta_k, 1\leq k\leq K} \, & \mathbb{E}\left[U_p(F_n(\mathcal{D}_n), \eta)\right], \\
\text{s.t. } & \mathbb{E}\left[U_a(F_n(\mathcal{D}_n), \eta)\right] \geq U_0.
\end{split}
\label{eq:first_best_discrete}
\end{equation}
We can see that the IR constraint for the continuous first-best is binding; then, the optimal value for each $\eta$ forms a convex function with unique minimum taken at $\eta = \eta^*$ (more specifically, $\mathrm{F}(\eta)$ in the proof of Theorem~\ref{thm:binary}; see Appendix~\ref{subsec:proof_of_binary} for details). The discretized benchmark, however, restricts the possible values of $\eta$ on $\mathcal{S}_r$, implying that the distance between $\eta^*$ and $\mathcal{S}_r$, $\min_{k} |\eta_k - \eta^*|$, is no greater than $r/2$. As a consequence, the gap $\mathcal{C} - \mathcal{C}^{\mathrm{disc}}(r)$ can be upper bounded by $O(r^2)$ by taking a Taylor expansion.

On the other hand, for each fixed finite grid, the discrete-action analysis of \citet{frick2023monitoring} gives exponential convergence in $n$ for binary contracts to the discretized benchmark, which is faster compared to the sub-exponential rates in Theorem~\ref{thm:binary}. The exponential rate has a dependence on the rate distortion function, which is of the same order as the minimum KL divergence between the assessment distributions induced by two different $\eta_k$'s in the previous example. Again, the KL divergence is a convex function, so we can expect the minimal gap to be of order $O(r^2)$. As a consequence, the convergence rate to $\mathcal{C}^{\mathrm{disc}}(r)$ is of order $\exp(-O(r^2) n)$. Hence the finite-action model can converge faster in $n$, but toward a smaller benchmark. Combining the above analysis, we have an overall gap of $O(r^2) + \exp(-O(r^2) n)$ compared to $\mathcal{C}$, where the optimal choice of $r^2$ is $r^2 = \Theta(\log n / n)$, inducing a result of $\Theta(\log n / n)$ in general.

At first glance, the rate of the discretized action space's binary contract is superior to the continuous action space's binary contract. However, in the discretized action space model, we have forced the agent to only take finitely many possible actions (here, annotation qualities), which is quite unnatural in practice. In our application, annotation quality is therefore more naturally modeled as a continuous choice variable, while a finite grid is best viewed as a conceptual benchmark or a numerical approximation.

Another interesting comparison is via the information theory: In the discrete case, the principal only needs to separate between finitely many actions with a clear gap, and tools from large deviation theory can give exponential rates. In the continuous case, the principal has to distinguish $\eta$ from local alternatives, and local distinguishability is governed by Fisher information.

\subsection{Linear contracts}

Another simple type of contracts is the linear contract. The biggest feature of linear contracts is that the payoff to the agent relies on the \emph{average} outcome of the assessment. More formally stating, under linear contracts, each contract $F_n$ is specified by a function $f_n$ that maps from each assessment sample $D_i$ to a positive value and pays the agent the average of $f_n(D_i).$ The procedure is described by Algorithm \ref{alg_lin_contr}. For example, under self-consistent monitoring, each sample $D_i$ corresponds to the agreement variable $A_i$ in Algorithm \ref{alg_self}. Then the payment $w_a$ is equal to the proportion of consistent labels annotated by the annotator. 

\begin{algorithm}[ht!]
\centering
\caption{Linear contract}
\label{alg_lin_contr}
\begin{algorithmic}
\State \textbf{Input:} A dataset $\mathcal{D}_n=\{D_1,...,D_n\}$ used to assess the annotator performance and a linear contract $F_n = (f_n)$
\State \textcolor{blue}{\%\% $f_n(D_i)\in \mathbb{R}$ maps each assessment data sample $D_i$ to a payment amount (say, based on its quality)}
\State  The company pays the annotator $$w_a = \frac1n \sum_{i=1}^n f_n(D_i)$$
\end{algorithmic}
\end{algorithm}

We present the result of the linear contracts, which is of better rate compared to the binary contracts. The analysis is largely the same as that in \citet{frick2023monitoring}, while we make the dependence on the Fisher information explicit.

\begin{theorem}
\label{thm:linear_contract}
By restricting the contract class $\mathcal{F}_n$ to linear contracts $\mathcal{F}_n^{\text{lin}}$ as described in Algorithm~\ref{alg_lin_contr} and inspecting the monitoring method $m$, we denote the induced second-best values by $\mathcal{C}_{n,m}^{\text{lin}}$ and $\tilde{\mathcal{C}}_{n,m}^{\text{lin}}$, accordingly. Under Assumptions \ref{assm:basic_assm_1}, \ref{assm:basic_assm_2}, \ref{assm:basic_assm_3}, \ref{assm:basic_assm_4}, and \ref{assm:basic_assm_5}, if \eqref{eq:restricted_second_best} is feasible, then we have
\[
\mathcal{C} - \tilde{\mathcal{C}}_{n,m}^{\text{lin}} = \Theta\left(1\middle/ (\mathcal{I}(m, \eta^*)\, n)\right); \quad 
\mathcal{C} - \mathcal{C}_{n,m}^{\text{lin}} = \Theta\left(1\middle/ (\mathcal{I}(m, \eta^*)\, n)\right)
\]
where $n$ is the number of samples in the assessment dataset $\mathcal{D}_n$. Here, $\mathcal{I}(m, \eta^*)$ is the same Fisher information defined in Theorem~\ref{thm:binary}.
\end{theorem}

Theorem \ref{thm:linear_contract} gives the gap between the first-best and second-best solutions under linear contracts. Compared to the case of binary contracts, the rate of linear contracts is better, which we shall see later is indeed optimal (where the lower bound is also of $\Omega\left(1\middle/ \mathcal{I}(m, \eta^*)\, n\right)$). Here, we can see that the quantity $\mathcal{I}(m, \eta^*)$ again plays a vital role in the final result. The intuition is the same as Theorem~\ref{thm:binary}: the better the monitoring, the better the economic outcome.

The sketch proof of this theorem is not so complicated as in Theorem~\ref{thm:binary}. As in the previous analysis, the gap is approximately the payoff's variance; then, for linear contracts, the variance decays as the order of $1/n$. The constants of Fisher information can also be derived from detailed calculations. This basically describes the crux of the proof; we defer the full proof to Appendix \ref{apd:proof_of_linear}.

\begin{figure}[ht!]
  \centering
  \includegraphics[width=\textwidth]{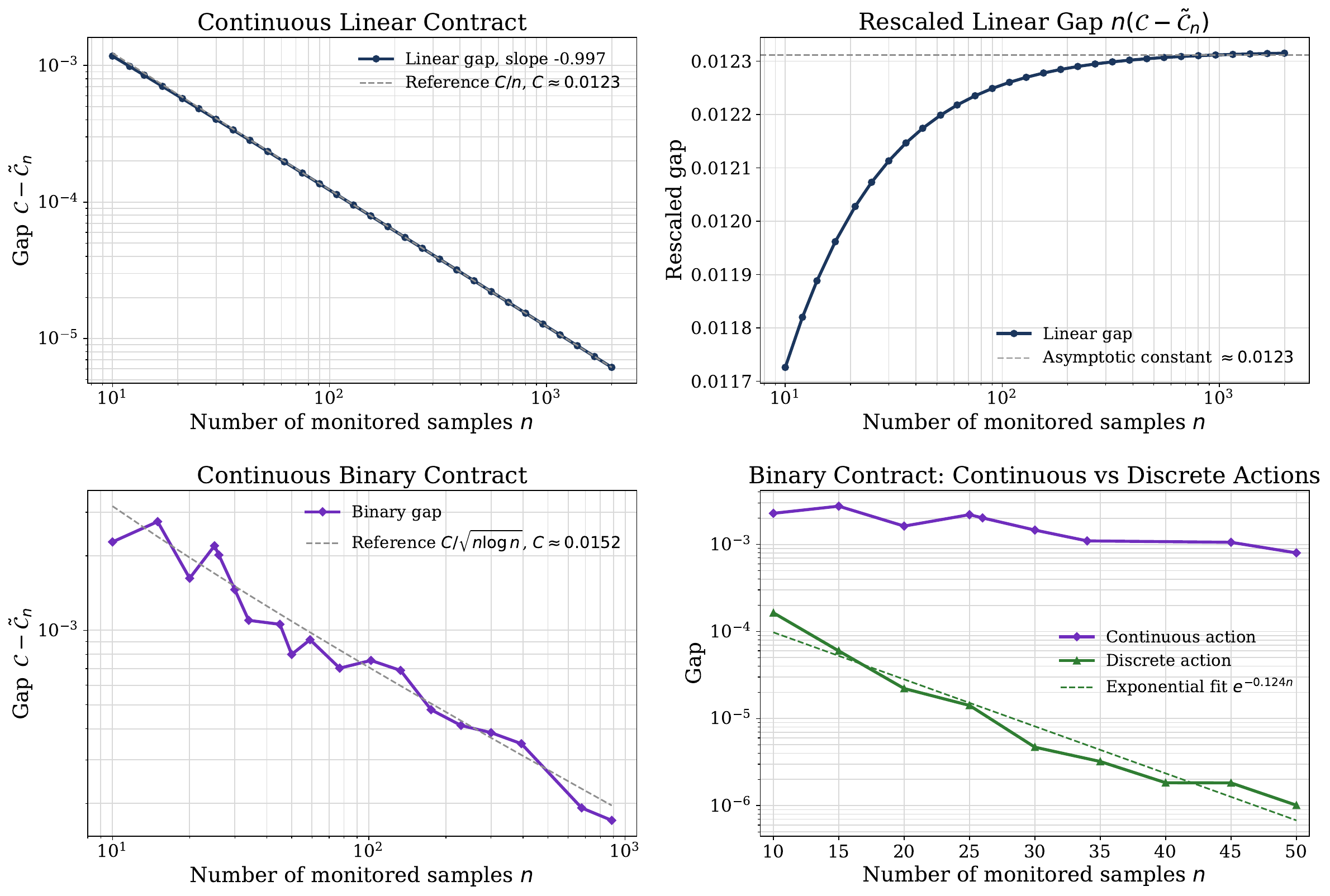}
  \caption{\small Synthetic illustration of the gap scaling. Top-left: for linear contracts under the continuous action space, the restricted gap $\mathcal{C}-\tilde{\mathcal{C}}_n$ is nearly parallel to a $1/n$ reference line on the log-log scale. Top-right: the rescaled gap $n(\mathcal{C}-\tilde{\mathcal{C}}_n)$ stabilizes to a positive constant, which numerically confirms $\Theta(1/n)$. Bottom-left: for continuous binary contracts, the restricted gap follows the slower $1/\sqrt{n\log n}$ benchmark. Bottom-right: under binary contracts, the continuous-action gap is much slower than the discrete-action benchmark, whose semi-log behavior is consistent with the exponential convergence established by \citet{frick2023monitoring}.}
\label{fig:reviewer_gap_scaling}
\end{figure}

Figure~\ref{fig:reviewer_gap_scaling} provides a synthetic illustration of the gap rates in the continuous and discrete settings. In the continuous-action model, the restricted linear-contract gap $\mathcal{C}-\tilde{\mathcal{C}}_n$ is nearly parallel to a $1/n$ reference line, and the rescaled quantity $n(\mathcal{C}-\tilde{\mathcal{C}}_n)$ approaches a positive constant, which is the numerical signature of $\Theta(1/n)$. The same synthetic setup also shows that the restricted continuous binary-contract gap follows the slower $1/\sqrt{n\log n}$ benchmark. The last panel then compares binary contracts under continuous and discrete action spaces: once the action space is discretized, the gap becomes approximately linear on the semi-log scale, which is consistent with the exponential convergence proved by \citet{frick2023monitoring}. Thus, the figure makes the qualitative difference between the continuous and discrete cases visually explicit.
\subsection{Lower bound}
\label{subsec:lower_bound}

In the previous subsections, we analyze the principal-agent model and present the utility gap between the first-best \eqref{eq:first_best} and the (restricted) second-best \eqref{eq:second_best} (\eqref{eq:restricted_second_best}) with respect to the monitoring number $n$ for the binary contracts and the linear contracts. One may wonder what the optimal rate should be for general contracts. In this subsection, we will prove that the lower bound of the gap $\mathcal{C} - \mathcal{C}_n$ is of $\Omega\left(1\middle/ \mathcal{I}(m, \eta^*)\, n\right)$ in mild conditions, which indicates that the linear contracts are of the optimal rate with respect to $n$ and Fisher information among all types of contracts.

\begin{theorem}
\label{thm:lower_bound}
     Under Assumptions \ref{assm:basic_assm_1}, \ref{assm:basic_assm_2}, \ref{assm:basic_assm_3}, \ref{assm:basic_assm_4}, and \ref{assm:basic_assm_5}, we have
    \[
    \mathcal{C} - \mathcal{C}_n = \Omega\left(1\middle/ (\mathcal{I}(m, \eta^*)\, n)\right); \quad \mathcal{C} - \tilde{\mathcal{C}}_n = \Omega\left(1\middle/ (\mathcal{I}(m, \eta^*)\, n)\right).
    \]
\end{theorem}

The full proof is deferred to Appendix \ref{apd:proof_lower_bound}. On a high level, the proof is based on the Cram\'{e}r-Rao lower bound based on the Fisher information. A similar idea is adopted in Proposition 1 of \citet{kirkegaard2017fisher}.

We write some short notes on the lower bound. First, this lower bound explains why continuous and discrete action spaces behave so differently. In the discrete case, the incentive problem is about separating finitely many well-separated actions, so large deviations can give exponential rates. In the continuous case, the binding IC constraint is local: the contract must distinguish $\eta$ from local alternatives, and local distinguishability is governed by Fisher information, which only grows linearly in $n$. That is why the fundamental rate drops to $1/n$. Second, consider the case of Algorithms~\ref{alg_expert} and \ref{alg_self}, where each sample in $\mathcal{D}_n$ is denoted as $A_i$ and is of distribution $p_\eta(A_i)$. As in the bounds of binary and linear contracts (Theorem~\ref{thm:binary} and \ref{thm:linear_contract}), by taking our Bernoulli distribution model where $\mathbb{P}(A_i = 1 | \eta) = \frac{1+c_m\eta}{2}$, the lower bound also depends on the Fisher information
\[
\mathcal{I}(m,\eta) = \frac{c_m^2}{1-c_m^2(\eta^*)^2},
\]
as $\mathcal{C} - \mathcal{C}_n = \Omega(1/(\mathcal{I}(m, \eta^*) n))$. The core intuition remains the same: the larger the Fisher information is, the lower the bound would be. 
So, at a higher level, a larger $c_m$ will not only benefit the upper bound of specific contract types but also the lower bound, which justifies our self-consistency monitoring for a larger $c_m$ again. Last but not least, we want to note that this result does not imply that the linear contracts are constant-optimal; in fact, we only prove that they are rate-optimal with respect to $n$ and $\mathcal{I}(m, \eta^*)$.

\subsection{Implications on Algorithms in Section~\ref{sec:assess}}
\label{subsec:link_between_assess_and_incentivize}

In the previous subsections, we have derived some theories on how good the company (the principal) can expect when different monitoring methods are adopted with different types of contracts. We set up the idealized case (the first-best) as a benchmark for comparison, and the closer to the first-best, the better the performance is. A core economic intuition is running throughout the entire text: a better monitoring method means a smaller randomness in the outcome and a better utility overall. This justifies why people should adopt better monitoring methods. We formalize this intuition in the following corollary.

\begin{corollary}[Which monitoring technology is better?]
\label{cor:monitor_compare}
Suppose two monitoring technologies $m_1$ and $m_2$ satisfy \eqref{eq:generic_monitoring_strength} with
$c_{m_1}>c_{m_2}>0$. Then $m_1$ yields a strictly smaller asymptotic principal-utility gap than $m_2$ for (1) binary contracts; (2) linear contracts; and (3) lower bounds. In particular, self-consistency monitoring is asymptotically better than expert-based monitoring whenever
\[
c_{\mathrm{self}} > c_{\mathrm{exp}}.
\]
\end{corollary}

Corollary \ref{cor:monitor_compare} is especially relevant when the preference task is subjective or ambiguous. In that case, the expert-based parameter $c_{\mathrm{exp}}=\bar c$ can be small because either $|p(x,y_1,y_2)-1/2|$ or $|p_e(x,y_1,y_2)-1/2|$ is small on average, or because the expert and the annotator are not perfectly aligned. By contrast, if the same annotator is stable over time, then $c_{\mathrm{self}}$ is large, so self-consistency monitoring has the better convergence constant. This is exactly the regime where the statistical advantage in Section \ref{sec:assess} carries over to the contract problem. Expert-based monitoring can be preferable only when the expert is highly confident and well aligned with the annotator, making $\bar c$ larger than the corresponding self-consistency parameter. Moreover, the dependence on the monitoring technology is stronger for linear contracts than for binary contracts: the dependence on $\mathcal{I}(m, \eta^*)$ scales as $\mathcal{I}(m, \eta^*)^{-1}$ for linear contracts but only as $\mathcal{I}(m, \eta^*)^{-\frac12}$ for binary contracts. Equivalently, to attain the same target utility gap, a weaker monitoring technology requires on the order of $\mathcal{I}(m, \eta^*)^{-1}$ more monitored samples (up to the extra logarithmic term in the binary-contract case).

\begin{figure}[ht!]
  \centering
  \begin{minipage}[b]{0.23\textwidth}
    \centering
\includegraphics[width=\textwidth]
{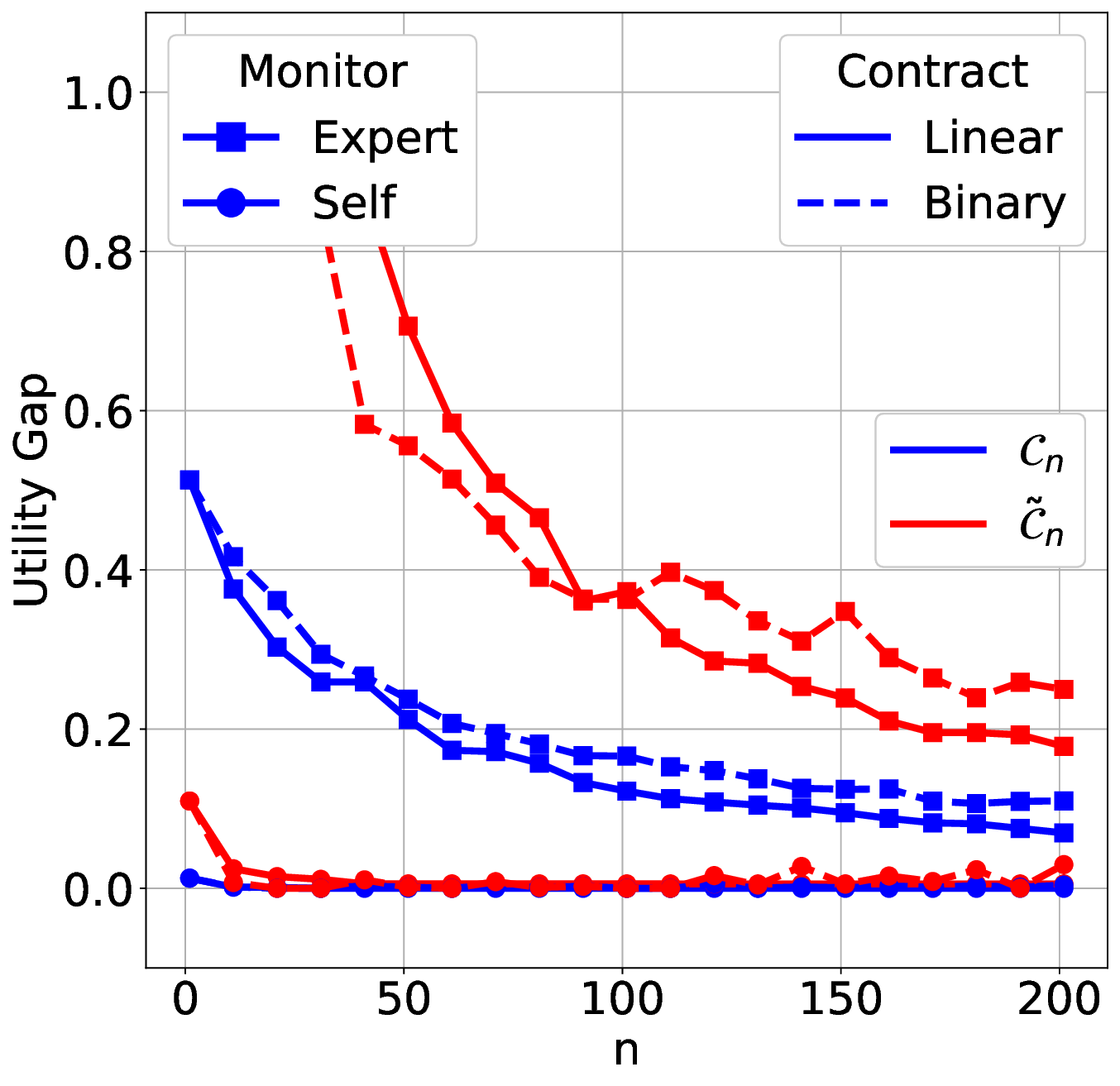}
    \vspace{2mm} {\small (a) PKU}
  \end{minipage}
    \hfill
  \begin{minipage}[b]{0.23\textwidth}
    \centering
\includegraphics[width=\textwidth]
{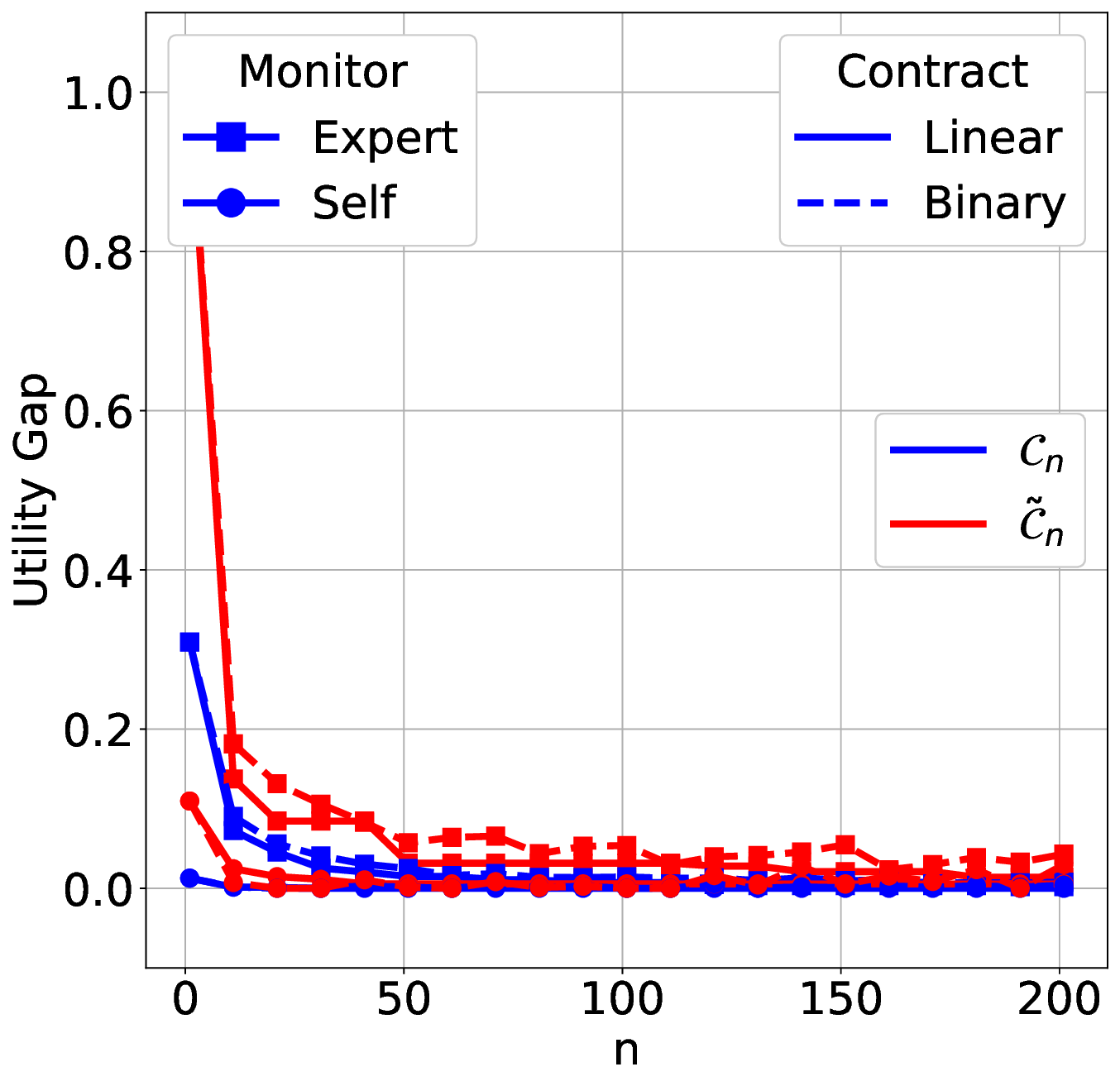}
    \vspace{2mm} {\small (b) Helpsteer}
  \end{minipage}
    \hfill
    \begin{minipage}[b]{0.23\textwidth}
    \centering
\includegraphics[width=\textwidth]
{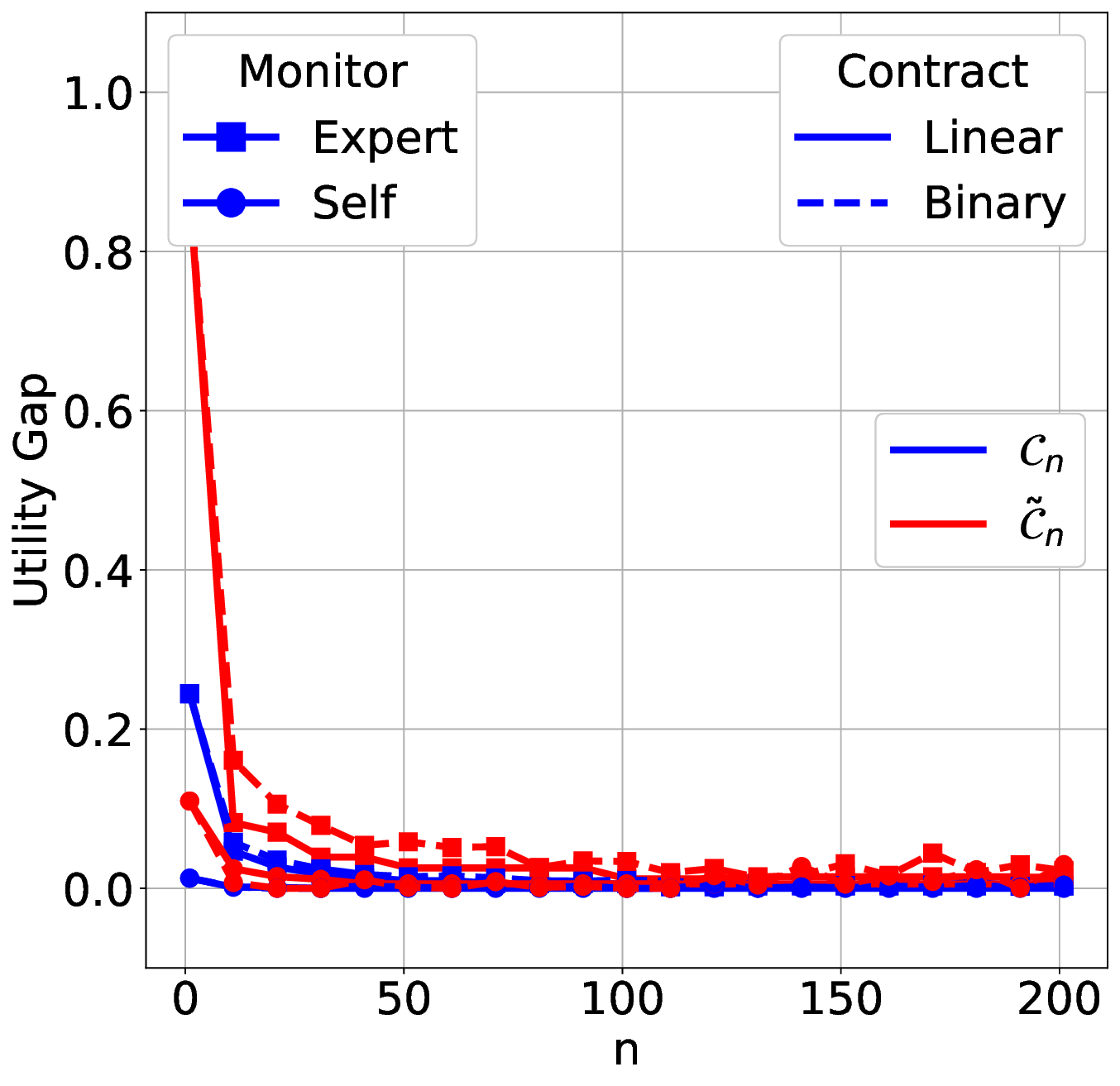}
    \vspace{2mm} {\small (c) Ultra}
  \end{minipage}
  \hfill  
  \begin{minipage}[b]{0.23\textwidth}
    \centering
\includegraphics[width=\textwidth]
{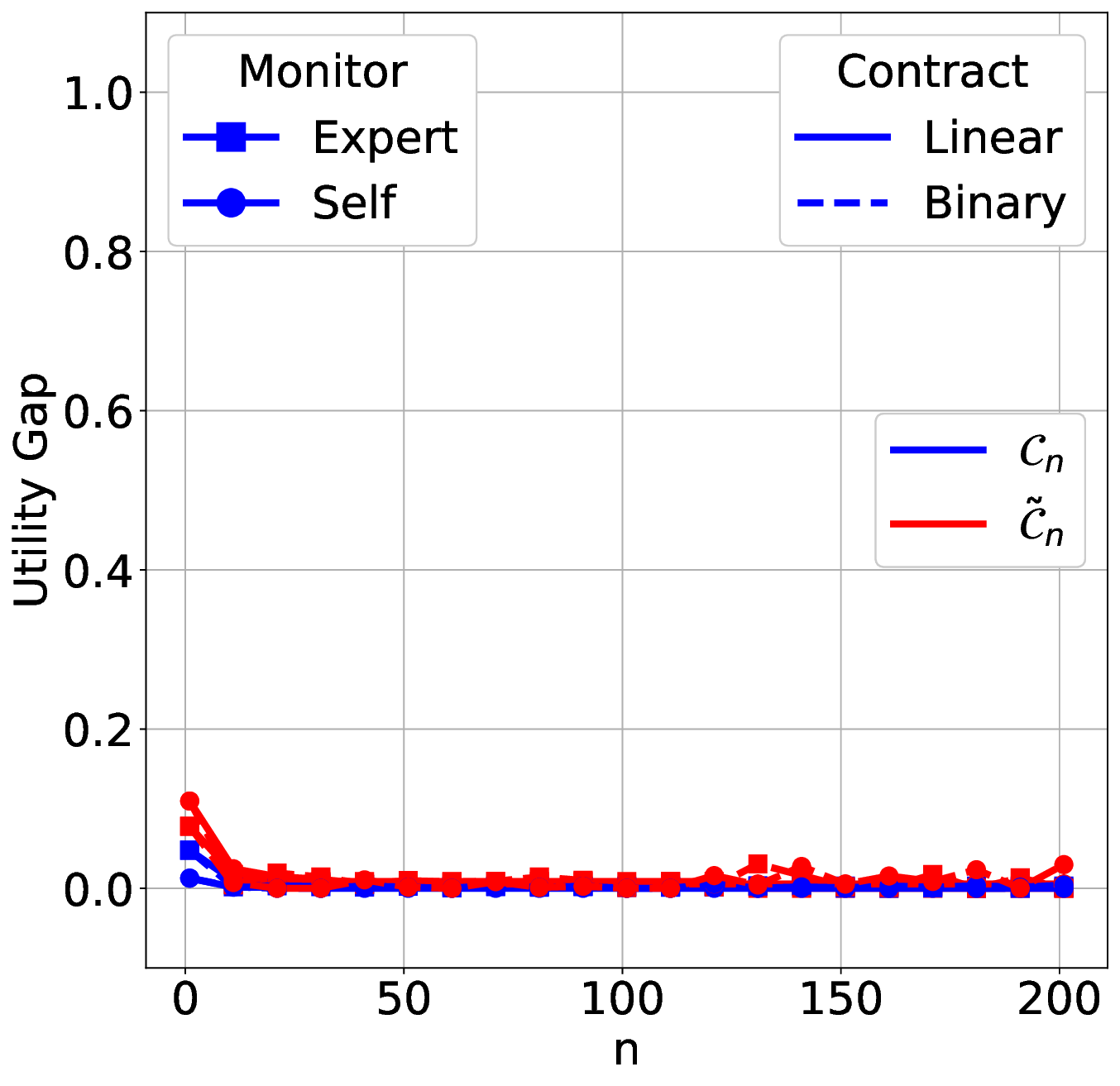}
    \vspace{2mm} {\small (d) Skywork}
  \end{minipage}
  \caption{Normalized principal utility gap ($\mathcal{C}-\mathcal{C}_n$ and $\mathcal{C}-\tilde{\mathcal{C}}_n$) under different monitoring and contract settings. In these experiments, we set $U_0=0$, $\delta=0.02$, $\mu(\eta)=1/2\eta^{4/5}$, $G_a(w_a)=1-\exp(-w_a)$, and $E(\eta)=0.18\eta^2$ (see Appendix~\ref{appx:fig_contract_rank} and Appendix~\ref{appx:delta_utility_sensitivity} for the setup,  sensitivity to $\delta$ and additional utility profiles). In these calibrations, self-consistency monitoring generally yields better utility gaps than expert-based monitoring under the same contract class, with the largest advantage on ambiguous datasets such as PKU and smaller differences on clearer datasets such as Helpsteer, Ultra, and Skywork. The figure also validates Theorems~\ref{thm:binary} and \ref{thm:linear_contract}: for a fixed monitoring technology, the linear contract closes the gap faster in $n$ than the binary contract.}
\label{fig:main_contract_rank}
\end{figure}

Figure \ref{fig:main_contract_rank} closes the loop between Sections \ref{sec:assess} and \ref{sec:incentivizing}. The statistical advantage of self-consistency monitoring in Section \ref{sec:assess} becomes an economic advantage here: once the same simple contracts are written on the monitored signal, self-consistency monitoring generally delivers a smaller principal-utility gap than expert-based monitoring. It as well echoes Corollary~\ref{cor:monitor_compare}: the gain from self-consistency is largest on datasets such as PKU, where preferences concentrate near $1/2$, and smaller on clearer datasets such as Skywork. The figure also shows the second layer of the paper's design question. For a fixed monitoring technology, both contract classes improve as the number of monitored samples $n$ grows, but linear contracts extract value from the monitored signal faster than binary contracts. Thus, the paper's message is not only that better contracts matter, but also that better monitoring technologies and better contract forms are complementary levers for improving the annotation system. Appendix~\ref{appx:delta_utility_sensitivity} further shows that the qualitative pattern is robust to broader values of $\delta$ and to alternative utility specifications.

\subsection{Extensions}
\subsubsection*{Discussion on Risk-Neutral Benchmark ($G_a(w)=w$)}
From previous discussions, we observe that if the annotator is risk-averse, then any randomness in the assessment outcome will affect the final utility that the principal can hope for. In this subsection, we extend the discussions a bit to the risk-neutral cases. Suppose the annotator is risk-neutral, i.e., $G_a(w)=w$.
Then only the \emph{expected} payment matters to the annotator, so the randomness of the
monitoring signal is irrelevant.
In the monitoring setting above, let $X_i\in\{0,1\}$ denote the pass indicator for inspected
sample $D_i$ (e.g., $X_i=A_i$ under self-consistency), so that $\mathbb{P}(X_i=1\mid \eta)=p(\eta)$,
and let $\bar{X}\coloneqq \frac{1}{n}\sum_{i=1}^n X_i$.
Consider the linear contract
\begin{equation}
F_n(\mathcal{D}_n)\;=\;\alpha+\beta\,\bar{X}.
\label{eq:risk_neutral_affine}
\end{equation}
Under risk-neutrality, the agent's expected utility is
$\mathbb{E}[F_n(\mathcal{D}_n)\mid \eta]-E(\eta)=\alpha+\beta\,p(\eta)-E(\eta)$,
so the first-order condition is $\beta\,p'(\eta)=E'(\eta)$.
Hence choosing
\[
\beta=\frac{E'(\eta^*)}{p'(\eta^*)}
\qquad
\Bigl(\text{in particular, for }p(\eta)=\tfrac{1+c\eta}{2}\text{ we have }\beta=\tfrac{2E'(\eta^*)}{c}\Bigr)
\]
makes $\eta^*$ incentive compatible, and setting
$\alpha=U_0+E(\eta^*)-\beta\,p(\eta^*)$ binds the IR constraint.
The principal's expected payoff is then
\[
\mathbb{E}[U_p(F_n(\mathcal{D}_n),\eta^*)]
=\mu(\eta^*)-\bigl(\alpha+\beta\,p(\eta^*)\bigr)
=\mu(\eta^*)-U_0-E(\eta^*)=\mathcal{C},
\]
so $\mathcal{C}_n=\tilde{\mathcal{C}}_n=\mathcal{C}$ for all $n$.
This benchmark echoes classic results: when the agent is risk-neutral and transfers are
unconstrained, first-best can be achieved by a high-powered linear/residual-claim contract and
monitoring has no additional value \citep{harris1979optimal}.
In contrast, with strictly concave $G_a$, the variance of $\bar{X}$ affects the agent through
Jensen's inequality, generating the finite-sample risk premium that drives the convergence-rate
gaps studied in Theorems~\ref{thm:binary}, \ref{thm:linear_contract}, and \ref{thm:lower_bound}.

\subsubsection*{Generalizing beyond Bernoulli Distribution}

In the previous discussions, we have assumed that the assessment dataset $\mathcal{D}_n$ consists of $n$ Bernoulli random variables following the model used in Section~\ref{sec:assess}, especially for the expert monitoring Algorithm~\ref{alg_expert} and our self-consistency monitoring Algorithm~\ref{alg_self}. One may wonder if our main results in Section~\ref{sec:incentivizing} can be generalized to other models where the assessment dataset may consist of data from other distributions than Bernoulli distributions. The answer is yes, as long as the variance for the assessment exists and the assumptions are satisfied. For Theorem~\ref{thm:binary}, a follow-up work \citet{liu2025incentivizing} generalizes our result to normal distributions and other distributions that are asymptotically normal with central limit theorem under mild conditions. For Theorem~\ref{thm:linear_contract}, the result holds for general distributions with finite variance. For Theorem~\ref{thm:lower_bound}, the Cram\'{e}r-Rao arguments also hold for general distributions as long as the assumptions are satisfied (for example, the exponential family distributions).

\section{Conclusion}

In this paper, we study the problem of assessing and incentivizing human annotators in the context of human preference annotation, which is an emerging problem in the era of large language models and generative AI. A central message of the paper is that monitoring design and contract design are two complementary levers for improving LLM annotation systems. Section \ref{sec:assess} is the practitioner-facing, micro/operational part of the paper: expert-based monitoring is the natural workhorse for quality control, we show why it can be sample-inefficient in preference annotation, and we propose self-consistency monitoring as a more suitable agreement-based alternative. Section \ref{sec:incentivizing} is the macro/economic part: taking the monitoring technology as given, we show how the same monitored samples support incentives and determine the principal's utility gap. In this sense, the two sections study a common sample-complexity question from two angles---statistical and economic. Better monitoring reduces the number of samples needed to reliably assess an annotator, and it also reduces the practical sample burden for simple contracts to approach first-best performance.

In the face of two new challenges (1) the lack of a ``true label'' due to human annotators' heterogeneity and (2) the hardness in evaluating downstream performance and tracing back the data quality, we propose a self-consistency monitoring approach to assess human annotators, which requires the human annotators to re-label a random subset of the samples. This approach addresses both concerns and enjoys a smaller error as well compared to classic expert-based monitoring (in terms of minimax lower bounds). To solve the problem of designing better contracts to incentivize human annotators, we give a new analysis under the framework of the principal-agent model. The analysis of the continuous action space shows that the binary contract's gap between the first-best and the second-best is $\Theta(1/\sqrt{n \log n})$, inferior to the linear contract's $\Theta(1/n)$. This result partly answers an open question proposed by \citet{frick2023monitoring}, where the discrete case is analyzed such that the binary contract is of $\exp(-\Theta(n))$ gap. Our result also underlines the importance of reducing the monitoring costs (to increase test number $n$), which further supports the superiority of self-consistency monitoring. The analysis is of independent interest since traditional tools such as large deviation theory fail in the continuous case. Our results are not restrictive: while our proof in Theorem \ref{thm:binary} assumes the binomial distribution, the framework can also be extended to other distributions such as normal distributions shown in the follow-up work \citet{liu2025incentivizing}.

In addition, our formal analysis focuses on binary pairwise preference labels and Bernoulli agreement variables. This matches the dominant reward-model/RLHF data format in LLM training \citep{dong2024rlhf,amini2024direct} and lets us isolate the interaction between monitoring design and contract design in the cleanest possible way. The theory should therefore be interpreted as a benchmark for pairwise preference annotation, rather than as a universal model for all human-labeling tasks. For ordinal ratings, multi-class judgments, rankings, or other structured outputs, the same self-consistency idea can be implemented by replacing the binary agreement indicator with an appropriate task-specific consistency score. Extending the minimax and contract-theoretic results to those richer monitoring channels requires new analysis and is left for future work.
\bibliographystyle{informs2014}
\bibliography{main}

@book{Polyanskiy_Wu_2025, place={Cambridge}, title={Information Theory: From Coding to Learning}, publisher={Cambridge University Press}, author={Polyanskiy, Yury and Wu, Yihong}, year={2025}}

@inproceedings{bacon2012predicting,
  title={Predicting your own effort},
  author={Bacon, David F and Chen, Yiling and Kash, Ian and Parkes, David C and Rao, Malvika and Sridharan, Manu},
  booktitle={Proceedings of the 11th International Conference on Autonomous Agents and Multiagent Systems-Volume 2 (AAMAS)},
  pages={695--702},
  year={2012}}

@article{alon2022bayesian,
  title={Bayesian analysis of linear contracts},
  author={Alon, Tal and D{\"u}tting, Paul and Li, Yingkai and Talgam-Cohen, Inbal},
  journal={arXiv preprint arXiv:2211.06850},
  year={2022}
}

@inproceedings{dutting2019simple,
  title={Simple versus optimal contracts},
  author={D{\"u}tting, Paul and Roughgarden, Tim and Talgam-Cohen, Inbal},
  booktitle={Proceedings of the 2019 ACM Conference on Economics and Computation},
  pages={369--387},
  year={2019}
}

@article{dutting2021complexity,
  title={The complexity of contracts},
  author={Dutting, Paul and Roughgarden, Tim and Talgam-Cohen, Inbal},
  journal={SIAM Journal on Computing},
  volume={50},
  number={1},
  pages={211--254},
  year={2021},
  publisher={SIAM}
}

@inproceedings{ananthakrishnan2024delegating,
  title={Delegating data collection in decentralized machine learning},
  author={Ananthakrishnan, Nivasini and Bates, Stephen and Jordan, Michael and Haghtalab, Nika},
  booktitle={International Conference on Artificial Intelligence and Statistics},
  pages={478--486},
  year={2024},
  organization={PMLR}
}

@article{saig2024delegated,
  title={Delegated classification},
  author={Saig, Eden and Talgam-Cohen, Inbal and Rosenfeld, Nir},
  journal={Advances in Neural Information Processing Systems},
  volume={36},
  year={2024}
}

@article{ivanov2024principal,
  title={Principal-Agent Reinforcement Learning: Orchestrating AI Agents with Contracts},
  author={Ivanov, Dima and D{\"u}tting, Paul and Talgam-Cohen, Inbal and Wang, Tonghan and Parkes, David C},
  journal={arXiv preprint arXiv:2407.18074},
  year={2024}
}

@inproceedings{chen2022selling,
  title={Selling data to a machine learner: Pricing via costly signaling},
  author={Chen, Junjie and Li, Minming and Xu, Haifeng},
  booktitle={International Conference on Machine Learning},
  pages={3336--3359},
  year={2022},
  organization={PMLR}
}

@inproceedings{goldwasser2021interactive,
  title={Interactive proofs for verifying machine learning},
  author={Goldwasser, Shafi and Rothblum, Guy N and Shafer, Jonathan and Yehudayoff, Amir},
  booktitle={12th Innovations in Theoretical Computer Science Conference (ITCS 2021)},
  year={2021},
  organization={Schloss-Dagstuhl-Leibniz Zentrum f{\"u}r Informatik}
}

@inproceedings{cai2015optimum,
  title={Optimum statistical estimation with strategic data sources},
  author={Cai, Yang and Daskalakis, Constantinos and Papadimitriou, Christos},
  booktitle={Conference on Learning Theory},
  pages={280--296},
  year={2015},
  organization={PMLR}
}

@inproceedings{dasgupta2013crowdsourced,
  title={Crowdsourced judgement elicitation with endogenous proficiency},
  author={Dasgupta, Anirban and Ghosh, Arpita},
  booktitle={Proceedings of the 22nd international conference on World Wide Web},
  pages={319--330},
  year={2013}
}

@article{miller2005eliciting,
  title={Eliciting informative feedback: The peer-prediction method},
  author={Miller, Nolan and Resnick, Paul and Zeckhauser, Richard},
  journal={Management Science},
  volume={51},
  number={9},
  pages={1359--1373},
  year={2005},
  publisher={INFORMS}
}

@inproceedings{duetting2024mechanism,
  title={Mechanism design for large language models},
  author={Duetting, Paul and Mirrokni, Vahab and Paes Leme, Renato and Xu, Haifeng and Zuo, Song},
  booktitle={Proceedings of the ACM on Web Conference 2024},
  pages={144--155},
  year={2024}
}

@article{sun2024mechanism,
  title={Mechanism Design for LLM Fine-tuning with Multiple Reward Models},
  author={Sun, Haoran and Chen, Yurong and Wang, Siwei and Chen, Wei and Deng, Xiaotie},
  journal={arXiv preprint arXiv:2405.16276},
  year={2024}
}

@article{hao2024online,
  title={Online Learning from Strategic Human Feedback in LLM Fine-Tuning},
  author={Hao, Shugang and Duan, Lingjie},
  journal={arXiv preprint arXiv:2412.16834},
  year={2024}
}

@article{harris2023algorithmic,
  title={Algorithmic Persuasion Through Simulation: Information Design in the Age of Generative AI},
  author={Harris, Keegan and Immorlica, Nicole and Lucier, Brendan and Slivkins, Aleksandrs},
  journal={arXiv preprint arXiv:2311.18138},
  year={2023}
}

@article{dutting2024algorithmic,
  title={Algorithmic contract theory: A survey},
  author={D{\"u}tting, Paul and Feldman, Michal and Talgam-Cohen, Inbal and others},
  journal={Foundations and Trends{\textregistered} in Theoretical Computer Science},
  volume={16},
  number={3-4},
  pages={211--412},
  year={2024},
  publisher={Now Publishers, Inc.}
}

@techreport{NBERw13480,
 title = "Personnel Economics",
 author = "Lazear, Edward P and Oyer, Paul",
 institution = "National Bureau of Economic Research",
 type = "Working Paper",
 series = "Working Paper Series",
 number = "13480",
 year = "2007",
 month = "October",
 doi = {10.3386/w13480},
 URL = "http://www.nber.org/papers/w13480",
}

@inproceedings{callison2010creating,
  title={Creating speech and language data with amazon’s mechanical turk},
  author={Callison-Burch, Chris and Dredze, Mark},
  booktitle={Proceedings of the NAACL HLT 2010 workshop on creating speech and language data with Amazon’s Mechanical Turk},
  pages={1--12},
  year={2010}
}

@article{krippendorff1989content,
  title={Content analysis},
  author={Krippendorff, Klaus and others},
  journal={International encyclopedia of communication},
  volume={1},
  number={1},
  pages={403--407},
  year={1989}
}

@book{monarch2021human,
  title={Human-in-the-Loop Machine Learning: Active learning and annotation for human-centered AI},
  author={Monarch, Robert Munro},
  year={2021},
  publisher={Simon and Schuster}
}

@article{ghosal2022cicero,
  title={Cicero: A dataset for contextualized commonsense inference in dialogues},
  author={Ghosal, Deepanway and Shen, Siqi and Majumder, Navonil and Mihalcea, Rada and Poria, Soujanya},
  journal={arXiv preprint arXiv:2203.13926},
  year={2022}
}

@article{bastan2020author,
  title={Author's sentiment prediction},
  author={Bastan, Mohaddeseh and Koupaee, Mahnaz and Son, Youngseo and Sicoli, Richard and Balasubramanian, Niranjan},
  journal={arXiv preprint arXiv:2011.06128},
  year={2020}
}

@article{bareket2021neural,
  title={Neural modeling for named entities and morphology (nemoˆ2)},
  author={Bareket, Dan and Tsarfaty, Reut},
  journal={Transactions of the Association for Computational Linguistics},
  volume={9},
  pages={909--928},
  year={2021},
  publisher={MIT Press One Rogers Street, Cambridge, MA 02142-1209, USA journals-info~…}
}

@article{northcutt2021confident,
  title={Confident learning: Estimating uncertainty in dataset labels},
  author={Northcutt, Curtis and Jiang, Lu and Chuang, Isaac},
  journal={Journal of Artificial Intelligence Research},
  volume={70},
  pages={1373--1411},
  year={2021}
}

@article{qian2021annotation,
  title={Annotation inconsistency and entity bias in MultiWOZ},
  author={Qian, Kun and Beirami, Ahmad and Lin, Zhouhan and De, Ankita and Geramifard, Alborz and Yu, Zhou and Sankar, Chinnadhurai},
  journal={arXiv preprint arXiv:2105.14150},
  year={2021}
}

@article{artstein2008inter,
  title={Inter-coder agreement for computational linguistics},
  author={Artstein, Ron and Poesio, Massimo},
  journal={Computational linguistics},
  volume={34},
  number={4},
  pages={555--596},
  year={2008},
  publisher={MIT Press One Rogers Street, Cambridge, MA 02142-1209, USA journals-info~…}
}

@article{krippendorff2004reliability,
  title={Reliability in content analysis: Some common misconceptions and recommendations},
  author={Krippendorff, Klaus},
  journal={Human communication research},
  volume={30},
  number={3},
  pages={411--433},
  year={2004},
  publisher={Wiley Online Library}
}

@book{pustejovsky2012natural,
  title={Natural Language Annotation for Machine Learning: A guide to corpus-building for applications},
  author={Pustejovsky, James and Stubbs, Amber},
  year={2012},
  publisher={" O'Reilly Media, Inc."}
}

@article{klie2024analyzing,
  title={Analyzing dataset annotation quality management in the wild},
  author={Klie, Jan-Christoph and Castilho, Richard Eckart de and Gurevych, Iryna},
  journal={Computational Linguistics},
  volume={50},
  number={3},
  pages={817--866},
  year={2024},
  publisher={MIT Press 255 Main Street, 9th Floor, Cambridge, Massachusetts 02142, USA~…}
}

@article{klie2024efficient,
  title={On Efficient and Statistical Quality Estimation for Data Annotation},
  author={Klie, Jan-Christoph and Haladjian, Juan and Kirchner, Marc and Nair, Rahul},
  journal={arXiv preprint arXiv:2405.11919},
  year={2024}
}

@article{liang2024robust,
  title={Robust preference optimization with provable noise tolerance for llms},
  author={Liang, Xize and Chen, Chao and Wang, Jie and Wu, Yue and Fu, Zhihang and Shi, Zhihao and Wu, Feng and Ye, Jieping},
  journal={arXiv preprint arXiv:2404.04102},
  year={2024}
}

@article{touvron2023llama,
  title={Llama 2: Open foundation and fine-tuned chat models},
  author={Touvron, Hugo and Martin, Louis and Stone, Kevin and Albert, Peter and Almahairi, Amjad and Babaei, Yasmine and Bashlykov, Nikolay and Batra, Soumya and Bhargava, Prajjwal and Bhosale, Shruti and others},
  journal={arXiv preprint arXiv:2307.09288},
  year={2023}
}

@article{munos2023nash,
  title={Nash learning from human feedback},
  author={Munos, R{\'e}mi and Valko, Michal and Calandriello, Daniele and Azar, Mohammad Gheshlaghi and Rowland, Mark and Guo, Zhaohan Daniel and Tang, Yunhao and Geist, Matthieu and Mesnard, Thomas and Michi, Andrea and others},
  journal={arXiv preprint arXiv:2312.00886},
  year={2023}
}

@article{chowdhury2024provably,
  title={Provably robust dpo: Aligning language models with noisy feedback},
  author={Chowdhury, Sayak Ray and Kini, Anush and Natarajan, Nagarajan},
  journal={arXiv preprint arXiv:2403.00409},
  year={2024}
}

@inproceedings{
dai2024safe,
title={Safe {RLHF}: Safe Reinforcement Learning from Human Feedback},
author={Josef Dai and Xuehai Pan and Ruiyang Sun and Jiaming Ji and Xinbo Xu and Mickel Liu and Yizhou Wang and Yaodong Yang},
booktitle={The Twelfth International Conference on Learning Representations},
year={2024},
url={https://openreview.net/forum?id=TyFrPOKYXw}
}

@article{ji2024pku,
  title={PKU-SafeRLHF: Towards Multi-Level Safety Alignment for LLMs with Human Preference},
  author={Ji, Jiaming and Hong, Donghai and Zhang, Borong and Chen, Boyuan and Dai, Josef and Zheng, Boren and Qiu, Tianyi and Li, Boxun and Yang, Yaodong},
  journal={arXiv preprint arXiv:2406.15513},
  year={2024}
}

@article{wang2023helpsteer,
  title={Helpsteer: Multi-attribute helpfulness dataset for steerlm},
  author={Wang, Zhilin and Dong, Yi and Zeng, Jiaqi and Adams, Virginia and Sreedhar, Makesh Narsimhan and Egert, Daniel and Delalleau, Olivier and Scowcroft, Jane Polak and Kant, Neel and Swope, Aidan and others},
  journal={arXiv preprint arXiv:2311.09528},
  year={2023}
}

@article{gao2024impact,
  title={Impact of preference noise on the alignment performance of generative language models},
  author={Gao, Yang and Alon, Dana and Metzler, Donald},
  journal={arXiv preprint arXiv:2404.09824},
  year={2024}
}

@article{shen2023large,
  title={Large language model alignment: A survey},
  author={Shen, Tianhao and Jin, Renren and Huang, Yufei and Liu, Chuang and Dong, Weilong and Guo, Zishan and Wu, Xinwei and Liu, Yan and Xiong, Deyi},
  journal={arXiv preprint arXiv:2309.15025},
  year={2023}
}

@article{sun2024rethinking,
  title={Rethinking Bradley-Terry Models in Preference-Based Reward Modeling: Foundations, Theory, and Alternatives},
  author={Sun, Hao and Shen, Yunyi and Ton, Jean-Francois},
  journal={arXiv preprint arXiv:2411.04991},
  year={2024}
}

@article{bradley1952rank,
  title={Rank analysis of incomplete block designs: I. The method of paired comparisons},
  author={Bradley, Ralph Allan and Terry, Milton E},
  journal={Biometrika},
  volume={39},
  number={3/4},
  pages={324--345},
  year={1952},
  publisher={JSTOR}
}

@article{zhao2023slic,
  title={Slic-hf: Sequence likelihood calibration with human feedback},
  author={Zhao, Yao and Joshi, Rishabh and Liu, Tianqi and Khalman, Misha and Saleh, Mohammad and Liu, Peter J},
  journal={arXiv preprint arXiv:2305.10425},
  year={2023}
}

@article{bai2022training,
  title={Training a helpful and harmless assistant with reinforcement learning from human feedback},
  author={Bai, Yuntao and Jones, Andy and Ndousse, Kamal and Askell, Amanda and Chen, Anna and DasSarma, Nova and Drain, Dawn and Fort, Stanislav and Ganguli, Deep and Henighan, Tom and others},
  journal={arXiv preprint arXiv:2204.05862},
  year={2022}
}

@article{wang2024secrets,
  title={Secrets of rlhf in large language models part ii: Reward modeling},
  author={Wang, Binghai and Zheng, Rui and Chen, Lu and Liu, Yan and Dou, Shihan and Huang, Caishuang and Shen, Wei and Jin, Senjie and Zhou, Enyu and Shi, Chenyu and others},
  journal={arXiv preprint arXiv:2401.06080},
  year={2024}
}

@article{tseng2024two,
  title={Two tales of persona in llms: A survey of role-playing and personalization},
  author={Tseng, Yu-Min and Huang, Yu-Chao and Hsiao, Teng-Yun and Hsu, Yu-Ching and Foo, Jia-Yin and Huang, Chao-Wei and Chen, Yun-Nung},
  journal={arXiv preprint arXiv:2406.01171},
  year={2024}
}

@article{dong2024rlhf,
  title={Rlhf workflow: From reward modeling to online rlhf},
  author={Dong, Hanze and Xiong, Wei and Pang, Bo and Wang, Haoxiang and Zhao, Han and Zhou, Yingbo and Jiang, Nan and Sahoo, Doyen and Xiong, Caiming and Zhang, Tong},
  journal={arXiv preprint arXiv:2405.07863},
  year={2024}
}

@article{ouyang2022training,
  title={Training language models to follow instructions with human feedback},
  author={Ouyang, Long and Wu, Jeffrey and Jiang, Xu and Almeida, Diogo and Wainwright, Carroll and Mishkin, Pamela and Zhang, Chong and Agarwal, Sandhini and Slama, Katarina and Ray, Alex and others},
  journal={Advances in neural information processing systems},
  volume={35},
  pages={27730--27744},
  year={2022}
}

@article{liu2024skywork,
  title={Skywork-Reward: Bag of Tricks for Reward Modeling in LLMs},
  author={Liu, Chris Yuhao and Zeng, Liang and Liu, Jiacai and Yan, Rui and He, Jujie and Wang, Chaojie and Yan, Shuicheng and Liu, Yang and Zhou, Yahui},
  journal={arXiv preprint arXiv:2410.18451},
  year={2024}
}

@article{cui2023ultrafeedback,
  title={Ultrafeedback: Boosting language models with high-quality feedback},
  author={Cui, Ganqu and Yuan, Lifan and Ding, Ning and Yao, Guanming and Zhu, Wei and Ni, Yuan and Xie, Guotong and Liu, Zhiyuan and Sun, Maosong},
  journal={arXiv preprint arXiv:2310.01377},
  year={2023}
}

@article{liu2024reward,
  title={Reward Learning From Preference With Ties},
  author={Liu, Jinsong and Ge, Dongdong and Zhu, Ruihao},
  journal={arXiv preprint arXiv:2410.05328},
  year={2024}
}

@article{ziegler2019fine,
  title={Fine-tuning language models from human preferences},
  author={Ziegler, Daniel M and Stiennon, Nisan and Wu, Jeffrey and Brown, Tom B and Radford, Alec and Amodei, Dario and Christiano, Paul and Irving, Geoffrey},
  journal={arXiv preprint arXiv:1909.08593},
  year={2019}
}

@article{acemoglu2022too,
  title={Too much data: Prices and inefficiencies in data markets},
  author={Acemoglu, Daron and Makhdoumi, Ali and Malekian, Azarakhsh and Ozdaglar, Asu},
  journal={American Economic Journal: Microeconomics},
  volume={14},
  number={4},
  pages={218--256},
  year={2022},
  publisher={American Economic Association 2014 Broadway, Suite 305, Nashville, TN 37203-2425}
}

@article{bergemann2019markets,
  title={Markets for information: An introduction},
  author={Bergemann, Dirk and Bonatti, Alessandro},
  journal={Annual Review of Economics},
  volume={11},
  number={1},
  pages={85--107},
  year={2019},
  publisher={Annual Reviews}
}

@inproceedings{ho2014adaptive,
  title={Adaptive contract design for crowdsourcing markets: Bandit algorithms for repeated principal-agent problems},
  author={Ho, Chien-Ju and Slivkins, Aleksandrs and Vaughan, Jennifer Wortman},
  booktitle={Proceedings of the fifteenth ACM conference on Economics and computation},
  pages={359--376},
  year={2014}
}

@inproceedings{
saig2024incentivizing,
title={Incentivizing Quality Text Generation via Statistical Contracts},
author={Eden Saig and Ohad Einav and Inbal Talgam-Cohen},
booktitle={The Thirty-eighth Annual Conference on Neural Information Processing Systems},
year={2024},
url={https://openreview.net/forum?id=wZgw4CrxwK}
}

@inproceedings{ananthakrishnanknowledge,
  title={Is Knowledge Power? On the (Im) possibility of Learning from Strategic Interactions},
  author={Ananthakrishnan, Nivasini and Haghtalab, Nika and Podimata, Chara and Yang, Kunhe},
year={2024},
  booktitle={The Thirty-eighth Annual Conference on Neural Information Processing Systems}
}

@inproceedings{agarwal2019marketplace,
  title={A marketplace for data: An algorithmic solution},
  author={Agarwal, Anish and Dahleh, Munther and Sarkar, Tuhin},
  booktitle={Proceedings of the 2019 ACM Conference on Economics and Computation},
  pages={701--726},
  year={2019}
}

@article{rafailov2024direct,
  title={Direct preference optimization: Your language model is secretly a reward model},
  author={Rafailov, Rafael and Sharma, Archit and Mitchell, Eric and Manning, Christopher D and Ermon, Stefano and Finn, Chelsea},
  journal={Advances in Neural Information Processing Systems},
  volume={36},
  year={2024}
}

@article{amini2024direct,
  title={Direct preference optimization with an offset},
  author={Amini, Afra and Vieira, Tim and Cotterell, Ryan},
  journal={arXiv preprint arXiv:2402.10571},
  year={2024}
}

@article{askell2021general,
  title={A general language assistant as a laboratory for alignment},
  author={Askell, Amanda and Bai, Yuntao and Chen, Anna and Drain, Dawn and Ganguli, Deep and Henighan, Tom and Jones, Andy and Joseph, Nicholas and Mann, Ben and DasSarma, Nova and others},
  journal={arXiv preprint arXiv:2112.00861},
  year={2021}
}

@article{kaufmann2023survey,
  title={A survey of reinforcement learning from human feedback},
  author={Kaufmann, Timo and Weng, Paul and Bengs, Viktor and H{\"u}llermeier, Eyke},
  journal={arXiv preprint arXiv:2312.14925},
  year={2023}
}

@article{choshen2019weaknesses,
  title={On the weaknesses of reinforcement learning for neural machine translation},
  author={Choshen, Leshem and Fox, Lior and Aizenbud, Zohar and Abend, Omri},
  journal={arXiv preprint arXiv:1907.01752},
  year={2019}
}

@article{yuan2024rrhf,
  title={RRHF: Rank responses to align language models with human feedback},
  author={Yuan, Hongyi and Yuan, Zheng and Tan, Chuanqi and Wang, Wei and Huang, Songfang and Huang, Fei},
  journal={Advances in Neural Information Processing Systems},
  volume={36},
  year={2024}
}

@article{corbett2005optimal,
  title={Optimal shared-savings contracts in supply chains: Linear contracts and double moral hazard},
  author={Corbett, Charles J and DeCroix, Gregory A and Ha, Albert Y},
  journal={European journal of operational research},
  volume={163},
  number={3},
  pages={653--667},
  year={2005},
  publisher={Elsevier}
}

@article{jain2013optimal,
  title={Optimal contracts for outsourcing of repair and restoration services},
  author={Jain, Nitish and Hasija, Sameer and Popescu, Dana G},
  journal={Operations Research},
  volume={61},
  number={6},
  pages={1295--1311},
  year={2013},
  publisher={INFORMS}
}

@article{silva2023classifier,
  title={Classifier calibration: a survey on how to assess and improve predicted class probabilities},
  author={Silva Filho, Telmo and Song, Hao and Perello-Nieto, Miquel and Santos-Rodriguez, Raul and Kull, Meelis and Flach, Peter},
  journal={Machine Learning},
  volume={112},
  number={9},
  pages={3211--3260},
  year={2023},
  publisher={Springer}
}

@inproceedings{guo2017calibration,
  title={On calibration of modern neural networks},
  author={Guo, Chuan and Pleiss, Geoff and Sun, Yu and Weinberger, Kilian Q},
  booktitle={International conference on machine learning},
  pages={1321--1330},
  year={2017},
  organization={PMLR}
}

@inproceedings{zadrozny2001obtaining,
  title={Obtaining calibrated probability estimates from decision trees and naive bayesian classifiers},
  author={Zadrozny, Bianca and Elkan, Charles},
  booktitle={Icml},
  volume={1},
  pages={609--616},
  year={2001}
}

@inproceedings{10.1145/3670865.3673607,
author = {Collina, Natalie and Gupta, Varun and Roth, Aaron},
title = {Repeated Contracting with Multiple Non-Myopic Agents: Policy Regret and Limited Liability},
year = {2024},
isbn = {9798400707049},
publisher = {Association for Computing Machinery},
address = {New York, NY, USA},
url = {https://doi.org/10.1145/3670865.3673607},
doi = {10.1145/3670865.3673607},
booktitle = {Proceedings of the 25th ACM Conference on Economics and Computation},
pages = {640–668},
numpages = {29},
location = {New Haven, CT, USA},
series = {EC '24}
}

@article{adida2019contracts,
  title={Contracts for healthcare referral services: Coordination via outcome-based penalty contracts},
  author={Adida, Elodie and Bravo, Fernanda},
  journal={Management Science},
  volume={65},
  number={3},
  pages={1322--1341},
  year={2019},
  publisher={INFORMS}
}

@article{corbett1999designing,
  title={Designing supply contracts: Contract type and information asymmetry},
  author={Corbett, Charles J and Tang, Christopher S},
  journal={Quantitative models for supply chain management},
  pages={269--297},
  year={1999},
  publisher={Springer}
}

@article{de2019designing,
  title={Designing contracts and sourcing channels to create shared value},
  author={de Zegher, Joann F and Iancu, Dan A and Lee, Hau L},
  journal={Manufacturing \& Service Operations Management},
  volume={21},
  number={2},
  pages={271--289},
  year={2019},
  publisher={INFORMS}
}

@incollection{grossman1992analysis,
  title={An analysis of the principal-agent problem},
  author={Grossman, Sanford J and Hart, Oliver D},
  booktitle={Foundations of Insurance Economics: Readings in Economics and Finance},
  pages={302--340},
  year={1992},
  publisher={Springer}
}

@incollection{laffont2009theory,
  title={The theory of incentives: the principal-agent model},
  author={Laffont, Jean-Jacques and Martimort, David},
  booktitle={The theory of incentives},
  year={2009},
  publisher={Princeton university press}
}

@article{harris1979optimal,
  title={Optimal incentive contracts with imperfect information},
  author={Harris, Milton and Raviv, Artur},
  journal={Journal of economic theory},
  volume={20},
  number={2},
  pages={231--259},
  year={1979},
  publisher={Academic Press}
}

@article{georgiadis2020optimal,
  title={Optimal monitoring design},
  author={Georgiadis, George and Szentes, Balazs},
  journal={Econometrica},
  volume={88},
  number={5},
  pages={2075--2107},
  year={2020},
  publisher={Wiley Online Library}
}

@article{herweg2010binary,
  title={Binary payment schemes: Moral hazard and loss aversion},
  author={Herweg, Fabian and M{\"u}ller, Daniel and Weinschenk, Philipp},
  journal={American Economic Review},
  volume={100},
  number={5},
  pages={2451--2477},
  year={2010},
  publisher={American Economic Association}
}

@article{lopomo2011knightian,
  title={Knightian uncertainty and moral hazard},
  author={Lopomo, Giuseppe and Rigotti, Luca and Shannon, Chris},
  journal={Journal of Economic Theory},
  volume={146},
  number={3},
  pages={1148--1172},
  year={2011},
  publisher={Elsevier}
}

@article{holmstrom1987aggregation,
  title={Aggregation and linearity in the provision of intertemporal incentives},
  author={Holmstrom, Bengt and Milgrom, Paul},
  journal={Econometrica: Journal of the Econometric Society},
  pages={303--328},
  year={1987},
  publisher={JSTOR}
}

@article{walton2022general,
  title={A general framework for robust contracting models},
  author={Walton, Daniel and Carroll, Gabriel},
  journal={Econometrica},
  volume={90},
  number={5},
  pages={2129--2159},
  year={2022},
  publisher={Wiley Online Library}
}

@article{bretagnolle1978estimation,
  title={Estimation des densit{\'e}s: risque minimax},
  author={Bretagnolle, Jean and Huber, Catherine},
  journal={S{\'e}minaire de probabilit{\'e}s de Strasbourg},
  volume={12},
  pages={342--363},
  year={1978}
}

@article{karlin1956theory,
  title={The theory of decision procedures for distributions with monotone likelihood ratio},
  author={Karlin, Samuel and Rubin, Herman},
  journal={The Annals of Mathematical Statistics},
  pages={272--299},
  year={1956},
  publisher={JSTOR}
}

@article{boyd2004convex,
  title={Convex optimization},
  author={Boyd, Stephen},
  journal={Cambridge UP},
  year={2004}
}

@book{le2012asymptotic,
  title={Asymptotic methods in statistical decision theory},
  author={Le Cam, Lucien},
  year={2012},
  publisher={Springer Science \& Business Media}
}

@article{kim1995efficiency,
  title={Efficiency of an information system in an agency model},
  author={Kim, Son Ku},
  journal={Econometrica: Journal of the Econometric Society},
  pages={89--102},
  year={1995},
  publisher={JSTOR}
}

@article{singh1985monitoring,
  title={Monitoring and hierarchies: The marginal value of information in a principal-agent model},
  author={Singh, Nirvikar},
  journal={Journal of Political Economy},
  volume={93},
  number={3},
  pages={599--609},
  year={1985},
  publisher={The University of Chicago Press}
}

@article{Jewitt,
author = {Jewitt, Ian},
year = {2006},
month = {01},
pages = {},
title = {Information Order in Decision and Agency Problems}
}

@article{moscarini2002law,
  title={The law of large demand for information},
  author={Moscarini, Giuseppe and Smith, Lones},
  journal={Econometrica},
  volume={70},
  number={6},
  pages={2351--2366},
  year={2002},
  publisher={Wiley Online Library}
}

@article{carroll2015robustness,
  title={Robustness and linear contracts},
  author={Carroll, Gabriel},
  journal={American Economic Review},
  volume={105},
  number={2},
  pages={536--563},
  year={2015},
  publisher={American Economic Association 2014 Broadway, Suite 305, Nashville, TN 37203}
}

@article{barron2020optimal,
  title={Optimal contracts with a risk-taking agent},
  author={Barron, Daniel and Georgiadis, George and Swinkels, Jeroen},
  journal={Theoretical Economics},
  volume={15},
  number={2},
  pages={715--761},
  year={2020},
  publisher={Wiley Online Library}
}

@article{holmstrom1979moral,
  title={Moral hazard and observability},
  author={Holmstr{\"o}m, Bengt},
  journal={The Bell journal of economics},
  pages={74--91},
  year={1979},
  publisher={JSTOR}
}

@article{frick2023monitoring,
  title={Monitoring with rich data},
  author={Frick, Mira and Iijima, Ryota and Ishii, Yuhta},
  journal={arXiv preprint arXiv:2312.16789},
  year={2023}
}

@article{liao2019sharpening,
  title={Sharpening Jensen's inequality},
  author={Liao, JG and Berg, Arthur},
  journal={The American Statistician},
  year={2019},
  publisher={Taylor \& Francis}
}

@article{kirkegaard2017fisher,
  title={Fisher Information and Incentives in Contracting},
  author={Kirkegaard, Ren{\'e}},
  year={2017}
}

@article{liu2025incentivizing,
  title={Incentivizing High-Quality Human Annotations with Golden Questions},
  author={Liu, Shang and Cai, Zhongze and Wang, Hanzhao and Ma, Zhongyao and Li, Xiaocheng},
  journal={arXiv preprint arXiv:2505.19134},
  year={2025}
}

\newpage
\appendix

\section{Proofs and Additional Discussions}

\subsection{Proofs and Discussions for Section \ref{sec:assess}}
\subsubsection{More Discussions on Proposition \ref{prop:info_lower_bound}}

For the Kullback--Leibler divergence appearing on the right-hand side of the lower bound, we can decompose it into two terms:
\[
D_{\mathrm{KL}}(\mathcal{P}_{\eta_0} \,\|\, \mathcal{P}_{\eta_1}) = \mathrm{CrossEntropy}(\mathcal{P}_{\eta_0}, \mathcal{P}_{\eta_1}) - \mathrm{Entropy}(\mathcal{P}_{\eta_0}).
\]
Recall that a smaller \(D_{\mathrm{KL}}(\mathcal{P}_{\eta_0} \,\|\, \mathcal{P}_{\eta_1})\) induces a larger/better lower bound on the sum of the two types of errors. This observation has two implications:
\begin{itemize}
    \item A larger \(\mathrm{CrossEntropy}(\mathcal{P}_{\eta_0}, \mathcal{P}_{\eta_1})\) is favorable. Denote $c = 2(p - 1/2)$, so that we have
    \[
    \mathrm{CrossEntropy}(\mathcal{P}_{\eta_0}, \mathcal{P}_{\eta_1}) = -\frac{1+\eta_0 c}{2}\log\left(\frac{1+\eta_1 c}{2}\right)
    -\frac{1-\eta_0 c}{2}\log\left(\frac{1-\eta_1 c}{2}\right).
    \]
    One can verify that this expression is increasing in \(\eta_1\) for \(\eta_1 > \eta_0\). Intuitively, a larger gap between the target level \(\eta_1\) and the penalty threshold \(\eta_0\) makes the hypothesis testing problem less challenging.
    
    \item A smaller \(\mathrm{Entropy}(\mathcal{P}_{\eta_0})\) is favorable. Intuitively, this corresponds to a distribution \(\mathcal{P}_{\eta_0}\) with less uncertainty, meaning that the annotator's low-annotation behavior adheres to a more predictable, ``fixed pattern'', thereby facilitating the detection.
\end{itemize}

\subsubsection{Proof of Proposition \ref{prop:info_lower_bound}}
We first define the total variation distance with two lemmas.
\begin{definition}
For any two distributions $\mathcal{Q}$ and $\mathcal{P}$ over the measurable space $(\Omega, \mathcal{F})$, the total variation distance is defined as
\[
\mathrm{TV} \coloneqq \sup_{A \in \mathcal{F}} \{|\mathcal{Q}(A) - \mathcal{P}(A)|\}.
\]
\end{definition}

\begin{lemma}[Le Cam's Lemma \citep{le2012asymptotic}]
\label{lemma:le_cam}
For any two distributions $\mathcal{Q}$ and $\mathcal{P}$ over the space $(\Omega, \mathcal{F})$, and denote $\Psi$ as a measurable function from $\Omega$ to $\{0,1\}$. Then
\[
\inf_{\Psi} \big\{\mathcal{Q}(\Psi(\omega) = 0) + \mathcal{P}(\Psi(\omega) = 1) \big\} = 1 - \mathrm{TV}(\mathcal{Q}, \mathcal{P}).
\]
Furthermore, such an infimum is met with the following function
\[
\Psi^*(s) := \mathbbm{1} \Big\{\frac{\mathrm{d}\mathcal{Q}}{\mathrm{d}\mathcal{P}}(s) \geq 1\Big\}.
\]
\end{lemma}
\begin{lemma}[Bretagnolle-Huber's Inequality \citep{bretagnolle1978estimation}]
\label{lemma:BH_ineq}
For any two distributions $\mathcal{Q}$ and $\mathcal{P}$, we have
\[
\mathrm{TV}(\mathcal{Q}, \mathcal{P}) \leq \sqrt{1 - \exp(-D_{\mathrm{kl}}(\mathcal{Q}\|\mathcal{P}))} \leq 1 - \frac12 \cdot \exp(-D_{\mathrm{kl}}(\mathcal{Q}\|\mathcal{P})).
\]
\end{lemma}

For any two distributions $\mathcal{Q}$ and $\mathcal{P}$ over the measurable space $(\Omega, \mathcal{F})$, the total variation distance is defined as
\[
\mathrm{TV} \coloneqq \sup_{A \in \mathcal{F}} \{|\mathcal{Q}(A) - \mathcal{P}(A)|\}.
\]
\textbf{Proof of Proposition \ref{prop:info_lower_bound}.}
\begin{proof}{Proof}
 Since each of $n$ samples from $\mathcal{D}_{\text{test}}$  is i.i.d. sampled under the law of \eqref{eqn:Z_annotation_quality} and \eqref{eqn:annotation_eta}, we denote their distributions as $\mathcal{P}_{\eta_0}^n,\mathcal{P}_{\eta_1}^n$ under  $\eta=\eta_0$ and $\eta=\eta_1$, respectively. By the definition of KL-divergence, we have (e.g., see Theorem 2.16 in \cite{Polyanskiy_Wu_2025}):
$$D_{\mathrm{kl}}(\mathcal{P}_{\eta_0}^n\|\mathcal{P}_{\eta_1}^n)=n\cdot D_{\mathrm{kl}}(\mathcal{P}_{\eta_0}\|\mathcal{P}_{\eta_1}).$$

Then by directly using Lemma \ref{lemma:le_cam} and \ref{lemma:BH_ineq}, we have 
\begin{align*}
&\phantom{=}\ \inf_{\Psi}\left\{\mathbb{P}\left(\Psi(\mathcal{D}_{\text{test}}) = 1 \middle | \eta \leq \eta_0\right) + \mathbb{P}\left(\Psi(\mathcal{D}_{\text{test}}) = 0 \middle | \eta \geq \eta_1\right)\right\} \\
&\geq \inf_{\Psi}\left\{ \inf_{\eta\leq \eta_0, \eta^\prime \geq \eta_1} \Big\{\mathbb{P}\left(\Psi(\mathcal{D}_{\text{test}}) = 1 \middle | \eta\right) + \mathbb{P}\left(\Psi(\mathcal{D}_{\text{test}}) = 0 \middle | \eta^\prime\right) \Big\} \right\}\\
&= \inf_{\eta\leq \eta_0, \eta^\prime \geq \eta_1} \inf_{\Psi}\left\{ \mathbb{P}\left(\Psi(\mathcal{D}_{\text{test}}) = 1 \middle | \eta\right) + \mathbb{P}\left(\Psi(\mathcal{D}_{\text{test}}) = 0 \middle | \eta^\prime\right) \right\}\\
&\geq \inf_{\eta\leq \eta_0, \eta^\prime \geq \eta_1} \frac12\cdot \exp(- n D_{\mathrm{kl}}(\mathcal{P}_{\eta} \| \mathcal{P}_{\eta^\prime}))\\
    &= \frac12 \cdot \exp(- n D_{\mathrm{kl}}(\mathcal{P}_{\eta_0} \| \mathcal{P}_{\eta_1})),
\end{align*}
where the first equality uses the exchangeability among taking infimum.
 \end{proof}

\subsubsection{Proof of Proposition \ref{prop:UMP}}
\begin{proof}{Proof}
By definition, given $\delta$, the distribution of each $A_i$ is Bernoulli with success probability
\[
p(\eta) = \frac{\eta \bar{c}+1}{2}.
\]
To see this, we note that for any $\eta$, $\frac{1+c\eta}{2}$ is a linear function of $c$, such that
\begin{align*}
    \mathbb{P}(A_i=1) & = \mathbb{E}\left[\frac{1+c(x, y_1, y_2)\eta}{2}\right] \\
    & = \frac{1 + \mathbb{E}[c(x, y_1, y_2)] \eta}{2}\\
    & = \frac{1 + \bar{c} \eta}{2}.
\end{align*}

Then the (joint) likelihood for the observations $\bm{A} = \{A_i\}_{i=1}^n$ is:
\[
L(\eta; \bm{A}) = \prod_{i=1}^n \left( \frac{1+\eta\bar{c}}{2} \right)^{A_i} \left( \frac{1-\eta\bar{c}}{2} \right)^{1-A_i} = \left( \frac{1+\eta\bar{c}}{2} \right)^{\sum_{i=1}^n A_i} \left( \frac{1-\eta\bar{c}}{2} \right)^{n-\sum_{i=1}^n A_i}.
\]
The likelihood purely depends on $\bar{A} = \frac1n \sum_{i=1}^n A_i$, which implies that $\bar{A}$ is a sufficient statistic. The likelihood ratio for two different parameter values $\eta_0$ and $\eta_1$ is given by
\[
\Lambda(\bm{A}) = \frac{L(\eta_1; \bm{A})}{L(\eta_0; \bm{A})} = \exp\Bigl\{ n\bar{A} \cdot \log\frac{1+\eta_1\bar{c}}{1+\eta_0\bar{c}}  + n(1-\bar{A})  \cdot \log\frac{1-\eta_1\bar{c}}{1-\eta_0\bar{c}}  \Bigr\}.
\]

Note that $\Lambda(\bm{A})$ is an increasing function of $\bar{A}$ for any given $\eta_1 \geq \eta_0$ because
\[
\log\frac{1+\eta_1\bar{c}}{1+\eta_0\bar{c}} \geq 0 \quad \text{and} \quad \log\frac{1-\eta_1\bar{c}}{1-\eta_0\bar{c}} \leq  0.
\]
This implies that the test based on $\bar{A}$ is uniformly most powerful by the Karlin--Rubin theorem \citep{karlin1956theory}.

 \end{proof}

\subsubsection{Proof of Proposition \ref{prop:lecam_lower_bound}}
\begin{proof}{Proof}
The first two inequalities are the direct consequence of the following Le Cam's two-point method and Bretagnolle-Huber inequality (Lemma \ref{lemma:BH_ineq}) by taking $\ell$ to be the $L_1$ distance. The proof of Le Cam's two-point method is standard and can be found in textbooks.
\begin{lemma}[Le Cam's two-point method, e.g. Theorem 31.1 in \citet{Polyanskiy_Wu_2025}]
Suppose the loss function $\ell: \Theta \times \Theta \rightarrow \mathbb{R}_+$ satisfies $\ell(\theta, \theta) = 0$ for all $\theta \in \Theta$ and the following $\alpha$-triangle inequality for some $\alpha > 0$: For all $\theta_0, \theta_1, \theta \in \Theta$,
\[
\ell(\theta_0, \theta_1) \leq \alpha\left(\ell(\theta_0, \theta) + \ell(\theta_1, \theta)\right).
\]
Then
\[
\inf_{\hat{\theta}} \sup_{\theta \in \Theta} \mathbb{E}[\ell(\theta, \hat{\theta})] \geq \sup_{\theta_0, \theta_1 \in \Theta} \frac{\ell(\theta_0, \theta_1)}{2\alpha} (1- \mathrm{TV}(\mathcal{Q}_{\theta_0}, \mathcal{Q}_{\theta_1})).
\]
\end{lemma}

Now we start to approximate the amount $\frac14 \cdot \sup_{\eta_0, \eta_1} |\eta_0 - \eta_1| \exp(-n D_{\mathrm{kl}}(\mathcal{Q}_{\eta_0} \| \mathcal{Q}_{\eta_1}))$. Note that $\mathcal{P}_{\eta} = \mathrm{Bernoulli}\left(\frac{1+\bar{c} \eta}{2}\right)$, we set $\eta_1 = 0$, $\eta_0 = 1/(\bar{c}\sqrt{n})$. Using the Taylor's expansion at $\eta = 0$ for the first argument $\eta_0$, we have
\[
D_{\mathrm{kl}}(\mathcal{Q}_{\eta_0} \| \mathcal{Q}_{\eta_1}) = \frac{1}{2n} + o(1/n).
\]
This gives our desired bound.
 \end{proof}

\subsubsection{Proof of Proposition \ref{prop:KL_compare}}
\begin{proof}{Proof}
The proof is almost the same as that of Proposition~\ref{prop:lecam_lower_bound}, which we omit for simplicity.
 \end{proof}

\subsection{Proofs and Discussions for Section \ref{sec:incentivizing}}
\subsubsection{Proof of Proposition \ref{prop:first_best_sol}}
\begin{proof}{Proof}
\medskip
\noindent\textbf{The Optimal Payment is Constant.} \\
Recall that the agent's utility is given by
\[
U_a(w_a,\eta) = G_a(w_a) - E(\eta),
\]
where \(G_a\) is strictly increasing and strictly concave (reflecting risk aversion), and \(E(\eta)\) is an increasing effort function. Suppose the principal offers a contract in which the wage depends on the realization of the test annotations, that is, a random variable \(F_n(\mathcal{D}_n)\) with
\[
\mathbb{E}\Bigl[U_a(F_n(\mathcal{D}_n), \eta)\Bigr] = \mathbb{E}\bigl[G_a(F_n(\mathcal{D}_n))\bigr] - E(\eta).
\]
By Jensen's inequality and the strict concavity of \(G_a\),
\[
\mathbb{E}\bigl[G_a(F_n(\mathcal{D}_n))\bigr] < G_a\Bigl(\mathbb{E}\bigl[F_n(\mathcal{D}_n)\bigr]\Bigr)
\]
whenever \(F_n(\mathcal{D}_n)\) is not almost surely constant. Therefore, any randomization in the wage reduces the agent's expected utility. As a result, if a random wage contract satisfies the individual rationality (IR) constraint, then the principal can instead offer a constant wage
\[
w_a = \mathbb{E}\bigl[F_n(\mathcal{D}_n)\bigr]
\]
and achieve a strictly higher utility for the agent without increasing the expected cost of the principal (note the cost $G_p(w_p)\equiv w_p$ is risk-neutral). Hence, the optimal contract has a constant payment.

\medskip

\noindent\textbf{The IR Constraint is Binding.}\\
When the wage is constant, we can remove the expectation operator and the agent's utility becomes
\[
U_a(w_a,\eta) = G_a(w_a) - E(\eta),
\]
and the IR constraint is now
\[
G_a(w_a) - E(\eta) \ge U_0.
\]

To show the IR constraint must be binding under the optimal solution, suppose there is an optimal solution $(w'_a,\eta')$ where $G_a(w'_a) - E(\eta') > U_0$, then by the monotonicity of $G_a$, we can always find another feasible solution $(w''_a,\eta')$ such that $w''_a<w'_a$ and $G_a(w''_a) - E(\eta') = U_0$. However, we have the objective function now satisfying 
$$-w''_a + \mu(\eta')> -w'_a + \mu(\eta'),$$
which contradicts that $(w'_a,\eta')$ is the optimal solution.  Thus we can conclude the IR constraint must bind.

\medskip

\noindent\textbf{Optimal Solution.}\\
Therefore, we have at the optimum,
\[
G_a(w_a) - E(\eta) = U_0.
\]
Then for any given $\eta$, by the monotonicity (and invertibility) of \(G_a\), we can solve for the corresponding \(w_a\) which satisfies the binding constraint:
\[
w_a = G_a^{-1}\Bigl(E(\eta) + U_0\Bigr).
\]

Substituting this expression into the principal's utility, we obtain
\[
U_p = -G_a^{-1}\Bigl(E(\eta) + U_0\Bigr) + \mu(\eta).
\]
Thus, the principal's problem reduces to choosing the annotation quality \(\eta \in [0, 1]\) to maximize
\[
\max_{\eta\in [0,1]} \Bigl\{-G_a^{-1}\Bigl(E(\eta) + U_0\Bigr) + \mu(\eta)\Bigr\}.
\]
Denote the maximizer by \(\eta^*\). Then, the corresponding optimal wage is
\[
w_a^* = G_a^{-1}\Bigl(E(\eta^*) +U_0\Bigr).
\]

\medskip
\noindent\textbf{Uniqueness.} \\
To show the uniqueness, by noting the strict monotonicity of $G^{-1}_a$, $w^*_a$ is unique if $\eta^*$ is unique. Thus it is sufficient to prove the uniqueness of $\eta^*$.

We first show $G^{-1}_a$ is convex. To see this, let \(y_1, y_2\) be in the range of \(G_a\) and define \(x_1=G_a^{-1}(y_1)\) and \(x_2=G_a^{-1}(y_2)\). For any \(\lambda\in[0,1]\), by the concavity of \(G_a\),
\[
G_a\bigl(\lambda x_1+(1-\lambda)x_2\bigr) \ge \lambda G_a(x_1) + (1-\lambda) G_a(x_2) = \lambda y_1 + (1-\lambda)y_2.
\]
Since \(G_a\) is strictly increasing, applying \(G_a^{-1}\) (which preserves order) yields
\[
\lambda x_1+(1-\lambda)x_2 \ge G_a^{-1}\Bigl(\lambda y_1 + (1-\lambda)y_2\Bigr),
\]
which shows the convexity for \(G_a^{-1}\).

Thus, by the preservation of convexity in the composition and affine function (Section 3.2 in \cite{boyd2004convex}), since $E(\eta)$ is convex and non-decreasing, and $\mu$ is concave, we have 
$$-G_a^{-1}\Bigl(E(\eta) + U_0\Bigr) + \mu(\eta)$$
is a concave function.

Hence, if \eqref{eq:first_best} is feasible, then the first-best solution is unique to this concave optimization problem:
\[
\eta^* = \arg\max_{\eta \in [0, 1]} \Bigl\{-G_a^{-1}\Bigl(E(\eta) + U_0\Bigr) + \mu(\eta)\Bigr\},
\]
with the agent receiving the wage
\[
w_a^* = G_a^{-1}\Bigl(E(\eta^*) + U_0\Bigr).
\]
 \end{proof}

\subsubsection{Discussions on Feasibility of Second-Best Problems}
For the unrestricted second-best problem \eqref{eq:second_best}, it is straightforward to see that at least one feasible solution exists by paying the agent enough salary. Consider those 
$F_n$'s defined on the space of measurable functions of $\mathcal{D}_n$ and using the $L_1$ norm, for example. Then, due to the continuity of $E$ and $\mu$, the feasible set must be closed. Furthermore, it is without loss of generality to assume that we are considering only a uniformly bounded subset of $F_n$; otherwise, the agent would quit (due to low salary) or the payment is unnecessarily high for the principal to afford. Then, the feasible set is compact. Since the target function is continuous, it must have a maximum on the compact feasible set.

The core argument is that any non-empty compact set must have a maximum for a continuous function. Thus, for the restricted second-best problem \eqref{eq:restricted_second_best}, the conclusion could also be derived in the same way. It is easy to check that all the other requirements (closed feasible set, continuous target function, and bounded feasible set) are met in \eqref{eq:restricted_second_best}. However, the key point is: is the feasible set of \eqref{eq:restricted_second_best} non-empty? We will prove the claim for the consistency/agreement-based tests based on $\bm{A} = (A_1, \dots, A_n)$ and the binary contract case $\mathcal{F}_n = \mathcal{F}_n^{\text{bin}}$. Such a case covers all the discussions in Theorem \ref{thm:binary}.

\begin{proposition}
\label{prop:feasible_binary}
Assume $n$ is large enough. For $\mathcal{F}_n = \mathcal{F}_n^{\text{bin}}$ and $\mathcal{F}_n^{\text{bin}}(\mathcal{D}_n) = \mathcal{F}_n^{\text{bin}}(\bm{A})$, the problem \eqref{eq:restricted_second_best} is feasible.
\end{proposition}
\begin{proof}{Proof}
Recall that each consistency happens with probability $p(\eta) \coloneqq \frac{1+c\eta}{2}$. We consider a particular subclass of contracts based on
\[
\bar{A} \coloneqq \frac1n \sum_{i=1}^n A_i
\]
or equivalently,
\[
\hat{\eta} \coloneqq \frac{2\bar{A} - 1}{c}.
\]
The agent gets a higher wealth if $\bar{A}$ is no less than a threshold $\frac{k}{n}$ with monetary utility $S_1$ with probability $\mathbb{P}(\bar{A} \geq \frac{k}{n})$. 
Denotes the probability by $P(\eta, n, k)$. 
Otherwise, the agent gets a monetary utility of $S_0$.

From Lemma \ref{lemma:binomial_properties} (f) and (g), we see that if we set $k = \frac{1+c \eta^*}{2} \cdot n - \Theta(1/\sqrt{n \log n})$ and the difference between two wages $S_1 - S_0$ moderately large (for example, larger than $E(1)$), then the induced $\eta_a$ must be the right intersection between $(S_1 - S_0) \cdot \frac{\partial }{\partial \eta} P(\eta, n, k)$ and $E(\eta)$; the intersection happens very close to $\eta^*$ by a distance of at most $\Theta(1/\sqrt{n \log n})$. Suppose we adjust $S_1 - S_0$ by multiplying a factor of $(1+\gamma)$. By the first-order condition, we have
\[
(1+\gamma) \cdot (S_1 - S_0) \cdot \frac{\partial }{\partial \eta} P(\eta, n, k)\bigg\vert_{\eta = \eta_a} - \frac{\mathrm{d} E}{\mathrm{d} \eta}\bigg \vert_{\eta = \eta_a} = 0.
\]
By the implicit function theorem, we can write $\eta_a$ as a function of $\gamma$ with
\[
(S_1 - S_0) \cdot \frac{\partial }{\partial \eta} P(\eta, n, k)\bigg\vert_{\eta = \eta_a} + (1+\gamma) \cdot (S_1 - S_0) \cdot \frac{\partial^2 }{\partial \eta^2} P(\eta, n, k)\bigg\vert_{\eta = \eta^*} \cdot \frac{\partial \eta_a}{\partial \gamma} - \frac{\mathrm{d}^2 E}{\mathrm{d} \eta^2}\bigg \vert_{\eta = \eta_a} \cdot \frac{\partial \eta_a}{\partial \gamma} = 0.
\]
By the fact that $P(\eta, n, k)$ is concave and monotonically increasing at $\eta_a$ (Lemma \ref{lemma:binomial_properties} (e) and (b)) and the assumption that $E$ is convex, we have that
\[
\frac{\partial \eta_a}{\partial \gamma} = \Omega\left(\frac{1}{1+\gamma}\right) > 0,
\]
where $\Omega(\cdot)$ hides the dependence on $n, P,$ and $E$. We can always find an $\gamma = \exp\left(-\mathcal{O}\left(\sqrt{n\log n}\right)\right) = o(1)$ such that the $\eta_a$ is adjusted to $\eta_a = \eta^*$.
 \end{proof}

\subsubsection{Proof of Theorem \ref{thm:binary}}
\label{subsec:proof_of_binary}
We first note that, except the discussions on the constants of Fisher information, the dependence on $m$ is abbreviated in our proofs for notation simplicity. We follow the convention in the principal-agent model that the gap between the first-best and the second-best solutions is due to the concavity of the agent's monetary utility. Such a concave function induces a Jensen's gap, which can be approximated by the variance when $n$ is large (see \citet{frick2023monitoring} for an example). Then, all the efforts could be reduced to estimating the payoff variance for large enough $n$'s. We summarize them into the following:\\
\textbf{Step 1. First-best/second-best gap is Jensen's gap; which approximately is the variance when $n$ is large enough.}\\
\textbf{Step 2. Estimating binary contracts' variance. The variance is approximately $\Theta(1/\sqrt{n \log n})$ when $n$ is large enough.}

During our discussions on the variance, the only required property of the yielding outputs $\eta_a(F_n)$'s is that
\begin{equation}
\exists \varepsilon, N > 0, \text{ s.t. } \forall n \geq N, \eta_a(F_n) \in [\varepsilon, 1-\varepsilon].
\label{eq:uniformly_bounded_away}
\end{equation}
In other words, the yielding outputs $\eta^*$'s are uniformly bounded away from 0 and 1 for large enough $n \geq N$. This requirement is to make sure that the agent's output at least induces some randomness in the outcome (a counter-example is that if $\eta_a = 1$ and the agreement probability $p(\eta) = \frac{1+c \eta}{2} = 1$ for some $c=1$, then the outcome is deterministic, and the variance of the outcome is zero). This property is satisfied by the restricted program \eqref{eq:restricted_second_best} by the constraint that $\eta_a = \eta^*$ and $\eta^* \in (0, 1)$. For the unrestricted second-best program \eqref{eq:second_best}, the property is also satisfied in the sense that the optimal solution will converge to $\eta^*$ (later proved in Proposition \ref{prop:second_converge_to_first}).

Now, we start to derive Theorem \ref{thm:binary} for the restricted second-best \eqref{eq:restricted_second_best}'s gap from the first-best. The proof also holds for the unrestricted program \eqref{eq:second_best} and will be omitted for simplicity.
\begin{proof}{Proof}
\textbf{Part I: Upper bound.}
According to our discussions in Proposition \ref{prop:first_best_sol}, there is only one constraint (IR) in the first-best problem \eqref{eq:first_best}, and the constraint is binding. The first-best solution is to pay the agent a fixed amount of $G_a^{-1}(E(\eta^*) + U_0)$.

For the restricted second-best problem \eqref{eq:restricted_second_best}, the paid amount must obey the IR constraint as well, induces
\[
\mathbb{E}[G_a(F_n) - E(\eta^*)] \geq U_0.
\]
For the binary contract, suppose the agent's monetary utility is one of the following two levels: a higher level $S_1$ and a lower level $S_0$. We can always adjust two levels of payment $G_a^{-1}(S_0)$ and $G_a^{-1}(S_1)$ in the following way: we keep the agent's monetary utility gap $\Delta S = S_1 - S_0$ between those two levels and only adjust the lower level $S_0$ of the agent's monetary utility. Such adjustments do not change the utility-maximizing choice of the agent, which separates the IC constraint from the IR. Recall that we want to prove that the IR constraint is binding. By adjusting $S_0$, the expected monetary utility of the agent changes according to
\[
\frac{\partial \mathbb{E}[G_a(F_n)]}{\partial S_0} = -\mathbb{P}(\text{agent gets }G_a^{-1}(S_0)) - \mathbb{P}(\text{agent gets }G_a^{-1}(S_0 + \Delta S)) = -1.
\]
The principal's expected paid amount changes accordingly:
\[
\frac{\partial \mathbb{E}[F_n]}{\partial S_0} = -\mathbb{P}(\text{agent gets }G_a^{-1}(S_0)) \cdot \frac{\partial G_a^{-1}}{\partial S} \bigg\vert_{S=S_0} - \mathbb{P}(\text{agent gets }G_a^{-1}(S_0 + \Delta S)) \cdot \frac{\partial G_a^{-1}}{\partial S} \bigg\vert_{S=S_0+\Delta S} < 0.
\]
Thus, the principal's optimal choice must be to adjust the payment such that the IR constraint is binding: the principal pays the minimal amount such that the agent accepts the contract. Due to the continuity, the adjustment is always achievable. Hence, we can now restrict our attention to the binding IR constraint cases:
\[
\mathbb{E}[G_a(F_n) - E(\eta^*)] = U_0.
\]

The expected paid amount in the restricted second-best problem is now
\[
\mathbb{E}[F_n] = \mathbb{E}[G_a^{-1}(U_0 + E(\eta^*))],
\]
which is larger than the first-best problem's optimal payment
\[
G_a^{-1}(E(\eta^*) + U_0) = G_a^{-1}\left(\mathbb{E}[E(\eta^*) + U_0]\right).
\]
By a sharpened version of Jensen's inequality \citep{liao2019sharpening}, we have
\[
\inf_{s} \frac{\partial^2 G_a^{-1}(s)}{\partial s^2} \cdot \mathrm{Var}(F_n(\mathcal{D}_n)) \leq \mathbb{E}[G_a^{-1}(U_0 + E(\eta^*))] - G_a^{-1}(E(\eta^*) + U_0) \leq \sup_{s} \frac{\partial^2 G_a^{-1}(s)}{\partial s^2} \cdot \mathrm{Var}(F_n(\mathcal{D}_n)).
\]
Since the agent decision $\eta_a \in [0, 1]$, without loss of generality assume that the amount paid to the agent must be almost surely bounded:
\[
F_n(\mathcal{D}_n) \in [G_a^{-1}(U_0 + E(0)), G_a^{-1}(U_0 + E(1))] \eqqcolon \mathcal{S},\quad  \text{a.s.}
\]
Then $\inf_{s \in \mathcal{S}} \frac{\partial^2 G_a^{-1}(s)}{\partial s^2}$ and $\sup_{s \in \mathcal{S}} \frac{\partial^2 G_a^{-1}(s)}{\partial s^2}$ must be bounded from below and above by Assumption \ref{assm:basic_assm_1} and Assumption \ref{assm:basic_assm_3}. We now conclude that the Jensen gap between the first-best and the second-best solutions is of the order of $\mathrm{Var}(F_n(\mathcal{D}_n))$:
\begin{equation}
\mathbb{E}[G_a^{-1}(U_0 + E(\eta^*))] - G_a^{-1}(E(\eta^*) + U_0) = \Theta(\mathrm{Var}(F_n(\mathcal{D}_n))).
\label{eq:Jensen_gap_variance}
\end{equation}

The task now reduces to estimating the variance of the payment $\mathrm{Var}(F_n(\mathcal{D}_n))$. Again, for the binary contracts, we assume that the agent's monetary utility is of two levels: a higher level $S_1$ and a lower level $S_0$. If we set two levels of payments to be $G_a^{-1}(S_1) > G_a^{-1}(S_0)$, the variance is
\[
\mathrm{Var}(F_n(\mathcal{D}_n)) = [G_a^{-1}(S_1) - G_a^{-1}(S_0)]^2 \cdot \mathbb{P}(\text{agent gets } G_a^{-1}(S_1)) \cdot \mathbb{P}(\text{agent gets} G_a^{-1}(S_0)).
\]
From the discussions in Proposition \ref{prop:feasible_binary}, we can see that the $\eta^*$ can be achieved by setting the gap $G_a^{-1}(S_1) - G_a^{-1}(S_0) = \mathcal{O}(E(1))$, setting a test based on average agreements $\bar{A}$
\[
\bar{A} = \frac{1}{n} \sum_{i=1}^n A_i,
\]
and paying the agent a higher amount if $\bar{A} \geq \frac{1+c\eta^*}{2}  - \Theta(1/\sqrt{n \log n})$. Under this setting, we shall see from Lemma \ref{lemma:binomial_properties} (f) that
\[
\mathbb{P}(\text{agent gets }G_a^{-1}(S_1)) = 1 - \Theta\left(1/\sqrt{n \log n}\right), \quad \mathbb{P}(\text{agent gets } G_a^{-1}(S_0)) = \Theta\left(1/\sqrt{n \log n}\right).
\]
We reach the conclusion that there is a feasible solution to \eqref{eq:restricted_second_best} such that
\[
\mathrm{Var}(F_n(\mathcal{D}_n)) = \Theta\left(1/\sqrt{n \log n}\right),
\]
which proves the upper bound.

\textbf{Part II: Lower bound.}
The consistencies happen according to $n$ independent Bernoulli distribution; or equivalently, the total consistency follows a binomial distribution. Fix the number 
$n$. Since the likelihood of Bernoulli/binomial distributions purely depends on the consistency probability $p(\eta) = \frac{1+c\eta}{2}$ for some $c > 0$, we, without loss of generality, assume that we consider only the contracts based on the consistencies/agreements $\bm{A}$ (such a consideration can be justified by the Proposition \ref{prop:UMP}).

From the proof of the upper bound, we observe that if $F_n^{\text{bin}}(\bm{A})$ follows a threshold scheme such that the agent gets a higher payoff when $\bar{A}$ exceeds a certain threshold and gets a lower payoff otherwise, then the gap between the first-best and the second-best solutions is of $\Theta(1/\sqrt{n \log n})$. It suffices to show that such a threshold scheme is indeed optimal in the sense of minimizing the Jensen's gap/variance. We will prove the claim for (a slightly extended version of) binary contracts $\mathcal{F}_n^{\text{bin}}$.

To prove the claim, we first extend $\mathcal{F}_n^{\text{bin}}$ to be measurable functions of not only $\bm{A}$ but also an additional uniform random variable $U$ on $[0, 1]$, which introduces some extra randomness for fixed $\bm{A}$. More specifically, the original binary contract $F_n^{\text{bin}}$ pays the agent a higher wealth $G_a^{-1}(S_1)$ (so that the agent's monetary utility is $S_1$) if the total agreements $n \cdot \bar{A} \in \mathcal{K} \subset [n]$ and pays a lower amount $G_a^{-1}(S_0)$ (so that the agent's monetary utility is $S_0$) otherwise. We can write down as
\[
\mathbb{E}\left[F_n^{\text{bin}}(\bm{A})\right] = G_a^{-1}(S_0) + \left(G_a^{-1}(S_1) - G_a^{-1}(S_0)\right) \cdot \sum_{k \in \mathcal{K}} \mathbb{P}\left(n \cdot \bar{A} = k\right).
\]
We can view the payment process as: if $n \cdot \bar{A} \in \text{or }\notin \mathcal{K}$, pay a bonus with probability one or zero. If we "soften" the probability from one/zero to a real number $w_k \in [0, 1]$ (by using the additional uniform random variable $U$ and setting the pay probability to be $\mathbb{P}(U \leq w_k)$), we get a richer class of contracts, including
\[
\mathbb{E}\left[\tilde{F}_n^{\text{bin}}(\bm{A})\right] \coloneqq G_a^{-1}(S_0) + \left(G_a^{-1}(S_1) - G_a^{-1}(S_0)\right) \cdot \sum_{k \in [n]} \mathbb{P}\left(n \cdot \bar{A} = k\right) \cdot w_k.
\]
We denote the extended class by $\tilde{\mathcal{F}}_n^{\text{bin}}$. We denote $\mathbb{P}\left(n \cdot \bar{A} = k\middle| \eta^*\right)$ by $p_k$ and the amount $\sum_{k \in [n]} \mathbb{P}\left(n \cdot \bar{A} = k\middle| \eta^*\right) \cdot w_k$ by $P = \sum_{k \in [n]} w_k p_k$. Since more information is included, the principal's second-best solution could not be worse (because at least the principal can ignore the additional $U$ to keep the original contract). Thus, it suffices to show the optimality of the threshold scheme in the extended class to prove the lower bound.

Before we dive into the derivations, we present a property of the binomial distribution: for any $p \in (0,1 )$ and $X_n(p) \sim \mathrm{Binomial}(n,p)$,
\[
\frac{\partial }{\partial p}\mathbb{P}\left(X_n(p) = k\right) = \frac{k-np}{p(1-p)} \cdot \mathbb{P}\left(X_n(p) = k\right).
\]
Denote $(\frac{k}{n} - p(\eta^*))$ by $u_k$. Then $u_k$ is strictly increasing with respect to $k$, and
\begin{equation}
\begin{split}
\frac{\partial }{\partial \eta}\mathbb{P}\left(n \cdot \bar{A} = k\middle| \eta\right)\bigg\vert_{\eta = \eta^*}
&=\frac{\partial }{\partial \eta}\mathbb{P}(X_n(p(\eta)) = k)\bigg\vert_{\eta = \eta^*}\\
&= \frac{c}{2}\cdot \frac{k-np(\eta^*)}{p(\eta^*)(1-p(\eta^*))} \cdot \mathbb{P}(X_n(p(\eta^*)) = k)\\
&= \frac{cn}{2p(\eta^*)(1-p(\eta^*))} \cdot u_k \cdot \mathbb{P}\left(n \cdot \bar{A} = k\middle| \eta^*\right),
\end{split}
\label{eq:lower_bound_eq_1}
\end{equation}
which induces
\begin{equation}
\frac{\partial P}{\partial \eta}\bigg\vert_{\eta = \eta^*} = \frac{cn}{2p(\eta^*)(1-p(\eta^*))} \cdot \sum_{k \in [n]} w_k u_k p_k.
\label{eq:lower_bound_eq_2}
\end{equation}

From the discussions in the upper bound, we have the observation that: (1) Jensen's gap between the first-best and the second-best solutions is of the order of $\Theta(\mathrm{Var}(F_n(\mathcal{D}_n)))$; and (2) at least one feasible solution has the property that $\mathrm{Var}(F_n(\mathcal{D}_n)) \rightarrow 0$ as $n$ grows to infinity. Hence, for the optimal solution to \eqref{eq:restricted_second_best}, it must also have $\mathrm{Var}(F_n(\mathcal{D}_n)) \rightarrow 0$.
By the delta's method, we have that
\[
\mathrm{Var}(\tilde{F}_n) = (1+o(1)) \cdot \frac{\partial ^2 G_a^{-1}}{\partial s^2}\bigg\vert_{s=(\mathrm{E[U_0 + E(\eta^*)]})}\cdot \mathrm{Var}(G_a(\tilde{F}_n)),
\]
meaning that the gap can also be approximated by $\mathrm{Var}(G_a(\tilde{F}_n))$ or equivalently, $\mathrm{Var}(U_a)$. We have
\begin{equation}
\mathrm{Var}(U_a) = (S_1 - S_0)^2 P (1-P).
\label{eq:lower_bound_eq_3}
\end{equation}
By the first-order condition of the agent's IC constraint, the agent's marginal expected utility must be zero at $\eta = \eta^*$, which means that the agent's monetary utility gap between two wages $S_1 - S_0$ must satisfy
\begin{equation}
(S_1 - S_0) \cdot \frac{\partial P}{\partial \eta}\bigg\vert_{\eta = \eta^*} - \frac{\mathrm{d} E}{\mathrm{d} \eta}\bigg\vert_{\eta = \eta^*} = 0.
\label{eq:FOC_lower_bound}
\end{equation}
Substituting \eqref{eq:FOC_lower_bound} into \eqref{eq:lower_bound_eq_3}, we have
\begin{equation}
\mathrm{Var}(U_a) = \frac{P (1-P) \Big(\frac{\mathrm{d} E}{\mathrm{d} \eta}\big\vert_{\eta = \eta^*}\Big)^2}{\Big(\frac{\partial P}{\partial \eta}\big\vert_{\eta = \eta^*}\Big)^2}.
\label{eq:lower_bound_eq_4}
\end{equation}
If our target is to minimize the variance, it is equivalent to maximizing $\frac{\big(\frac{\partial P}{\partial \eta}\big\vert_{\eta = \eta^*}\big)^2}{P (1-P)}$, which means that our target is also equivalent to:
\begin{equation}
\max_{w_1,\dots,w_n} \ \frac{\left(\sum_{k \in [n]} w_k u_k p_k\right)^2}{\left(\sum_{k\in[n]} w_k p_k\right)\left(1 - \sum_{k\in[n]} w_k p_k\right)} \eqqcolon g(w_1,\dots,w_n).
\label{eq:lower_bound_target}
\end{equation}
Taking the derivatives with respect to each $w_k$, we have
\begin{equation}
\frac{\partial g}{\partial w_k} = \frac{\left(\sum_{j \in [n]} w_j u_j p_j\right) p_k \cdot\left[2u_k P(1-P) - \left(\sum_{j \in [n]} w_j u_j p_j\right)\cdot (1-2P)\right]}{P^2(1-P)^2}.
\label{eq:lower_bound_derivative}
\end{equation}
We can see that for any value of $\left(\sum_{k \in [n]} w_k u_k p_k\right)$ and $\left(\sum_{k \in [n]} w_k p_k\right)$, the sign of the derivative purely depends on the value of $u_k = \frac{k}{n} - p(\eta^*)$. Thus, there must exist some $k_0 \in [n]$, such that
\begin{equation}
\mathrm{sgn}\left(\frac{\partial g}{\partial w_k}\right) = \begin{cases}
-1, \quad \text{if }k < k_0,\\
+1, \quad \text{if } k > k_0,
\end{cases}
\end{equation}
which implies that any maximum must be taken at
\[
w_k = 
\begin{cases}
0, \quad \text{if }k < k_0,\\
1, \quad \text{if }k > k_0.
\end{cases}
\]
We have proved that any optimal contract in the extended case must be of the threshold form (except for one single $k_0$). Since $n$ is large enough and $p(\eta^*)$ is uniformly bounded away from 0 and 1, any single term $p_{k_0}$ does not have much influence over the entire contract; adding one term or removing one term does not change the structure of the problem. Therefore, it does not hurt to consider only those threshold contracts by letting $w_{k_0} = 0$ or $w_{k_0} = 1$. From Lemma \ref{lemma:binomial_properties}, we can see that the first-order condition is met at two points: one $p_1$ on the left side of $\frac{k_0}{n}$, and one $p_2$ on the right side of $\frac{k_0}{n}$. But the left one $p_1$ is not a local maximum of the expected utility with respect to the agent's IC constraint. In conclusion, the only possible way is to set $k_0$ such that $p_2 = p^*$. In that case,
\[
\mathrm{Var}(U_a) = \Theta\left(1/\sqrt{n \log n}\right).
\]

\textbf{Part III: Constants of Fisher information.} Fix a monitoring technology $m$ and write $p_m(\eta)=(1+c_m\eta)/2$ and $p_m^*=p_m(\eta^*)$. By the lower-bound part of the proof of Theorem \ref{thm:binary}, any optimal binary contract is thresholding (up to one boundary mass point). Thus we may consider a threshold contract with utility levels $S_0<S_1$ and utility gap
\[
\Delta S_m \coloneqq S_1-S_0.
\]
Let
\[
P_m(\eta,n,k)\coloneqq \mathbb{P}(X_n(p_m(\eta))\ge k),
\qquad X_n(p)\sim\mathrm{Binomial}(n,p).
\]
Under such a contract, the agent's expected utility is
\[
S_0+\Delta S_m\, P_m(\eta,n,k)-E(\eta).
\]

We first show that $\Delta S_m=\Theta(1)$. Since $\eta^*$ is incentive compatible in the restricted problem,
\[
S_0+\Delta S_m\, P_m(\eta^*,n,k)-E(\eta^*)
\ge
S_0+\Delta S_m\, P_m(0,n,k)-E(0),
\]
so
\[
\Delta S_m\bigl(P_m(\eta^*,n,k)-P_m(0,n,k)\bigr)\ge E(\eta^*)-E(0).
\]
Because $0\le P_m(0,n,k)\le P_m(\eta^*,n,k)\le 1$, we obtain
\[
\Delta S_m \ge E(\eta^*)-E(0)>0.
\]
On the other hand, Assumption \ref{assm:basic_assm_3} implies
\[
\Delta S_m \le G_a(\overline w)-G_a(\underline w)<\infty.
\]
Hence $\Delta S_m=\Theta(1)$.

The first-order condition at $\eta=\eta^*$ now yields
\[
\Delta S_m\,
\frac{\partial}{\partial \eta}P_m(\eta,n,k)\bigg|_{\eta=\eta^*}
=
E'(\eta^*),
\]
and therefore
\[
\frac{\partial}{\partial \eta}P_m(\eta,n,k)\bigg|_{\eta=\eta^*}
=\Theta(1).
\]
Since
\[
\frac{\partial}{\partial \eta}P_m(\eta,n,k)
=
\frac{c_m}{2}\,
\frac{\partial}{\partial p}\mathbb{P}(X_n(p)\ge k)\bigg|_{p=p_m(\eta)},
\]
we get
\[
\frac{\partial}{\partial p}\mathbb{P}(X_n(p)\ge k)\bigg|_{p=p_m^*}
=
\Theta(c_m^{-1}).
\]
By Lemma \ref{lemma:binomial_properties}(d)--(f), this is possible only when the threshold is placed at
\[
\left|\frac{k}{n}-p_m^*\right|
=
\Theta\!\left(
\sqrt{\frac{p_m^*(1-p_m^*)\log n}{n}}
\right).
\]

At such a threshold,
\[
\mathbb{P}(X_n(p_m^*)=k)
=
\frac{p_m^*}{k}\,
\frac{\partial}{\partial p}\mathbb{P}(X_n(p)\ge k)\bigg|_{p=p_m^*}
=
\Theta\!\left(\frac{1}{c_m n}\right).
\]
Moreover, for $j$ in the relevant tail window,
\[
1-\frac{\mathbb{P}(X_n(p_m^*)=j+1)}{\mathbb{P}(X_n(p_m^*)=j)}
=
\Theta\!\left(
\sqrt{\frac{\log n}{n\,p_m^*(1-p_m^*)}}
\right).
\]
Hence, by the same geometric-series argument as in Lemma \ref{lemma:binomial_properties}(g),
\[
1-P_m(\eta^*,n,k)
=
\Theta\!\left(
\frac{\sqrt{p_m^*(1-p_m^*)}}{c_m\sqrt{n\log n}}
\right),
\]
and thus
\[
P_m(\eta^*,n,k)\bigl(1-P_m(\eta^*,n,k)\bigr)
=
\Theta\!\left(
\frac{\sqrt{p_m^*(1-p_m^*)}}{c_m\sqrt{n\log n}}
\right).
\]

Finally, because $G_a^{-1}$ is continuously differentiable on the compact interval
$[G_a(\underline w),G_a(\overline w)]$ and $\Delta S_m=\Theta(1)$, the wealth gap
\[
\Delta w_m \coloneqq G_a^{-1}(S_1)-G_a^{-1}(S_0)
\]
also satisfies $\Delta w_m=\Theta(1)$. Therefore
\[
\mathrm{Var}(F_n)
=
\Delta w_m^2\,
P_m(\eta^*,n,k)\bigl(1-P_m(\eta^*,n,k)\bigr)
=
\Theta\!\left(
\frac{\sqrt{p_m^*(1-p_m^*)}}{c_m\sqrt{n\log n}}
\right).
\]
Combining this with \eqref{eq:Jensen_gap_variance} proves the stated bound for the restricted second-best problem. The unrestricted second-best case follows from the asymptotic convergence of the unrestricted second-best agent's action to the restricted second-best agent's action (or equivalently, the first-best) as $n$ approaches infinity, which is the same argument as Proposition \ref{prop:second_converge_to_first}.
 \end{proof}

We now give a proof to show that the optimal solution to \eqref{eq:second_best} must be close to that of \eqref{eq:restricted_second_best} when $n$ is large enough.
\begin{proposition}
\label{prop:second_converge_to_first}
Denote the optimal solution to \eqref{eq:second_best} within the binary contract class by $\eta_n^{\text{bin}}$. Then, under the condition of Theorem \ref{thm:binary}, we have that $\eta_n^{\text{bin}}$ converges to $\eta^*$ with
\[
\left(\eta_n^{\text{bin}} - \eta^*\right)^2 = \mathcal{O}\left(1/\sqrt{n \log n}\right).
\]
\end{proposition}
\begin{proof}{Proof}
\textbf{Proof.} For any $\eta \in [0, 1]$, define two optimization problems in analogy to the first-best/second-best problems:
\begin{equation}
\begin{split}
\mathrm{F}(\eta) \coloneqq \max_{F_n} \ &\mathbb{E}[U_p(F_n(\mathcal{D}_n), \eta)]\\
\text{s.t. }& \mathbb{E}[U_a(F_n(\mathcal{D}_n), \eta)]\geq U_0. \quad \text{(IR)}
\end{split}
\label{eq:first_best_auxilliary}
\end{equation}
\begin{equation}
\begin{split}
\mathrm{S}_n^{\text{bin}}(\eta) \coloneqq \max_{F_n^{\text{bin}}} \ &\mathbb{E}[U_p(F_n^{\text{bin}}(\mathcal{D}_n), \eta)],\\
\text{s.t. }& \mathbb{E}[U_a(F_n^{\text{bin}}(\mathcal{D}_n), \eta)]\geq U_0, \quad \text{(IR)}\\
& \mathbb{E}[U_a(F_n^{\text{bin}}(\mathcal{D}_n), \eta)]\geq \sup_{\eta^\prime} \mathbb{E}[U_a(F_n^{\text{bin}}(\mathcal{D}_n), \eta^\prime)]. \quad \text{(IC)}
\end{split}
\label{eq:second_best_auxilliary}
\end{equation}
Note that the maximum to the first-best problem \eqref{eq:first_best} equals $\mathrm{S}_n^{\text{bin}}(\eta^*)$. Also, the maximum to the second-best problem \eqref{eq:second_best} equals $\mathrm{F}(\eta_n^{\text{bin}})$ and the maximum to the restricted second-best problem \eqref{eq:restricted_second_best} equals $\mathrm{F}(\eta^*)$. We have
\begin{align*}
\mathrm{F}(\eta_n^{\text{bin}}) & \geq \mathrm{S}_n^{\text{bin}}(\eta_n^{\text{bin}})\\
& \geq \mathrm{S}_n^{\text{bin}}(\eta^*)\\
& \geq \mathrm{F}(\eta^*) - \Theta\left(1/\sqrt{n \log n}\right).
\end{align*}
Here, the first inequality is due to \eqref{eq:first_best_auxilliary}'s removing the IC constraint and enlarging the contract class could only induce an optimum no smaller than \eqref{eq:second_best_auxilliary}, the second inequality is due to the optimality of $\eta_n^{\text{bin}}$ subject to the original second-best problem \eqref{eq:second_best}, and the last from the proof of Theorem \ref{thm:binary} for the restricted second-best problem.

Now, as we have already proved the optimal contract form for the first-best problem must be a constant wage that exactly compensates the efforts made by the agent $E(\eta)$ plus the leisure utility $U_0$ (see that in the proof of Proposition \ref{prop:first_best_sol}), we have
\[
\mu(\eta_n^{\text{bin}}) - G_a^{-1}(U_0 + E(\eta_n^{\text{bin}})) \geq \mu(\eta^*) - G_a^{-1}(U_0 + E(\eta^*)) - \Theta\left(1/\sqrt{n \log n}\right).
\]
From the assumption made in Theorem \ref{thm:binary}, the function $G_a^{-1}(U_0 + E(\eta))$ must be strictly convex for $\eta$. Also, $\mu(\eta)$ is concave. Thus, the function $\mu(\eta) - G_a^{-1}(U_0 + E(\eta))$ must be strictly concave. From Taylor's expansion at $\eta=\eta^*$, we have
\[
\left(\eta_n^{\text{bin}} - \eta^*\right)^2 = \mathcal{O}\left(1/\sqrt{n \log n}\right).
\]
 \end{proof}

\subsubsection{Proof of Theorem \ref{thm:linear_contract}}
\label{apd:proof_of_linear}
Similar to the binary contract case, the only difference between the proof for the restricted second-best \eqref{eq:restricted_second_best} and the unrestricted \eqref{eq:second_best} in our proof is that for the restricted case, we directly assume $\eta^* \in (0, 1)$ such that \eqref{eq:uniformly_bounded_away} is satisfied (so there is at least some randomness in the outcome), while for the unrestricted case, we need to prove that the action taken by the agent under linear contracts, $\eta_n^{\text{lin}}$, converges to $\eta^*$ such that for large enough $n$'s, the decisions are also uniformly bounded away from 0 and 1 (proved later in Proposition \ref{prop:second_converge_to_first_linear}). Thus, we only prove Theorem \ref{thm:linear_contract} for the restricted case and do not repeat for the unrestricted case. 
\begin{proof}{Proof}
Suppose we are considering some contract $F_n^{\text{lin}}(\bm{A}) = \frac{1}{n}\sum_{i=1}^n f_n(A_i)$, where the agent gets a lower payment $f_n(0)$ per $A_i = 0$ and a higher payment $f_n(1)$ per $A_i = 1$. From previous discussions in the proof of Theorem \ref{thm:binary}, we know that the gap between the first-best and the second-best is the Jensen's gap due to the risk-averse nature of the agent. This gap is of the same order as $\mathrm{Var}(F_n(\mathcal{D}_n))$, where
\begin{align*}
\mathrm{Var}(F_n(\mathcal{D}_n)) & = \mathrm{Var}\left(\frac1n \sum_{i=1}^n f_n(A_i)\right) \\
& = \frac{1}{n} \mathrm{Var}(f_n(A_1)).
\end{align*}
It suffices to show that $\mathrm{Var}(f_n(A_1)) = \Theta(1)$.

\textbf{Upper bound.}
By Assumption \ref{assm:basic_assm_3}, the wage is uniformly bounded, thus
\[
\mathrm{Var}(f_n(A_1)) = \mathcal{O}(1).
\]

\textbf{Lower bound.}
If $\mathrm{Var}(f_n(A_1)) = \Omega(1)$ does not hold, then we must have $\mathrm{Var}(f_n(A_1)) \rightarrow 0$ as $n \rightarrow \infty$. Since $p(\eta^*) \in (0, 1)$ (due to $\eta^* \in (0, 1)$), we have $\mathrm{Var}(A_1) = \Theta(1)$. It implies that
\[
f_n(1) - f_n(0) \rightarrow 0, \quad \text{as }n \rightarrow \infty.
\]
Then the agent's utility at $\eta = \eta^*$ satisfies
\begin{align*}
\mathbb{E}[U_a(F_n^{\text{lin}}(\mathcal{D}_n), \eta^*)] & = \mathbb{E}\left[G_a\left(\frac1n \sum_{i=1}^n f_n(A_i)\right)\middle| \eta = \eta^*\right] - E(\eta^*) \\
& < G_a\left(\mathbb{E}[f_n(A_1)|\eta = \eta^*]\right) - E(\eta^*) \quad \rightarrow G_a(f_n(0)) - E(\eta^*).
\end{align*}
Similarly, the agent's utility at $\eta = 0$ also has the property that
\begin{align*}
\mathbb{E}[U_a(F_n^{\text{lin}}(\mathcal{D}_n), 0)] & = \mathbb{E}\left[G_a\left(\frac1n \sum_{i=1}^n f_n(A_i)\right)\middle| \eta = 0\right] - E(0) \\
& < G_a\left(\mathbb{E}[f_n(A_1)|\eta = 0]\right) - E(0) \quad \rightarrow G_a(f_n(0)) - E(0).
\end{align*}
However, the fact that $\eta^* > 0$ and $E$ is strictly monotonically increasing due to Assumption \ref{assm:basic_assm_4} and Assumption \ref{assm:basic_assm_1} implies that
\[
\mathbb{E}[U_a(F_n^{\text{lin}}(\mathcal{D}_n), \eta^*)] < \mathbb{E}[U_a(F_n^{\text{lin}}(\mathcal{D}_n), 0)], \quad \text{for large enough }n.
\]
This contradicts the IC constraint that $\eta^*$ maximizes the agent's utility.

\textbf{Constants depending on Fisher information.} Fix a monitoring technology $m$. Write the linear contract as
\[
F_n=a_n+b_n \bar A,
\qquad
a_n\coloneqq f_n(0),\quad b_n\coloneqq f_n(1)-f_n(0),
\]
where $n\bar A \sim \mathrm{Binomial}(n,p_m(\eta))$ and $p_m(\eta)=(1+c_m\eta)/2$. Then
\[
\mathrm{Var}(F_n\mid \eta^*)
=
\frac{b_n^2 p_m^*(1-p_m^*)}{n}.
\]
It remains to identify the order of $b_n$.

Let $X_n\sim \mathrm{Binomial}(n,p)$ and define $h_n(x)\coloneqq G_a(a_n+b_n x/n)$. Since wages are uniformly bounded, there exist constants $0<\underline g \le \overline g <\infty$ such that
\[
\underline g \le G_a'(w)\le \overline g
\qquad
\text{for all }w\in[\underline w,\overline w].
\]
Hence, for all $x\in[0,n]$,
\[
\underline g\,\frac{b_n}{n}\le h_n'(x)\le \overline g\,\frac{b_n}{n}.
\]

Using
\[
\frac{\partial}{\partial p}\mathbb{E}[h_n(X_n)]
=
\frac{\mathrm{Cov}(h_n(X_n),X_n)}{p(1-p)},
\]
together with the symmetrization identity
\[
2\,\mathrm{Cov}(h_n(X_n),X_n)
=
\mathbb{E}\!\left[(h_n(X_n)-h_n(X_n'))(X_n-X_n')\right]
\]
for an i.i.d.\ copy $X_n'$, we obtain
\[
\underline g\, b_n
\le
\frac{\partial}{\partial p}\mathbb{E}[G_a(F_n)]
\le
\overline g\, b_n.
\]
Since $\partial p_m(\eta)/\partial \eta = c_m/2$, the first-order condition of the IC constraint at $\eta^*$ implies
\[
\underline g\, b_n \frac{c_m}{2}
\le
E'(\eta^*)
\le
\overline g\, b_n \frac{c_m}{2},
\]
so
\[
b_n=\Theta(c_m^{-1}).
\]
Substituting this into the variance formula yields
\[
\mathrm{Var}(F_n\mid \eta^*)
=
\Theta\!\left(
\frac{p_m^*(1-p_m^*)}{c_m^2 n}
\right).
\]
Combining this with \eqref{eq:Jensen_gap_variance} proves the stated bound for the restricted second-best problem. The unrestricted second-best case follows exactly as in Proposition \ref{prop:second_converge_to_first_linear}.
 \end{proof}

The convergence of $\eta_n^{\text{lin}}$ could also be proved once we have the following auxiliary problem
\begin{equation}
\begin{split}
\mathrm{S}_n^{\text{lin}}(\eta) \coloneqq \max_{F_n^{\text{lin}}} \ &\mathbb{E}[U_p(F_n^{\text{lin}}(\mathcal{D}_n), \eta)],\\
\text{s.t. }& \mathbb{E}[U_a(F_n^{\text{lin}}(\mathcal{D}_n), \eta)]\geq U_0, \quad \text{(IR)}\\
& \mathbb{E}[U_a(F_n^{\text{lin}}(\mathcal{D}_n), \eta)]\geq \sup_{\eta^\prime} \mathbb{E}[U_a(F_n^{\text{lin}}(\mathcal{D}_n), \eta^\prime)]. \quad \text{(IC)}
\end{split}
\label{eq:second_best_auxilliary_linear}
\end{equation}
Equipped with that problem, we have similarly to Proposition~\ref{prop:second_converge_to_first} that
\begin{align*}
\mathrm{F}(\eta_n^{\text{lin}}) & \geq \mathrm{S}_n^{\text{lin}}(\eta_n^{\text{lin}})\\
& \geq \mathrm{S}_n^{\text{lin}}(\eta^*)\\
& \geq \mathrm{F}(\eta^*) - \Theta\left(1/n\right).
\end{align*}
We present the result here without for simplicity. The remaining proof is almost the same as that in Proposition \ref{prop:second_converge_to_first} except for the convergence rate.
\begin{proposition}
\label{prop:second_converge_to_first_linear}
Denote the optimal solution to \eqref{eq:second_best} under the linear contract class by $\eta_n^{\text{lin}}$. Then, under the condition of Theorem \ref{thm:linear_contract}, we have that $\eta_n^{\text{lin}}$ converges to $\eta^*$ with
\[
\left(\eta_n^{\text{lin}} - \eta^*\right)^2 = \mathcal{O}\left(1/n\right).
\]
\end{proposition}

\subsubsection{Proof of Theorem \ref{thm:lower_bound}}
\label{apd:proof_lower_bound}
According to previous discussions, the gap between the first-best and the second-best solutions is due to the randomness in the assessment result and is of the same order as $\mathrm{Var}(F_n(\mathcal{D}_n))$. In the following, we will see that the variance is lower bounded by the Cram\'{e}r-Rao arguments based on the Fisher information.
\begin{proof}{Proof}
   We first prove the claim for the restricted second-best value $\tilde{\mathcal{C}}_n$ with $\eta_a = \eta^*$. To start with, recall that the joint probability of the outcome $\mathcal{D}_n$ under the agent's action $\eta$ as $p_\eta^n(\mathcal{D}_n)$.
   Then the expected pay-off is $\mathbb{E}[F_n(\mathcal{D}_n)]$, which is
   \[
   \mathbb{E}[F_n(\mathcal{D}_n)] = \int F_n(\mathcal{D}_n)\, p_\eta^n(\mathcal{D}_n) \, \mathrm{d} \mathcal{D}_n.
   \]
   By defining the score as
   \[
   S_n(\mathcal{D}_n) \coloneqq \frac{\partial}{\partial \eta} \log p_\eta^n(\mathcal{D}_n),
   \]
   we have
   \allowdisplaybreaks
   \begin{align*}
   \frac{\partial}{\partial \eta}\mathbb{E}[F_n(\mathcal{D}_n)] & = \frac{\partial}{\partial \eta}\int F_n(\mathcal{D}_n)\, p_\eta^n(\mathcal{D}_n) \, \mathrm{d} \mathcal{D}_n\\
   & = \frac{\partial}{\partial \eta}\int \Big(F_n(\mathcal{D}_n) - \mathbb{E}[F_n(\mathcal{D}_n)]\Big)\, p_\eta^n(\mathcal{D}_n) \, \mathrm{d} \mathcal{D}_n\\
   & = \int \Big(F_n(\mathcal{D}_n) - \mathbb{E}[F_n(\mathcal{D}_n)]\Big)\, \frac{\partial}{\partial \eta}p_\eta^n(\mathcal{D}_n) \, \mathrm{d} \mathcal{D}_n\\
   & = \int \Big(F_n(\mathcal{D}_n) - \mathbb{E}[F_n(\mathcal{D}_n)]\Big)\, \Big(\frac{\partial}{\partial \eta} \log p_\eta^n(\mathcal{D}_n) \Big)\, p_\eta^n(\mathcal{D}_n)\, \mathrm{d} \mathcal{D}_n\\
   & = \int \Big(F_n(\mathcal{D}_n) - \mathbb{E}[F_n(\mathcal{D}_n)]\Big)\, S_n(\mathcal{D}_n)\, p_\eta^n(\mathcal{D}_n)\, \mathrm{d} \mathcal{D}_n\\
   & \leq \Bigg(\int \Big(F_n(\mathcal{D}_n) - \mathbb{E}[F_n(\mathcal{D}_n)]\Big)^2\, p_\eta^n(\mathcal{D}_n)\, \mathrm{d} \mathcal{D}_n \int S_n(\mathcal{D}_n)^2\, p_\eta^n(\mathcal{D}_n)\, \mathrm{d} \mathcal{D}_n \Bigg)^{\frac12}\\
   & = \Big(\mathrm{Var}(F_n(\mathcal{D}_n)) \cdot \mathcal{I}_n\Big)^{\frac12}
   \end{align*}
   Here, the first equality is by definition of the expectation, the second is due to the fact that $\int p_\eta^n(\mathcal{D}_n) \, \mathrm{d} \mathcal{D}_n \equiv 1$ and henceforth $\frac{\partial }{\partial \eta} \int p_\eta^n(\mathcal{D}_n) \, \mathrm{d} \mathcal{D}_n \equiv 0$, the third is due to the continuity, the fourth is the derivative of the logarithm, the fifth is by definition, the inequality is Cauchy-Schwarz inequality with respect to distribution $p_\eta^n(\mathcal{D}_n)$, and the last equality by defining
   \[
   \mathcal{I}_n \coloneqq \mathbb{E}[S_n^2(\mathcal{D}_n)].
   \]

   Under Assumption~\ref{assm:basic_assm_4}, the optimal $\eta^*$ is taken within $(0, 1)$, which implies that the first-order-condition is satisfied:
   \[
   \frac{\partial }{\partial \eta}G_a^{-1}(E(\eta) + U_0) \bigg|_{\eta = \eta^*} = \frac{\partial}{\partial \eta}\mathbb{E}[F_n(\mathcal{D}_n)] \leq \Big(\mathrm{Var}(F_n(\mathcal{D}_n)) \cdot \mathcal{I}_n\Big)^{\frac12}.
   \]
   Then, the variance of the payoff is lower bounded by:
   \[
   \mathrm{Var}(F_n(\mathcal{D}_n)) \geq \frac{\Big(\frac{\partial }{\partial \eta}G_a^{-1}(E(\eta) + U_0) \big|_{\eta = \eta^*}\Big)^2}{\mathcal{I}_n}.
   \]
   
   Since we have assumed that the outcome $\mathcal{D}_n$ is of $n$ i.i.d. samples, we have the following relationship for the Fisher information $\mathcal{I}_n$:
   \[
   \mathcal{I}_n = n \mathcal{I}_1.
   \]
   As a direct consequence, we have
   \[
   \mathrm{Var}(F_n(\mathcal{D}_n)) \geq \frac{1}{n} \cdot \frac{\Big(\frac{\partial }{\partial \eta}G_a^{-1}(E(\eta) + U_0) \big|_{\eta = \eta^*}\Big)^2}{\mathcal{I}_1}.
   \]

   For the unrestricted second-best problem~\eqref{eq:second_best}, as long as we prove the fact that the action taken by the agent in the second-best problem, $\eta_n^*$, remains interior of $(0, 1)$, the other parts of the proof remain exactly the same except for just one replacement (from $\Big(\frac{\partial }{\partial \eta}G_a^{-1}(E(\eta) + U_0) \big|_{\eta = \eta^*}\Big)^2$ to $\Big(\frac{\partial }{\partial \eta}G_a^{-1}(E(\eta) + U_0) \big|_{\eta = \eta_n^*}\Big)^2$). For the claim that $\eta_n^*$ remains interior of $(0, 1)$, we derive it from the following arguments: From the proof of Proposition~\ref{prop:second_converge_to_first} and Proposition~\ref{prop:second_converge_to_first_linear}, we know that $\mathrm{F}(\eta)$ is a convex function with minimum attained at $\eta = \eta^*$. Similarly, we can define
   \begin{equation}
    \begin{split}
    \mathrm{S}_n(\eta) \coloneqq \max_{F_n} \ &\mathbb{E}[U_p(F_n^{\text{bin}}(\mathcal{D}_n), \eta)],\\
    \text{s.t. }& \mathbb{E}[U_a(F_n(\mathcal{D}_n), \eta)]\geq U_0, \quad \text{(IR)}\\
    & \mathbb{E}[U_a(F_n(\mathcal{D}_n), \eta)]\geq \sup_{\eta^\prime} \mathbb{E}[U_a(F_n(\mathcal{D}_n), \eta^\prime)]. \quad \text{(IC)}
    \end{split}
    \label{eq:second_best_auxilliary_any}
    \end{equation}
    Then we can check that
    \begin{align*}
        \mathrm{F}(\eta_n^*) &\geq \mathrm{S}_n(\eta_n^*)\\
        &\geq \mathrm{S}_n(\eta_n^{\text{lin}})\\
        &\geq \mathrm{S}_n^{\text{lin}}(\eta_n^{\text{lin}})\\
        &\geq \mathrm{F}(\eta^*) - \Theta(1/n).
    \end{align*}
    As a consequence, a similar conclusion as Proposition~\ref{prop:second_converge_to_first_linear} holds for $\eta_n^*$, such that
    \[
    \big(\eta_n^* - \eta^* \big)^2 = O(1/n),
    \]
    which verifies the claim that $\eta_n^*$ remains interior of $(0, 1)$ for large enough $n$.
 \end{proof}



\subsubsection{Proof of Corollary \ref{cor:monitor_compare}}

By Theorems \ref{thm:binary} and \ref{thm:linear_contract}, the relevant asymptotic factors are
\[
\Gamma_{\mathrm{bin}}(c)=\frac{\sqrt{1-c^2(\eta^*)^2}}{c},
\qquad
\Gamma_{\mathrm{lin}}(c)=\frac{1-c^2(\eta^*)^2}{c^2}.
\]
Their derivatives are
\[
\Gamma_{\mathrm{bin}}'(c)
=
-\frac{1}{c^2\sqrt{1-c^2(\eta^*)^2}}
<0,
\qquad
\Gamma_{\mathrm{lin}}'(c)
=
-\frac{2}{c^3}
<0.
\]
Thus both asymptotic constants are strictly decreasing in $c$. Therefore, among any two monitoring technologies satisfying \eqref{eq:generic_monitoring_strength}, the one with the larger monitoring parameter $c$ yields the smaller asymptotic principal-utility gap. Applying this monotonicity to $c_{\mathrm{self}}$ and $c_{\mathrm{exp}}=\bar c$ proves the result.

\subsubsection{Technical Lemmas}

\begin{figure}[ht!]
  \centering
\includegraphics[width=0.7\textwidth]
{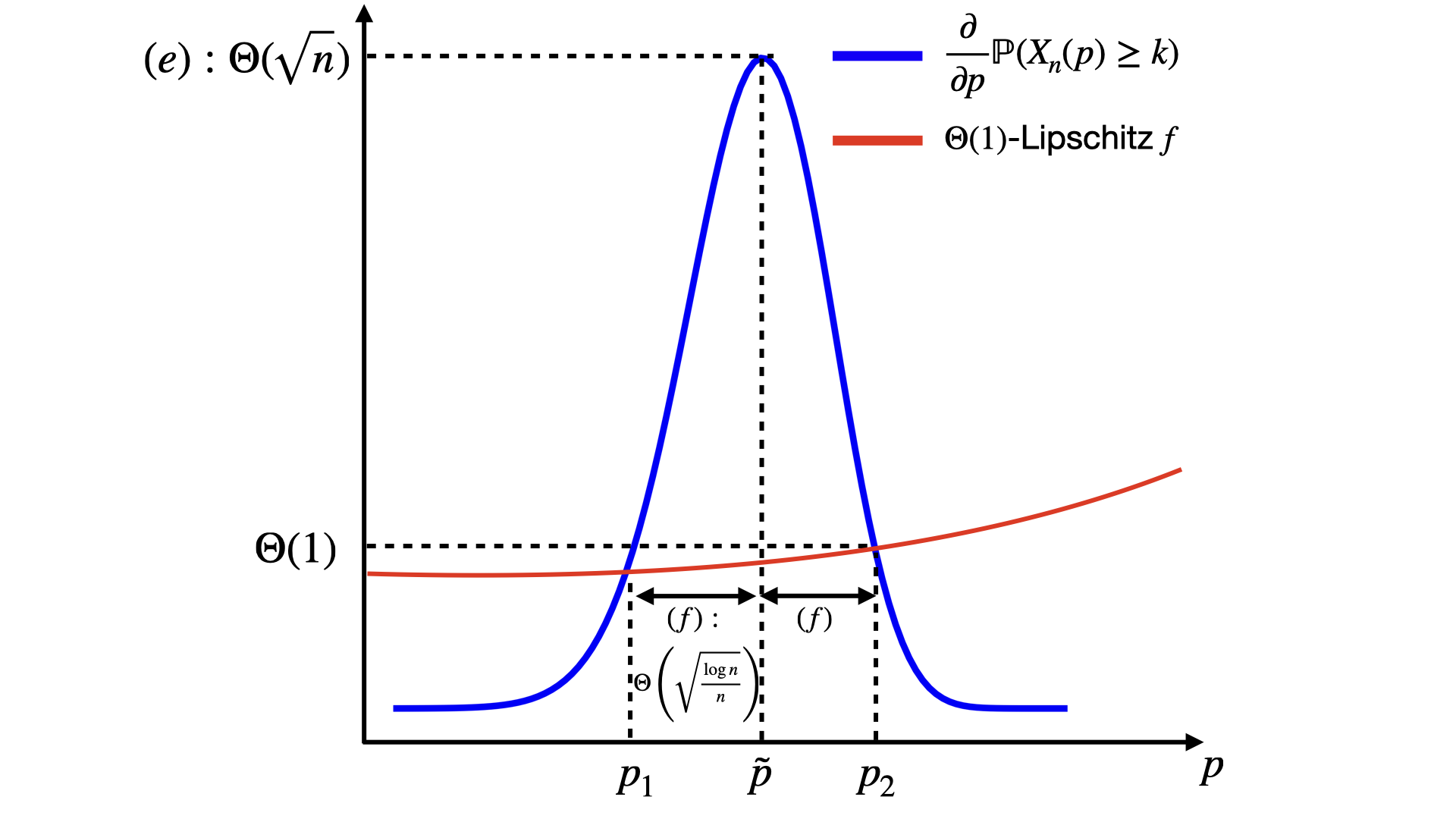}
    \caption{ Illustration for~Lemma \ref{lemma:binomial_properties} (e) (f). The blue-colored curve shows the ``bell curve'' nature of $\frac{\partial }{\partial p} \mathbb{P}(X_n(p) \geq k)$. The curve's peak value is of $\Theta(\sqrt{n})$. Two intersections with some $\Theta(1)$-Lipschitz function $f$ take place: one on the left, one on the right. Both are at a distance of $\Theta(\sqrt{\log n/n})$ to the center.}
  \end{figure}
 
\begin{figure}[ht!]
  \centering
\includegraphics[width=0.7\textwidth]
{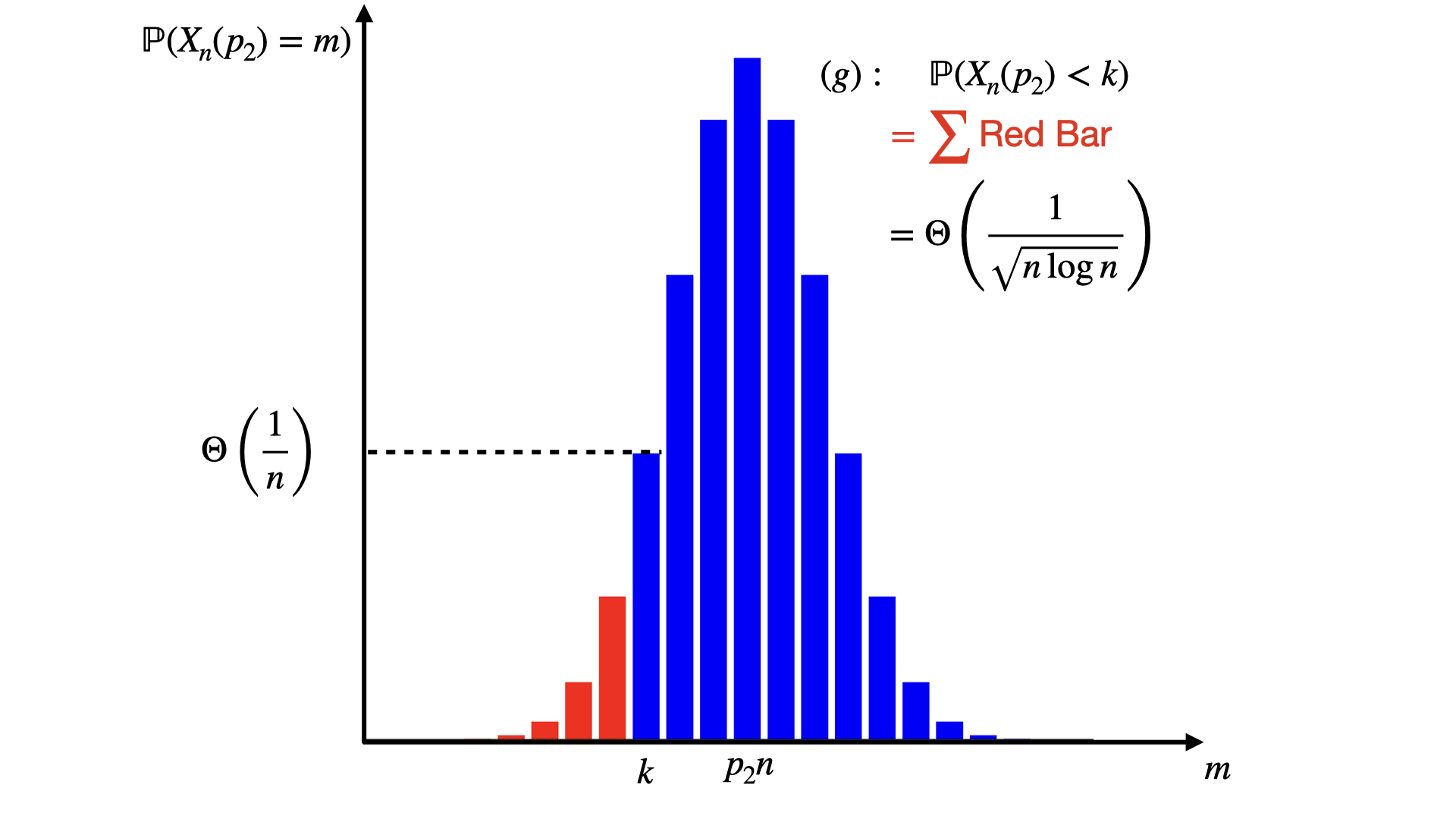}
    \caption{ Illustration for~Lemma \ref{lemma:binomial_properties} (g). Both the red-colored and the blue-colored bars are probability mass for different $\mathbb{P}(X_n(p_2) = m)$'s. By the first-order condition, the probability at the threshold $\mathbb{P}(X_n(p_2) = k)$ is of $\Theta(1/n)$. The tail probability (the red part) can be then estimated by the sum of geometric sequences, which is of $\Theta(1/\sqrt{n\log n})$.}
\label{fig:tech_lemma_illu}
\end{figure}

\begin{lemma}
\label{lemma:binomial_properties}
Suppose $X_n(p) \sim \mathrm{Binomial}(n, p)$. Then for any integer $k$ such that $1 < k < n$,
\begin{enumerate}
\item \[
\mathbb{P}(X_n(p) \geq k) = n\int_0^p \mathbb{P}(X_{n-1}(u) = k - 1) \mathrm{d} u
\]
\item \[
\frac{\partial}{\partial p} \mathbb{P}(X_n(p) \geq k) = n \mathbb{P}(X_{n-1}(p) = k - 1).
\]
\item \[
\frac{\partial^2}{\partial p^2} \mathbb{P}(X_n(p) \geq k) = \frac{n-1}{(1-p) p} \cdot (\tilde{p} - p) \cdot \frac{\partial}{\partial p} \mathbb{P}(X_n(p) \geq k),
\]
where $\tilde{p} = \frac{k-1}{n-1}$.
\item If $a \leq (1-p) p \leq b$, then \[
\exp\left(-\frac{n-1}{a} (p - \tilde{p})^2\right) \leq \frac{\frac{\partial }{\partial p} \mathbb{P}(X_n(p) \geq k)}{\frac{\partial }{\partial p} \mathbb{P}(X_n(p) \geq k) \vert_{p=\tilde{p}}} \leq \exp\left(-\frac{n-1}{b} (p - \tilde{p})^2\right),
\]
where $\tilde{p} = \frac{k-1}{n-1}$.
In other words, the curve of $\frac{\partial }{\partial p} \mathbb{P}(X_n(p) \geq k)$ is like a bell curve centered at $\tilde{p}$.
\item $\frac{\partial}{\partial p} \mathbb{P}(X_n(p) \geq k)$ monotonically increases for $p < \tilde{p}$ and monotonically decreases for $p > \tilde{p}$. If $k = cn + \mathcal{O}(1)$ for some $c \in (0, 1)$, then $\frac{\partial}{\partial p} \mathbb{P}(X_n(p) \geq k)$ reaches its maximum of order $\Theta(\sqrt{n})$ at $\tilde{p} = \frac{k-1}{n-1}$.
\item For any $\mathcal{O}(1)$-Lipschitz function $f$ independent of $n$ on $[0, 1]$ with $f(p) > 0$, for large enough $n$ and $k = cn + \mathcal{O}(1)$ for some $c \in (0, 1)$, $f$ intersects with $\frac{\partial }{\partial p} \mathbb{P}(X_n(p) \geq k)$ at two points $p_1$ and $p_2$, with
\[
p_1 = \tilde{p} - \Theta\left(\sqrt{\frac{(1-\tilde{p})\tilde{p}}{n}} \cdot \sqrt{\log n}\right),
\]
and
\[
p_2 = \tilde{p} + \Theta\left(\sqrt{\frac{(1-\tilde{p})\tilde{p}}{n}} \cdot \sqrt{\log n}\right).
\]
\item \[
\mathbb{P}(X_n(p_1) \geq k) = \Theta\left(\frac{1}{\sqrt{n \log n}}\right),
\]
and
\[
\mathbb{P}(X_n(p_2) \geq k) = 1 - \Theta\left(\frac{1}{\sqrt{n \log n}}\right).
\]
\end{enumerate}
\end{lemma}
\begin{proof}{Proof}
(a) This could be done via iterative integration by parts. For the completeness, we present the process here:
\begin{align*}
\int_0^p \mathbb{P}(X_{n-1}(u) = k - 1) \mathrm{d} u & = \int_0^p \frac{(n-1)!}{(k-1)!(n-k)!} \cdot u^{k-1}(1-u)^{n-k} \mathrm{d} u\\
& = \frac{(n-1)!}{k!(n-k)!} \cdot p^{k}(1-p)^{n-k} + \int_0^p \frac{(n-1)!}{(k)!(n-k-1)!} \cdot u^{k}(1-u)^{n-k-1} \mathrm{d} u\\
& = \frac{(n-1)!}{k!(n-k)!} \cdot p^{k}(1-p)^{n-k} + \frac{(n-1)!}{(k+1)!(n-k-1)!} \cdot p^{k+1}(1-p)^{n-k-1} \\
& \phantom{=} + \int_0^p \frac{(n-1)!}{(k+1)!(n-k-2)!} \cdot u^{k+1}(1-u)^{n-k-2} \mathrm{d} u \\
& = \cdots\\
& = \frac{1}{n}\sum_{j=k}^{n} \mathbb{P}(X_n(p) = j).
\end{align*}

(b) Taking the derivative with respect to $p$ on both sides of (a) yields the result.

(c) Taking the derivative with respect to $p$ on both sides of (b), we have
\begin{align*}
\frac{\partial^2}{\partial p^2} \mathbb{P}(X_n(p) \geq k) &= n \cdot \frac{(n-1)!}{(k-1)!(n-k)!} \cdot \frac{\partial}{\partial p} (p^{k-1} (1-p)^{n-k})\\
& = n \cdot \frac{(n-1)!}{(k-1)!(n-k)!} \cdot \left[(k-1)(1-p) - (n-k)p\right] \cdot p^{k-2}(1-p)^{n-k-1} \\
& = n \cdot \frac{(n-1)!}{(k-1)!(n-k)!} \cdot \left[(k-1) - (n-1)p\right] \cdot p^{k-2}(1-p)^{n-k-1} \\
& = n \cdot \frac{(k-1) - (n-1)p}{p(1-p)} \cdot \mathbb{P}(X_{n-1}(p) = k-1)\\
& = \frac{(k-1) - (n-1)p}{p(1-p)} \cdot \frac{\partial}{\partial p} \mathbb{P}(X_n(p) \geq k).
\end{align*}

(d) By the assumption that $a \leq (1-p)p$ and (c), for $p \geq \tilde{p}$ we have
\[
\frac{\partial^2}{\partial p^2} \mathbb{P}(X_n(p) \geq k) \geq \frac{n-1}{a} \cdot (\tilde{p} - p) \cdot \frac{\partial}{\partial p} \mathbb{P}(X_n(p) \geq k).
\]
Multiplying both sides with $\exp(\frac{n-1}{a}(p - \tilde{p})^2)$, we have
\[
\frac{\partial }{\partial p} \left[\exp\left(\frac{n-1}{a}(p - \tilde{p})^2\right) \cdot \frac{\partial }{\partial p} \mathbb{P}(X_n(p) \geq k)\right] \geq 0.
\]
Thus,
\[
\exp\left(\frac{n-1}{a}(p - \tilde{p})^2\right) \cdot \frac{\partial }{\partial p} \mathbb{P}(X_n(p) \geq k) \geq \frac{\partial }{\partial p} \mathbb{P}(X_n(p) \geq k) \bigg\vert_{p=\tilde{p}}.
\]
For the part $p \leq \tilde{p}$, the conclusion could be derived similarly. As for the second inequality, it follows from the assumption $(1-p)p \leq b$ and a similar argument.

(e) The monotonicity follows from (c). If $\frac{k}{n} = c + \mathcal{O}(\frac{1}{n})$ for some $c \in (0, 1)$, then the maximum can be computed as
\begin{align*}
\frac{\partial }{\partial p} \mathbb{P}(X_n(p) \geq k)\bigg\vert_{p = \tilde{p}} & = n \cdot \frac{(n-1)!}{(k-1)!(n-k)!} \cdot \tilde{p}^{k-1} (1-\tilde{p})^{n-k} \\
& = n \cdot \frac{(n-1)!}{(n-1)^{n-1}} \cdot \frac{(k-1)^{k-1}}{(k-1)!} \cdot \frac{(n-k)^{n-k}}{(n-k)!}\\
& = (1+o(1)) \cdot n \cdot \sqrt{2\pi n} \cdot \frac{1}{\sqrt{2\pi c n}} \cdot \frac{1}{\sqrt{2\pi(1-c)n}} \quad \text{by Stirling's approximation}\\
& = (1+o(1)) \cdot \frac{\sqrt{n}}{\sqrt{2\pi c(1-c)}}.
\end{align*}

(f) From the fact that the peak value of $\frac{\partial}{\partial p} \mathbb{P}(X_n(p) \geq k)$ is of order $\Theta(\sqrt{n})$ and (d)'s bell curve approximation, we have the following fact:\\
If $\frac{\partial}{\partial p} \mathbb{P}(X_n(p) \geq k)$ is of $\Theta(1)$ value, then $|p - \tilde{p}|$ must be of order $\Theta(\frac{\sqrt{(1-\tilde{p}){\tilde{p}} \log n}}{\sqrt{n}})$. In other words, if $\frac{\partial}{\partial p} \mathbb{P}(X_n(p) \geq k)$ intersects with a Lipschitz function $f$ independent of $n$, then the intersection must happens at a distance of $\Theta(\frac{\sqrt{(1-\tilde{p}){\tilde{p}} \log n}}{\sqrt{n}})$ from $\tilde{p}$.

From part (c), we see that the derivative of $\frac{\partial}{\partial p} \mathbb{P}(X_n(p) \geq k)$ (that is, the second-order derivative $\frac{\partial^2}{\partial p^2} \mathbb{P}(X_n(p) \geq k)$) must be of order $\Theta(\sqrt{n \log n})$, either positive if $p < \tilde{p}$ or negative if $p > \tilde{p}$. Since the derivative of $f$ is at most $\mathcal{O}(1)$, the intersection must happen at most twice (once on the left side of $\tilde{p}$ and once on the right).

We now prove that the intersection must happen at least twice. First, observe that
\[
\frac{\partial}{\partial p} \mathbb{P}(X_n(p) \geq k) \bigg\vert_{p = 0} = 0 < f(0),
\]
and
\[
\frac{\partial}{\partial p} \mathbb{P}(X_n(p) \geq k) \bigg\vert_{p = 1} = 0 < f(1).
\]
For large enough $n$, 
\[
\frac{\partial}{\partial p} \mathbb{P}(X_n(p) \geq k) \bigg\vert_{p = \tilde{p}} = \Theta(\sqrt{n}) > f(\tilde{p}).
\]
By the continuity of $\frac{\partial}{\partial p} \mathbb{P}(X_n(p) \geq k)$ and $f$, we must have at least two intersections.

(g) First, observe that
\begin{align*}
\mathbb{P}(X_n(p_1) = k) & = p_1 \cdot \mathbb{P}(X_{n-1}(p_1) = k-1) + (1-p_1) \cdot \mathbb{P}(X_{n-1}(p_1) = k)\\
& = p_1 \cdot \mathbb{P}(X_{n-1}(p_1) = k-1) + (1-p_1) \cdot \frac{p_1}{1-p_1} \cdot \frac{n-k}{k} \cdot \mathbb{P}(X_{n-1}(p_1) = k-1)\\
& = \frac{p_1 n}{k} \cdot \mathbb{P}(X_{n-1}(p_1) = k-1)\\
& = \frac{p_1}{k} \cdot \frac{\partial }{\partial p}\mathbb{P}(X_n(p_1) \geq k)\\
& = \frac{p_1}{k} \cdot f(p_1) = \Theta\left(\frac{1}{n}\right).
\end{align*}
For $k \leq j \leq k + \lceil C\sqrt{n \log n}\rceil$, we have
\[
\frac{\mathbb{P}(X_n(p_1) = j+1)}{\mathbb{P}(X_n(p_1) = j)} = \frac{p_1}{1-p_1} \cdot \frac{n-j}{j+1} = 1 - \Theta\left(\sqrt{\frac{\log n}{n}}\right).
\]
Hence,
\begin{align*}
\sum_{j=k}^{n} \mathbb{P}(X_n(p_1) = j) & \geq \sum_{j=k}^{k + \lceil C\sqrt{n \log n}\rceil} \mathbb{P}(X_n(p_1) = j)\\
& \geq \frac{\mathbb{P}(X_n(p_1) = k) \cdot \left(1-\left(\sup_{k\leq j\leq k + \lceil C\sqrt{n \log n}\rceil - 1} \frac{\mathbb{P}(X_n(p_1) = j+1)}{\mathbb{P}(X_n(p_1) = j)}\right)^{\lceil C\sqrt{n \log n}\rceil + 1}\right)}{1-\inf_{k\leq j\leq k + \lceil C\sqrt{n \log n}\rceil - 1} \frac{\mathbb{P}(X_n(p_1) = j+1)}{\mathbb{P}(X_n(p_1) = j)}}\\
& = \frac{\mathbb{P}(X_n(p_1) = k) \cdot (1-\exp(-\Theta(1)))}{\Theta\left(\sqrt{\frac{\log n}{n}}\right)}\\
& = \Theta\left(\frac{1}{\sqrt{n \log n}}\right).
\end{align*}
For the upper bound, it is easy to see that
\begin{align*}
\sum_{j=k}^{n} \mathbb{P}(X_n(p_1) = j) 
& \leq \frac{\mathbb{P}(X_n(p_1) = k)}{1-\sup_{k\leq j\leq n -1} \frac{\mathbb{P}(X_n(p_1) = j+1)}{\mathbb{P}(X_n(p_1) = j)}}\\
& = \Theta\left(\frac{1}{\sqrt{n \log n}}\right).
\end{align*}
The conclusion can also be similarly derived for $p_2$, which we omit for simplicity.
 \end{proof}

\section{Appendix for Numerical Experiments}

\subsection{Experiments Setup and Additional Results}
\subsubsection{Dataset and Preference Probability}\
\label{appx: dataset}

\medskip
Throughout this work, we use four preference datasets: PKU-SafeRLHF (PKU) \citep{ji2024pku}, Helpsteer \citep{wang2023helpsteer}, UltraFeedback (Ultra) \citep{cui2023ultrafeedback}, and Skywork-Reward-Preference-80K-v0.2 (Skywork) \citep{liu2024skywork}:

\begin{itemize}
\item \textbf{PKU:} \textit{PKU-SafeRLHF} consists of 83K preference entries annotated along two dimensions: harmlessness and helpfulness. In our experiments, we use the preference model \textit{beaver-7b-v1.0-reward} \citep{dai2024safe} to annotate data pairs and compute the preference probability.

\item \textbf{Helpsteer:} \textit{Helpsteer} contains 37K samples. Each response is annotated by humans on five attributes ranging from 0 to 4, with higher scores indicating better performance. We use the average score across these five attributes as the chosen score to compute the preference probability.

\item \textbf{Ultra:} \textit{UltraFeedback} is annotated by GPT-4 using a fine-grained instruction covering four aspects: instruction-following, truthfulness, honesty, and helpfulness. GPT-4 provides a chosen score for each response, which we directly use to compute the preference probability.
\item \textbf{Skywork:} \textit{Skywork-Reward-Preference-80K-v0.2} is a curated subset of publicly available preference data and contains 77K samples. In our experiments, we use the preference model \textit{Skywork-Reward-Gemma-2-27B-v0.2} \citep{liu2024skywork} to annotate data pairs and derive the preference probability.
\end{itemize}

\noindent \textbf{Preference Model Calibration} To address the potential overconfidence of the predicted preference probabilities from open-source preference models \citep{guo2017calibration}, we calibrate these models using \textit{Histogram Binning} \citep{zadrozny2001obtaining} with half of the dataset. Specifically, we create 30 bins and define their boundaries such that each bin contains an equal number of samples. We then set the calibrated preference probability for each bin to the average proportion of positive samples (i.e., $y_1$ is the preferred response) within that bin. Figure \ref{fig:appx_cal} compares (out-of-sample) performance before and after calibration for the PKU and Skywork datasets. For more details on calibration methods, we refer readers to the survey \cite{silva2023classifier}.

\begin{figure}[ht!]
  \centering
  \begin{minipage}[b]{0.46\textwidth}
    \centering
\includegraphics[width=\textwidth]
{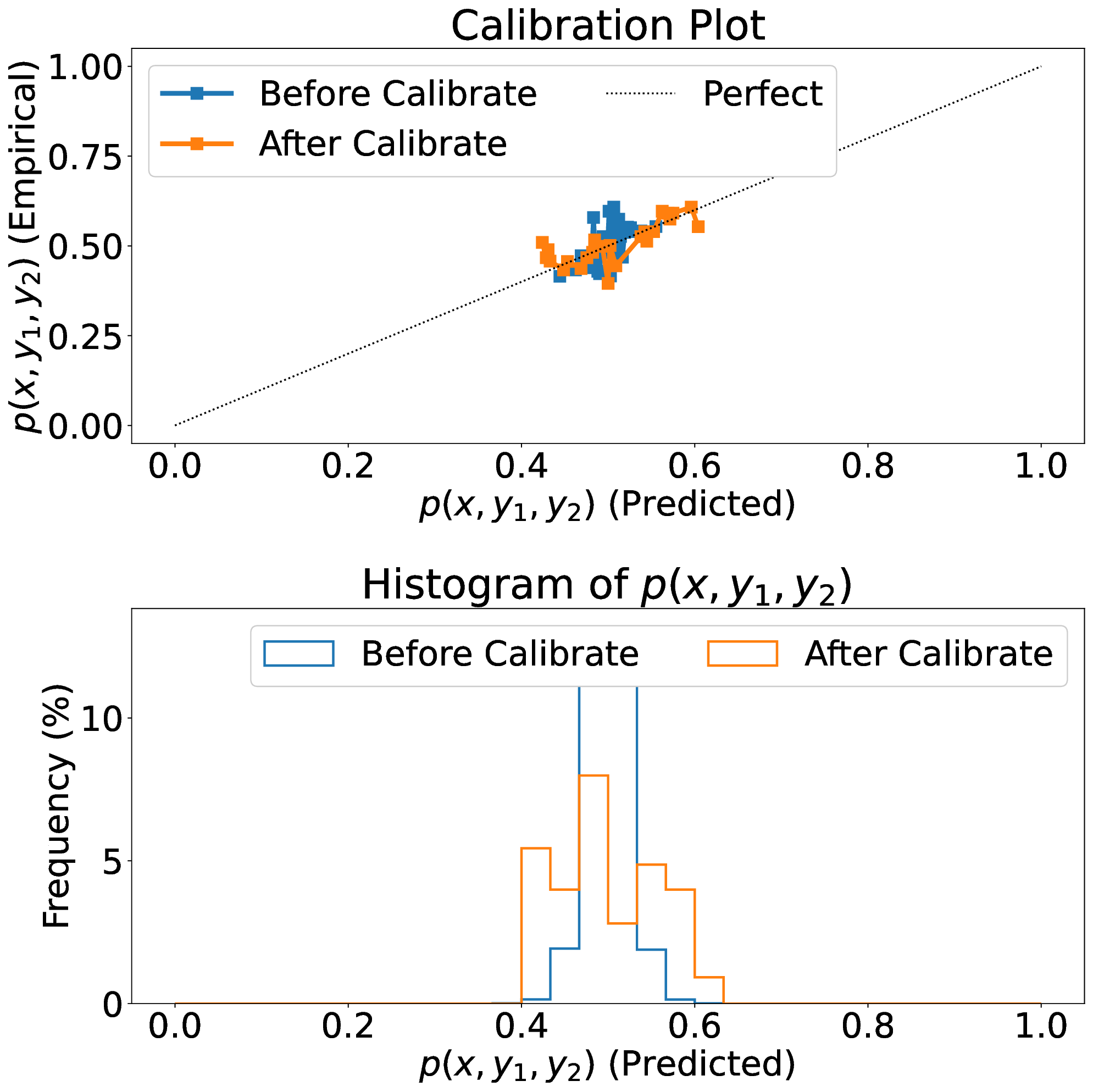}
    \vspace{2mm} {\small (a) PKU}
  \end{minipage}
    \hfill
  \begin{minipage}[b]{0.46\textwidth}
    \centering
\includegraphics[width=\textwidth]
{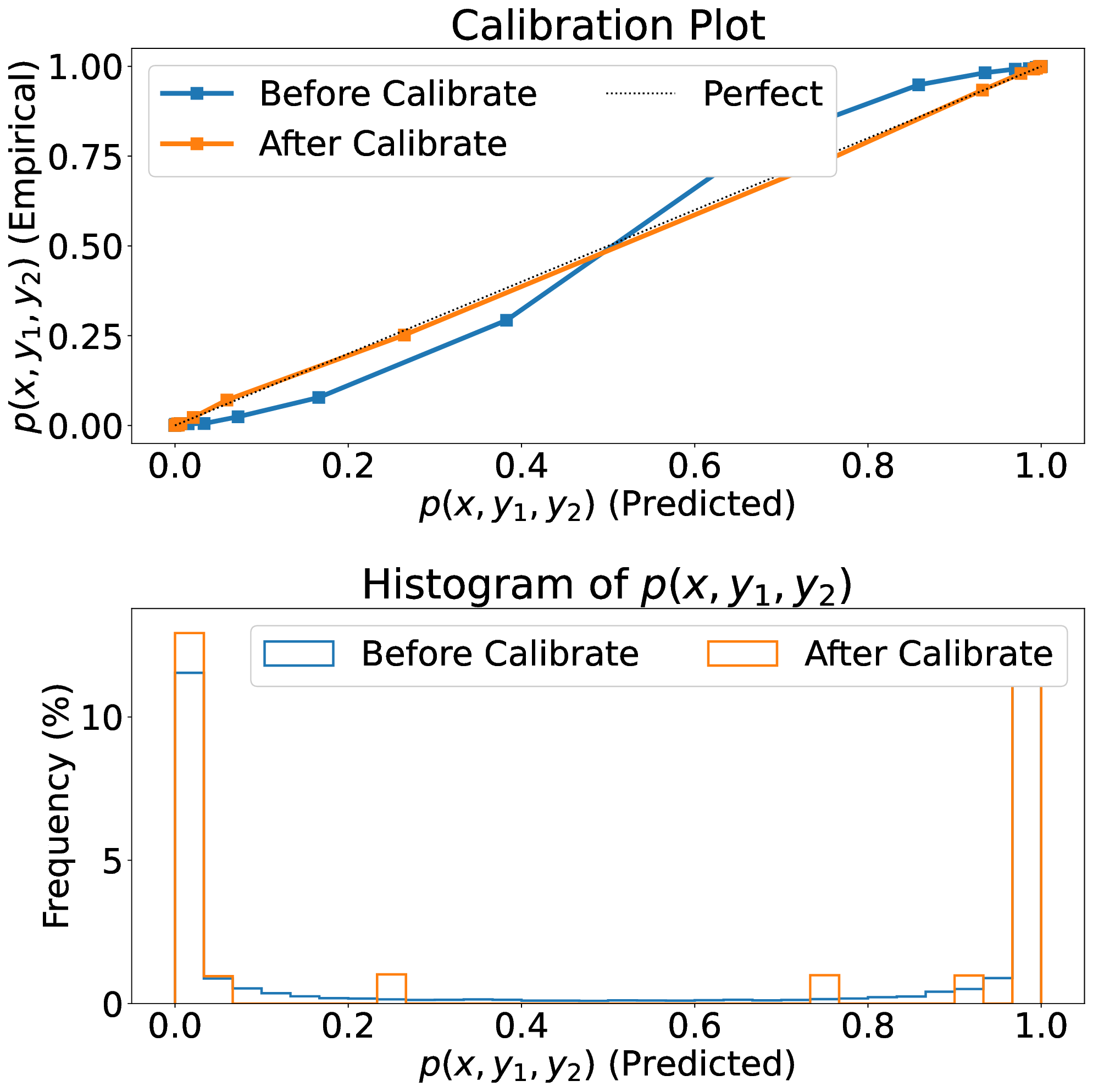}
    \vspace{2mm} {\small (b) Skywork}
  \end{minipage}
  \caption{Calibration for two datasets.  
(Top row) Empirical preference probability $p(x,y_1,y_2)$ vs. the predicted probability before and after calibration. The dashed line ($x=y$) represents perfect alignment between predictions and empirical observations.  
(Bottom row) Histogram of the (predicted) preference probability $p(x,y_1,y_2)$ before and after calibration. We can see the calibration procedure improves alignment between the predicted probabilities and the empirical observations for both datasets. }
\label{fig:appx_cal}
\end{figure}
\subsubsection{Setup and more experiments for Figure \ref{fig:main_hist_LB}}\
\label{appx:fig_hist_LB}

\medskip
\textbf{Setup:} We use half of each (calibrated) dataset (specifically, the part not used during the calibration step or randomly sampled data if there is no calibration)  to plot the histogram of the probability 
\[
\mathbb{P}\bigl(y_{\text{chosen}} \succ y_{\text{rejected}} \mid x\bigr),
\]
which is equivalent to 
\[
\mathbb{P}\bigl(y_{1} \succ y_{2} \mid x\bigr)
\]
if we always order the pair so that $\mathbb{P}\bigl(y_{1} \succ y_{2} \mid x\bigr) \ge 1/2$. In addition, we plot the lower bound on the sum of the two types of errors, as stated in Proposition~\ref{prop:info_lower_bound}, for various choices of $\eta_0$, $\eta_1$, and $n$.\\

\noindent \textbf{More results:} Figure~\ref{fig:appx_sup_1} complements Figure~\ref{fig:main_hist_LB} by further computing the lower bound from Proposition~\ref{prop:info_lower_bound} for $\eta_1 \in \{0.98,\, 0.96\}$. The results show that a lower $\eta_1$ yields a higher lower bound. Intuitively, a lower $\eta_1$ (or a higher $\eta_0$) causes the induced distributions on $\mathcal{D}_{\text{test}}$ corresponding to $\eta_1$ and $\eta_0$ to be more similar, thereby complicating the hypothesis testing.

\begin{figure}[ht!]
  \centering
  \begin{minipage}[b]{0.23\textwidth}
    \centering
\includegraphics[width=\textwidth]
{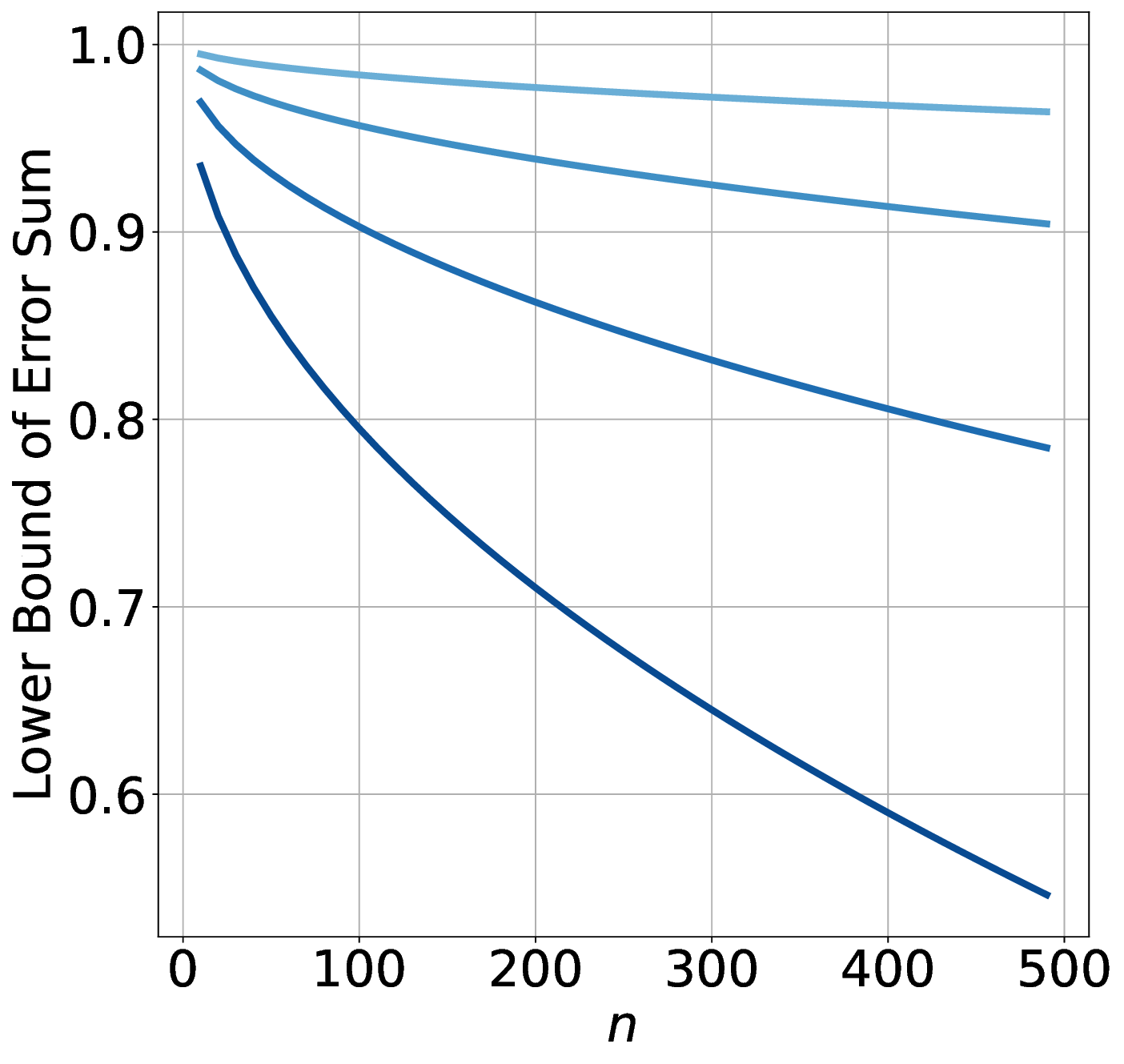}
    \vspace{2mm} {\small PKU, $\eta_1=0.98$}
  \end{minipage}
    \hfill
  \begin{minipage}[b]{0.23\textwidth}
    \centering
\includegraphics[width=\textwidth]
{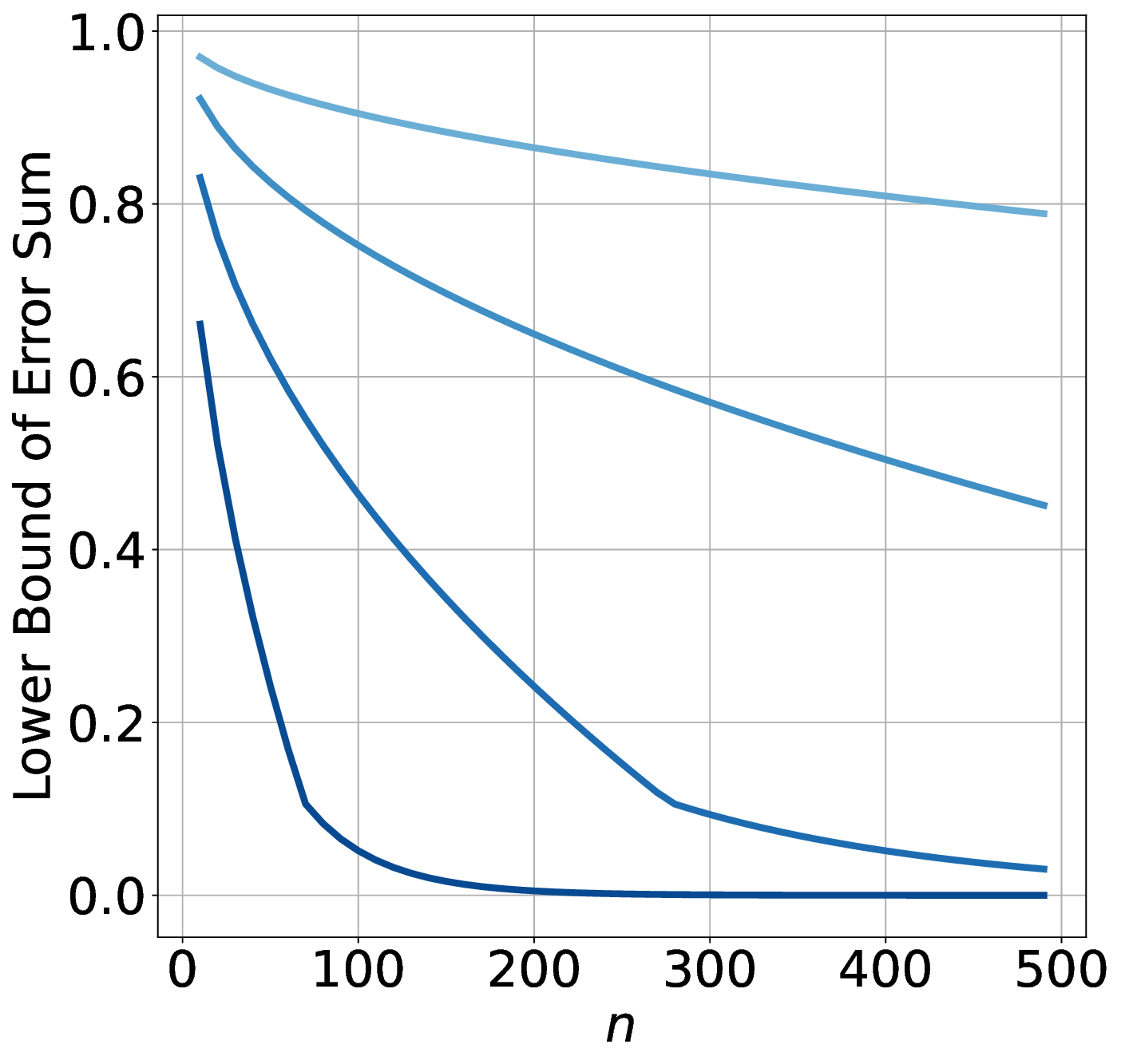}
    \vspace{2mm} {\small Helpsteer, $\eta_1=0.98$}
  \end{minipage}
    \hfill
    \begin{minipage}[b]{0.23\textwidth}
    \centering
\includegraphics[width=\textwidth]
{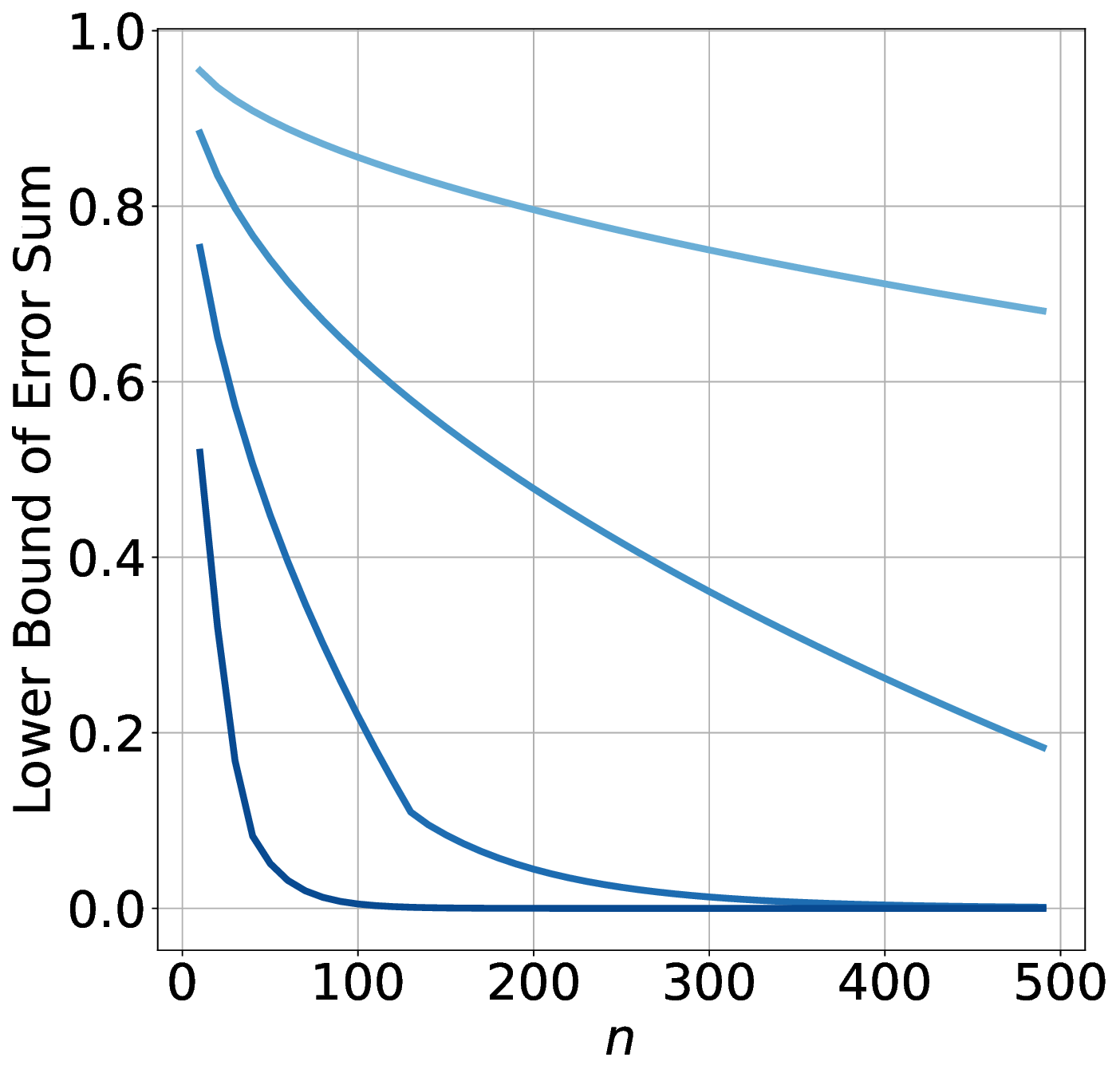}
    \vspace{2mm} {\small Ultra, $\eta_1=0.98$}
  \end{minipage}
  \hfill  
  \begin{minipage}[b]{0.23\textwidth}
    \centering
\includegraphics[width=\textwidth]
{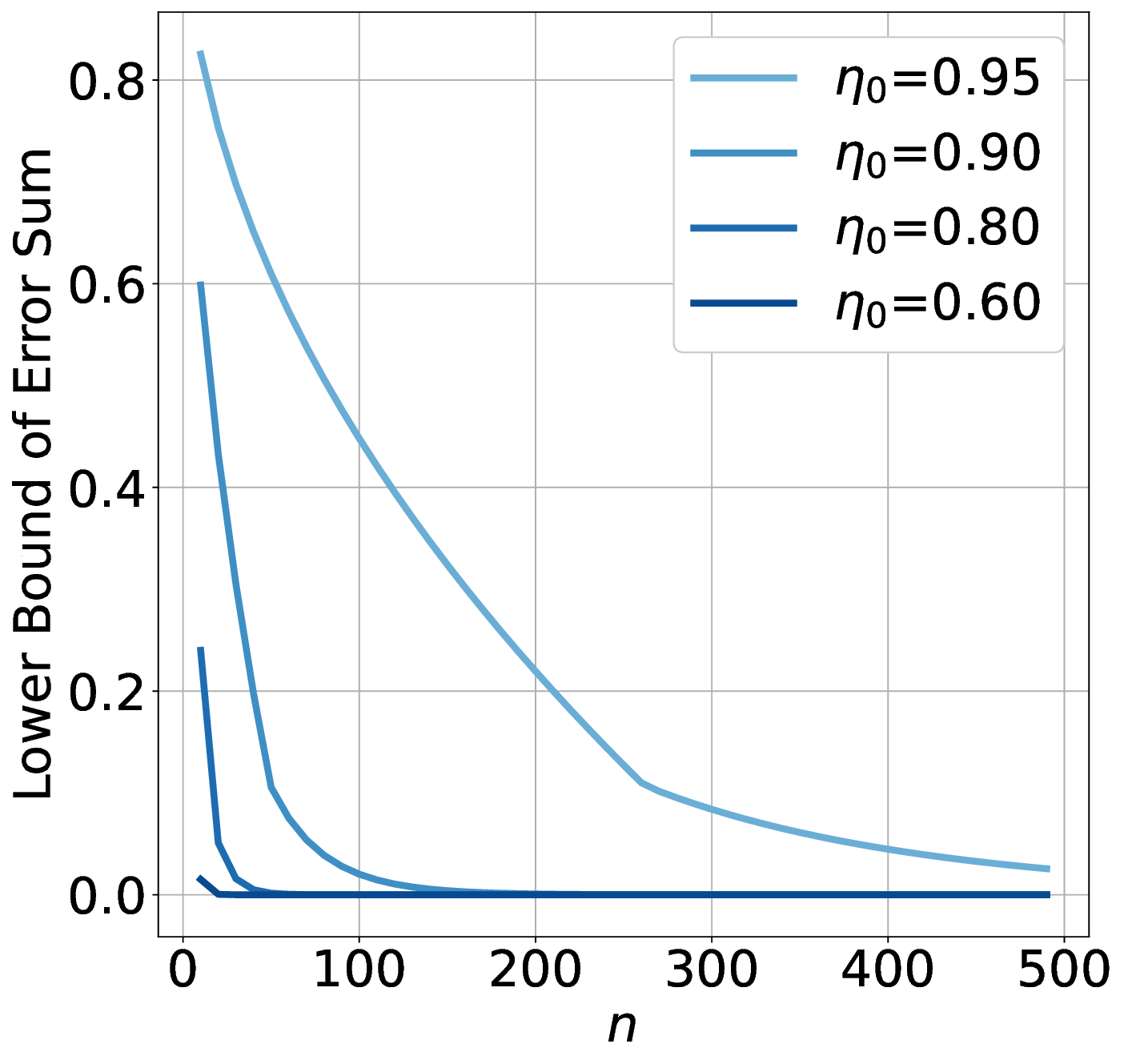}
    \vspace{2mm} {\small Skywork, $\eta_1=0.98$}
  \end{minipage}

  \begin{minipage}[b]{0.23\textwidth}
    \centering
\includegraphics[width=\textwidth]
{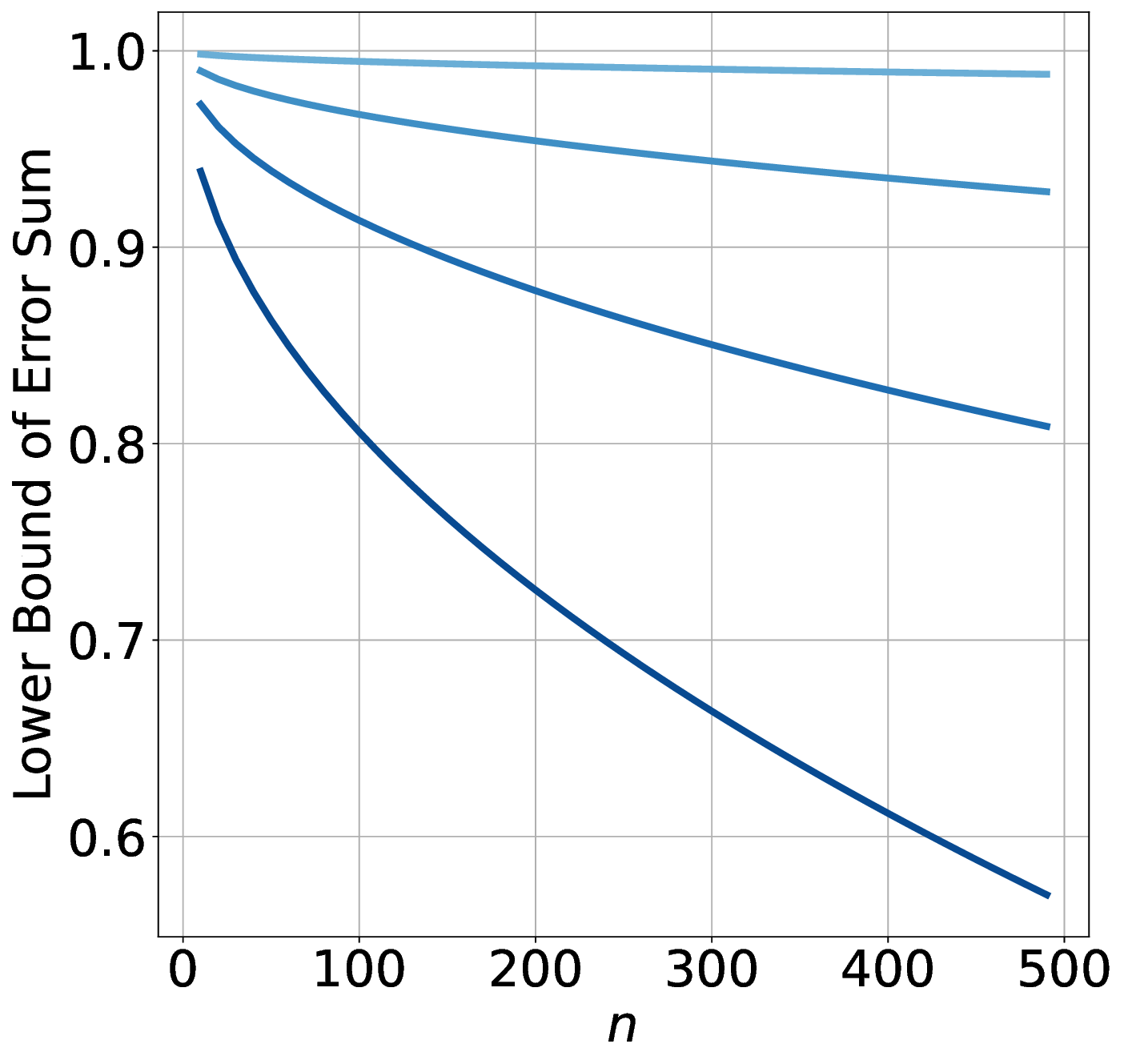}
    \vspace{2mm} {\small PKU, $\eta_1=0.96$}
  \end{minipage}
    \hfill
  \begin{minipage}[b]{0.23\textwidth}
    \centering
\includegraphics[width=\textwidth]
{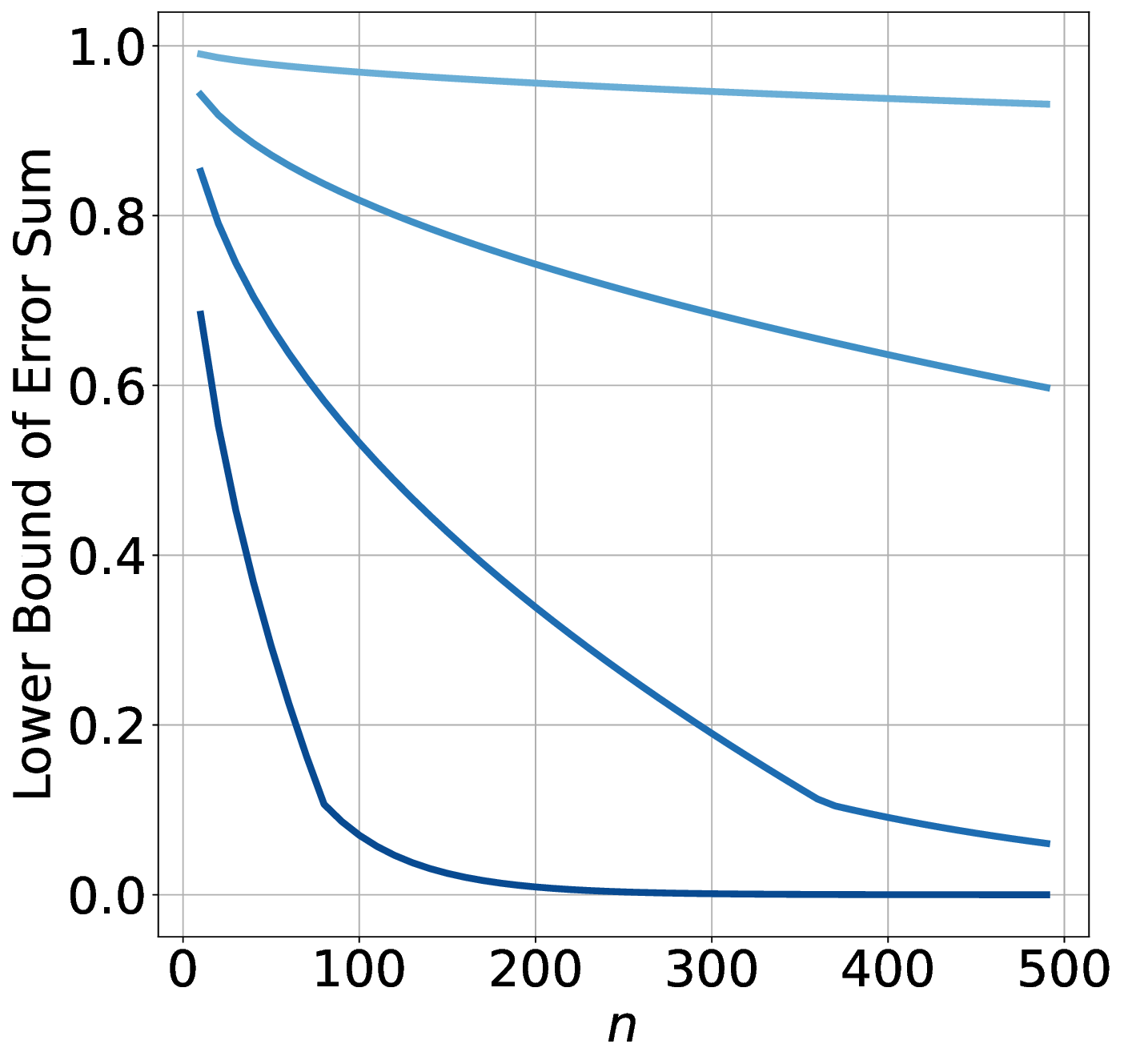}
    \vspace{2mm} {\small Helpsteer, $\eta_1=0.96$}
  \end{minipage}
    \hfill
    \begin{minipage}[b]{0.23\textwidth}
    \centering
\includegraphics[width=\textwidth]
{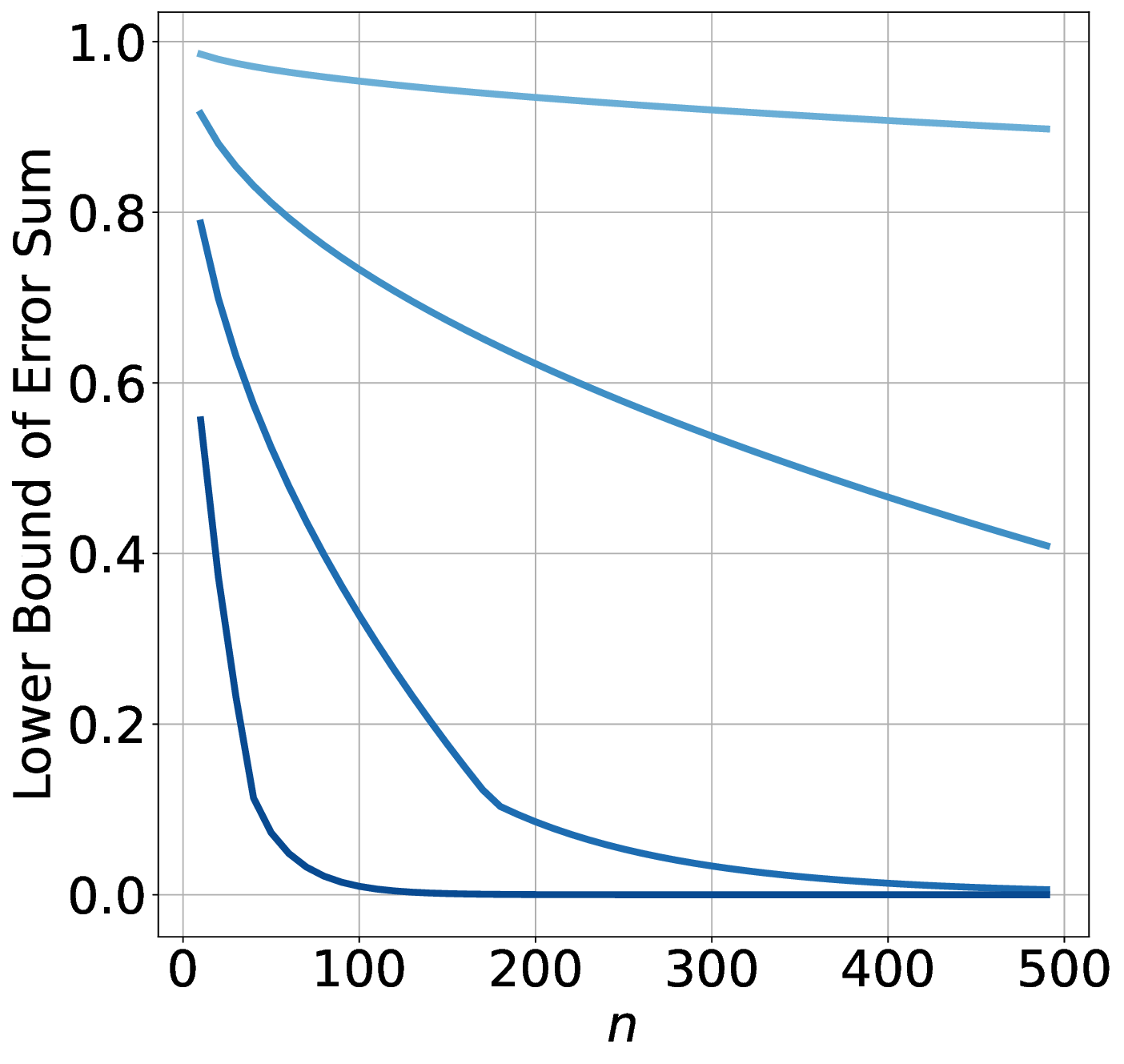}
    \vspace{2mm} {\small Ultra, $\eta_1=0.96$}
  \end{minipage}
  \hfill  
  \begin{minipage}[b]{0.23\textwidth}
    \centering
\includegraphics[width=\textwidth]
{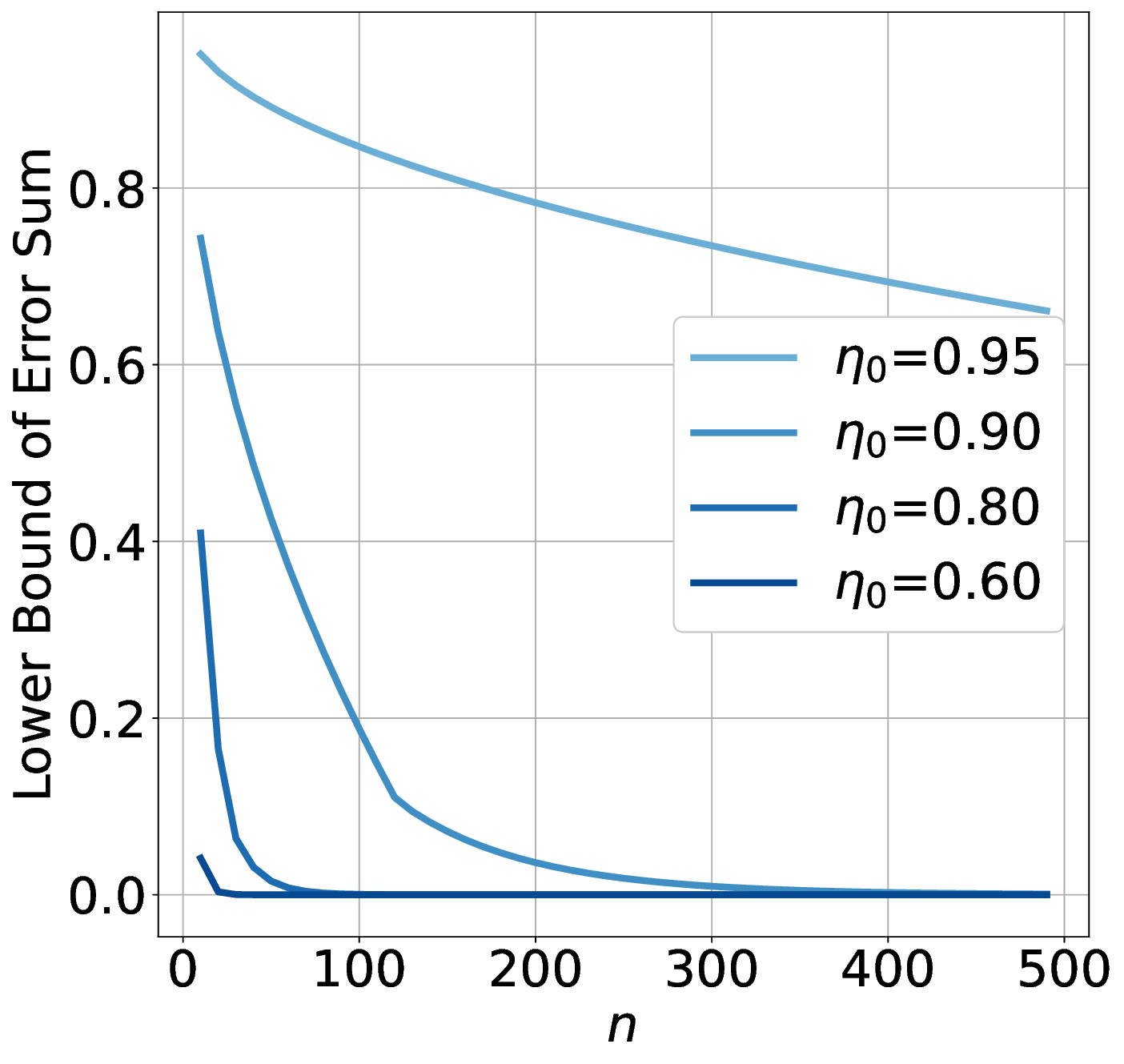}
    \vspace{2mm} {\small Skywork, $\eta_1=0.96$}
  \end{minipage}
  \caption{Additional results for Figure \ref{fig:main_hist_LB}: Lower bound of the sum of two types of errors from Proposition \ref{prop:info_lower_bound} vs. the number of tested annotations $n$ at different  $\eta_0$ with $\eta_1\in \{0.98, 0.96\}$.}
\label{fig:appx_sup_1}
\end{figure}

\subsubsection{Setup and more experiments for Figure \ref{fig:main_self_UB}}\
\label{appx:fig_self_UB}

\medskip
\textbf{Setup:} The lower bound on the two types of errors for the expert-based model is derived in the same manner as in Figure~\ref{fig:main_hist_LB}. For the upper bound of the self-consistency monitoring, we simulate the two types of errors using a likelihood-ratio test with a threshold of 1. Specifically, the test compares  $\mathbb{P}(\bm{A}\mid\delta_0)$ and $\mathbb{P}(\bm{A}\mid\delta_1)$.
If $
\mathbb{P}(\bm{A}\mid\delta_0) \ge \mathbb{P}(\bm{A}\mid\delta_1),$
we decide that $\bm{A}$ is generated from $H_0$, i.e., $\eta\leq \eta_0$
otherwise, we decide that it is generated from $H_1$, i.e., $\eta\geq \eta_1$.

To compute the two types of errors in our experiments, given $n$, $\eta_0$, $\eta_1$, and $\delta$, we perform $M = 10000$ trials as follows. In each trial, we sample $
\bm{A}^0 = \bigl(A^{0}_1, \ldots, A^{0}_n\bigr)$ with each $A_i^0$
i.i.d.\ from a Bernoulli distribution with mean $\eta_0(1-\delta)/2 + 1/2$. We then compute the frequency (over $M$ trials) with which 
$\mathbb{P}(\bm{A}^0\mid\delta_0) < \mathbb{P}(\bm{A}^0\mid\delta_1)$, and designate this frequency as the (simulated) Type-I error. Similarly, we perform another $M = 10000$ trials, where in each trial we sample $\bm{A}^1 = \bigl(A^{1}_1, \ldots, A^{1}_n\bigr)$ with $A^1_i$ i.i.d.\ from a Bernoulli distribution with mean $\eta_1(1-\delta)/2 + 1/2$. We then compute the frequency with which $
\mathbb{P}(\bm{A}^1\mid\delta_0) \ge \mathbb{P}(\bm{A}^1\mid\delta_1)$,
and designate this frequency as the (simulated) Type-II error. The sum of these two error frequencies yields the upper bound on the overall error for the self-consistency monitoring approach.

\textbf{More results:} Figure~\ref{fig:appx_self_UB} complements Figure~\ref{fig:main_self_UB} by further comparing self-consistency monitoring (upper bound) with expert-based monitoring (lower bound) at $\eta_1 = 0.95$. The results indicate that self-consistency monitoring continues to outperform expert-based monitoring over a wide range of $n$, underscoring its potential advantages in various settings.

\begin{figure}[ht!]
  \centering
  \begin{minipage}[b]{0.23\textwidth}
    \centering
\includegraphics[width=\textwidth]
{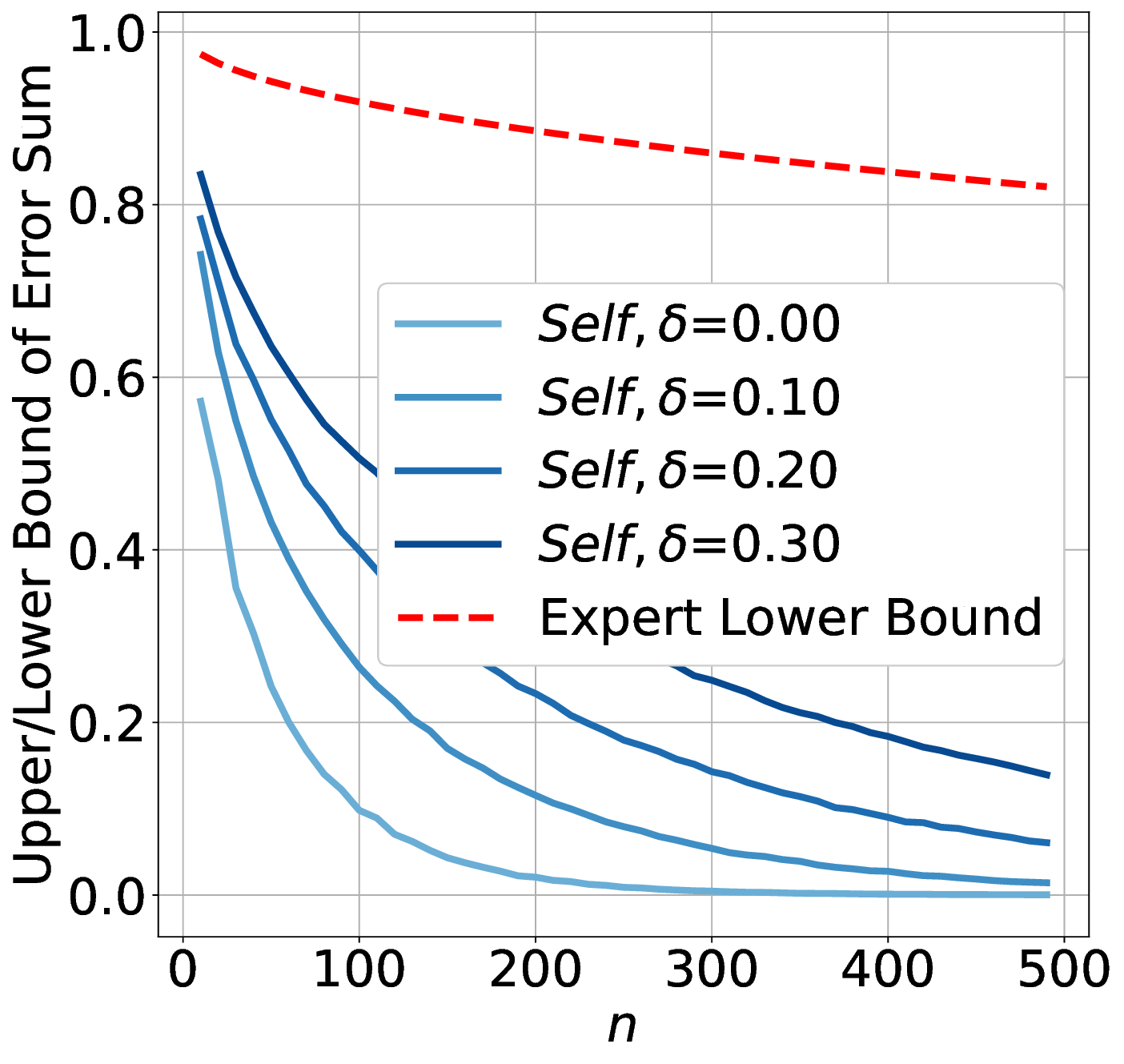}
    \vspace{2mm} {\small PKU, $\eta_0=0.8$}
  \end{minipage}
    \hfill
  \begin{minipage}[b]{0.23\textwidth}
    \centering
\includegraphics[width=\textwidth]
{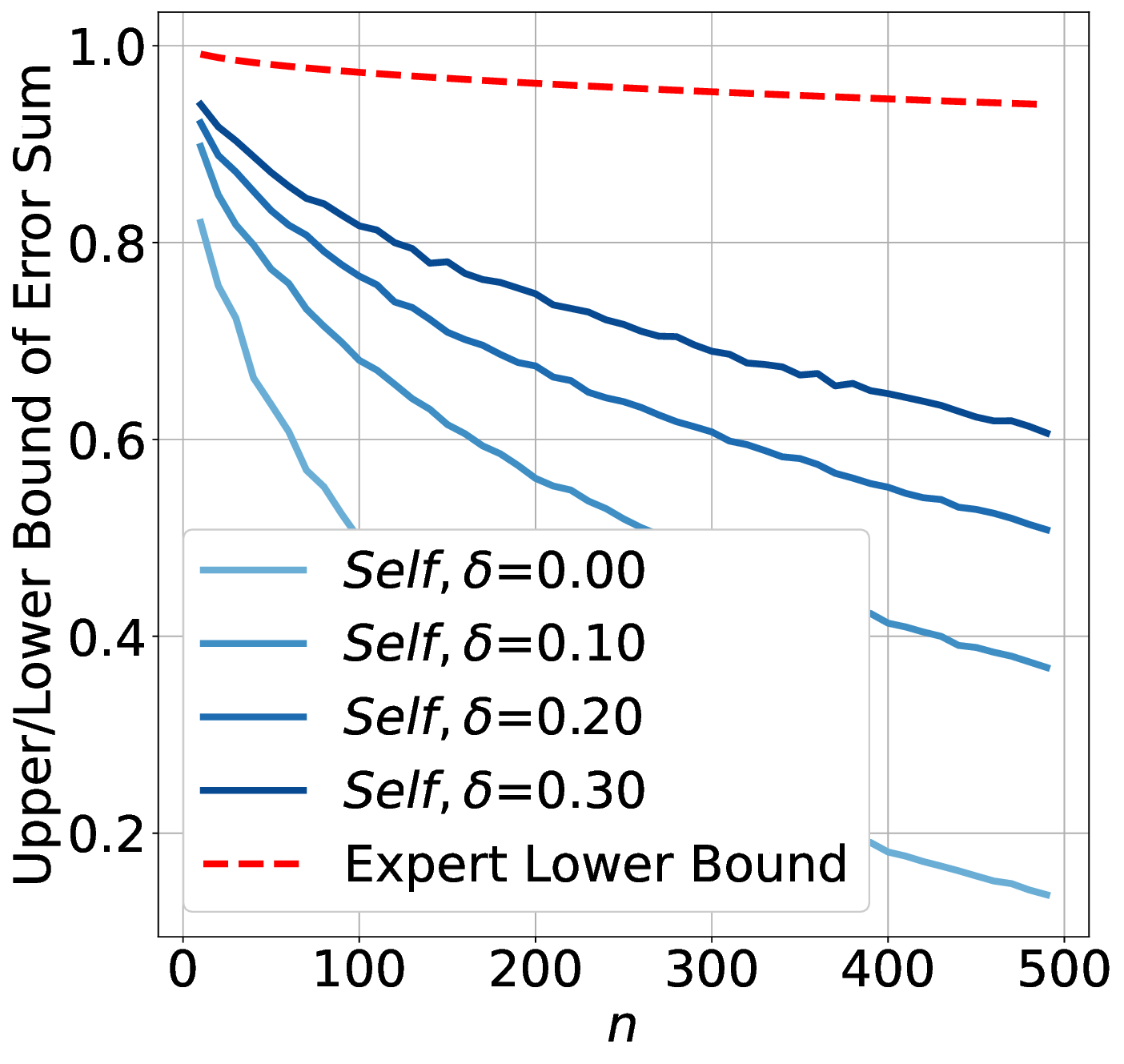}
    \vspace{2mm} {\small PKU, $\eta_0=0.9$}
  \end{minipage}
    \hfill
    \begin{minipage}[b]{0.23\textwidth}
    \centering
\includegraphics[width=\textwidth]
{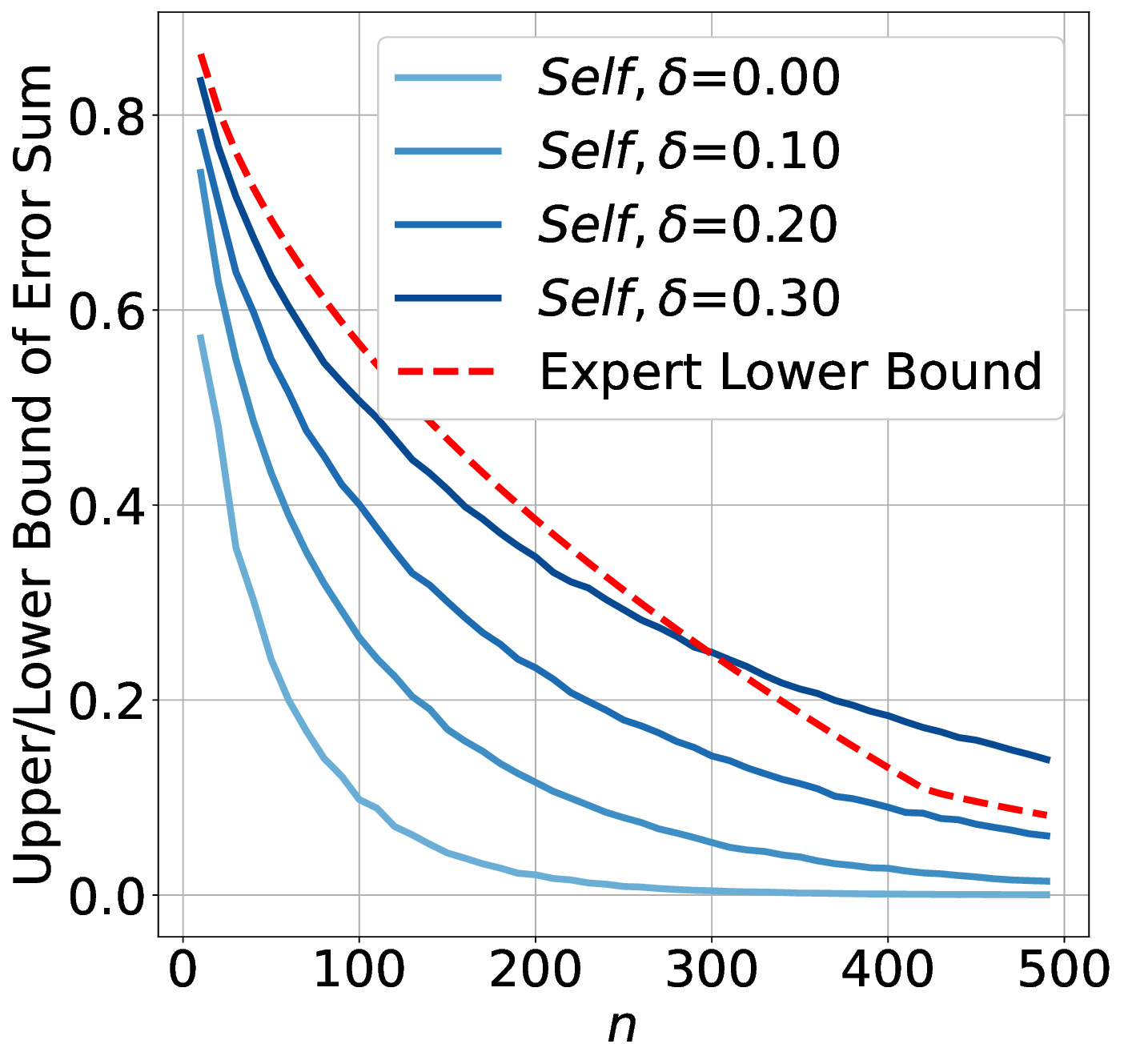}
    \vspace{2mm} {\small Helpsteer, $\eta_0=0.8$}
  \end{minipage}
  \hfill  
  \begin{minipage}[b]{0.23\textwidth}
    \centering
\includegraphics[width=\textwidth]
{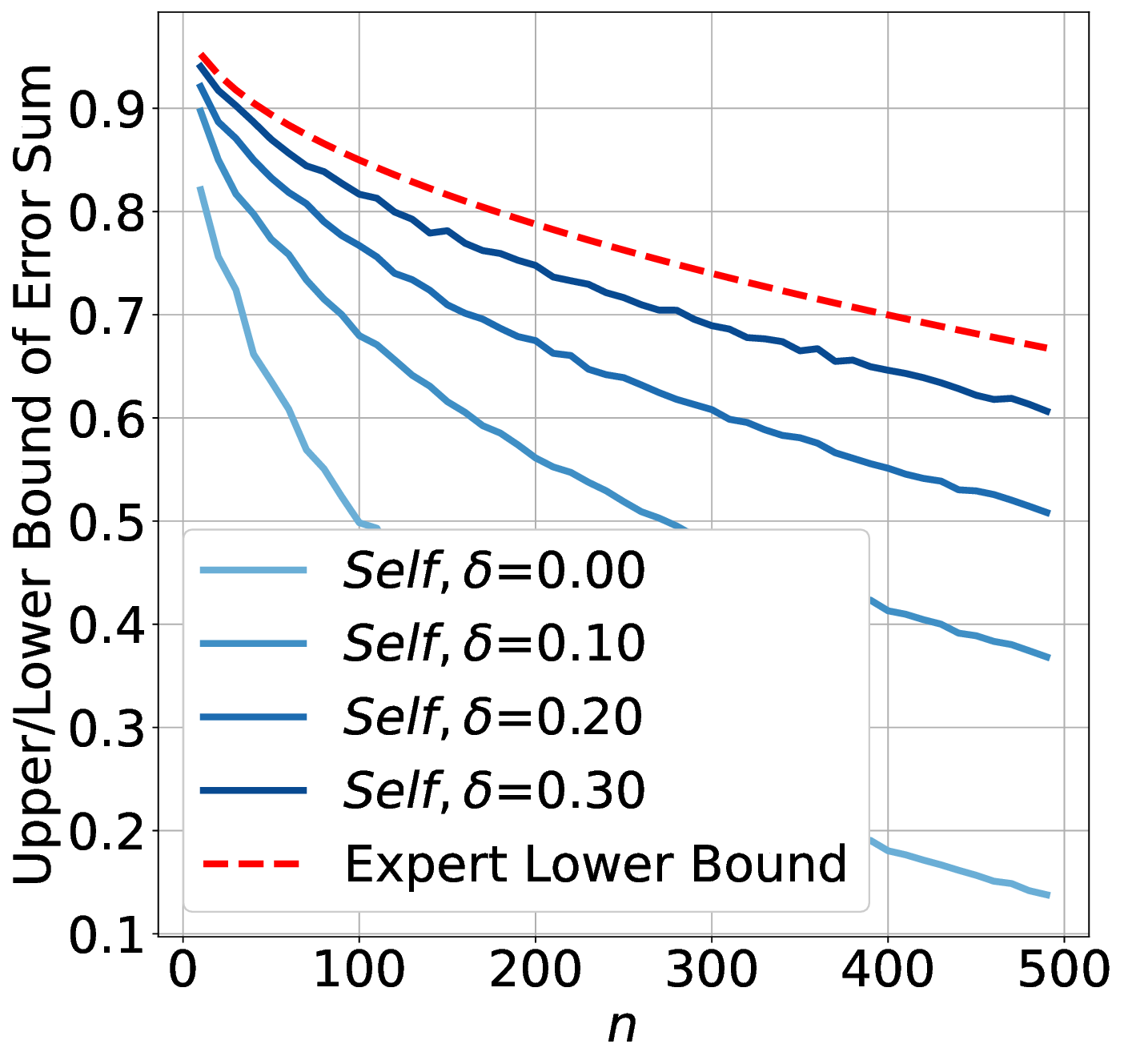}
    \vspace{2mm} {\small Helpsteer, $\eta_0=0.9$}
  \end{minipage}
  \caption{Additional results for Figure \ref{fig:main_self_UB}. Self-consistency monitoring (upper bound) vs. expert-based monitoring (lower bound). For the sum of two types of errors, we plot the \textbf{upper bound} of the self-monitoring with different $\delta$ (blue, thick line) and the \textbf{lower bound} of the expert-based monitoring (red, dashed line), at $\eta_0\in\{0.8,0.9\}$ and $\eta_1=0.95$ for two datasets.}
\label{fig:appx_self_UB}
\end{figure}

\subsubsection{Setup and Additional Experiments for Figure \ref{fig:main_contract_rank}}\
\label{appx:fig_contract_rank}

\medskip
\textbf{Setup:} In Figure \ref{fig:main_contract_rank} (and also Figure \ref{fig:appx_contract_rank}), we approximately solve the optimization problems $\mathcal{C}$, $\mathcal{C}_n$, and $\tilde{\mathcal{C}}_n$ as follows. Given $\delta$, $\mu(\eta)$, $G_a(w_a)$, and $E(\eta)$, we discretize the effort space and the contract space. In particular, we discretize the effort space into
\[
\mathcal{S}_\eta = \{0,0.01,\ldots,1\}.
\]
For the binary contract, we set
\[
\tilde{\mathcal{F}}^{\text{bin}}_n = \Bigl\{ w_0 + w_1\,\mathbbm{1}\Bigl\{\bar{A}\geq \tau\Bigr\} :\, w_0\in \{-10,-9.95,\ldots,9.95,10\},\; w_1\in \{0,0.05,\ldots,10\},\; \tau \in \{0,0.01,\ldots,1\} \Bigr\},
\]
where, by the proof of Theorem~\ref{thm:binary}, the optimal binary contract must follow a threshold contract form as above (although the space should be continuous).

For the linear contract, we set
\[
\tilde{\mathcal{F}}^{\text{lin}}_n = \Bigl\{ w_0 + w_1\,\bar{A} :\, w_0\in \{-10,-9.95,\ldots,9.95,10\},\; w_1\in \{0,0.05,\ldots,10\} \Bigr\}.
\]

We then compute the principal's and the agent's utilities for each pair of effort and contract design, and select the pair that maximizes the principal's utility while satisfying the corresponding constraints in $\mathcal{C}$, $\mathcal{C}_n$, and $\tilde{\mathcal{C}}_n$ as the approximate solution. In our experiments, we observed that the discretization can induce infeasibility in the restricted problem $\tilde{\mathcal{C}}_n$. To address this issue, we relax the constraint to
\[
\eta^* \in \Bigl[-0.01+\argmax_{\eta \in \mathcal{S}_\eta} \mathbb{E}\left[U_a\bigl(F_n(\mathcal{D}_n), \eta\bigr)\right],\; 0.01+\argmax_{\eta \in \mathcal{S}_\eta} \mathbb{E}\left[U_a\bigl(F_n(\mathcal{D}_n), \eta\bigr)\right]\Bigr],
\]
where $\eta^*\in \mathcal{S}_\eta$ is the approximate solution of $\mathcal{C}$. This relaxation further induces a non-decreasing gap between $\mathcal{C}$ and $\tilde{\mathcal{C}}_n$, as shown in our results.

In Figure \ref{fig:main_contract_rank}, we set $U_0=0$, $\delta=0.02$, $\mu(\eta)=\frac{1}{2}\eta^{4/5}$, $G_a(w_a)=1-\exp(-w_a)$, and $E(\eta)=0.18\eta^2$. In Figure \ref{fig:appx_contract_effort}, we plot the corresponding agent utility under the optimal solutions, which shows that the agent's utility closely matches the leisure utility $U_0$. This further numerically validates the binding condition of the leisure utility constraint proved in Proposition \ref{prop:first_best_sol} and Theorems \ref{thm:binary} and \ref{thm:linear_contract}.

In Figure \ref{fig:appx_contract_rank}, we further test alternative settings: $\delta=0$ (top row), $\tilde{\mu}(\eta)=\frac{1}{3}\eta^{4/5}$ (middle row), which has a smaller derivative (and value) than $\mu(\eta)$, and $\tilde{G}_a(w_a)=\frac{1}{2}-\frac{1}{2}\exp\Bigl(-\frac{1}{2}w_a\Bigr)$, representing a derivative that is less than that of $G_a$, i.e., the agent is less sensitive to the payment. The results continue to match our theoretical results and demonstrate the advantage of the self-consistency monitoring over the expert-based approach, especially for the PKU dataset.

\begin{figure}[ht!]
  \centering
  \begin{minipage}[b]{0.23\textwidth}
    \centering
\includegraphics[width=\textwidth]
{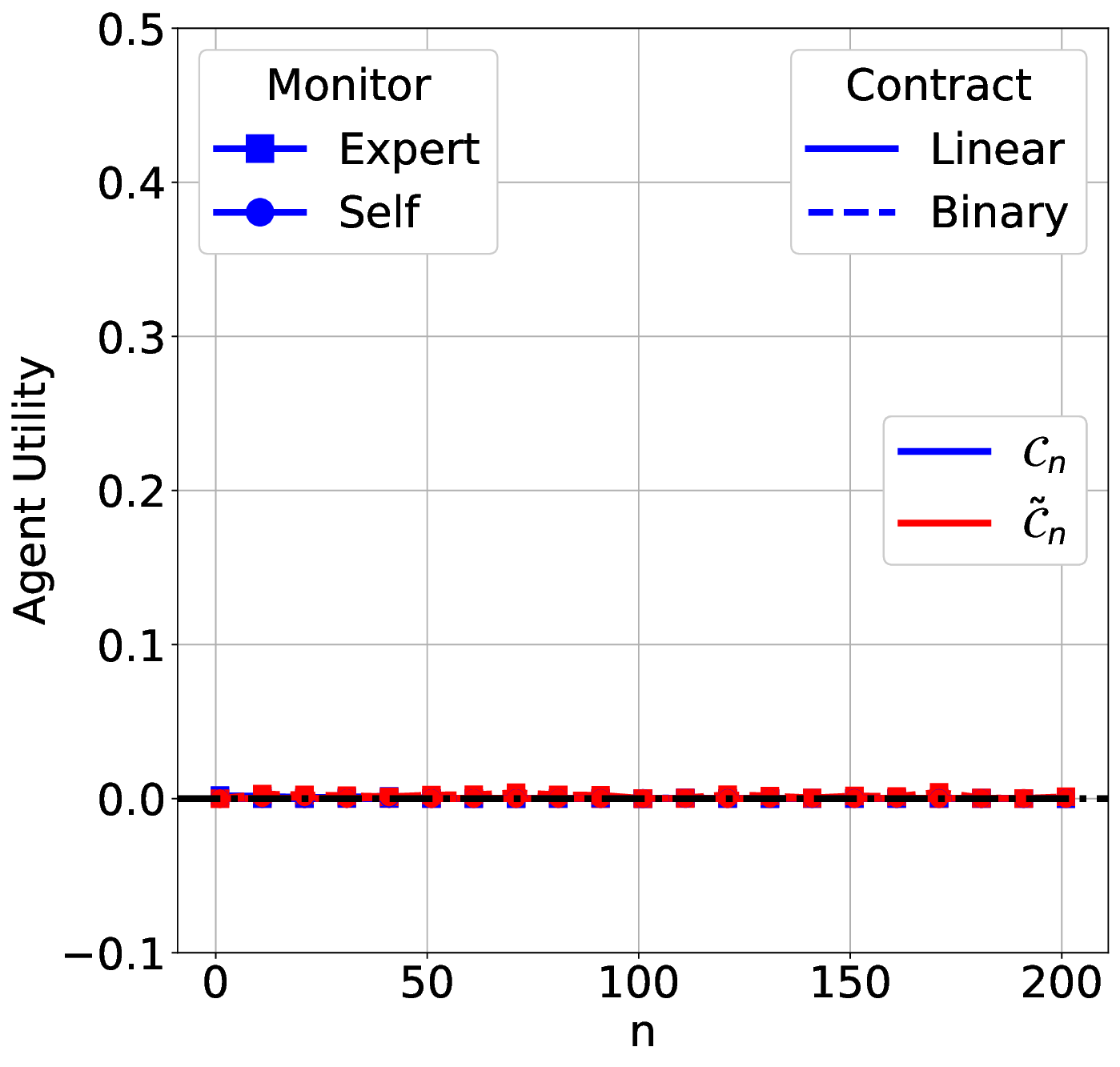}
\vspace{2mm} {\small PKU}
  \end{minipage}
    \hfill
  \begin{minipage}[b]{0.23\textwidth}
    \centering
\includegraphics[width=\textwidth]
{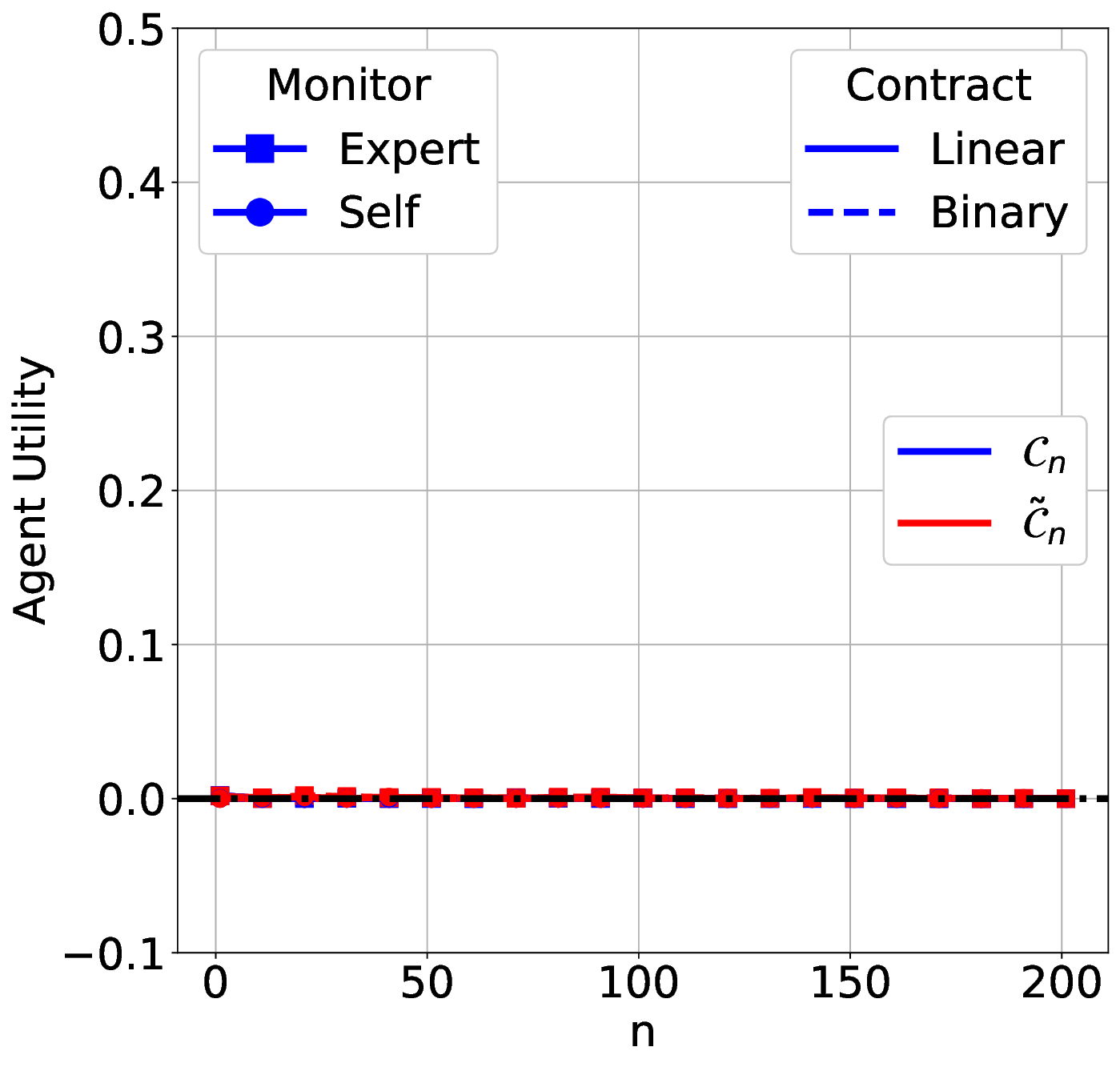}
\vspace{2mm} {\small Helpsteer}
  \end{minipage}
    \hfill
    \begin{minipage}[b]{0.23\textwidth}
    \centering
\includegraphics[width=\textwidth]
{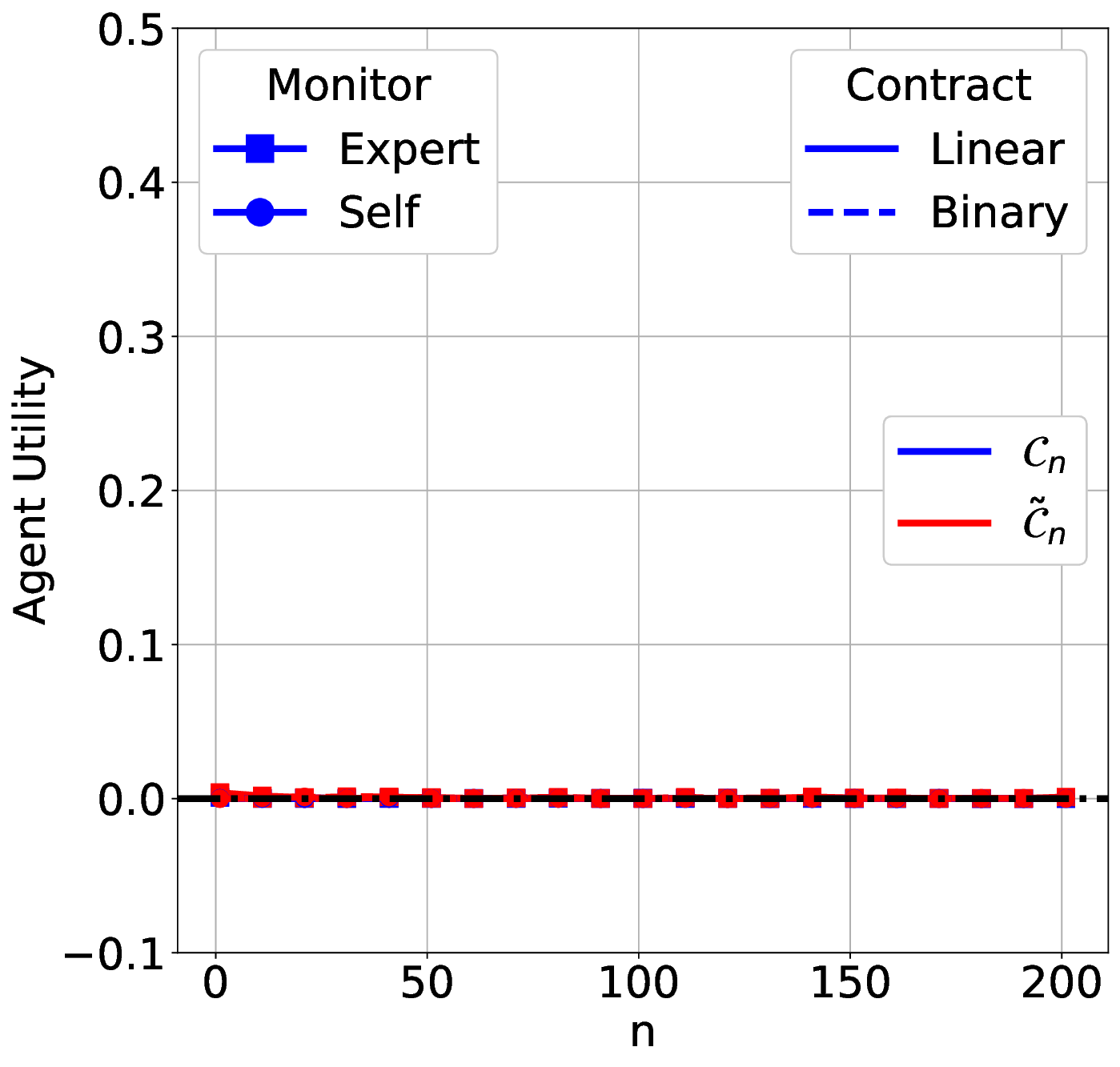}
\vspace{2mm} {\small Ultra}
  \end{minipage}
  \hfill  
  \begin{minipage}[b]{0.23\textwidth}
    \centering
\includegraphics[width=\textwidth]
{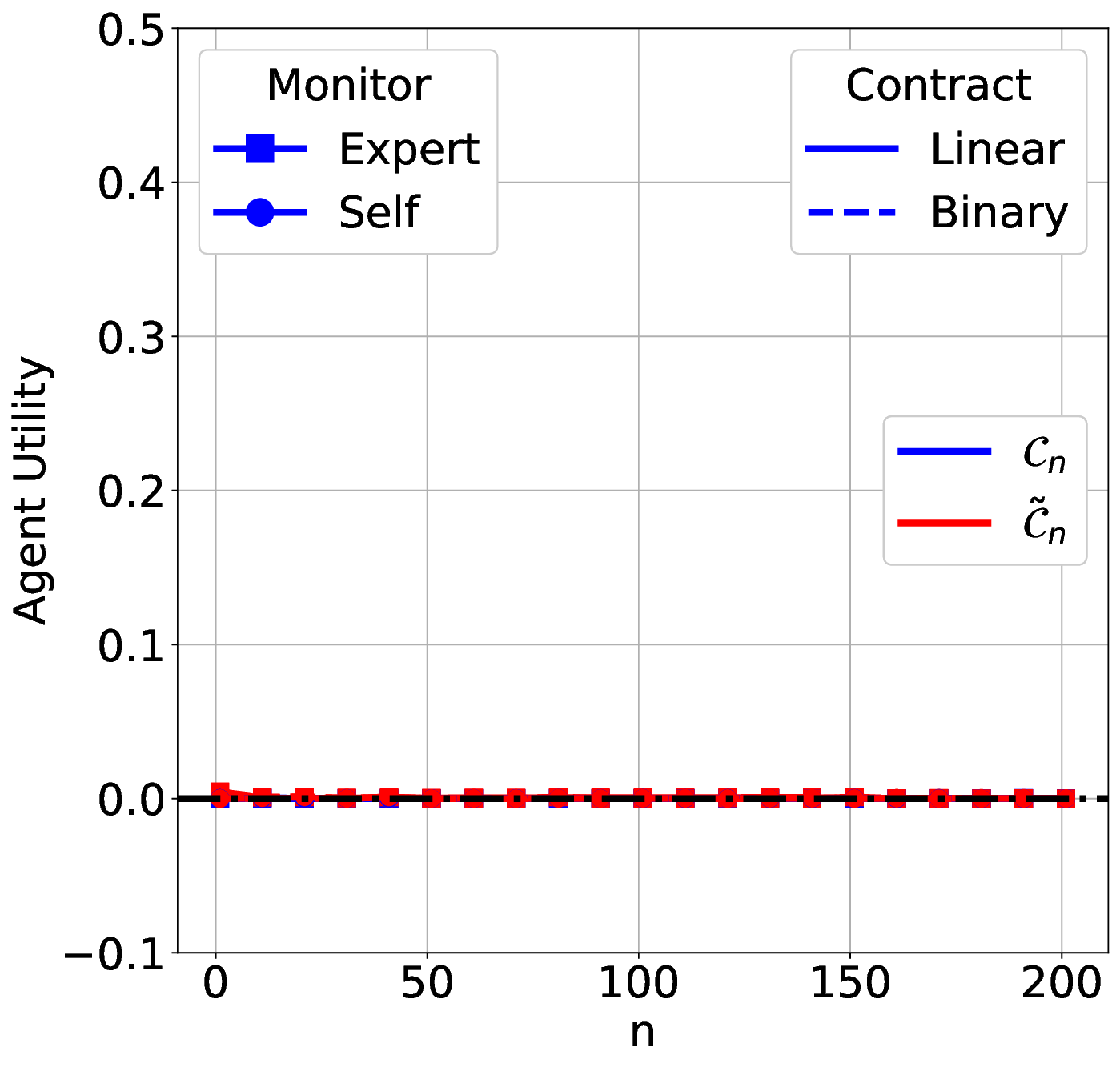}
\vspace{2mm} {\small Skywork}
  \end{minipage}
  \caption{Agent utility under the optimal solution, where we set the leisure utility $U_0=0$. For all datasets, monitoring method, contract type, and second-best formulation, the resulting agent utility matches the leisure utility, i.e., the corresponding constraint is binding.}
\label{fig:appx_contract_effort}
\end{figure}

\begin{figure}[ht!]
  \centering
  \begin{minipage}[b]{0.23\textwidth}
    \centering
\includegraphics[width=\textwidth]
{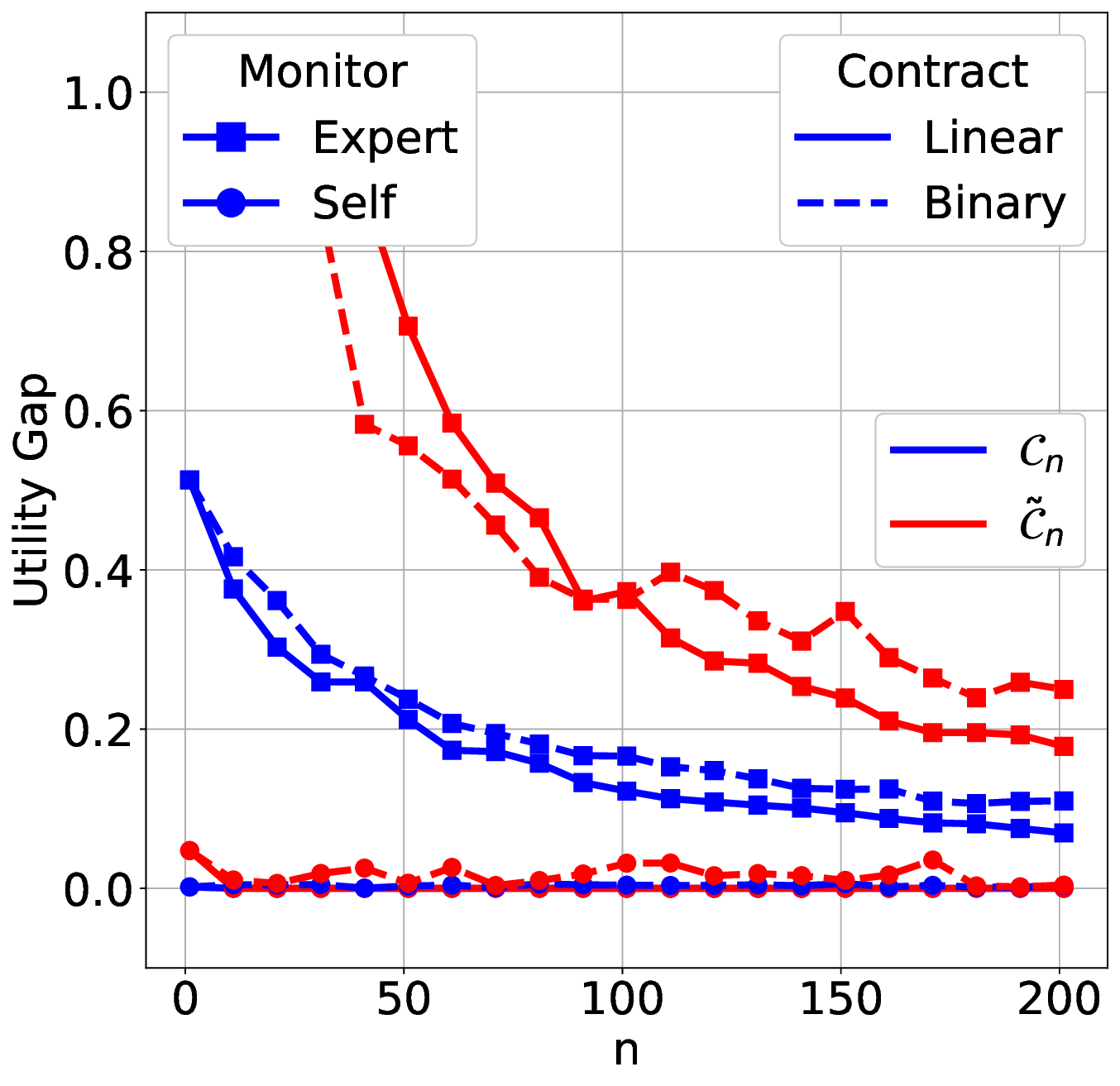}
\vspace{2mm} {\small PKU, $\delta=0$}
  \end{minipage}
    \hfill
  \begin{minipage}[b]{0.23\textwidth}
    \centering
\includegraphics[width=\textwidth]
{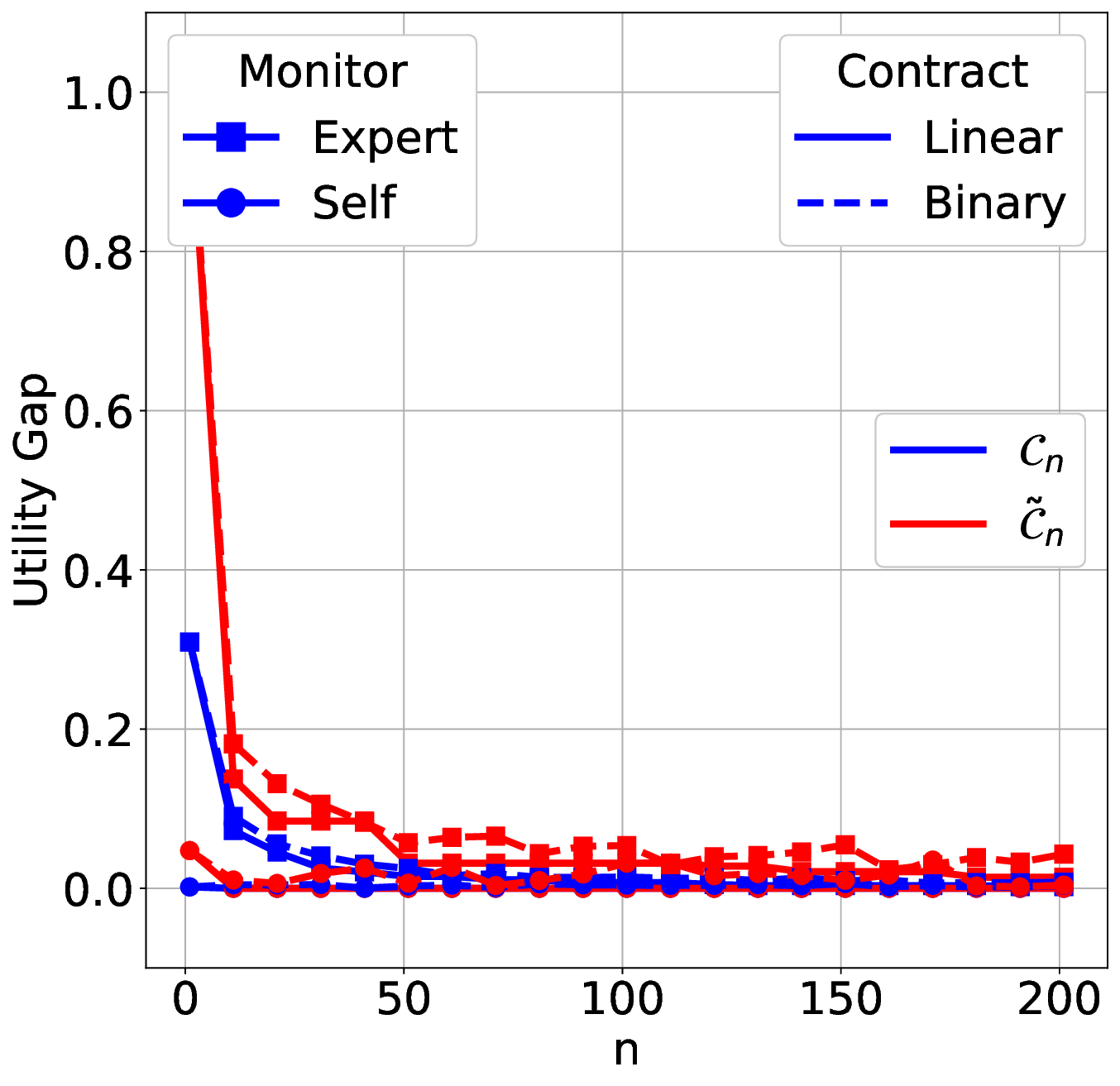}
\vspace{2mm} {\small Helpsteer, $\delta=0$}
  \end{minipage}
    \hfill
    \begin{minipage}[b]{0.23\textwidth}
    \centering
\includegraphics[width=\textwidth]
{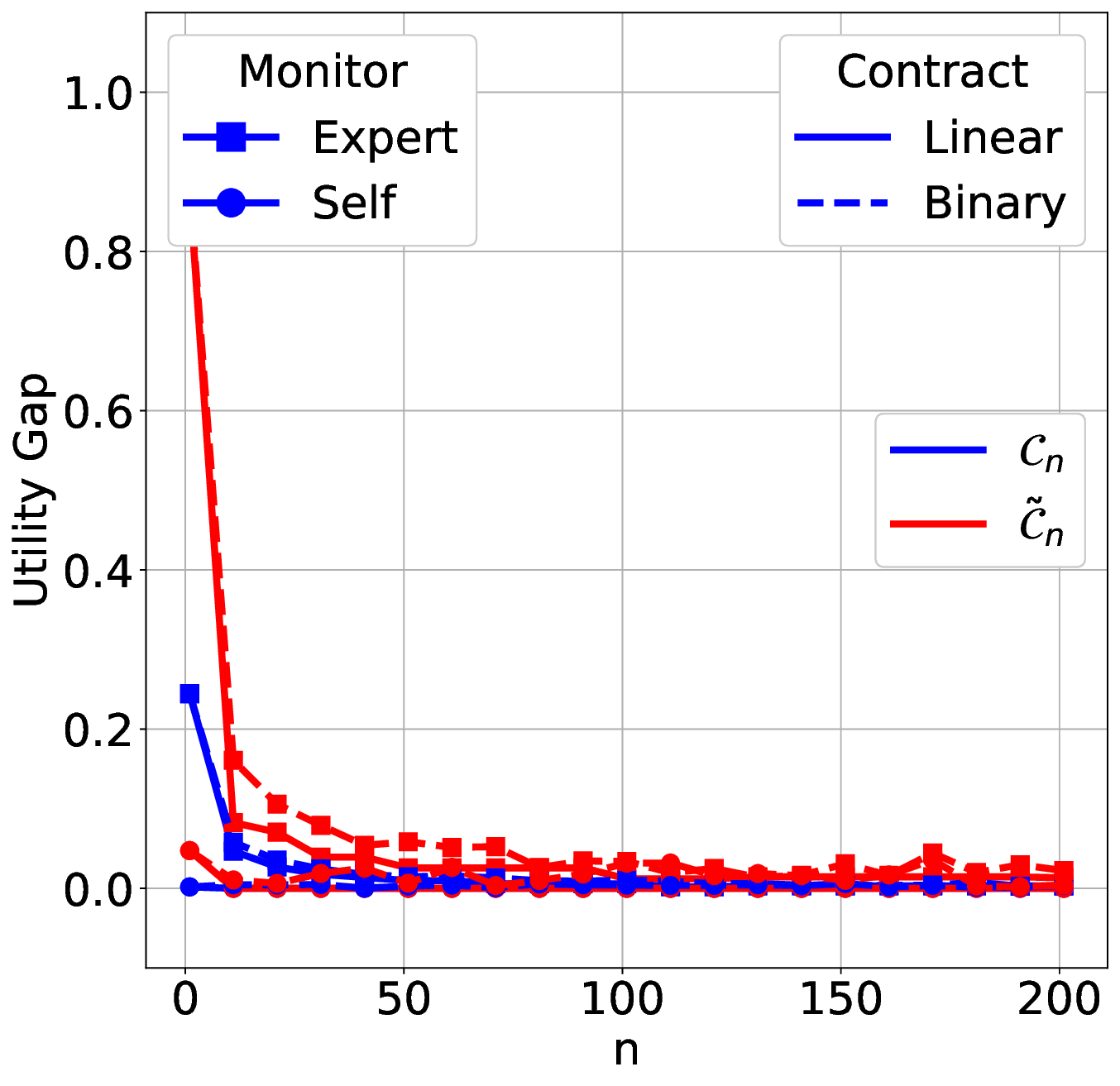}
 \vspace{2mm} {\small Ultra, $\delta=0$}
  \end{minipage}
  \hfill  
  \begin{minipage}[b]{0.23\textwidth}
    \centering
\includegraphics[width=\textwidth]
{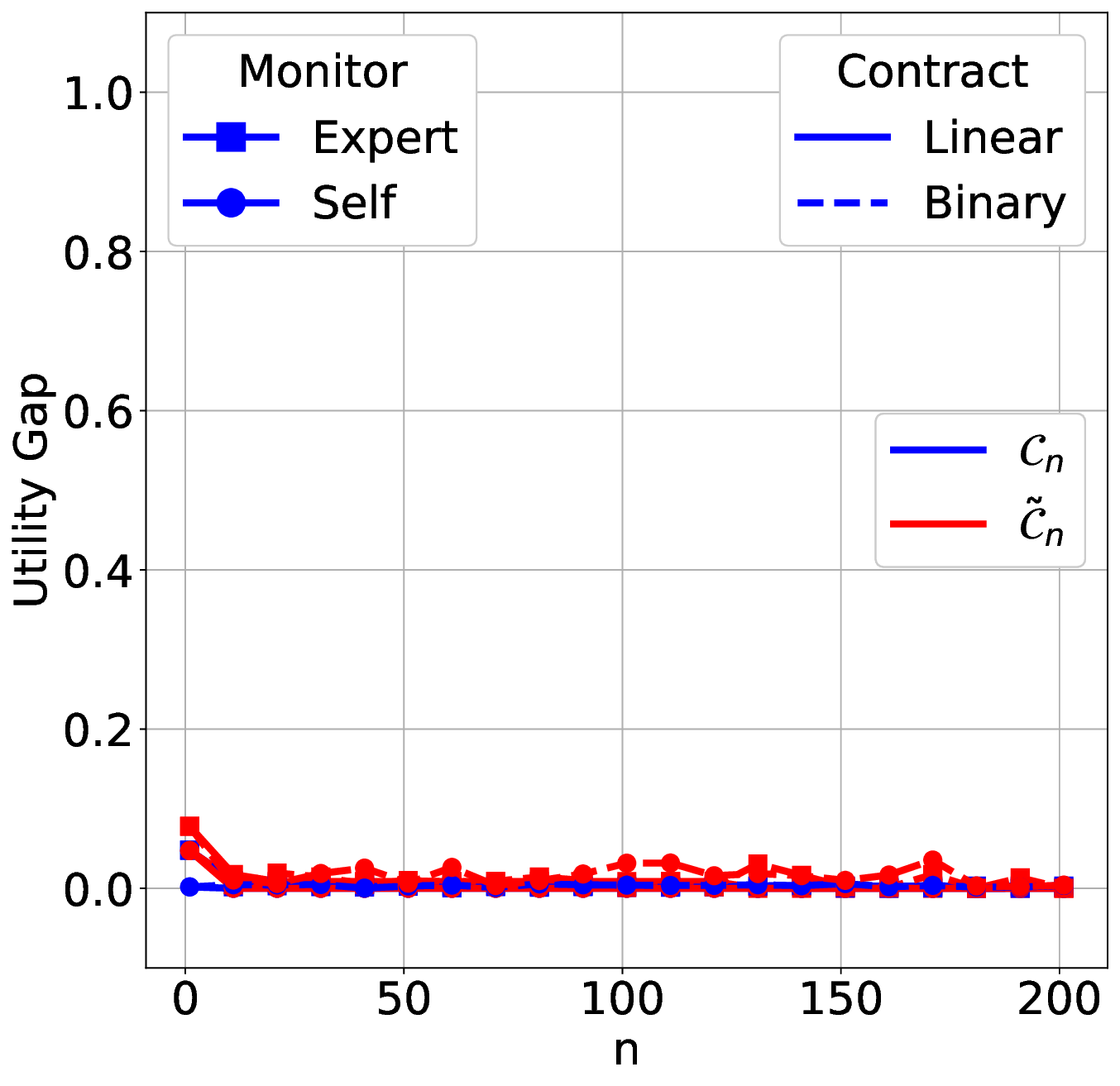}
\vspace{2mm} {\small Skywork, $\delta=0$}
  \end{minipage}
 \begin{minipage}[b]{0.23\textwidth}
    \centering
\includegraphics[width=\textwidth]
{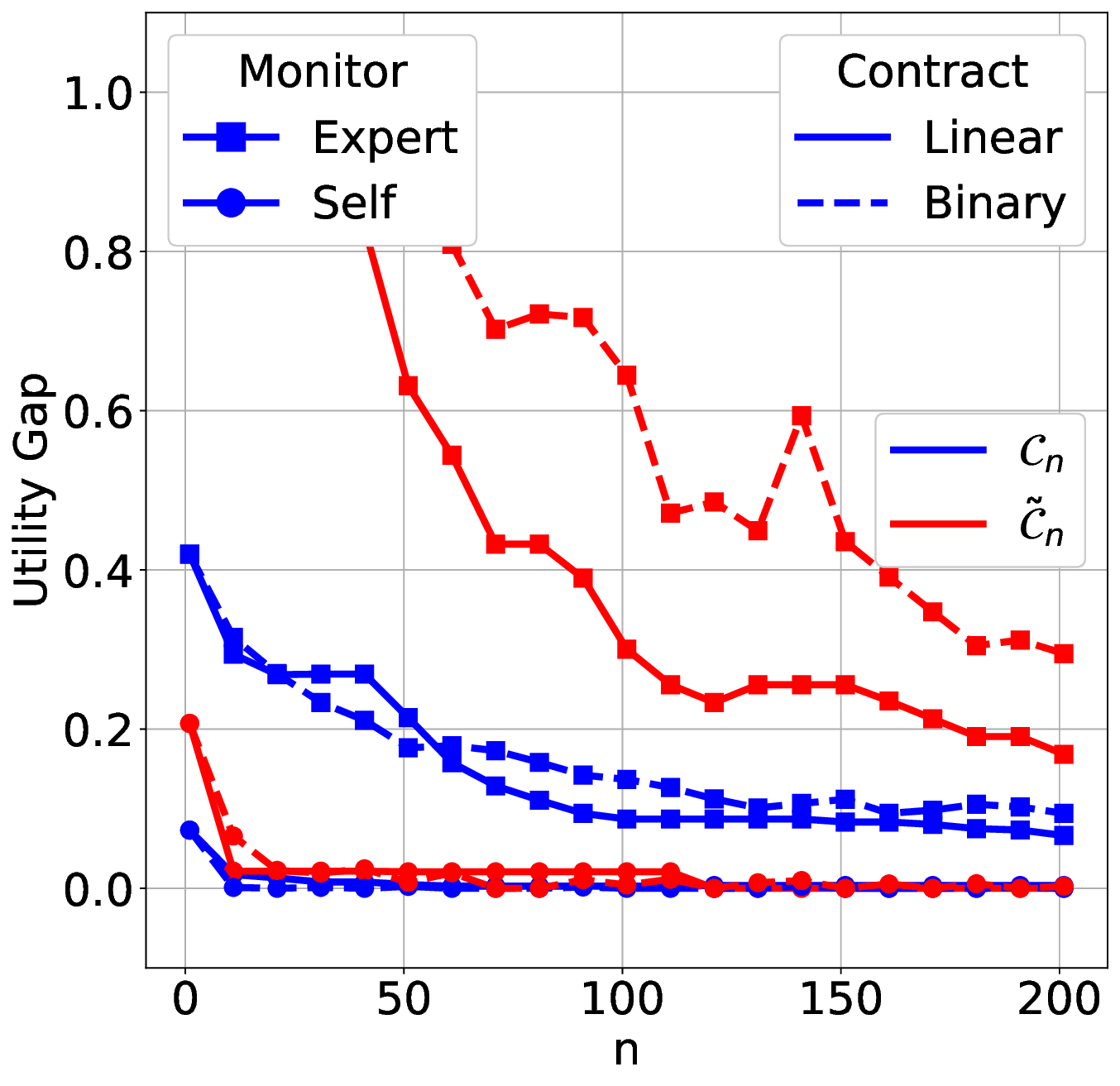}
\vspace{2mm} {\small PKU, $\tilde{\mu}(\eta)$}
  \end{minipage}
    \hfill
  \begin{minipage}[b]{0.23\textwidth}
    \centering
\includegraphics[width=\textwidth]
{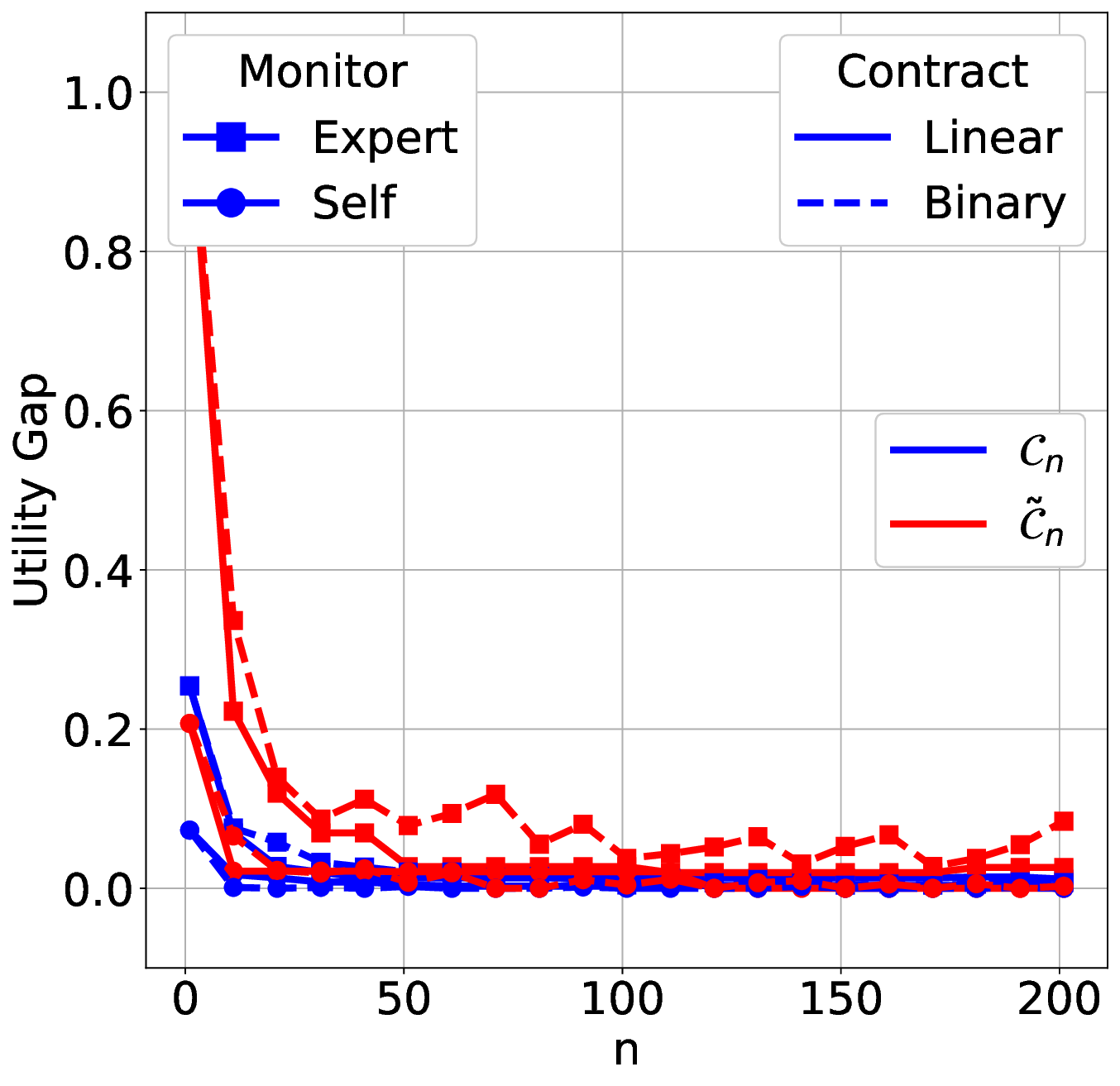}
\vspace{2mm} {\small Helpsteer, $\tilde{\mu}(\eta)$}
  \end{minipage}
    \hfill
    \begin{minipage}[b]{0.23\textwidth}
    \centering
\includegraphics[width=\textwidth]
{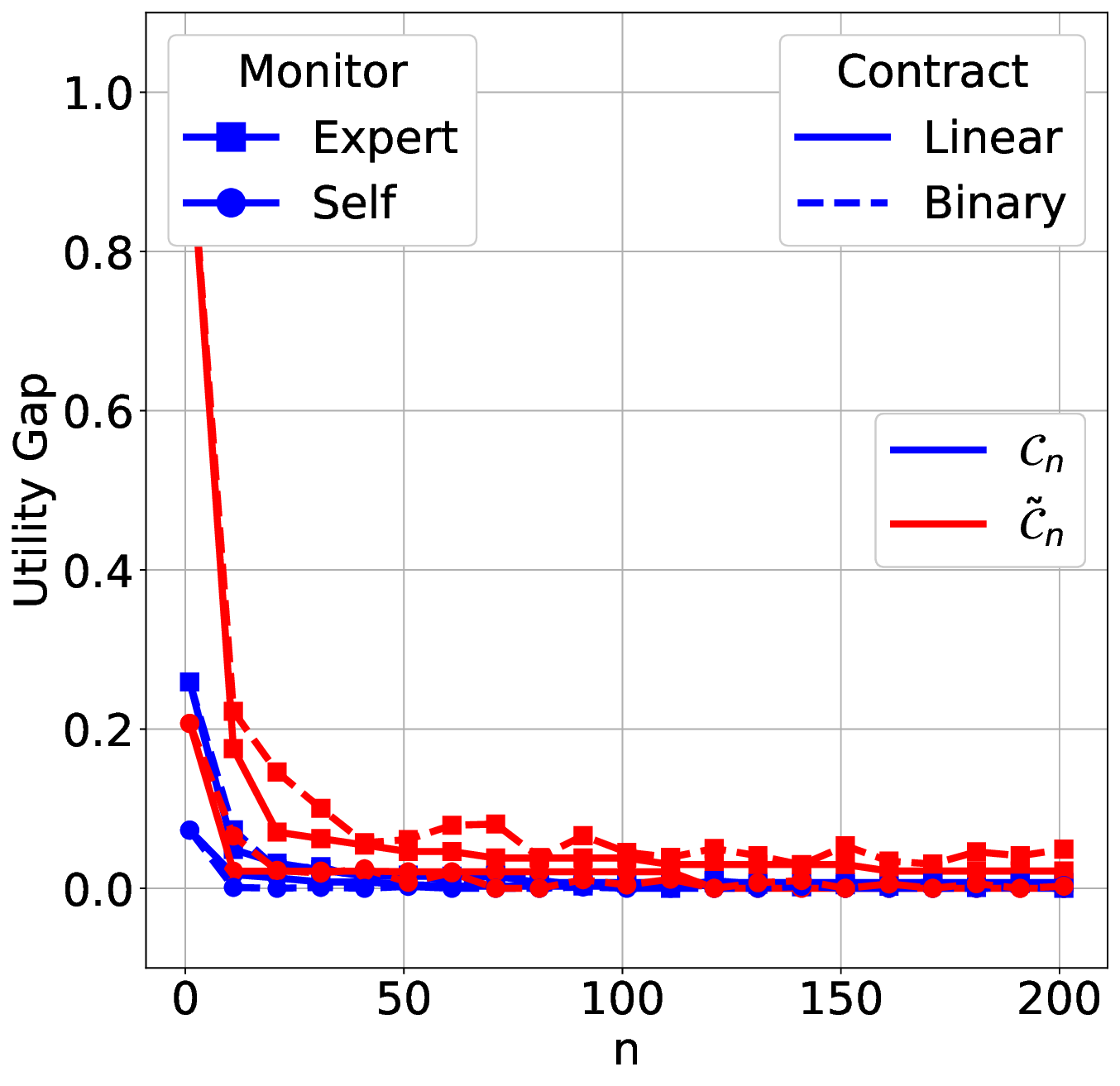}
\vspace{2mm} {\small Ultra, $\tilde{\mu}(\eta)$}
  \end{minipage}
  \hfill  
  \begin{minipage}[b]{0.23\textwidth}
    \centering
\includegraphics[width=\textwidth]
{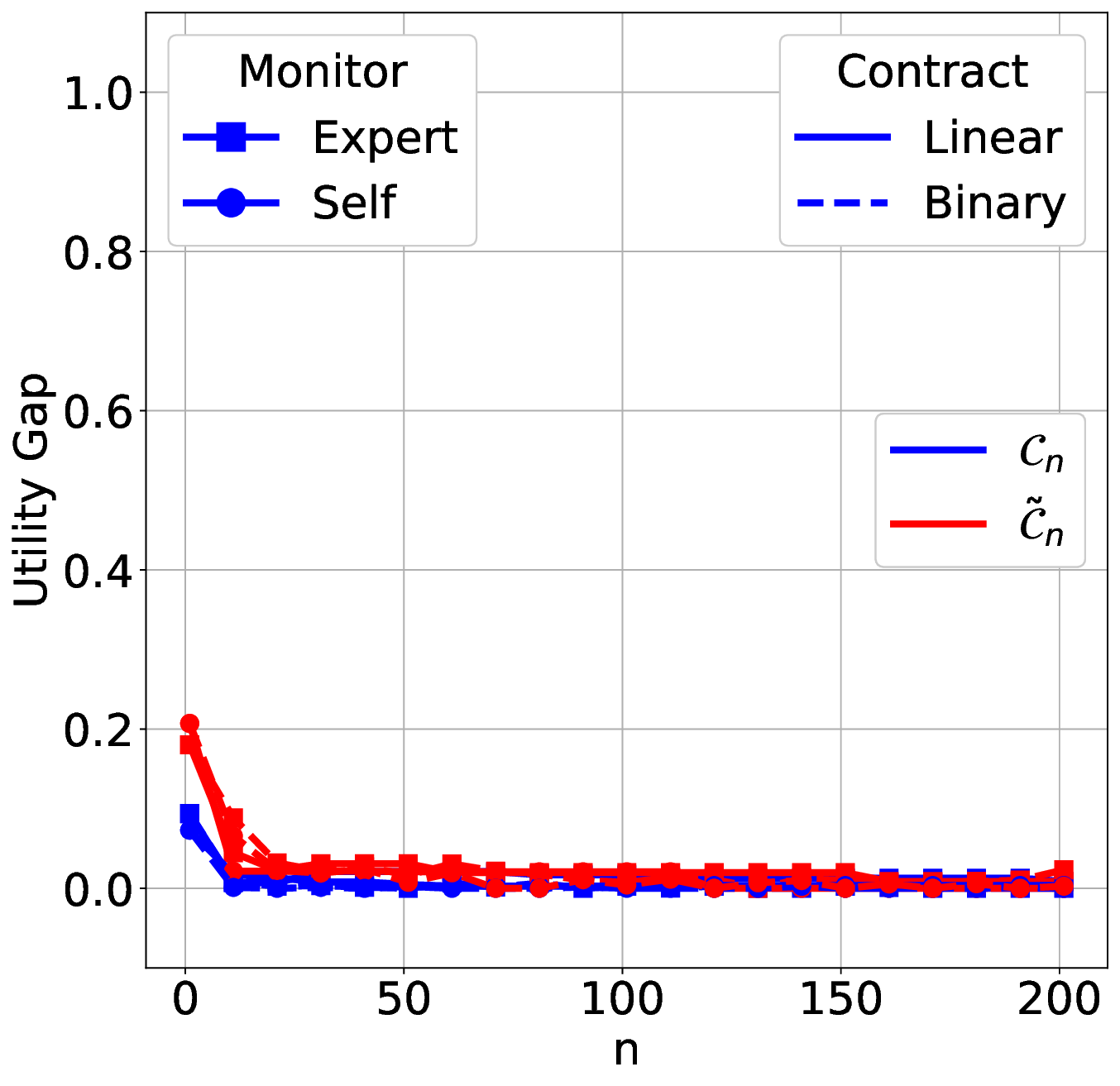}
\vspace{2mm} {\small Skywork, $\tilde{\mu}(\eta)$}
  \end{minipage}
  \begin{minipage}[b]{0.23\textwidth}
    \centering
\includegraphics[width=\textwidth]
{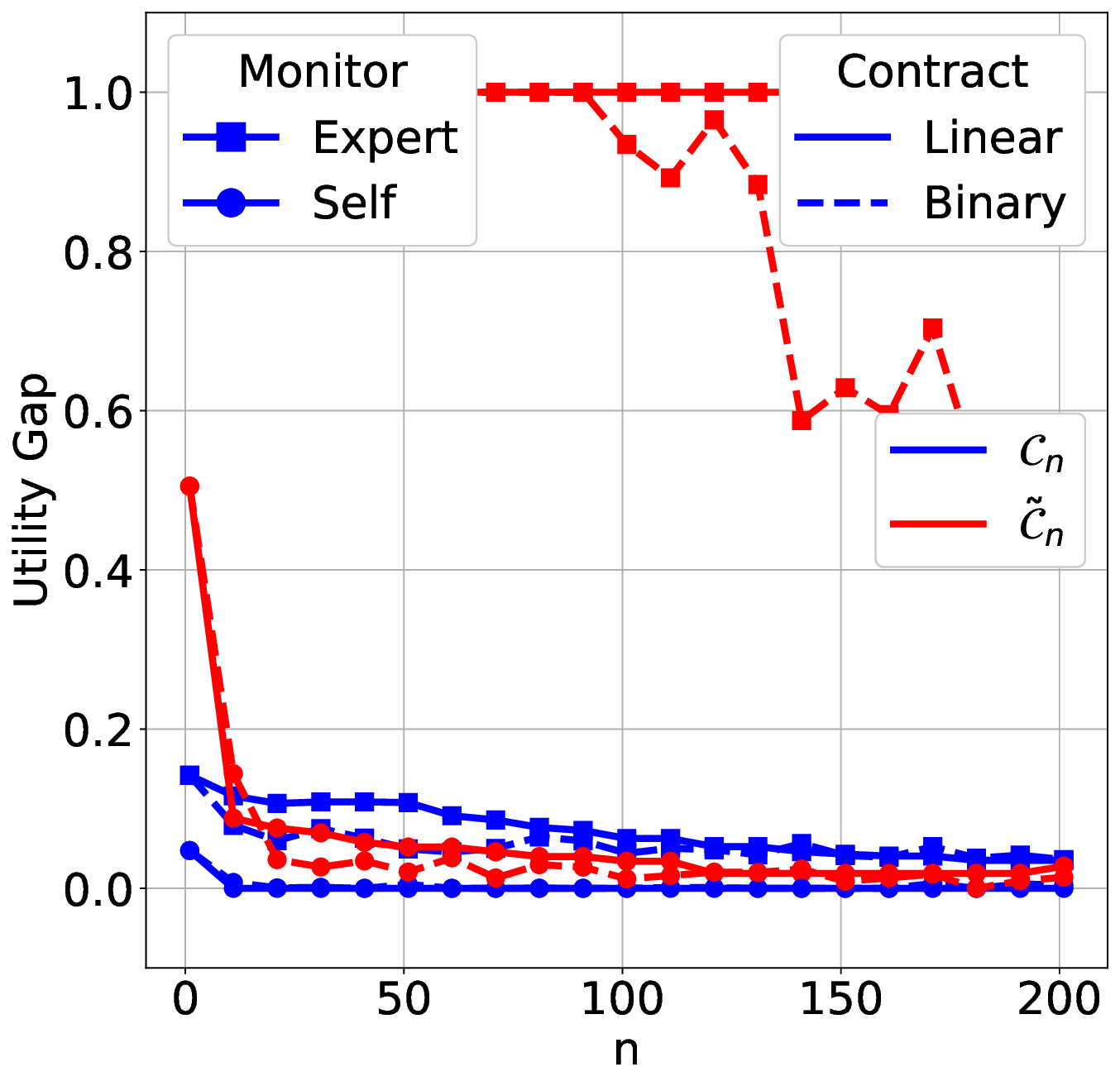}
\vspace{2mm} {\small PKU, $\tilde{G}_a(w_a)$}
  \end{minipage}
    \hfill
  \begin{minipage}[b]{0.23\textwidth}
    \centering
\includegraphics[width=\textwidth]
{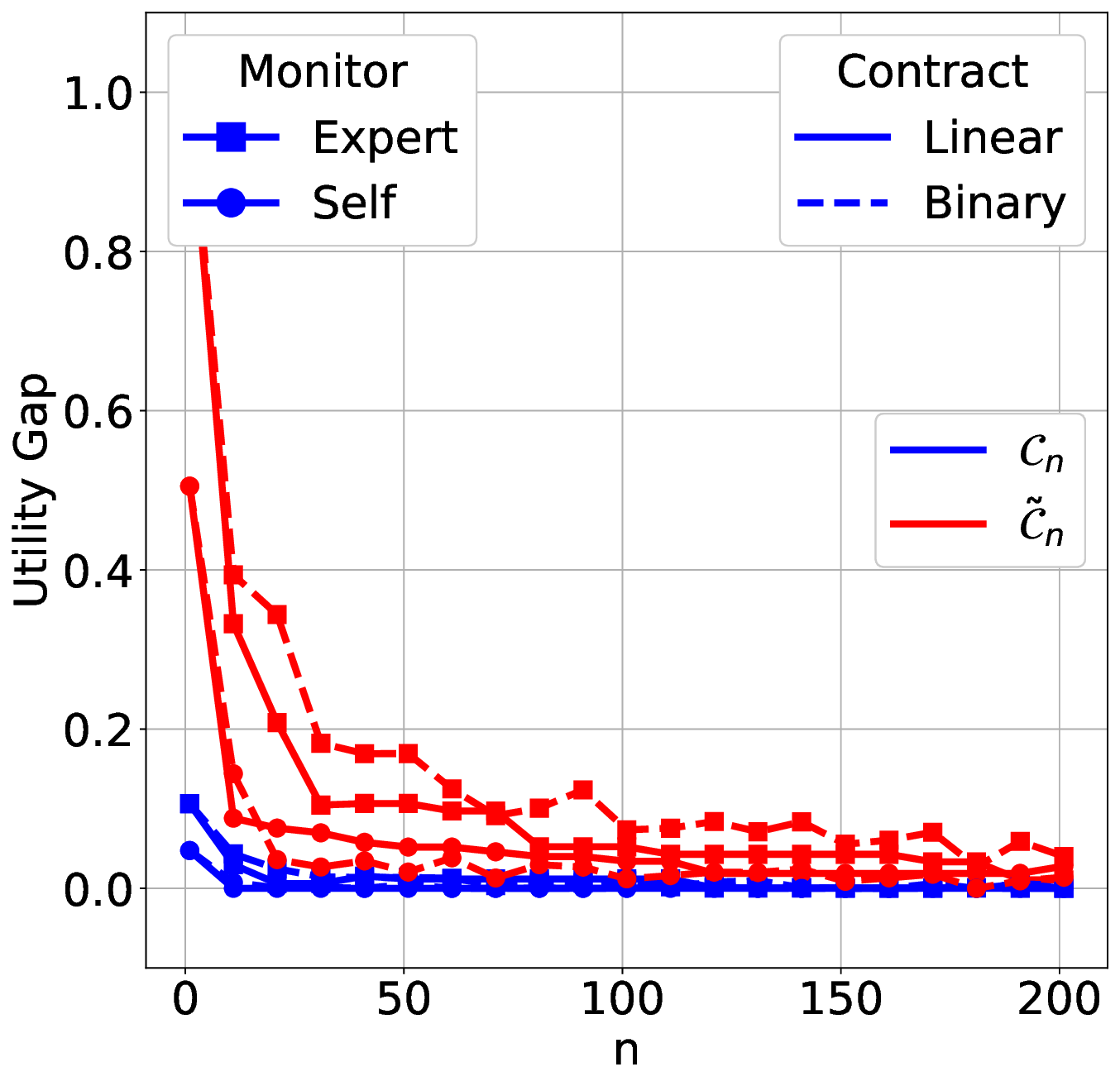}
\vspace{2mm} {\small Helpsteer, $\tilde{G}_a(w_a)$}
  \end{minipage}
    \hfill
    \begin{minipage}[b]{0.23\textwidth}
    \centering
\includegraphics[width=\textwidth]
{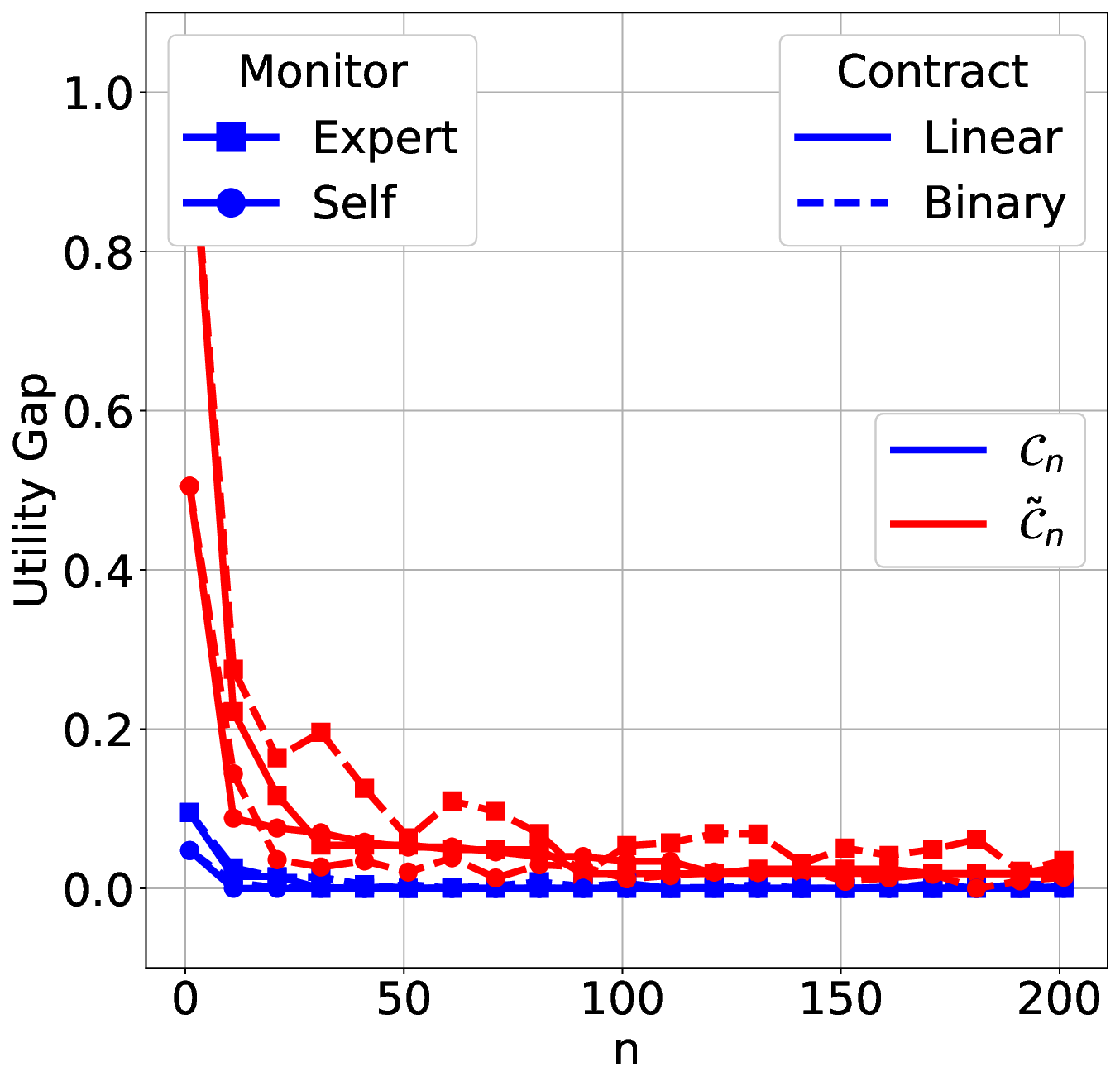}
\vspace{2mm} {\small Ultra, $\tilde{G}_a(w_a)$}
  \end{minipage}
  \hfill  
  \begin{minipage}[b]{0.23\textwidth}
    \centering
\includegraphics[width=\textwidth]
{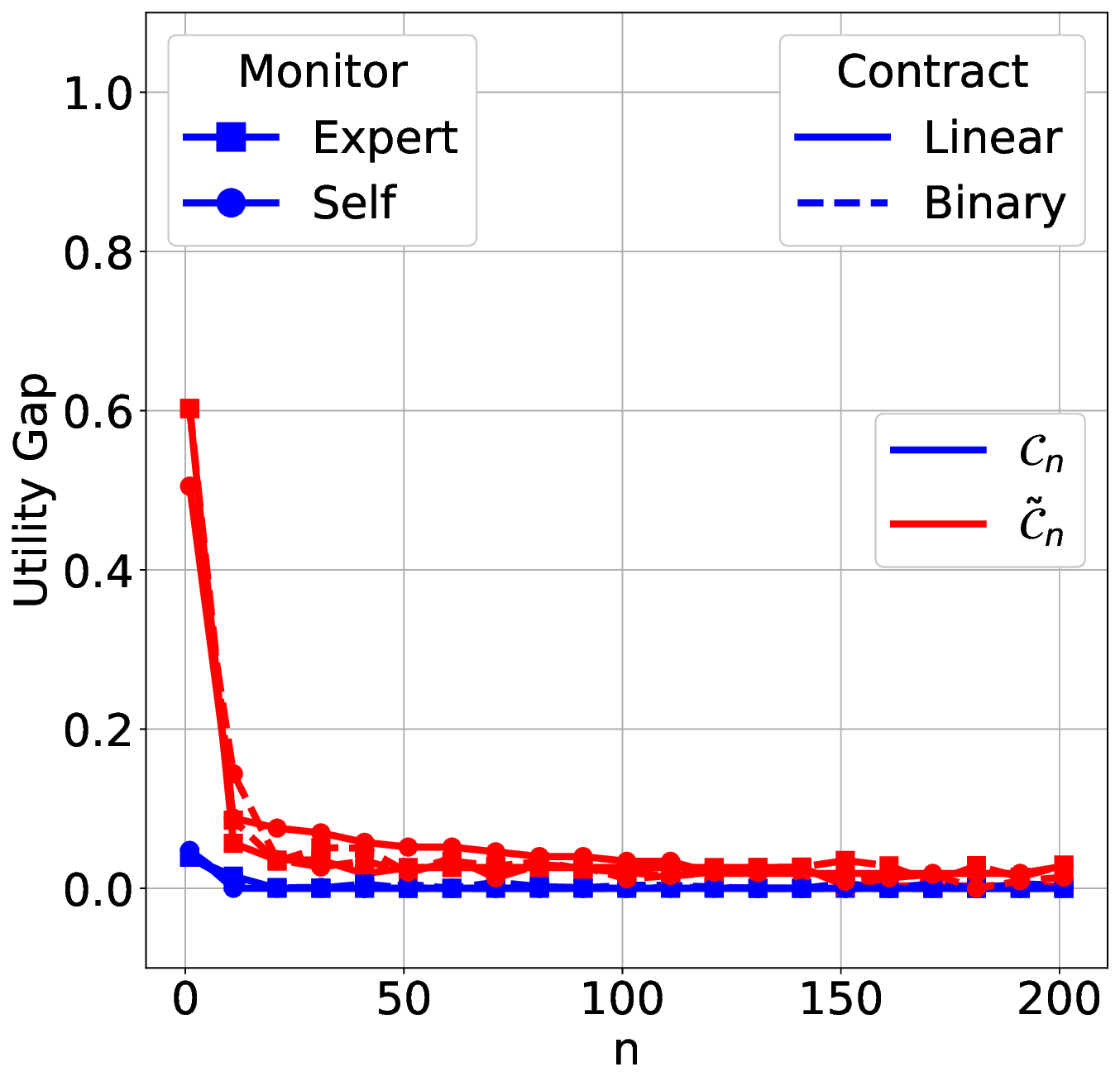}
\vspace{2mm} {\small Skywork, $\tilde{G}_a(w_a)$}
  \end{minipage}
  \caption{More principal utility gap results for Figure \ref{fig:main_contract_rank} under different configurations: (i) $\delta=0$; (ii) $\tilde{\mu}(\eta)=1/3\eta^{4/5}$; (iii) $\tilde{G}_a(w_a)=1/2-1/2\exp(-1/2w_a)$.}
\label{fig:appx_contract_rank}
\end{figure}

\medskip
\noindent\textbf{Remark on discretization.}
The discretization used in this appendix is purely computational. Throughout the theory, the agent's action space is the continuous interval $[0,1]$; the grids introduced below are only numerical devices for approximating the continuous optimization problems $\mathcal{C}$, $\mathcal{C}_n$, and $\tilde{\mathcal{C}}_n$. They should not be confused with the finite-action model of \citet{frick2023monitoring}, in which the grid itself is part of the economic environment and changes the benchmark from $\mathcal{C}$ to $\mathcal{C}^{\mathrm{disc}}(r)$.

We also report a synthetic experiment that compares three action-space specifications for binary contracts: a coarse grid (\textit{Discrete}), an intermediate grid (\textit{Dense}), and a very fine grid that numerically approximates the continuous model (\textit{Continuous}). In each case, the plotted utility gap is measured relative to the corresponding first-best benchmark. The coarse-grid model exhibits the fastest decay in $n$, but it also has the smallest benchmark due to the grid constraint; as the grid is refined, the benchmark rises toward the continuous one and the convergence curve moves toward the continuous-action case. The intermediate-grid curve lies between these two extremes, which is consistent with the idea that the separation between the target action and nearby deviations shrinks as the grid becomes finer.

\begin{figure}[ht!]
\centering
\includegraphics[width=0.62\textwidth]{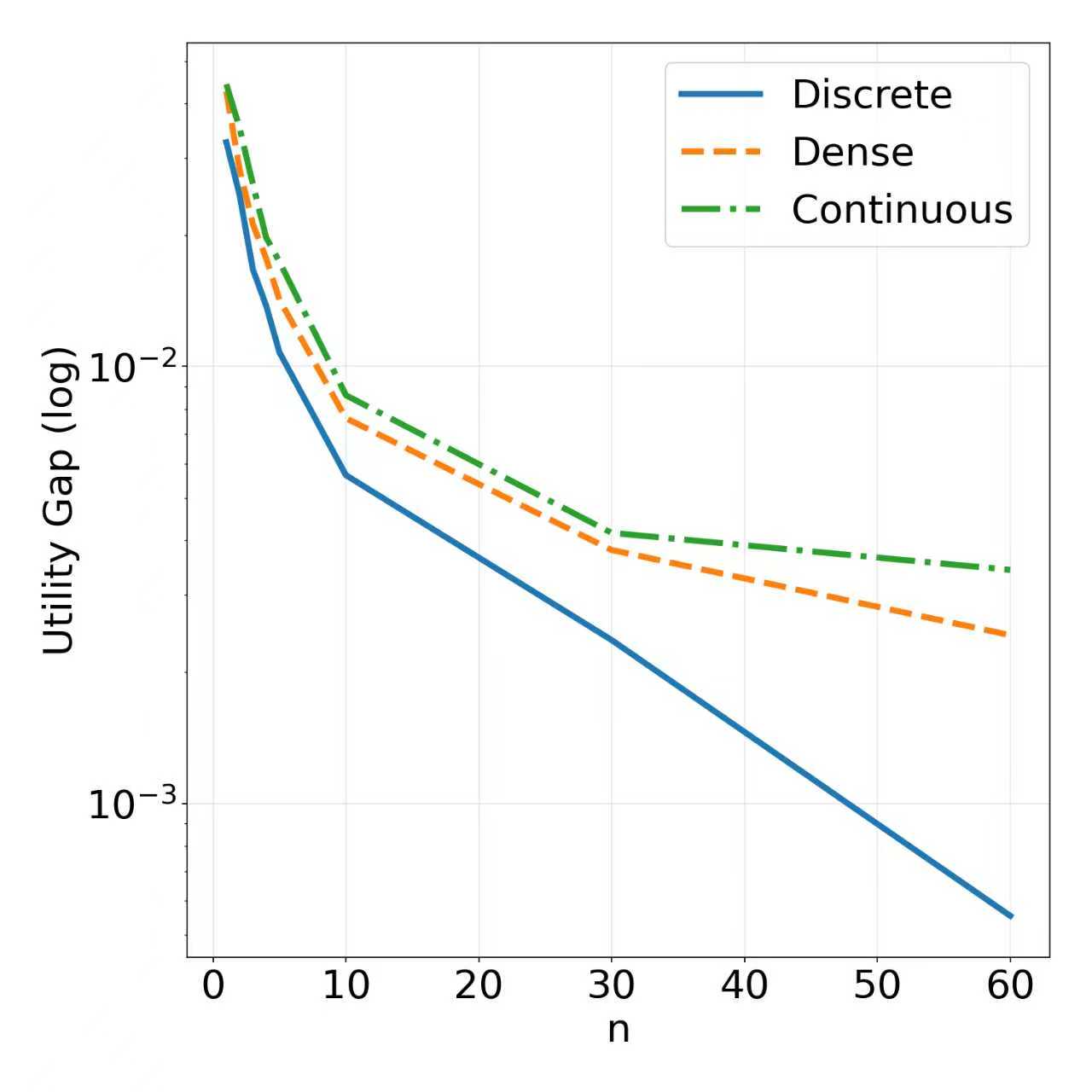}
\caption{\small Binary-contract utility gaps under three action-space specifications. ``Discrete'' uses a coarse finite grid, ``Dense'' uses an intermediate grid, and ``Continuous'' uses a very fine grid that approximates the continuous-action model. Each curve is plotted relative to its own first-best benchmark. Coarser grids yield faster convergence in $n$, but they also correspond to lower first-best benchmarks; finer grids raise the benchmark toward the continuous one and move the convergence curve toward the continuous-action case.}
\label{fig:grid_gap_tradeoff}
\end{figure}

\subsubsection{Setup and Additional Experiments for Sensitivity to $\delta$ and Alternative Utility Profiles}
\label{appx:delta_utility_sensitivity}

\paragraph{Sensitivity to $\delta$.}
To examine whether Figure~\ref{fig:main_contract_rank} is driven by the benchmark choice $\delta=0.02$, we vary $\delta\in\{0,0.05,0.15,0.20\}$ in the PKU calibration and report the resulting normalized utility gaps in Figure~\ref{fig:appx_delta_sensitivity}. The gaps change smoothly with $\delta$, and the qualitative ranking between linear and binary contracts is unchanged. 

\begin{figure}[ht!]
  \centering
  \includegraphics[width=0.72\textwidth]{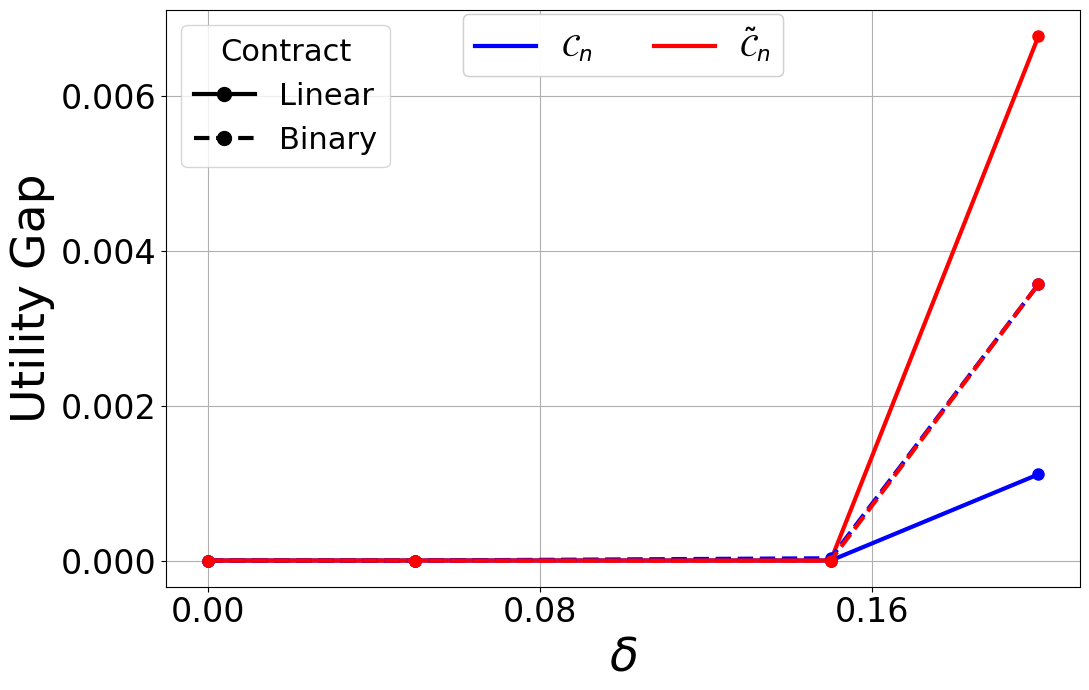}
  \caption{Sensitivity of the normalized principal utility gap to the benchmark disagreement parameter $\delta$ in the PKU calibration. We fix $n=1000$ and vary $\delta\in\{0,0.05,0.15,0.20\}$.}
  \label{fig:appx_delta_sensitivity}
\end{figure}

\paragraph{Alternative utility profiles.}
To examine whether Figure~\ref{fig:main_contract_rank} depends on the baseline utility primitives, we consider two additional assumption-compliant profiles:
\begin{align*}
G_a^{\mathrm{risk}}(w) &= 0.7\!\left(1-e^{-0.7w}\right) + 0.3\!\left(1-e^{-2.0w}\right),\\
E^{\mathrm{risk}}(\eta) &= 0.24\!\left(0.03\eta + 0.10\eta^2 + 0.12\eta^3\right),\\
\mu^{\mathrm{risk}}(\eta) &= 0.35\log\!\left(1 + 1.8\eta\right),
\end{align*}
and
\begin{align*}
G_a^{\mathrm{high}}(w) &= 1-e^{-w},\\
E^{\mathrm{high}}(\eta) &= 0.22\!\left(0.03\eta + 0.10\eta^2 + 0.12\eta^3\right),\\
\mu^{\mathrm{high}}(\eta) &= 0.20\log\!\left(1 + 1.8\eta\right).
\end{align*}
Figures~\ref{fig:appx_utility_risk} and \ref{fig:appx_utility_highcost} show that the qualitative message of Figure~\ref{fig:main_contract_rank} remains: self-consistency monitoring offers the largest gains on the more ambiguous PKU dataset, while the difference between the two monitoring technologies is smaller on Helpsteer, Ultra, and Skywork, where both signals are already informative.

\begin{figure}[ht!]
  \centering
  \begin{minipage}[b]{0.23\textwidth}
    \centering
    \includegraphics[width=\textwidth]{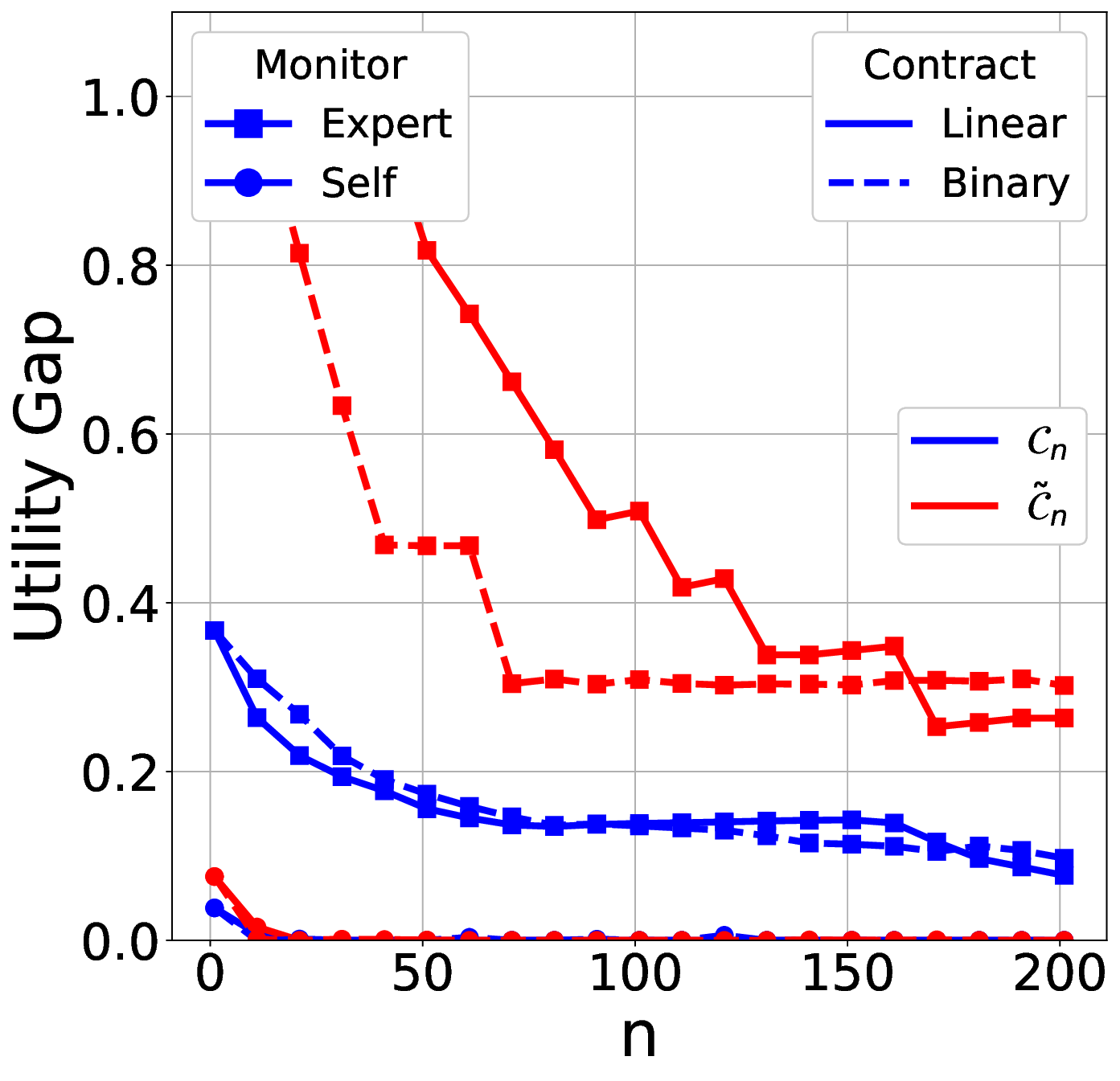}
    \vspace{2mm} {\small PKU}
  \end{minipage}
  \hfill
  \begin{minipage}[b]{0.23\textwidth}
    \centering
    \includegraphics[width=\textwidth]{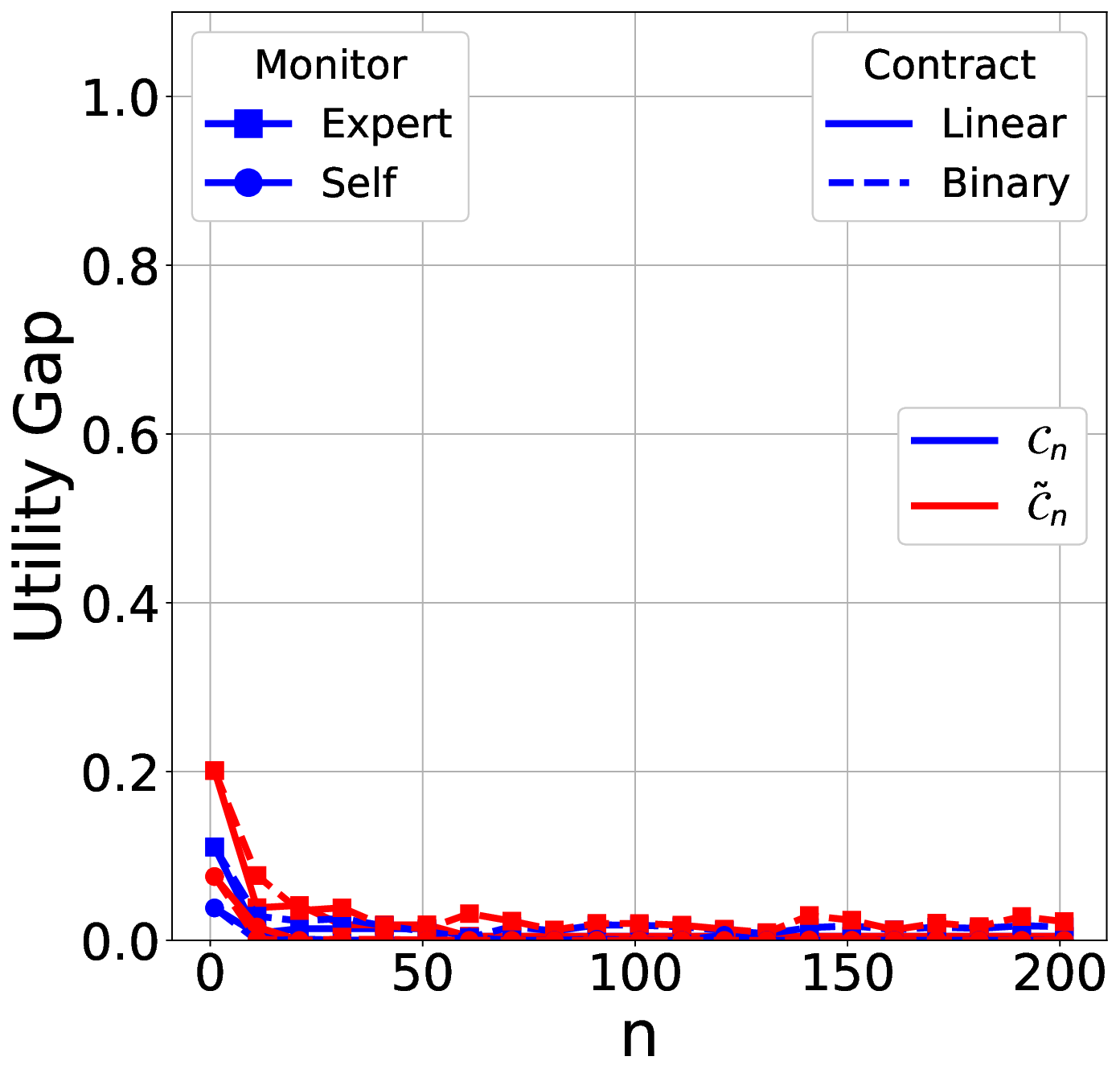}
    \vspace{2mm} {\small Helpsteer}
  \end{minipage}
  \hfill
  \begin{minipage}[b]{0.23\textwidth}
    \centering
    \includegraphics[width=\textwidth]{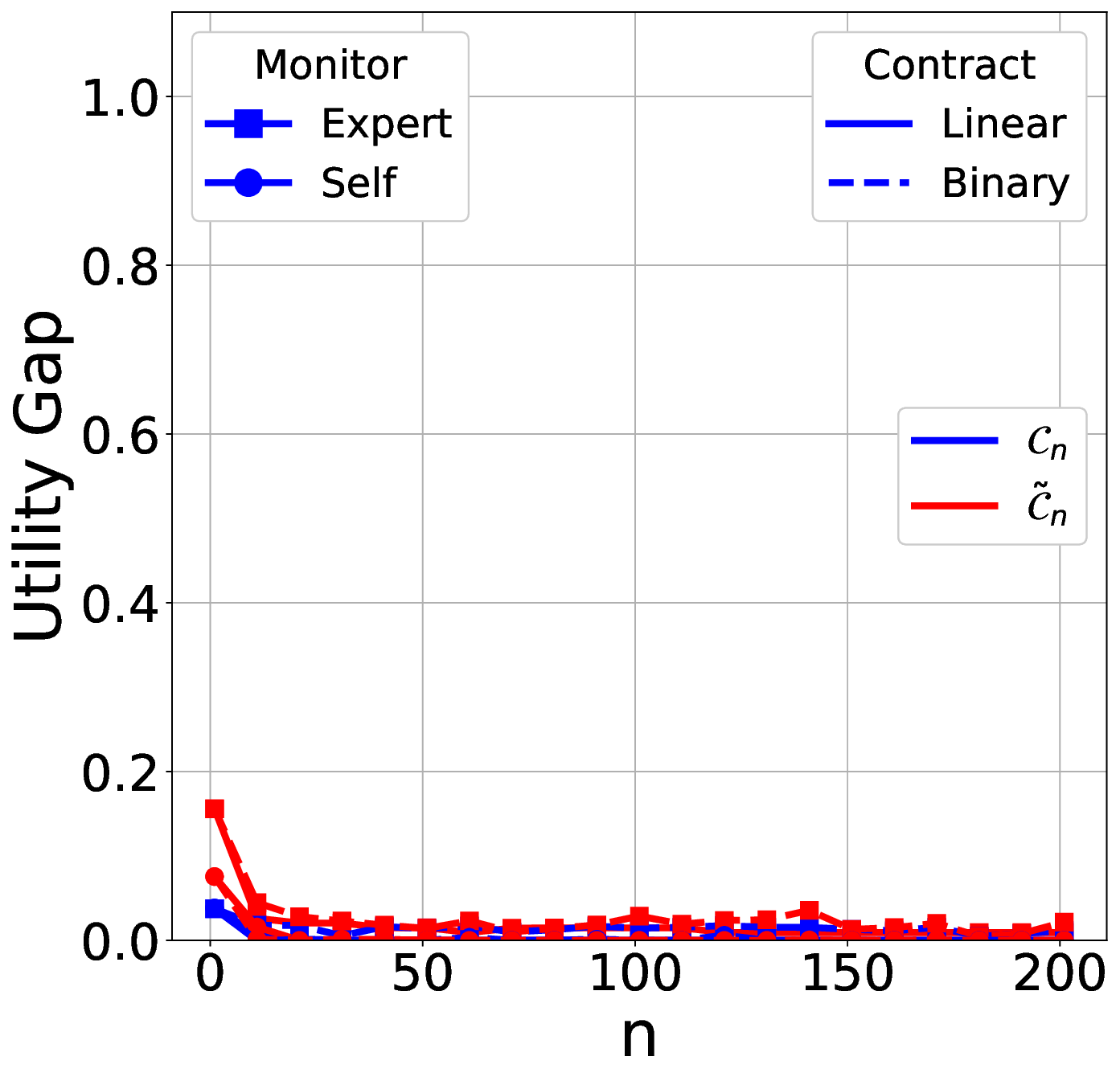}
    \vspace{2mm} {\small Ultra}
  \end{minipage}
  \hfill
  \begin{minipage}[b]{0.23\textwidth}
    \centering
    \includegraphics[width=\textwidth]{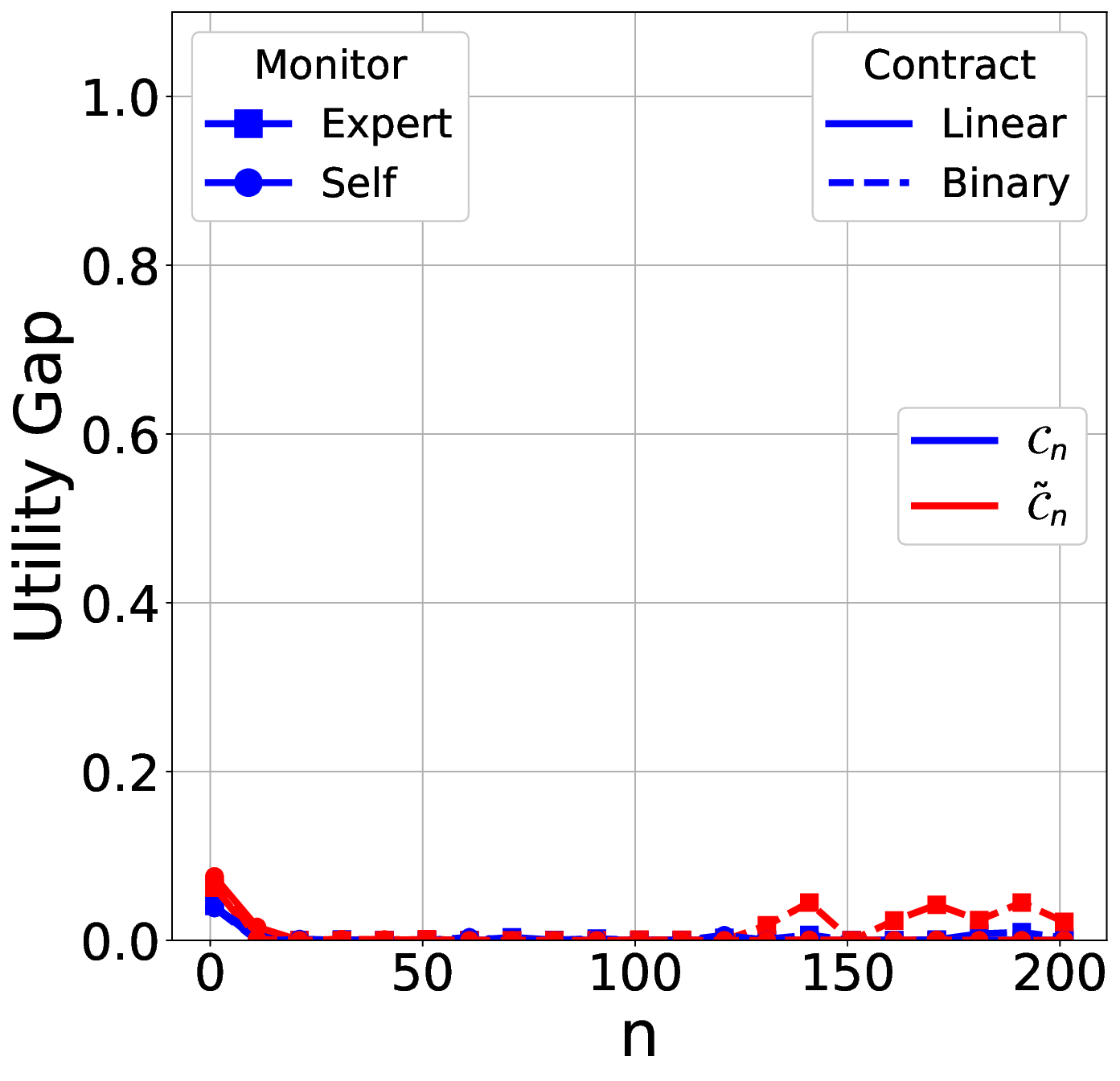}
    \vspace{2mm} {\small Skywork}
  \end{minipage}
  \caption{Normalized principal utility gap under the alternative utility profile \texttt{risk} on four datasets.}
  \label{fig:appx_utility_risk}
\end{figure}

\begin{figure}[ht!]
  \centering
  \begin{minipage}[b]{0.23\textwidth}
    \centering
    \includegraphics[width=\textwidth]{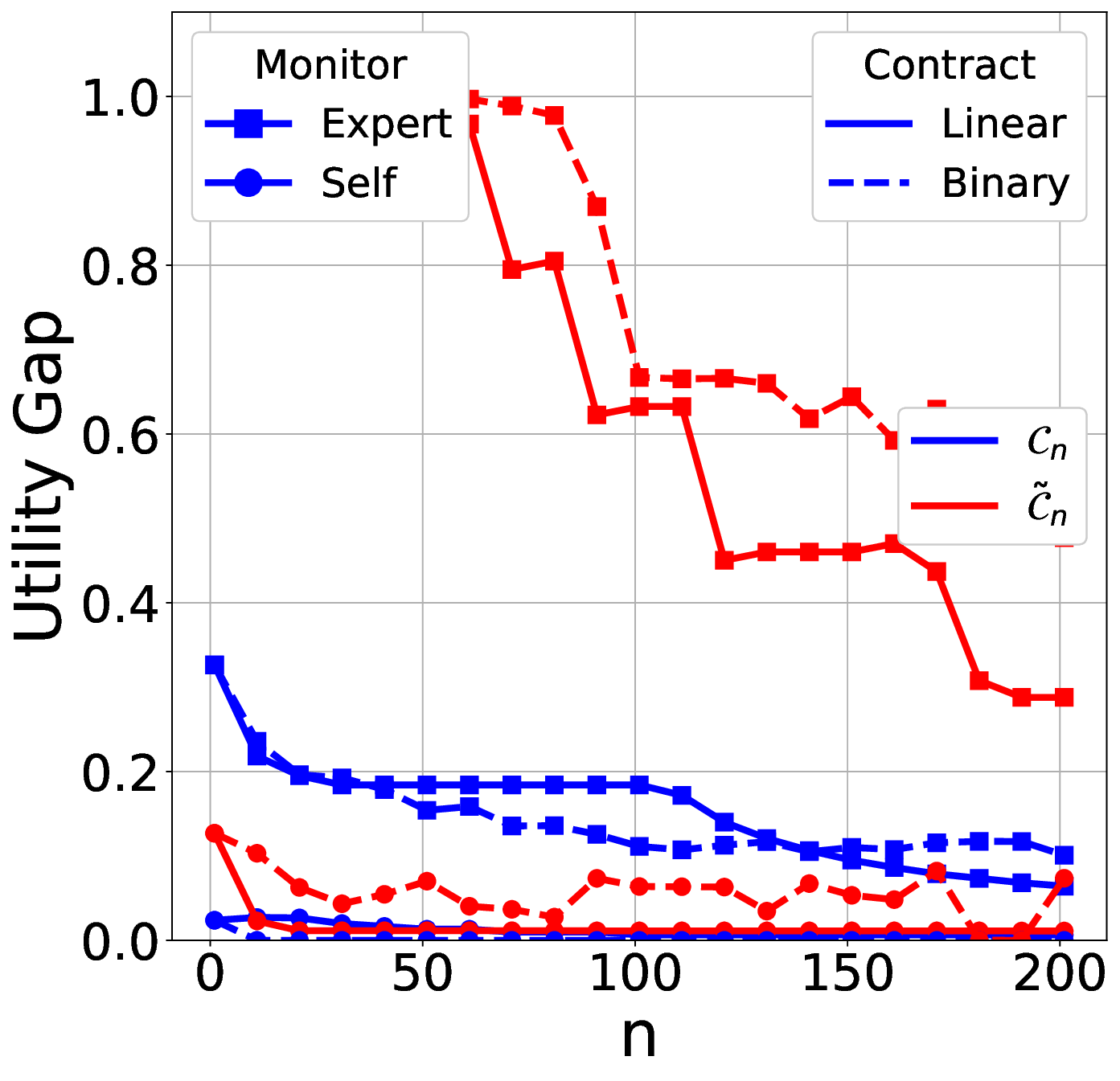}
    \vspace{2mm} {\small PKU}
  \end{minipage}
  \hfill
  \begin{minipage}[b]{0.23\textwidth}
    \centering
    \includegraphics[width=\textwidth]{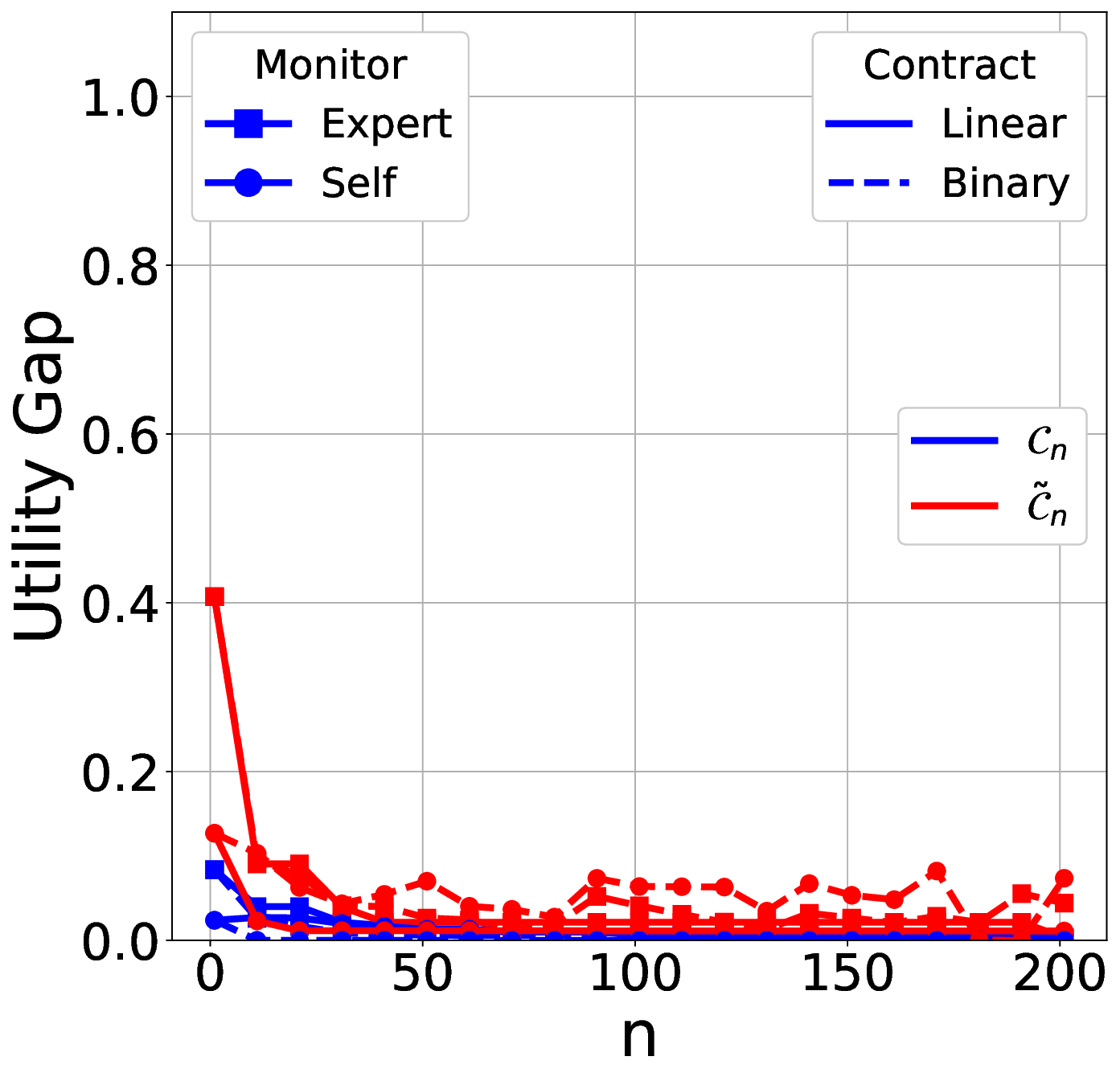}
    \vspace{2mm} {\small Helpsteer}
  \end{minipage}
  \hfill
  \begin{minipage}[b]{0.23\textwidth}
    \centering
    \includegraphics[width=\textwidth]{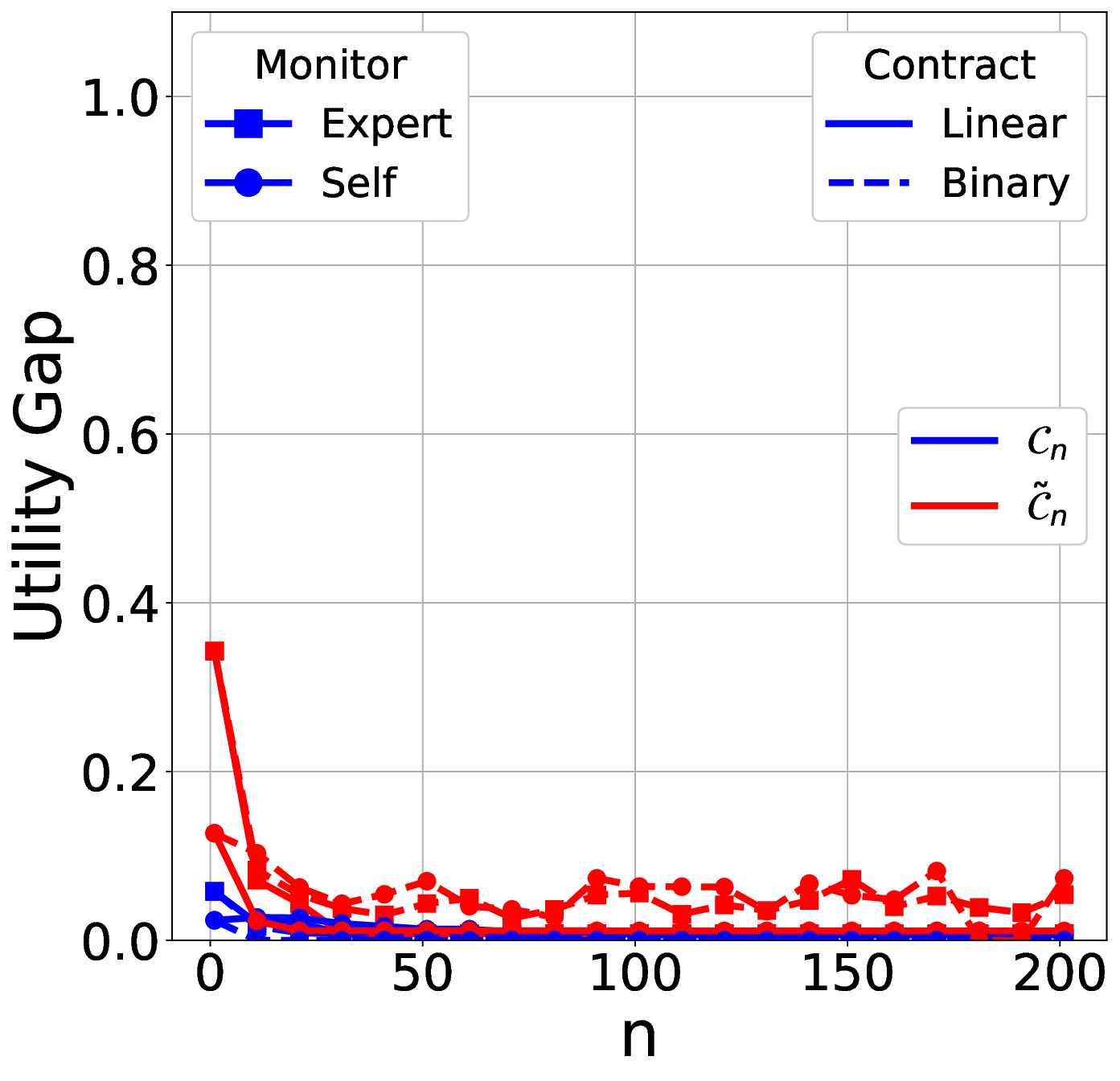}
    \vspace{2mm} {\small Ultra}
  \end{minipage}
  \hfill
  \begin{minipage}[b]{0.23\textwidth}
    \centering
    \includegraphics[width=\textwidth]{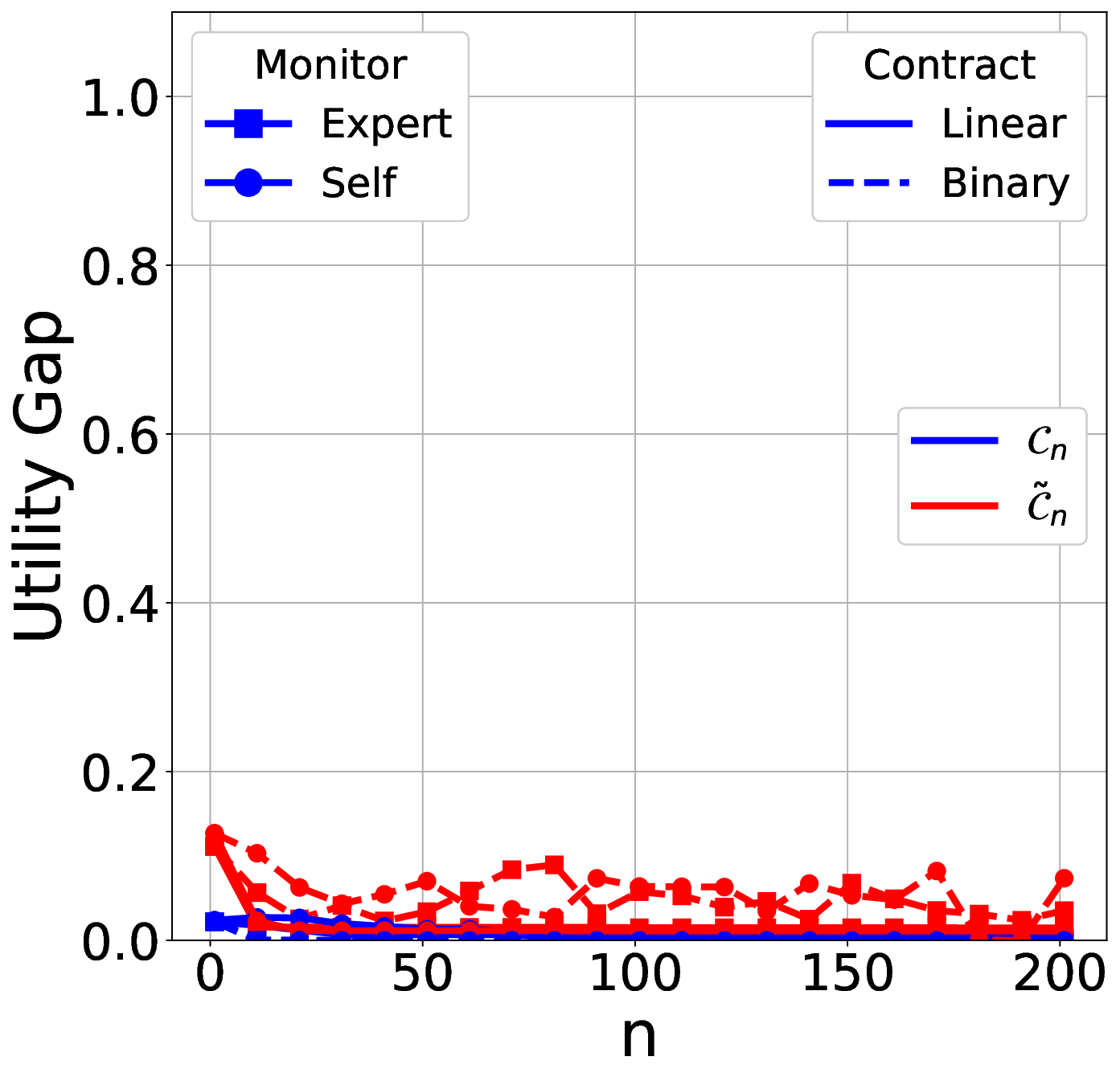}
    \vspace{2mm} {\small Skywork}
  \end{minipage}
  \caption{Normalized principal utility gap under the alternative utility profile \texttt{high} on four datasets. }
  \label{fig:appx_utility_highcost}
\end{figure}

\subsubsection{Setup and Additional Experiments for Figure~\ref{fig:reviewer_gap_scaling}}
\label{appx:reviewer_gap_scaling}

This subsection describes the synthetic experiment used in Figure~\ref{fig:reviewer_gap_scaling}. 

\paragraph{Monitoring model and utility primitives.}
We use the self-consistency monitoring model from Section~\ref{sec:incentivizing}. Conditional on the agent's action $\eta \in [0,1]$, the monitored outcomes $\{A_i\}_{i=1}^n$ are i.i.d. Bernoulli with
\[
p(\eta)=\mathbb{P}(A_i=1\mid \eta)=\frac{1+(1-\delta)\eta}{2},
\qquad \delta=0.02.
\]
The utility primitives are
\[
U_0=0,\qquad
\mu(\eta)=\frac12\eta^{4/5},\qquad
G_a(w)=1-\exp(-w),\qquad
E(\eta)=0.18\eta^2.
\]
Hence the first-best benchmark is
\[
\mathcal{C}
=
\max_{\eta\in[0,1]}
\left\{
\mu(\eta)-G_a^{-1}(E(\eta)+U_0)
\right\}
=
\max_{\eta\in[0,1]}
\left\{
\frac12\eta^{4/5}+\log(1-E(\eta))
\right\},
\]
since $G_a^{-1}(s)=-\log(1-s)$. Numerically, the maximizer is $\eta^*\approx 0.9438$.

\paragraph{Continuous-action linear contracts.}
For linear contracts we set
\[
F_n^{\mathrm{lin}}(\bm{A})=x_n+y_n\bar{A},
\qquad
\bar{A}=\frac1n\sum_{i=1}^n A_i.
\]
Because the $A_i$'s are i.i.d. Bernoulli and $G_a(w)=1-e^{-w}$, the agent's expected utility has the closed form
\[
\mathbb{E}[U_a(F_n^{\mathrm{lin}}(\bm{A}),\eta)]
=
1-e^{-x_n}\Bigl(1-p(\eta)\bigl(1-e^{-y_n/n}\bigr)\Bigr)^n-E(\eta).
\]
Thus both the principal's and the agent's objectives can be evaluated exactly, without Monte Carlo noise. To numerically represent the continuous action space, we first approximate the first-best action by maximizing over the grid
\[
\left\{0,\frac{1}{200000},\frac{2}{200000},\ldots,1\right\},
\]
which consists of $200{,}001$ equally spaced action levels. This gives the numerical first-best action $\eta^*\approx 0.9438$. For the restricted second-best problem $\tilde{\mathcal{C}}_n$, we then choose $(x_n,y_n)$ so that the individual-rationality constraint binds at $\eta^*$ and the first-order condition also holds at $\eta^*$. We report results for $30$ logarithmically spaced values of $n$ between $10$ and $2000$. The fitted log-log slope for $\mathcal{C}-\tilde{\mathcal{C}}_n$ is $-0.9968$, and the quantity $n(\mathcal{C}-\tilde{\mathcal{C}}_n)$ stabilizes around $0.0123$.

\paragraph{Continuous-action binary contracts.}
For binary contracts we consider threshold contracts of the form
\[
F_n^{\mathrm{bin}}(\bm{A})
=
\begin{cases}
w_0, & \sum_{i=1}^n A_i < k,\\
w_1, & \sum_{i=1}^n A_i \ge k,
\end{cases}
\qquad w_1>w_0.
\]
Let
\[
T_{n,k}(\eta)=\mathbb{P}\!\left(\sum_{i=1}^n A_i\ge k \,\middle|\, \eta\right).
\]
For each $n$ and threshold $k$, we set the monetary-utility gap by the first-order condition at $\eta^*$:
\[
\Delta S_{n,k}=\frac{E'(\eta^*)}{T'_{n,k}(\eta^*)},
\]
and then choose the lower monetary-utility level so that the individual-rationality constraint binds at $\eta^*$. We search over thresholds in a window centered at $np(\eta^*)$ with radius $\lceil 8\sqrt{n\log n}\rceil$, which is the theoretically relevant scale from the proof of Theorem~\ref{thm:binary}. To verify the restricted incentive-compatibility condition numerically in the continuous action space, we evaluate the agent's expected utility on the grid
\[
\left\{0,\frac{1}{4000},\frac{2}{4000},\ldots,1\right\},
\]
which consists of $4001$ equally spaced action levels, and retain only those thresholds for which $\eta^*$ is an approximate global maximizer. We use the union of $\{10,15,\ldots,50\}$ and $18$ logarithmically spaced values between $20$ and $2000$ for $n$. The resulting curve is well aligned with the $1/\sqrt{n\log n}$ reference in Figure~\ref{fig:reviewer_gap_scaling}.

\paragraph{Discrete-action binary benchmark.}
To connect with the discrete-action result of \citet{frick2023monitoring}, we discretize the action space to
\[
\mathcal{S}_{\eta}^{\mathrm{disc}}=\{0,0.05,\ldots,1\}
\]
and solve the corresponding threshold binary-contract problem exactly using binomial tail probabilities. For each $n\in\{10,15,\ldots,50\}$, we choose the contract that maximizes the principal's utility subject to incentive compatibility and individual rationality for the discrete first-best action. The resulting gap is approximately linear on the semi-log scale, which is consistent with exponential convergence in the discrete-action setting. We stress that this discrete experiment is used only as a synthetic visual benchmark for the continuous-versus-discrete contrast.

\newpage

\subsection{Examples for hard-to-choose responses}
\label{appx:example}
In the following, we present a few examples from Helpsteer \citep{wang2023helpsteer} for which we think it is unclear to say which response is better. In this case, annotators may assign different preference labels for this sample, even when they are fully committed. This is what we refer to as the \textit{heterogeneity} among human annotators. Hence we cannot tell if an annotator makes a mistake whatever label they assign.

\begin{framed}
\textbf{Prompt:}\\

What is the history of the bed? \\

\textbf{Response 1:}\\

The history of the bed is a long and complex one, with many different cultural and historical influences. Beds have been used for thousands of years, and their design and function have evolved significantly over time. In ancient times, beds were often simple platforms made of wood or stone, and were used primarily for sleeping and rest. However, as societies became more complex and wealthy, beds became more elaborate and comfortable, with mattresses, pillows, and even canopy beds. In medieval times, beds became even more luxurious, with mattresses stuffed with feathers, and canopy beds draped in fine fabrics. As the Industrial Revolution brought new manufacturing processes and materials, beds became more mass-produced and affordable, and new designs and features such as adjustable beds and water beds were introduced. Today, beds come in a wide variety of styles, materials, and sizes, and are used not only for sleeping but also for relaxing, reading, and other activities. \\

\textbf{Response 2:}\\

The bed is one of the most important inventions in human history, providing a place of comfort, relaxation, and safety for people around the world. In this question, we will explore the history of the bed and how it has evolved over time. The earliest beds were simple platforms or mats made of leaves, branches, or other natural materials. These beds were used by prehistoric humans and were often placed on the ground or in a cave. Over time, beds became more sophisticated, with the addition of cushions, pillows, and even mattresses made of straw or other materials. In ancient civilizations such as Egypt and Greece, beds were symbols of status and wealth, with wealthy individuals owning ornate beds made of wood and adorned with gold and other precious materials. In medieval Europe, beds became even more elaborate, with some beds featuring canopy frames and curtains for privacy. As materials and manufacturing processes advanced, beds became more affordable and accessible to the general public. In the 20th century, the bed became an important symbol of comfort and relaxation, with the invention of the adjustable bed and the development of the modern mattress. Today, the bed remains an essential piece of furniture in homes around the world, with a wide variety of styles, sizes, and materials available to suit different needs and preferences. From simple futons to luxurious canopy beds, the bed continues to play a crucial role in providing comfort and relaxation for people of all ages and backgrounds.\\
\end{framed}

\begin{framed}
\textbf{Prompt:}\\

What are the three most important things to consider when deciding what technology to use to build an assist device to help an elderly person with basic needs? \\

\textbf{Response 1:}\\

There are many different types of assistive devices that can be used to help an elderly person with their basic needs. When deciding what technology to use to build an assistive device, it is important to consider the following three factors: 1. The needs of the elderly person: The first and most important factor to consider is the specific needs of the elderly person who will be using the assistive device. It is important to understand their physical limitations, as well as their daily routine and activities, in order to design an assistive device that will be effective and useful for them. 2. The capabilities of the technology: It is also important to consider the capabilities of the technology that will be used to build the assistive device. For example, if the elderly person has limited vision, it may be important to use a technology that provides auditory feedback or haptic feedback in addition to visual feedback. 3. The ease of use of the technology: Finally, it is important to consider the ease of use of the technology. The assistive device should be designed in a way that is easy for the elderly person to understand and use, even if they have limited dexterity or cognitive abilities. This may involve using simple controls, providing clear instructions, or using a natural language interface. By considering these three factors, it is possible to build an assistive device that is effective, useful, and easy to use for the elderly person. \\                

\textbf{Response 2:}\\

 When deciding what technology to use to build an assist device to help an elderly person with basic needs, there are three most important things to consider: 1. The specific needs of the elderly person: The assist device should be tailored to the individual's specific needs and limitations. For example, if the person has limited mobility, the device should focus on assisting with mobility. 2. The ease of use: The assist device should be easy to use and understand, with clear and simple instructions. The elderly person should be able to operate the device without assistance. 3. The reliability of the technology: The assist device should be reliable and robust, with minimal risk of malfunction or failure. The technology should be tested thoroughly to ensure it is safe and effective.   \\
 
\end{framed}

\begin{framed}
\textbf{Prompt:}\\

List down some tips to make the most out of attending a seminar. \\

\textbf{Response 1:}\\

Sure! Here are some tips to make the most out of attending a seminar: 1. Prepare in advance by reading the seminar's material and any suggested readings. 2. Arrive on time and stay for the duration of the seminar. 3. Engage with the speaker and other attendees by asking questions and participating in discussions. 4. Take notes and actively listen to the speaker and other attendees. 5. Connect with other attendees to expand your network and share ideas. 6. Reflect on the seminar's content and consider how it applies to your own work or interests. 7. Follow up with the speaker or organizers to continue the conversation or ask further questions.        \\

\textbf{Response 2:}\\

Before attending a seminar, research the topic and speakers to understand the purpose and goals. Choose a comfortable seat near the front to improve your focus and engagement. Arrive early to settle in and mingle with other attendees. Use the event app or program to connect with other attendees and speakers. Take notes and ask questions to enhance your learning experience. Don't be afraid to introduce yourself to others and engage in conversations. Stay for the entire event, even if it's longer than expected, to reap the full benefits. Don't be afraid to ask for business cards and follow up with people after the event.            \\

\end{framed}

\subsection{Online Contract Adaptation}
\label{subsec:online_contract_adaptation}

The second-best programs \eqref{eq:second_best} and \eqref{eq:restricted_second_best} are static benchmarks: they identify the optimal contract when the principal already knows how the annotator responds to each candidate contract. In practice, however, the principal may not know the agent-side primitives $G_a$, $E$, and $U_0$, and hence does not know which contracts will be accepted or which annotation quality each accepted contract induces. This motivates an online implementation in which the principal learns a good contract from repeated interaction.

We keep the assessment batch size $n$ fixed and let
\[
\mathcal{G}_n \coloneqq \{F_{n,\xi}:\xi\in\Xi\}\subseteq \mathcal{F}_n
\]
be a finite low-dimensional candidate family. Throughout this subsection, we assume that the annotator's primitives $G_a$, $E$, and $U_0$ remain stable across rounds, so that repeatedly offering the same contract induces the same best-response quality. For a given contract $F_{n,\xi}$, the annotator accepts it if
\[
\max_{\eta\in[0,1]} \mathbb{E}\left[U_a(F_{n,\xi}(\mathcal{D}_n),\eta)\right]\ge U_0,
\]
and, conditional on acceptance, chooses
\[
\eta_a(F_{n,\xi}) \in \operatorname*{arg\,max}_{\eta\in[0,1]}
\mathbb{E}\left[U_a(F_{n,\xi}(\mathcal{D}_n),\eta)\right].
\]
The principal-side value of $F_{n,\xi}$ is then
\[
J_n(\xi)\coloneqq
\mathbb{E}\left[U_p(F_{n,\xi}(\mathcal{D}_n),\eta_a(F_{n,\xi}))\right],
\]
with the convention that $J_n(\xi)=0$ when the contract is rejected. In contrast with \eqref{eq:second_best}, the response map $\xi\mapsto \eta_a(F_{n,\xi})$ and hence the value function $\xi\mapsto J_n(\xi)$ are now unknown and must be learned online.

For the online implementation below, we specialize $\mathcal{D}_n$ to the self-consistency monitoring dataset in Algorithm \ref{alg_self}. Thus, whenever the contract is accepted and induces annotation quality $\eta_t\coloneqq \eta_a(F_{n,\xi_t})$, the principal observes
\[
\bm{A}_t \coloneqq \{A_{t,i}\}_{i=1}^n,
\qquad
A_{t,i}\sim \mathrm{Bernoulli}\!\left(\frac{1+(1-\delta)\eta_t}{2}\right),
\qquad
\bar{A}_t \coloneqq \frac{1}{n}\sum_{i=1}^n A_{t,i}.
\]
This is exactly the self-consistency model in \eqref{eqn:agreement_self_consist}. Assuming $\delta<1/2$, a natural plug-in estimator of the latent annotation quality is
\[
\hat{\eta}_t \coloneqq \Pi_{[0,1]}\!\left(\frac{2\bar{A}_t-1}{1-\delta}\right),
\]
where $\Pi_{[0,1]}$ denotes projection onto $[0,1]$. We then evaluate round $t$ by the proxy utility
\[
\hat{u}_t \coloneqq -F_{n,\xi_t}(\bm{A}_t)+\mu(\hat{\eta}_t),
\]
and set $\hat{u}_t=0$ if the contract is rejected. Since $\hat{\eta}_t$ is a consistent estimator of $\eta_t$ under self-consistency monitoring, $\hat{u}_t$ is a practical proxy for the one-round principal utility. Moreover, Assumption \ref{assm:basic_assm_3} bounds the wage, and $\mu$ is bounded on $[0,1]$ by continuity, so $\hat{u}_t$ is bounded as well.

For the online implementation, $\delta$ is treated as a calibrated property of the monitoring channel rather than of a particular contract, and $\mu$ is treated as a fixed quality-to-value score chosen by the principal. The latter need not equal the exact downstream utility curve; it can be a monotone surrogate obtained from a validation study, a business KPI calibration, or a policy weight that encodes how much the principal values higher-quality labels.

In our setup, the monitoring rule, the batch size $n$, the calibrated parameter $\delta$, and the principal's score function $\mu$ are treated as fixed. The unknown objects are the agent-side primitives $G_a$, $E$, and possibly $U_0$, or equivalently the acceptance/best-response map $F_{n,\xi}\mapsto \eta_a(F_{n,\xi})$. Hence the principal does not know ex ante which contracts will be accepted, which annotation quality each accepted contract induces, or which value $J_n(\xi)$ each contract delivers.

We compare the online policy against the best \emph{static} contract in the candidate family,
\[
\xi_n^{\mathrm{SB}} \in \operatorname*{arg\,max}_{\xi\in\Xi} J_n(\xi).
\]
This is the finite-family analogue of \eqref{eq:second_best}. If one instead wants to approximate \eqref{eq:restricted_second_best}, the same routine can be run on a filtered grid that retains only contracts whose induced annotation quality is estimated to be close to $\eta^*$.

An online policy $\pi$ chooses $\xi_t$ using the past history of accepted/rejected offers and observed monitoring outcomes. A natural performance metric is the pseudo-regret
\[
\mathrm{Reg}_{\pi}(T)
\coloneqq
T\,J_n(\xi_n^{\mathrm{SB}})
-
\mathbb{E}_{\pi}\!\left[\sum_{t=1}^T J_n(\xi_t)\right],
\]
or, equivalently, convergence of the average realized utility to $J_n(\xi_n^{\mathrm{SB}})$.

The static analysis above suggests exactly which low-dimensional families are worth searching over. Writing $\bm{A}=(A_1,\ldots,A_n)$ and $\bar{A}=n^{-1}\sum_{i=1}^n A_i$, for binary contracts we use the threshold/bonus family
\[
F_{n,\xi}^{\mathrm{bin}}(\bm{A}) = w + w_b\,\mathbbm{1}\{\bar{A} \ge \tau\},
\qquad \xi=(w,w_b,\tau).
\]
For linear contracts, it is natural to use
\[
F_{n,\xi}^{\mathrm{lin}}(\bm{A}) = a + b\bar{A},
\qquad \xi=(a,b).
\]
Hence the online learner searches only over a few scalar parameters rather than over an arbitrary function $F_n\in\mathcal{F}_n$.

\begin{algorithm}[ht!]
\centering
\caption{UCB learning on a finite contract grid}
\label{alg_online_ucb_contract}
\begin{algorithmic}
\State \textbf{Input:} Finite grid $\Xi$, exploration coefficient $\beta>0$, horizon $T$
\State \textbf{Warm start:} Offer each contract $F_{n,\xi}$ with $\xi\in\Xi$ once, observe the resulting proxy utility, and initialize $N(\xi)=1$ and $\widehat{J}(\xi)$ by that observation
\For{$t=|\Xi|+1,\ldots,T$}
    \State Choose
    \[
    \xi_t \in \operatorname*{arg\,max}_{\xi\in\Xi}
    \left\{
    \widehat{J}(\xi)
    + \beta \sqrt{\frac{\log(t+1)}{N(\xi)}}
    \right\}
    \]
    \State Offer $F_{n,\xi_t}$ to the annotator
    \If{the contract is rejected}
        \State Set $\hat{u}_t=0$
    \Else
        \State Collect $\bm{A}_t$, compute $\bar{A}_t$, $\hat{\eta}_t=\Pi_{[0,1]}\!\left(\frac{2\bar{A}_t-1}{1-\delta}\right)$, and payment $w_{a,t}=F_{n,\xi_t}(\bm{A}_t)$
        \State Set $\hat{u}_t=-w_{a,t}+\mu(\hat{\eta}_t)$
    \EndIf
    \State Update $N(\xi_t) \leftarrow N(\xi_t)+1$
    \State Update $\widehat{J}(\xi_t) \leftarrow \widehat{J}(\xi_t)+\frac{\hat{u}_t-\widehat{J}(\xi_t)}{N(\xi_t)}$
\EndFor
\State \textbf{Return} $\operatorname*{arg\,max}_{\xi\in\Xi}\widehat{J}(\xi)$
\end{algorithmic}
\end{algorithm}

Algorithm \ref{alg_online_ucb_contract} is a direct repeated-learning implementation of the second-best logic. For a fixed contract $F_{n,\xi}$, the induced annotation quality $\eta_a(F_{n,\xi})$ and the distribution of the proxy utility are stationary, so the empirical mean $\widehat{J}(\xi)$ concentrates with repeated sampling. The upper-confidence term prevents premature commitment to a noisy early winner. Self-consistency monitoring converts the hidden action $\eta_t$ into the observable statistic $\bar{A}_t$, so the principal can learn good contracts without retraining a downstream model at each round. Finally, a larger assessment batch size $n$ helps twice: it sharpens the static approximation guarantees in Theorems \ref{thm:binary} and \ref{thm:linear_contract}, and it reduces the noise in $\hat{\eta}_t$.

\subsubsection{Numerical Experiment}
We conduct a numerical experiment of Algorithm \ref{alg_online_ucb_contract}  on the binary threshold contract family. We use the same primitives as in Figure \ref{fig:main_contract_rank}: $U_0=0$, $\delta=0.02$, $\mu(\eta)=\frac12\eta^{4/5}$, $G_a(w_a)=1-\exp(-w_a)$, and $E(\eta)=0.18\eta^2$. The assessment batch contains $n=100$ duplicated samples under self-consistency monitoring, with
\[
A_{t,i}\sim \mathrm{Bernoulli}\!\left(\frac{1+(1-\delta)\eta_t}{2}\right).
\]
We think a simplified version of binary contract due to the high computation requirement of Monte Carlo simulation: we fix $w=0$  and search over the grid
\[
w_b\in\{0.05,0.10,0.15,0.20,0.25,0.30\},
\qquad
\tau\in\{0.80,0.85,0.90,0.95,0.97\}.
\]
We run Algorithm \ref{alg_online_ucb_contract} for $T=1500$ rounds with exploration coefficient $\beta=0.05$ and average the results over $200$ Monte Carlo replications.

As a simple baseline, we also report a one-shot explore-then-commit policy. This baseline samples each contract once, computes the same proxy utility, and then commits for all remaining rounds to the contract with the highest empirical score after this initial sweep. It uses exactly the same information as UCB but omits continued exploration.

\textbf{Results.} On this grid, the best static contract is $(w_b,\tau)=(0.15,0.85)$, which induces annotation quality approximately $0.819$ and expected principal utility $0.2860$. The UCB learner attains a last-$200$-round average realized utility of $0.2850$, which is $99.6\%$ of the best static-grid benchmark, and it selects the exact best grid contract in $90.6\%$ of the last $200$ rounds.

By contrast, the explore-then-commit baseline attains a last-$200$-round average utility of $0.2533$, or $88.6\%$ of the benchmark, and it selects the exact best grid contract only $29.0\%$ of the time in the last $200$ rounds. Figure \ref{fig:online_binary_ucb} shows that the $50$-round moving average under UCB rapidly approaches the static benchmark and remains above explore-then-commit, whose noisy initial sweep can lock the principal into a suboptimal threshold/bonus pair. An analogous one-dimensional implementation can be used for the linear family.

\begin{figure}[ht!]
    \centering
    \includegraphics[width=0.78\textwidth]{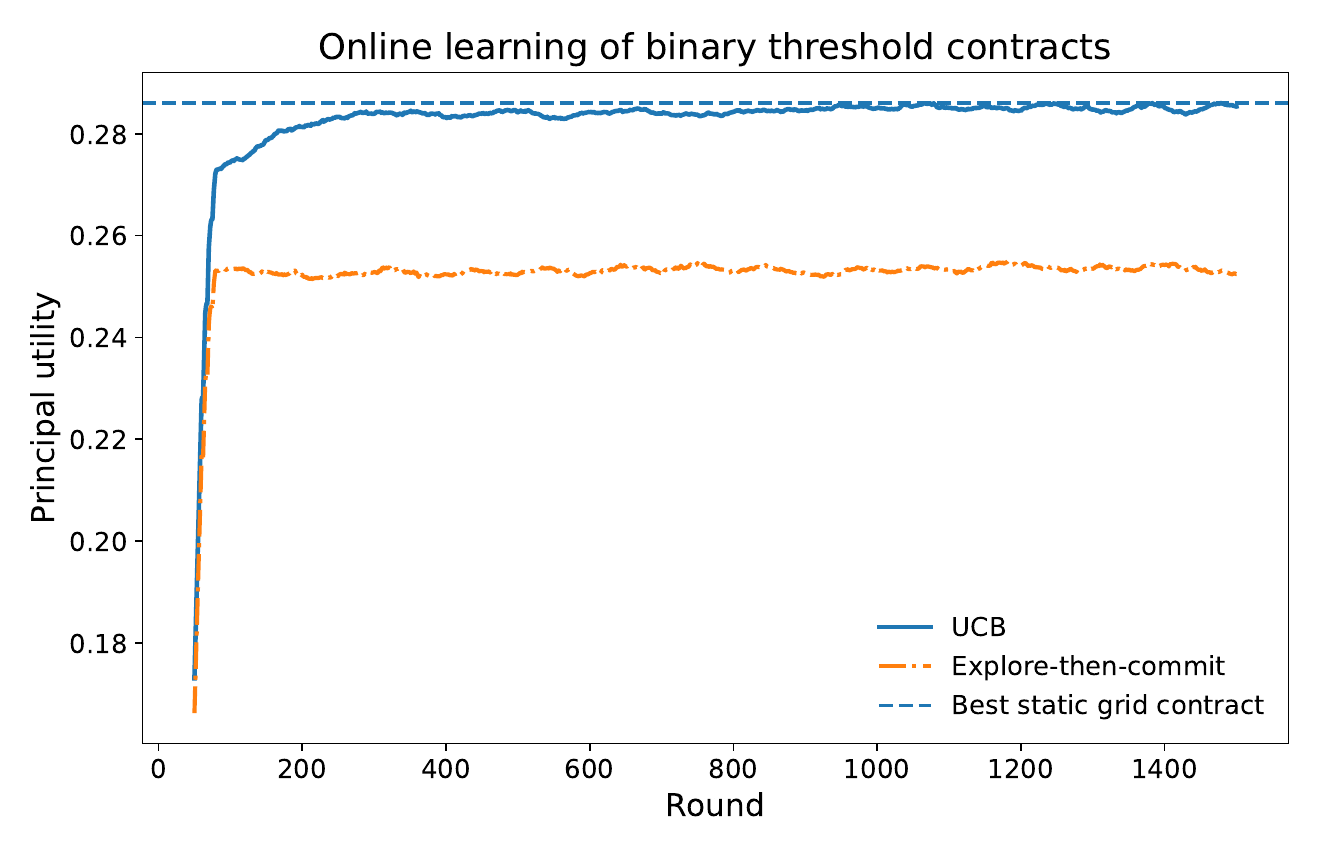}
    \caption{Simulated online-learning experiment for binary threshold contracts under self-consistency monitoring. The solid line is the $50$-round moving average of the realized principal utility under UCB; the dash-dotted line is the corresponding $50$-round moving average under explore-then-commit; and the dashed horizontal line is the best static contract on the searched grid. UCB rapidly concentrates on the optimal threshold/bonus pair and nearly matches the static benchmark, while explore-then-commit remains lower because it can lock into a suboptimal contract after the initial sweep.}
    \label{fig:online_binary_ucb}
\end{figure}
\end{document}